\PassOptionsToPackage{unicode}{hyperref}
\PassOptionsToPackage{hyphens}{url}
\PassOptionsToPackage{dvipsnames,svgnames,x11names}{xcolor}
\documentclass[
  12pt]{article}

\usepackage{amsmath,amssymb}
\usepackage{iftex}
\ifPDFTeX
  \usepackage[T1]{fontenc}
  \usepackage[utf8]{inputenc}
  \usepackage{textcomp} 
\else 
  \usepackage{unicode-math}
  \defaultfontfeatures{Scale=MatchLowercase}
  \defaultfontfeatures[\rmfamily]{Ligatures=TeX,Scale=1}
\fi
\usepackage{lmodern}
\ifPDFTeX\else  
\fi
\IfFileExists{upquote.sty}{\usepackage{upquote}}{}
\IfFileExists{microtype.sty}{
  \usepackage[]{microtype}
  \UseMicrotypeSet[protrusion]{basicmath} 
}{}
\makeatletter
\@ifundefined{KOMAClassName}{
  \IfFileExists{parskip.sty}{%
    \usepackage{parskip}
  }{
    \setlength{\parindent}{0pt}
    \setlength{\parskip}{6pt plus 2pt minus 1pt}}
}{
  \KOMAoptions{parskip=half}}
\makeatother
\usepackage{xcolor}
\makeatletter
\ifx\paragraph\undefined\else
  \let\oldparagraph\paragraph
  \renewcommand{\paragraph}{
    \@ifstar
      \xxxParagraphStar
      \xxxParagraphNoStar
  }
  \newcommand{\xxxParagraphStar}[1]{\oldparagraph*{#1}\mbox{}}
  \newcommand{\xxxParagraphNoStar}[1]{\oldparagraph{#1}\mbox{}}
\fi
\ifx\subparagraph\undefined\else
  \let\oldsubparagraph\subparagraph
  \renewcommand{\subparagraph}{
    \@ifstar
      \xxxSubParagraphStar
      \xxxSubParagraphNoStar
  }
  \newcommand{\xxxSubParagraphStar}[1]{\oldsubparagraph*{#1}\mbox{}}
  \newcommand{\xxxSubParagraphNoStar}[1]{\oldsubparagraph{#1}\mbox{}}
\fi
\makeatother

\usepackage{longtable,booktabs,array}
\usepackage{calc} 
\usepackage{etoolbox}
\makeatletter
\patchcmd\longtable{\par}{\if@noskipsec\mbox{}\fi\par}{}{}
\makeatother
\IfFileExists{footnotehyper.sty}{\usepackage{footnotehyper}}{\usepackage{footnote}}
\makesavenoteenv{longtable}
\usepackage{graphicx}
\makeatletter
\def\maxwidth{\ifdim\Gin@nat@width>\linewidth\linewidth\else\Gin@nat@width\fi}
\def\maxheight{\ifdim\Gin@nat@height>\textheight\textheight\else\Gin@nat@height\fi}
\makeatother
\setkeys{Gin}{width=\maxwidth,height=\maxheight,keepaspectratio}
\makeatletter
\def\fps@figure{htbp}
\makeatother

\usepackage[subtle]{savetrees}
\usepackage{amsmath}
\usepackage{graphicx}
\usepackage{enumerate}
\usepackage{natbib}
\usepackage{url} 
\usepackage{float}
\usepackage{ragged2e}
\usepackage{amsthm}
\usepackage{mathtools}
\usepackage{amsfonts}
\usepackage{amssymb}
\usepackage{hyperref}  
\usepackage{caption}
\usepackage{subcaption}
\usepackage{apptools}
\usepackage{xcolor}
\usepackage{xr}
\usepackage{adjustbox}
\usepackage{wrapfig}
\usepackage{caption}
\usepackage{subcaption}
\usepackage{cprotect}
\usepackage{enumitem}
\usepackage{array} 
\usepackage{placeins}
\input{0_commands}
\theoremstyle{definition}

\newtheorem{theorem}{Theorem}
\newtheorem{assumption}{Assumption}
\newtheorem{algorithm}{Algorithm}
\newtheorem{proposition}{Proposition}

\newtheorem{remark}{Remark}
\newtheorem{corollary}{Corollary}
\newtheorem{lemma}{Lemma}
\newtheorem{example}{Example}

\AtAppendix{\counterwithin{theorem}{section}}
\AtAppendix{\counterwithin{assumption}{section}}
\AtAppendix{\counterwithin{algorithm}{section}}
\AtAppendix{\counterwithin{proposition}{section}}
\AtAppendix{\counterwithin{definition}{section}}
\AtAppendix{\counterwithin{remark}{section}}
\AtAppendix{\counterwithin{corollary}{section}}
\AtAppendix{\counterwithin{lemma}{section}}
\AtAppendix{\counterwithin{example}{section}}

\newcommand{\anon}{1}

\begin{document}

\def\spacingset#1{\renewcommand{\baselinestretch}%
{#1}\small\normalsize} \spacingset{1}


\if1\anon
{
  \title{\bf Nested Nonparametric Instrumental Variable Regression}
  \date{Orginal draft: December 2021. This draft: May 2025.}
  \author{Isaac Meza \\
    Department of Economics, Harvard University\\
    \and 
    Rahul Singh\thanks{
    We thank Alberto Abadie, Victor Chernozhukov, Raj Chetty, Avi Feller, Anna Mikusheva, Whitney Newey, James Robins, Andrea Rotnitzky, Vasilis Syrgkanis, and Suhas Vijaykumar for helpful comments. We thank Miriam Nelson and Moses Stewart for excellent research assistance. Rahul Singh thanks the Jerry Hausman Dissertation Fellowship. Part of this work was done while Rahul Singh visited the Simons Institute for the Theory of Computing. }\hspace{.2cm}\\ 
    Society of Fellows and Department of Economics, Harvard University}
  \maketitle
} \fi

\if0\anon
{
  \bigskip
  \bigskip
  \bigskip
  \begin{center}
    {\LARGE\bf Nested Nonparametric Instrumental Variable Regression}
\end{center}
  \medskip
} \fi

\bigskip

\begin{abstract}
  Several causal parameters in short panel data models are functionals of a nested nonparametric instrumental variable regression (nested NPIV). 
Recent examples include mediated, time varying, and long term treatment effects identified using proxy variables.
In econometrics, examples arise in triangular simultaneous equations and hedonic price systems.
However, it appears that explicit mean square convergence rates for nested NPIV are unknown, preventing inference on some of these parameters with generic machine learning. 
A major challenge is compounding ill posedness due to the nested inverse problems. 
To limit how ill posedness compounds, we introduce two techniques: relative well posedness, and multiple robustness to ill posedness.
With these techniques, we provide explicit mean square rates for nested NPIV and efficient inference for recently identified causal parameters. 
Our nonasymptotic analysis accommodates neural networks, random forests, and reproducing kernel Hilbert spaces. It extends to causal functions, e.g. heterogeneous long term treatment effects.
\end{abstract}

\noindent%
\vfill

\newpage
\spacingset{1.8} 

\section{Introduction and related work}\label{sec:intro}

Mediated, time varying, and long term treatment effects are causal parameters defined in short panel data models. In the presence of unobserved confounding, e.g. latent ability, several recent works have proposed nonparametric identification strategies for these causal parameters using auxiliary variables called proxies that satisfy relevance and exclusion conditions \citep{miao2018identifying,deaner2018nonparametric,dukes2023proximal,ying2023proximal,ghassami2022combining,imbens2022long}. Across settings, each causal parameter $\theta_0$ turns out to be a scalar summary of a \textit{nested} nonparametric instrumental variable regression (nested NPIV) function $h_0$. 
Structural parameters in several econometric models have a similar form, e.g. parameters in triangular simultaneous equations and hedonic price systems \citep{hausman1977errors,wooldridge1996estimating,newey1999nonparametric,ekeland2004identification,ai2007estimation}.

A nested NPIV function $h_0$ is a solution to an inverse problem of the form $\E\{h(B)|C\}=\E\{g_0(A)|C\}$, where $g_0$ is itself an NPIV function that solves an inverse problem. 
For example, $g_0$ is a solution to $\E\{g(A)|C'\}=\E(Y|C')$. In general, $C\not \subset C'$ and $C'\not \subset C$. 
Solving each equation requires inverting a conditional expectation operator, which is ill posed.
The nested NPIV $h_0$ is more challenging than the NPIV $g_0$ because ill posedness of the nested inverse problems may compound, in potentially complex ways.
Our research question is how to derive mean square rates for nested NPIV, and how to conduct inference on its functionals, with machine learning. 

\textbf{Contributions}. Our primary contribution is to derive mean square rates for the nested NPIV function $h_0$ over general function spaces, e.g. neural networks, random forests, and reproducing kernel Hilbert spaces (RKHSs). 
It appears that previous mean square consistency results for the nested NPIV function $h_0$ are specific to series estimation, and do not give an explicit mean square rate of convergence \citep{ai2007estimation}; they could arbitrarily slow. 
Stronger results, namely mean square rates, are necessary to conduct semiparametric inference for the causal parameter $\theta_0$ using machine learning. 

We derive increasingly strong rate results under increasingly strong assumptions. For mean square rates, we assume that (i) the function space $\mathcal{H}$ used in estimation is not too complex, satisfying a critical radius condition that is standard in the $M$ estimation literature; and (ii) $h_0$ is smooth, satisfying a source condition that is standard in the NPIV literature. For faster mean square rates, we formulate what appears to be a new condition: (iii) the nested inverse problems have a well behaved \textit{relative} measure of ill posedness.

As a secondary contribution, we translate our nested NPIV mean square convergence rates into guarantees for causal inference in short panel models, with sharper dependence on ill posedness.
Similar to previous works on targeted and debiased machine learning, we combine $(\hat{h},\hat{g})$ and their dual analogues into an estimator $\hat{\theta}$ \citep{zheng2011cross,chernozhukov2018original,chernozhukov2016locally}. Unlike previous works, we prove $\hat{\theta}$ 
 is \textit{multiple} robust to ill posedness: 
 it tolerates moderate ill posedness of multiple inverse problems, as long as other inverse problems are mildly ill posed, generalizing techniques previously developed in cross sectional models  \citep{chernozhukov2021simple}.

Our results apply to not only causal scalars but also causal functions in panel settings, e.g. heterogeneous long term treatment effects. Our class of causal parameters includes several for which machine learning estimation and inference were not previously given; see Section~\ref{sec:examples}.
%

Our techniques uncover new insights in economic data. We extend the program evaluation of the US Job Corps, which randomly assigned eligibility for job training. A previous, parametric approach corrects for latent motivation with proxy variables, and finds a zero or imprecise change in arrests directly due to job training.\footnote{This direct effect is the component of the total effect that is not mediated by employment.} Our RKHS approach finds a small and precise decrease in arrests directly due to job training.  We also extend the program evaluation of the Tennessee Student Teacher Achievement Ratio experiment (Project STAR), which randomly assigned kindergarten students to small or large class sizes. We document heterogeneity in the long term effects of enrollment in a small class, with the strongest long term effects for students with the lowest prior scores. 


\textbf{Related work.} We contribute new estimation and inference results to the literature on simultaneous equation models with different instruments for different equations \citep{wooldridge1996estimating,ai2007estimation}.  Much early work focused on parametric three stage least squares \citep{amemiya1977maximum} or  optimal instruments when $C'\subset C$ \citep{chamberlain1992comment,brown1998efficient,ai2012semiparametric}. We study a version of the problem motivated by short panel proxy models: an NPIV $g_0$ with instrument $C'$ enters the conditional moment for the nested NPIV $h_0$ with instrument $C$, where possibly $C\not \subset C'$ and $C'\not \subset C$ \citep{dukes2023proximal,ying2023proximal,ghassami2022combining,imbens2022long}. Its essential feature is that ill posedness compounds. This key challenge is the reason why prior works on proxies define the nested NPIV, but pose as a question how to derive its mean square rate.
The key challenge is assumed away in panel models with Markov or linear factor structure \citep{deaner2018nonparametric,imbens2021controlling}.
Compared to \cite{ai2007estimation}'s series analysis, we study a different machine learning estimator and provide different main results: explicit mean square rates, and inference beyond Donsker spaces.

Our critical radius and source assumptions generalize those of classical NPIV  \citep{newey2003instrumental,ai2003efficient,hall2005nonparametric,blundell2007semi,darolles2011nonparametric,chen2011rate,chen2012estimation,santos2012inference,severini2012efficiency}, while our relative well posedness assumption appears to be new. 
Similar to many NPIV papers, we use the source condition to control the bias of Tikhonov regularization \citep{darolles2011nonparametric,hall2005nonparametric,horowitz2005nonparametric,carrasco2007linear, chen2012estimation,gagliardini2012tikhonov,singh2019kernel}.

Whereas we provide adversarial estimation and inference results for nested NPIV and panel models, previous works provide adversarial estimation and inference results for NPIV and  cross sectional models. These earlier results do not face the key challenge of compounding ill posedness, and do not apply to our setting. For NPIV, several works prove projected mean square error \citep{dikkala2020minimax} and mean square error \citep{liao2020provably,bennett2023inference,bennett2023minimax,bennett2023source} rates under critical radius and source conditions.\footnote{See references therein for the vast literature on NPIV with non-adversarial machine learning, as well as adversarial approaches to NPIV without rate guarantees.} 
By contrast, we prove such rates for nested NPIV and introduce relative well posedness. 
%
For treatment effects in cross sections, several works prove Gaussian approximation \citep{hirshberg2021augmented,chernozhukov2020adversarial,kallus2021causal,ghassami2021minimax}. By contrast, we prove Gaussian approximation for treatment effects in short panels, which involve nested inverse problems.

An earlier draft circulated under a different title \citep{singh2021finite}. 
 This paper appears to be the first to:
(i) propose nested NPIV estimators over general machine function spaces with mean square rates;
(ii) define a notion of relative well posedness across inverse problems;
(iii) prove inference for machine learning estimators of proxy mediation analysis, heterogeneous long term effects, and other important causal parameters;
(iv) characterize multiple robustness to ill posedness.
None of (i), (ii), (iii), or (iv) appear to be contained in previous works.

 \textbf{Structure.} Section~\ref{sec:examples} demonstrates that several important causal parameters are functionals of a nested NPIV. 
 Section~\ref{sec:algo} proposes our machine learning procedure. 
 Section~\ref{sec:rate} proves new mean square rate guarantees, with and without the technique of relative well posedness. 
Section~\ref{sec:inference} translates the rates into semiparametric inference on causal parameters, with multiple robustness to ill posedness.
 Section~\ref{sec:experiments} presents real world applications: the proxy direct treatment effect of the US Job Corps, and heterogeneous long term treatment effects of Project STAR.
 
\section{Examples in statistics and econometrics}\label{sec:examples}

The nested NPIV is a recurring challenge in several areas of statistics and econometrics, which motivates our investigation. In this section, we list some salient examples from the modern causal inference literature and the traditional economic literature on supply and demand estimation. To the best of our knowledge, this paper provides the first explicit estimation and inference guarantees with machine learning for Examples~\ref{ex:direct},~\ref{ex:time}, and~\ref{ex:het} below.

The initial examples use proxy variables, i.e. auxiliary variables that satisfy relevance and exclusion conditions, to adjust for unobserved confounding in short panel models. Throughout these examples, we denote covariates, treatments, and outcomes by $(X,D,Y)$. Treatment and outcome proxies are $(Z,W)$. When applicable, the mediators are $M$. The potential outcomes are $Y^{(d)}$ or $Y^{(d,m)}$, and the potential mediators are $M^{(d)}$.

\begin{example}[Proxy mediation analysis]\label{ex:direct}
    Let $h_0$ be an outcome confounding bridge that solves the inverse problem $\E\{h(X,D,W)|X,Z,D=0\}=\E\{g_0(X,D=1,M,W)|X,Z,D=0\}$, where $g_0$ solves the inverse problem
$\E\{g(X,D,M,W)|X,Z,D=1,M\}=\E(Y|X,Z,D=1,M)$. Under proxy variable assumptions \citep{dukes2023proximal}, the pure direct effect $\textsc{direct}=\E[Y^{\{1,M^{(0)}\}}]$ is a functional of $h_0$, and $h_0$ is a term in its efficient influence function.\footnote{This definition follows \cite{robins1992identifiability}. See \cite{richardson2013single} for alternatives.}
\end{example}

Concretely, we may wish to measure the direct effect of US Job Corps job training $D$ on subsequent arrests $Y$ that is not mediated by employment $M$. Unobserved motivation $U$ may confound employment and arrests. Researchers have used the time spent with a Job Corps recruiter as an auxiliary variable $Z$ that reflects motivation yet does not directly cause employment or arrests. Researchers have used pre-training expectations as an auxiliary variable $W$ that reflects motivation yet cannot be caused by training or employment. We replicate and extend this empirical strategy in Section~\ref{sec:experiments}, relaxing previous parametric assumptions.

In the next example, we index variables by the time period when they are observed.

\begin{example}[Proxy time varying treatment effect]\label{ex:time}
     Let $h_0$ be an outcome confounding bridge that solves the inverse problem 
    $\E\{h(X_1,D_1,W_1,D_2=d_2)|X_1,Z_1,D_1=d_1\}=\E\{g_0(X_1,D_1,W_1,X_2,D_2=d_2,W_2)|X_1,Z_1,D_1=d_1\}$, 
    where $g_0$ solves the inverse problem 
    $$\E\{g(X_1,D_1,W_1,X_2,D_2,W_2)|X_1,Z_1,D_1=d_1,X_2,Z_2,D_2=d_2\}=\E(Y|X_1,Z_1,D_1=d_1,X_2,Z_2,D_2=d_2).$$
    Under proxy variable assumptions \citep{ying2023proximal}, the time varying treatment effect $\textsc{time}(d_1,d_2)=\E\{Y^{(d_1,d_2)}\}$ is a functional of $h_0$, and $h_0$ is a term in its efficient influence function.
\end{example}

When combining short term experimental data $(G=\EXP)$ with long term observational data $(G=\OBS)$ for long term causal inference, a similar statistical problem arises.

\begin{example}[Proxy long term effect]\label{ex:long}
    Let $h_0$ be an outcome confounding bridge that solves the problem 
    $
   h_0(X,D,G=\EXP)=\E\{g_0(X,D,M,W)|X,D,G=\EXP\},
    $
where      
    $g_0$ solves the inverse problem 
    $\E\{g(X,D,M,W)|X,D,Z,M,G=\OBS\}=\E(Y|X,D,Z,M,G=\OBS)$.
    Under proxy variable assumptions \citep{ghassami2021minimax,imbens2022long}, the long term effect 
    $
    \textsc{long}=\E\{Y^{(d)}|G=\OBS\}
    $ is a functional of $h_0$, and $h_0$ is a term in its efficient influence function.
\end{example}

Similar to Example~\ref{ex:direct}, Examples~\ref{ex:time} and~\ref{ex:long} use auxiliary variables to detect and correct for unobserved confounding. 
While Examples~\ref{ex:direct} and~\ref{ex:time} each have two inverse problems, Example~\ref{ex:long} has one. 
It is a special case of our framework, since a nonparametric regression is a special case of an NPIV.\footnote{In the notation of Section~\ref{sec:intro}, $B=C$ in Example~\ref{ex:long}, so $h(C)=\E\{h(C)|C\}=\E\{g_0(A)|C\}$.}
Example~\ref{ex:long} was proposed subsequently to the initial draft of this paper, highlighting how proxy panel models are an active area of research, generating new nested NPIV examples.\footnote{Some additional, subsequent examples include \cite{park2024proximal,bai2025proximal}. Neither work provides a mean square rate for the nested NPIV, highlighting the importance of our main result.}
 
The examples so far are causal scalars. Our results extend to causal functions, e.g. heterogeneous treatment effects. For clarity, we set aside proxy variables and let $X_1\in\R$ be the baseline covariate of interest used to index heterogeneity. Note that a nested  regression is another special case of a nested NPIV.\footnote{In the notation of Section~\ref{sec:intro}, $A=C'$ and $B=C$ in Example~\ref{ex:het}.}

\begin{example}[Heterogeneous long term effects]\label{ex:het}
Let $h_0$ be an outcome bridge that solves the problem 
    $h_0(X_1,X_2,D)=\E\{g_0(X_1,X_2,M,G=\OBS)|X_1,X_2,D,G=\EXP\}$, 
    where $g_0$ is the outcome mechanism 
    $g_0(X_1,X_2,M,G)=\E(Y|X_1,X_2,M,G).$
    Under surrogate variable assumptions \citep{athey2019surrogate}, the heterogeneous long term effects $\textsc{long}(x_1)=\E\{Y^{(d)}|X_1=x_1\}=\lim_{\lambda\rightarrow 0} \E\{\ell_{\lambda}(X_1)Y^{(d)}\}$ are functionals of $\ell_{\lambda}h_0$, and $\ell_{\lambda} h_0$ is a term in their influence functions.
\end{example}

Example~\ref{ex:het} generalizes from average to heterogeneous effects by introducing 
a local weighting $\ell_{\lambda}(X_1)$ around the value $x_1$ with bandwidth $\lambda$. Formally, 
 $\ell_{\lambda}(X_1)=\omega^{-1} K\{(X_1-x_1)/\lambda\}$, where $K$ is a bounded symmetric kernel that integrates to one, and $\omega=\E[K\{(X_1-x_1)/\lambda\}]$ is its normalization. Our results complement those of \cite{kallus2018removing}, who study a different problem. Our results also apply to heterogeneous direct effects and heterogeneous time varying effects.

Concretely, we may wish to extrapolate long term effects of kindergarten class size $D$ on middle school test scores $Y$. The experimental group $(G=\EXP)$ are students in Project STAR, for whom we see elementary school test scores $M$ rather than middle school test scores $Y$. The observational group $(G=\EXP)$ are students in New York City public schools, for whom we see elementary and middle school test scores but not kindergarten class size. In Section~\ref{sec:experiments}, we ask whether the long term effects of kindergarten class size vary based on baseline aptitude $X_1$.
 
 By taking $\ell_{\lambda}(X_1)=1$, we provide estimation and inference for the average long term effect with machine learning. This modest yet practical contribution builds on the efficient influence function of \cite{chen2021semiparametric}.\footnote{Contemporaneously to \cite{singh2021finite}, \cite{chen2022semiparametric} provide a similar result for average long term  effects, but not for heterogeneous long term  effects.}

Finally, we turn to a classic example in economics. Consider the equations $Q=\textsc{demand}(P,W_Q)+U_Q$ and $P=\textsc{supply}(Q,W_P)+U_P$, where $Q$ is quantity sold, $P$ is price, $W_Q$ is consumer income, $W_P$ is input price, and $(U_Q,U_P)$ are unobservable demand and supply shocks \citep{matzkin2008identification}. A standard simultaneous equation model would consider $(W_Q,W_P)$ to be exogenous: $\E(U_Q|W_Q,W_P)=0$ and $\E(U_P|W_Q,W_P)=0$. 
A generalization that tolerates omitted variables is  $\E(U_Q|Z_Q,W_P)=0$ and $\E(U_P|W_Q,Z_P)=0$; in the terminology of econometrics, $Z_Q$ instruments for $W_Q$, and $Z_P$ instruments for $W_P$ \citep{hausman1977errors}.

\begin{example}[Simultaneous equations with omitted variables]\label{ex:supply}
    Let $h_0$ solve the inverse problem $\E\{h(P,Q,W_P, W_Q)|W_Q,Z_P\}=\E\{g_0(P,W_Q)|W_Q,Z_P\}$, where $g_0$ solves the inverse problem $\E\{g(P,W_Q)|Z_Q,W_P\}=\E(Q|Z_Q,W_P)$. Under instrumental variable assumptions \citep{hausman1977errors}, the demand curve $\textsc{demand}(P,W_Q)$ is $g_0(P,W_Q)$ and the supply curve $\textsc{supply}(Q,W_P)$ is $h_0(P,Q,W_P, W_Q)-\textsc{demand}(P,W_Q)+P$. Elasticities are their average derivatives.
\end{example}
\section{Nested NPIV estimators over general spaces}\label{sec:algo}

While there are several causal and structural models which give rise to nested NPIVs, they are all unified by the following problem: $h_0$ solves $\E\{h(B)|C\}=\E\{g_0(A)|C\}$, where $g_0$ solves $\E\{g(A)|C'\}=\E(Y|C')$, and in general $C\not \subset C'$ and $C'\not \subset C$.\footnote{In general, the solutions are not unique; our algorithms will target the minimal norms solutions.} 
Conditions for $h_0$ to exist are well documented in prior work, using Picard's criterion \cite[Theorem 15.18]{kress1989linear}. 
In this section, we define two natural estimators of the nested NPIV. Each estimator is defined over a general space $\mathcal{H}$, e.g. a neural network, random forest, or RKHS.

Since the nested NPIV reduces to NPIV when $g_0(A)=Y$, a \textit{sequential} approach replaces $Y$ with an initial NPIV estimator $\hat{g}(A)$ to estimate $\hat{h}(B)$. Building on an extensive literature for adversarial NPIV estimation, \cite[e.g.][]{dikkala2020minimax,liao2020provably,bennett2023minimax,bennett2023source}, we define a sequential nested NPIV estimator as follows. 

\begin{algorithm}[Sequential nested NPIV]\label{algo:sequential}
    Given observations $(A_i,B_i,C_i)$, an initial estimator $\hat{g}$ that may be estimated on the same data, and a hyperparameter $\mu>0$,
    $$
    \hat{h}=\argmin_{h\in\mathcal{H}}\left[\sup_{f\in\mathcal{F}}\left\{ 2\cdot \textsc{loss}(f,\hat{g},h)-\textsc{penalty}(f)\right\}+\mu\cdot \textsc{penalty}(h)\right],
    $$
    where $\textsc{loss}(f,g,h)=\mathbb{E}_n[\{h(B)-g(A)\}f(C)]$, $\textsc{penalty} (f)=\mathbb{E}_n\{f(C)^2\}$, $\textsc{penalty}(h)= \mathbb{E}_n\{h(B)^2\}$, and $\E_n(\cdot)$ averages over observations.
\end{algorithm}

To interpret the estimator, recall that the 
 inverse problem that defines $h_0$ may be viewed as the conditional moment restriction $\mathbb{E}\{h_0(B)-g_0(A)|C\}=0$, which implies a continuum of unconditional moment restrictions $\mathbb{E}[\{h_0(B)-g_0(A)\}f(C)]=0$ for all square integrable $f \in\mathbb{L}_2$.  The adversary attempts to violate the empirical moments by searching for the worst counterexample $f$. Our estimator attempts to preserve the empirical moments, by searching for an estimate $h$ that is the most robust to the adversary's worst counterexample. The following corollary shows that the penalty terms normalize the game in the right way.

 \begin{corollary}[Sequential limit]\label{cor:limit_sequential}
     Consider the population limit of Algorithm~\ref{algo:sequential} where $n\rightarrow \infty$, $\mu\rightarrow0$, and $\hat{g}$ converges to $g_0$: $\sup_{f\in\mathcal{F}} \mathbb{E}[2\{h(B)-g_0(A)\}f(C)-f(C)^2]$. If $\E\{h(B)-h_0(B)|C=(\cdot)\}\in \mathcal{F}$, then this limit equals $\E([\E\{h(B)-h_0(B)|C\}]^2)$.
 \end{corollary}

 In the final expression, the limit is the mean square error of the projection of $h(B)-h_0(B)$ onto the instrument $C$. It is the well known weak metric in the NPIV literature \cite[eq. 14]{ai2003efficient}. Thus Algorithm~\ref{algo:sequential}'s empirical criterion converges to the weak metric, even though an analyst does not have access to unbiased samples from $h_0(B)$.

The sequential approach is largely agnostic about the initial estimator $\hat{g}$. On the one hand, its analysis will be simpler. On the other hand, it presupposes that $\hat{g}$ converges to $g_0$ in a mean square sense, so ill posedness in estimating $g_0$ may sequentially compound when estimating $h_0$. Towards a better dependence on ill posedness, we define a simultaneous approach that jointly estimates $\hat{g}(A)$ and $\hat{h}(B)$, in the spirit of \cite{ai2007estimation}.

\begin{algorithm}[Simultaneous nested NPIV]\label{algo:simultaneous}
      Given observations $(A_i,B_i,C_i,C_i')$ and hyperparameter values $(\mu',\mu)$,
      \begin{align*}
    (\hat{g},\hat{h})&=\argmin_{g\in\mathcal{G}, h\in\mathcal{H}}\bigg[\sup_{f'\in\mathcal{F}'}\left\{ 2\cdot \textsc{loss}(f',Y,g)-\textsc{penalty}(f')\right\}+\mu' \cdot \textsc{penalty}(g) \\
   &\quad +
   \sup_{f\in\mathcal{F}}\left\{ 2\cdot \textsc{loss}(f,g,h)-\textsc{penalty}(f)\right\}+\mu \cdot \textsc{penalty}(h)
   \bigg],
   \end{align*}
    where $\textsc{loss}(f',Y,g)=\mathbb{E}_n[\{g(A)-Y\}f'(C')]$. The remaining terms are defined as before.
\end{algorithm}

As before, the inverse problem that defines $h_0$ implies a continuum of unconditional moment restrictions $\mathbb{E}[\{h_0(B)-g_0(A)\}f(C)]=0$. Now, 
the inverse problem that defines $g_0$ implies another continuum of unconditional moment restrictions $\mathbb{E}[\{g_0(A)-Y\}f'(C')]=0$. The adversary attempts to violate the empirical moments by searching for the worst counterexamples $(f',f)$, while our estimator attempts to preserve them by searching for robust estimates $(g,h)$. Crucially, $g$ must be chosen to balance the two aspects of the game.

\begin{corollary}[Simultaneous limit]\label{cor:limit_simultaneous}
     Consider the population limit of Algorithm~\ref{algo:simultaneous} where $n\rightarrow \infty$, $\mu,\mu'\rightarrow0$: 
     $\sup_{f'\in\mathcal{F}} \mathbb{E}[2\{g(A)-Y\}f'(C')-f'(C')^2]
     +\sup_{f\in\mathcal{F}} \mathbb{E}[2\{h(B)-g(A)\}f(C)-f(C)^2].$
     If $\E\{g(A)-g_0(A)|C'=(\cdot)\}\in\mathcal{F}'$, $\E\{h(B)-h_0(B)|C=(\cdot)\}\in \mathcal{F}$, and   $\E\{g(B)-g_0(B)|C=(\cdot)\}\in \mathcal{F}$, 
     then it equals $\E([\E\{g(A)-g_0(A)|C'\}]^2)+\E\{ (\E[\{h(B)-h_0(B)\}-\{g(A)-g_0(A)\}|C])^2\}$.
\end{corollary}

The limit contains two terms: (i) the bias of $g$ projected onto the initial instrument $C'$; (ii) the gap between the biases of $h$ and $g$ projected onto the nested instrument $C$. The limit coincides with the generalized weak metric of \citet[eq. 14]{ai2003efficient} and \citet[p. 16]{ai2007estimation}, i.e. the metric in which those authors provide rates for series estimators. By contrast, we provide rates in mean square error. Moreover, we show that if $(C',C)$ are compatible in some way, then these various biases can be small and compounding ill posedness can be avoided. 
\section{Main result: Mean square rates}\label{sec:rate}

Our main contribution is explicit mean square rates of convergence for machine learning estimators of nested NPIV, i.e. $\E[\{\hat{h}(B)-h_0(B)\}^2]$. Such a result is a crucial building block for efficient inference on the functionals in Section~\ref{sec:examples}. We provide rates for Algorithm~\ref{algo:sequential} using critical radius and source conditions, then faster rates for Algorithm~\ref{algo:simultaneous} under critical radius, source, and relative well posedness conditions. Relative well posedness appears to be a novel concept.

We denote the conditional expectation operators $T(h,g)=\mathbb{E}\{h(B)-g(A)|C=(\cdot)\}$ and $S(g)=\mathbb{E}\{g(A)|C'=(\cdot)\}$.
Since $(g_0,h_0)$ may be non-unique, we take $(g_0,h_0)$ to be solutions to $T(h,g)=0$ and $S(g)=\mathbb{E}(Y|C'=\cdot)$ with the smallest second moments. 
We assume they exist. Explicit conditions for their existence are given in the various identification papers listed in Section~\ref{sec:examples}.
To lighten notation, we also write $T_h(h)=T(h,0)$ and $T_g(g)=T(0,g)$. Finally, we write the $\mathbb{L}_2$ norm as $\|\tilde{f}\|^2=\E\{\tilde{f}(\tilde{A})^2\}$.

Throughout, we assume that the data are independent and identically distributed. Future work may relax this simplifying assumption.

\subsection{Rates with compounding ill posedness}

\textbf{Permissive assumption: Critical radius.} We assume that our estimator $\hat{h}$ is not too complex in a familiar sense. 
Consider a generic function space $\tilde{\mathcal{F}}$ that contains functions $f:\tilde{\mathcal{A}}\rightarrow \mathbb{R}$ which are uniformly and absolutely bounded by one.\footnote{This can be relaxed to functions that have some finite bound by a rescaling argument.} Let $\epsilon_i$ be independent random variables taking values on $\{-1,1\}$ equiprobably. The local Rademacher complexity of $\tilde{\mathcal{F}}$ over a neighborhood of radius $\delta$ is $\mathcal{R}_n(\tilde{\mathcal{F}},\delta)= \mathbb{E} \left\{ \sup_{f\in\tilde{\mathcal{F}}:\|f\|\leq\delta} \frac{1}{n}\sum_{i=1}^n \epsilon_i f(\tilde{A}_i)\right\}$. Building on a mature $M$ estimation literature, we bound $\|\hat{h}-h_0\|$ in terms of local Rademacher complexities.

We optimize these bounds into fast rates by an appropriate choice of the radius $\delta$, called the critical radius.  The critical radius $\delta_n$ is the smallest possible solution to the inequality $\mathcal{R}_n\{\textsc{star}(\tilde{\mathcal{F}}),\delta\}\leq\delta^2$, where $\textsc{star}(\tilde{\mathcal{F}})=\{cf:f\in\tilde{\mathcal{F}},c\in [0,1]\}$ is the star hull of $\tilde{\mathcal{F}}$. For many spaces, it is a converging sequence as $n\uparrow \infty$.

\begin{assumption}[Critical radius]\label{assumption:critical}
   For the space $\tilde{\mathcal{F}}$, which we will instantiate in various ways below, its critical radius converges: $\delta_n=\tilde{O}(n^{-\alpha})$ for some $\alpha\in(0,1/2]$, where $\tilde{O}(\cdot)$ means $O(\cdot)$ up to logarithmic factors.
\end{assumption}

\begin{example}[Gaussian kernel]\label{ex:kernel}
    Let $\tilde{\mathcal{F}}$ be the unit ball in the RKHS with the Gaussian kernel, with uniform data over $[0,1]$. Then $\delta_n=\sqrt{\frac{\ln(n)}{n}}$ and $\alpha=1/2$.
\end{example}

\begin{example}[Neural network]\label{ex:net}
    Let $\tilde{\mathcal{F}}$ be a rectified linear unit (ReLU) neural network with depth $L$ and a total of $W$ parameters.\footnote{For a multilayer perceptron with depth $L$ and width $H$, $W=O(H^2L)$.} Then $\delta_n=\sqrt{\frac{LW \ln(W)\ln(n)}{n}}$ and $\alpha<1/2$.
\end{example}

See e.g. \cite{foster2023orthogonal} for the critical radii of other RKHSs, other neural networks, random forests, and sparse linear spaces, demonstrating that the condition is quite permissive. See e.g. \cite{wainwright2019high} for an equivalent entropy integral condition. See \cite{chernozhukov2020adversarial} for a detailed exposition of how the critical radius condition relaxes the Donsker condition to accommodate ``simple'' machine learning. 

Appendix~\ref{sec:sourceless} gives projected mean squares rates using only critical radius conditions.

\textbf{Familiar assumption: Source condition.} To pass from a projected mean square rate to a mean square rate, i.e. from $\|T_h(\hat{h}-h_0)\|^2=\mathbb{E}([\E\{\hat{h}(B)-h_0(B)|C\}]^2)$ to  $\|\hat{h}-h_0\|^2=\mathbb{E}[\{\hat{h}(B)-h_0(B)\}^2]$, 
we assume that the nested NPIV $h_0$ is smooth in a familiar sense. In particular, we assume it well approximated by the top of the spectrum of the conditional expectation operator $T_h(h)=\mathbb{E}\{h(B)|C=(\cdot)\}$. To state the condition, we denote the adjoint operator by $T_h^*$. 
\begin{assumption}[Source]\label{assumption:source}
   $h_0=(T_h^*T_h)^{\beta_h/2}w_h$ for some $w_h\in\mathcal{H}$ and some $\beta_h\in (0,\infty)$. 
\end{assumption}

To interpret Assumption~\ref{assumption:source}, suppose $T_h$ admits a singular value decomposition $(\sigma_j,u_j,v_j)$, so that $T_hh=\sum_{j=1}^{\infty} \sigma_j \langle h,v_i\rangle_{\mathbb{L}_2} u_j$.  Assumption~\ref{assumption:source} is equivalent to imposing that $h_0\in \mathcal{H}^{\beta_h}= \{h\in \mathbb{L}_2: \sum_{j=1}^{\infty} 1_{\sigma_j\neq 0} \cdot  \sigma_j^{-2\beta_h} \langle h,v_j \rangle_{\mathbb{L}_2}^2<\infty \}$, where $\mathbb{L}_2$ is the space of square integrable functions. Clearly, $\mathcal{H}^0=\mathbb{L}_2$ when each $\sigma_j\neq 0$. For $\beta_h>0$, $\mathcal{H}^{\beta_h}$ is a subset of $\mathcal{H}^0$. Due to the singular value penalty $\sigma_j^{-2\beta_h}$, it consists of functions that are not too aligned with the higher order right singular functions $(v_j)$. Intuitively, it rules out ``rough'' functions primarily supported on the tail of the spectrum of $T_h$. In this way, it limits the ill posedness of the inverse problem. 

To lighten notation, let $\textsc{well}(\beta)=\frac{\min(\beta,1)}{\min(\beta,1)+1}$ be a measure of well posedness. At best, $\beta\geq 1$ and $\textsc{well}(\beta)=1/2$. At worst, $\beta\rightarrow 0$ and $\textsc{well}(\beta)\rightarrow 0$. Our final rates will depend on the well posedness $\textsc{well}(\beta_h)$.

See e.g. \cite{chen2011rate}
for comparisons of source conditions employed in the NPIV 
literature. 
In their series analysis, \citet[Assumption 3.8.3]{ai2007estimation} place an assumption equivalent to $\beta_h=2$ \cite[Proposition 5]{bennett2023inference} and prove mean square consistency, whereas we consider $\beta_h< 2$ and prove mean square rates.

\textbf{Compounding ill posedness throttles rates.} Under Assumptions~\ref{assumption:critical} and~\ref{assumption:source}, we derive rates for $\|\hat{h}-h_0\|^2$ that are throttled twice: by the ill posedness of inverting $S$ and then $T$. To simplify exposition, we assume that the conditional expectation operator $T$ is correctly specified by the adversary's function space $\mathcal{F}$ in Algorithm~\ref{algo:sequential}. This simplifying condition can be relaxed, incurring an additive approximation error; see our earlier draft.  

\begin{assumption}[Closedness]\label{assumption:closed}
$T(h-h_0,g-g_0) \in \mathcal{F}$ for all $h\in \mathcal{H}$ and $g\in\mathcal{G}$.
\end{assumption}

\begin{theorem}[Bound for Algorithm~\ref{algo:sequential}]\label{theorem:L2}
    Suppose Assumption~\ref{assumption:critical} holds for $\mathcal{F}$, $\mathcal{G}$, $\mathcal{H}$, and $\mathcal{H}\times \mathcal{F}$;
        Assumption~\ref{assumption:source} holds; and Assumption~\ref{assumption:closed} holds. Then with probability $1-\zeta$, when $\mu=O(1)$ and $\delta_n =\Omega[\{\ln\ln(n)+\ln(1/\zeta)\}^{1/2}n^{-1/2}]$, we have 
    $\|T_h(\hat{h}-h_0)\|^2=O(R_n)$ and $\|\hat{h}-h_0\|^2=O(\mu^{-1}R_n)$, where 
   $R_n=\mu^{\min(\beta_h+1,2)}\|w_h\|^2+\delta_n^2+\|\hat{g}-g_0\|^2.$ The former conclusion, $\|T_h(\hat{h}-h_0)\|^2=O(R_n)$, also holds without Assumption~\ref{assumption:critical} applied to $\mathcal{H}\times \mathcal{F}$.
\end{theorem}

Theorem~\ref{theorem:L2} is our first main result: a mean square rate for nested NPIV. 
The rate $R_n$ contains bias $\mu^{\min(\beta_h+1,2)}\|w_h\|$, variance $\delta_n^2$, and initial estimation error $\|\hat{g}-g_0\|^2$. 
It generalizes known results for NPIV: without the initial estimation error term, it recovers NPIV rates in mean square error \citep{liao2020provably,bennett2023minimax,bennett2023source} and projected mean square error \citep{dikkala2020minimax}. It seems to strengthen known projected mean square results for NPIV since it relaxes Assumption~\ref{assumption:critical} from $\mathcal{H}\times \mathcal{F}$ to $\mathcal{H}$, which may be of independent interest.

Theorem~\ref{theorem:L2} cleanly separates the ill posedness of $\hat{g}$ from $\hat{h}$.  The ill posedness of $\hat{g}$ appears via $\|\hat{g}-g_0\|^2$, which is additively separable from the other terms in $R_n$. The projected mean square rate $R_n$ is not affected by the ill posedness of $\hat{h}$. However, the rate for mean square error is slower than the rate for projected mean square error by a factor of $\mu^{-1}$, encoding the ill posedness of inverting $T_h$ to isolate the nested NPIV. 

The rates may be optimized by choosing $\mu$ to balance the bias term with the other terms. 

\begin{corollary}[Rate for Algorithm~\ref{algo:sequential}]~\label{cor:L2}
   Suppose the conditions of Theorem~\ref{theorem:L2} hold. Take $\mu=\max(\delta_n,\|\hat{g}-g_0\|)^{\frac{2}{\min(\beta_h,1)+1}}$. Then with probability $1-\zeta$, $\|T_h(\hat{h}-h_0)\|^2=O \left\{\max(\delta_n^2,\|\hat{g}-g_0\|^2)\right\} $ and $\|\hat{h}-h_0\|^2=O\left\{\max(\delta_n^2,\|\hat{g}-g_0\|^2)^{\textsc{well}(\beta_h)}\right\}$.
\end{corollary}

For concreteness, we instantiate a well known NPIV estimator for $g_0$, namely the NPIV special case of Algorithm~\ref{algo:sequential}. We place assumptions for $g_0$ analogous to those for $h_0$. 

\begin{assumption}[Source]\label{assumption:source2}
   $g_0=(S^*S)^{\beta'_g/2}w'_g$ for some $w_g'\in \mathcal{G}$ and some $\beta'_g\in(0,\infty)$.
\end{assumption}

\begin{assumption}[Closedness]\label{assumption:closed2}
$S(g-g_0) \in \mathcal{F}'$ for all $g\in \mathcal{G}$.
\end{assumption}

\begin{corollary}[Compounding ill posedness]~\label{cor:L2_compound}
    Suppose the conditions of Theorem~\ref{theorem:L2} hold, extending Assumption~\ref{assumption:critical} to $\mathcal{F}'$ and $\mathcal{G}\times\mathcal{F}'$; and 
    Assumptions~\ref{assumption:source2} and~\ref{assumption:closed2} hold. Write the regularizations as $(\mu_g,\mu_h)$. Set $\mu_g=\delta_n^{\frac{2}{\min(\beta'_g,1)+1}}$ and $\mu_h=\delta_n^{\frac{2}{\min(\beta_h,1)+1}\textsc{well}(\beta'_g)}$. Then with probability $1-\zeta$, $\|T_h(\hat{h}-h_0)\|^2=O \left\{\delta_n^{2\textsc{well}(\beta'_g)}\right\}  $ and $\|\hat{h}-h_0\|^2=O \left\{\delta_n^{2\textsc{well}(\beta_h)\textsc{well}(\beta'_g)}\right\}$.
\end{corollary}

\addtocounter{example}{-2}

\begin{example}[Gaussian kernel]
   By Corollary~\ref{cor:L2_compound}, 
    $\|\hat{h}-h_0\|^2=O_p \left[\left\{\frac{\ln(n)}{n}\right\}^{\textsc{well}(\beta_h)\textsc{well}(\beta'_g)}\right]$.
\end{example}

\begin{example}[Neural network]
     By Corollary~\ref{cor:L2_compound}, 
    $\|\hat{h}-h_0\|^2=O_p \left[\left\{\frac{LW \ln(W)\ln(n)}{n}\right\}^{\textsc{well}(\beta_h)\textsc{well}(\beta'_g)}\right]$.
\end{example}

Ill posedness compounds in a simple way for Algorithm~\ref{algo:sequential}: our mean square rate is a ``base rate'' $\delta_n^2$ slowed by the ill posedness of each inverse problem via the product $\textsc{well}(\beta'_g)\cdot \textsc{well}(\beta_h)$. Recall that $\delta_n^2$ is $\tilde{O}(n^{-1})$ for parametric classes, and it converges at well known, often optimal regression rates for many nonparametric classes.  At best, $\beta'_g ,\beta_h \geq 1$ and Algorithm~\ref{algo:sequential}'s mean square rate is $O(\delta_n^{1/2})$. Our mean square rates do not further improve for stronger source conditions, echoing the saturation effect of ridge regression \citep{bauer2007regularization}.  By contrast, Algorithm~\ref{algo:simultaneous}'s mean square rate will be $O(\delta_n)$ at best, under a new assumption.

\subsection{Rates without compounding ill posedness}

\textbf{New assumption: Relative well posedness.} For faster rates of convergence, we place what appears to be a new assumption: the $g_0$ inverse problem is relatively well posed compared to the $h_0$ inverse problem. For readability, we focus on the main conditions below and defer an additional, technical condition to Appendix~\ref{sec:extra}. Denote the singular values of the conditional expectation operator $S$ by $(\sigma_j')$ and those of $T_g$ by $(\sigma_j)$.

\begin{assumption}[Relative well posedness]\label{assumption:posedness}
$\|(S^*S+\mu'I)^{-1}T_g^*T_g\|_{\op}
=O(1)$ for all $\mu'>0$.
\end{assumption}

If $S$ and $T_g$ have the same right singular functions, then $\|(S^*S+\mu'I)^{-1}T_g^*T_g\|_{\op}=\sup_j\frac{\sigma_j^2}{(\sigma_j')^2+\mu'}$ is clearly recognizable as a relative measure of well posedness for the conditional expectation operators $S$ and $T_g$, up to a tolerance $\mu'$. Assumption~\ref{assumption:posedness} imposes that the singular values of $S$, aided by $\mu'>0$, are not too small relative to the singular values of $T_g$. 
It essentially requires that $C'$ is a strong enough instrument relative to $C$.

\begin{proposition}[Linear models with scaled data]\label{prop:linear}
    Suppose that $(g_0,h_0)$ are linear and $(A,B,C,C')$ are scalars with mean zero, unit variance,  $\textsc{corr}(A,C)=\rho_A$, $\textsc{corr}(B,C)=\rho_B$, and
    $\textsc{corr}(A,C')=\rho'$, where the correlations may be sequences decreasing in $n$. Assumption~\ref{assumption:posedness} holds when 
    $|\rho_A/\rho'|$ is $O(1)$. 
\end{proposition}

\begin{remark}[Intuition]\label{remark:intuition}
    Here, $g_0=\frac{\mathbb{E}(YC')}{\rho'}$ and $h_0=\frac{\mathbb{E}(YC')}{\rho'}\frac{\rho_A}{\rho_B}$. Without our condition, ill posedness manifests as the \textit{product} of instrument strength $|\rho'|\cdot |\rho_B|$. With our condition, it manifests as the \textit{minimum} of $|\rho'|$ and $|\rho_B|$, since $|g_0|\leq \frac{|\mathbb{E}(YC')|}{|\rho'|}$ and $|h_0| \lesssim \frac{|\mathbb{E}(YC')|}{|\rho_B|}$.
\end{remark}

\begin{proposition}[Nonlinear models with Gaussian data]\label{prop:gaussian}
    Suppose that $(A,B,C,C')$ are jointly normal scalars with mean zero, unit variance, 
    $\textsc{corr}(A,C)=\rho_A$,
   and $\textsc{corr}(A,C')=\rho'$.
   Assumption~\ref{assumption:posedness} holds when $|\rho_A/\rho'|$ is $O(1)$.
\end{proposition}

When $(A,B,C,C')$ are standard normal random vectors, then relative well posedness imposes that the canonical correlations between elements of $A$ and $C'$ 
are at least the same order as the canonical correlations between elements of $A$ and $C$. See Appendix~\ref{sec:analytic} for details.

In addition, we impose what appears to be a new notion of completeness.

\begin{assumption}[Relative completeness]\label{assumption:completeness}
    If $Sg=0$ then $T_g g=0$. The range of $S$ is closed.
\end{assumption}

Under Assumption~\ref{assumption:completeness}, if a function is indistinguishable to the initial conditional expectation operator $S:g(\cdot)\mapsto \E\{g(A)|C'=(\cdot)\}$, then it is also indistinguishable to the nested conditional expectation operator $T(0,g):g(\cdot)\mapsto \E\{0-g(A)|C=(\cdot)\}$. Intuitively, Assumption~\ref{assumption:completeness} ensures that solving the nested NPIV operator equation does not require information that cannot be learned from the initial NPIV operator equation. 

Assumption~\ref{assumption:completeness} naturally extends the standard completeness condition for NPIV: $Sg=0$ implies $g=0$ \citep{newey2003instrumental}. 
Our condition is weaker. It requires injectivity only on a restricted subspace defined relative to the next operator, rather than on the entire function space.  

\textbf{Faster rates via relative well posedness.} Assumption~\ref{assumption:posedness} improves convergence rates by limiting how the ill posedness of $g_0$ compounds the ill posedness of $h_0$. We arrive at a somewhat surprising result: under this auxiliary condition, the final rate for $h_0$ depends on the \textit{minimum} of $\textsc{well}(\beta_h)$ and $\textsc{well}(\beta_g')$ rather than their product, which is a dramatic improvement. In other words, rates are throttled only once. Such a result requires that $g_0$ is smooth with respect to not only $S$ but also $T$, i.e. another source condition.

\begin{assumption}[Source]\label{assumption:source3}
   $g_0=(T_g^*T_g)^{\beta_g/2}w_g$ for some $w_g\in\mathcal{G}$ and some $\beta_g\in (0,\infty)$.
\end{assumption}

\begin{theorem}[Bound for Algorithm~\ref{algo:simultaneous}]\label{theorem:joint}
    Suppose Assumption~\ref{assumption:critical} holds for $\mathcal{F}$, $\mathcal{G}$,  $\mathcal{H}$, $\mathcal{H}\times \mathcal{F}$, $\mathcal{F}'$, and $\mathcal{G}\times \mathcal{F}'$; 
    Assumptions~\ref{assumption:source} through~\ref{assumption:source3} hold; and an additional regularity condition holds given in Appendix~\ref{sec:extra}. Then with probability $1-\zeta$, when $\mu=\mu'=O(1)$ and $\delta_n =\Omega[\{\ln\ln(n)+\ln(1/\zeta)\}^{1/2}n^{-1/2}]$ (a simplifying lower bound), we have $\|S(\hat{g}-g_0)\|^2=O(R_n)$, $\|T(\hat{h}-h_0,\hat{g}-g_0)\|^2=O(R_n)$, $\|\hat{g}-g_0\|^2=O(\mu^{-1}R_n)$,  and $\|\hat{h}-h_0\|^2=O(\mu^{-1}R_n)$, where 
    $R_n= \mu^{\min(\beta_h+1,2)}\|w_h\|^2+\mu^{\min(\beta_g'+1,2)}\|w_g'\|^2+\mu^{\min(\beta_g+1,2)}\|w_g\|^2 +\delta_n^2$.
\end{theorem}

Theorem~\ref{theorem:joint} is a new result for the simultaneous  nested NPIV. The rate $R_n$ contains the bias terms $\mu^{\min(\beta_h+1,2)}\|w_h\|^2+\mu^{\min(\beta_g'+1,2)}\|w_g'\|^2+\mu^{\min(\beta_g+1,2)}\|w_g\|^2$ and a variance term $\delta_n^2$

Theorem~\ref{theorem:joint} cleanly separates the ill posedness of $\hat{g}$ from $\hat{h}$, and in a different way than Theorem~\ref{theorem:L2}.  The ill posedness of $\hat{g}$ appears via the second and third bias terms, which are additively separable from the other terms in $R_n$. 

The rate for mean square error is slower than the rate for projected mean square error by a factor of $\mu^{-1}$. Rates may be optimized by choosing $\mu$ to balance the bias and variance. Recall that $\textsc{well}(\beta)=\frac{\min(\beta,1)}{\min(\beta,1)+1}\in(0,1/2]$ measures well posedness. 


\begin{corollary}[Better rates due to relative well posedness]~\label{cor:L2_compound2}
   Suppose the conditions of Theorem~\ref{theorem:joint} hold. Take $\mu=\delta_n^{\frac{2}{\min(\underline{\beta},1)+1}}$ where $\underline{\beta}=\min(\beta_h,\beta_g',\beta_g)$. Then with probability $1-\zeta$, $\|T_h(\hat{h}-h_0)\|^2=O(\delta_n^2)$ and $\|\hat{h}-h_0\|^2=O\left\{\delta_n^{2\textsc{well}(\underline{\beta})}\right\}$.
\end{corollary}

\addtocounter{example}{-2}

\begin{example}[Gaussian kernel]
    By Corollary~\ref{cor:L2_compound2}, 
    $\|\hat{h}-h_0\|^2=O_p \left[\left\{\frac{\ln(n)}{n}\right\}^{\textsc{well}(\underline{\beta})}\right]$.
\end{example}

\begin{example}[Neural network]
     By Corollary~\ref{cor:L2_compound2}, 
    $\|\hat{h}-h_0\|^2=O_p \left[\left\{\frac{LW \ln(W)\ln(n)}{n}\right\}^{\textsc{well}(\underline{\beta})}\right]$.
\end{example}

Ill posedness compounds in a simple way for Algorithm~\ref{algo:simultaneous}: our mean square rate is a ``base rate'' $\delta_n^2$ slowed by the minimal well posedness $\textsc{well}(\beta_h \wedge \beta_g'\wedge \beta_g)$. At best, $\beta_h,\beta'_g,\beta_g \geq 1$ and the mean square rate is $O(\delta_n)$. Meanwhile, the projected mean square rate is $O(\delta_n^2)$.

Comparing Corollary~\ref{cor:L2_compound2} with Corollary~\ref{cor:L2_compound}, we see that Assumption~\ref{assumption:posedness} improves rates by restricting how ill posedness may compound. Whereas Corollary~\ref{cor:L2_compound} gives $\|\hat{h}-h_0\|^2=O \left\{\delta_n^{2\textsc{well}(\beta_h)\textsc{well}(\beta'_g)}\right\}$, Corollary~\ref{cor:L2_compound2} gives $\|\hat{h}-h_0\|^2=O\left\{\delta_n^{2\textsc{well}(\beta_h \wedge \beta_g'\wedge \beta_g)}\right\}$. The base rate $\delta_n^2$ is throttled less in the latter. By placing an assumption on the relative measure of well posedness, we ensure that the ill posedness compounds in a much more benign way. Remarkably, it is the \textit{minimum} rather than the \textit{product}. Remark~\ref{remark:intuition} gives intuition.

Comparing Corollary~\ref{cor:L2_compound2} with previous series analysis, we see the main assumptions generally align, yet Corollary~\ref{cor:L2_compound2} strengthens asymptotic mean square consistency \cite[Lemma 3.1]{ai2007estimation} to a nonasymptotic mean square rate. It also sharpens the projected mean square rate from $o_p(n^{-1/2})$ \cite[Theorem 3.1]{ai2007estimation} to $O(\delta_n^2)$. Assumption~\ref{assumption:critical} is similar to series complexity, as measured by covering numbers \cite[Assumption 3.7]{ai2007estimation}. Assumption~\ref{assumption:source} relaxes previous source conditions from $\beta=2$ to $\beta<2$ \cite[Assumption 3.8.3]{ai2007estimation}. 
Assumption~\ref{assumption:posedness} is not necessary for the mean square rate in Theorem~\ref{theorem:L2}. We conjecture that it may verify previous high level conditions \cite[Assumption 4.1]{ai2007estimation}.
\section{Semiparametric inference}\label{sec:inference}

Theorem~\ref{theorem:joint} alleviates some of the compounding ill posedness of nested NPIV that appears in Theorem~\ref{theorem:L2}. In this section, we further alleviate how compounding ill posedness affects inference on causal parameters that are functionals of the nested NPIV. In particular, we characterize a multiple robustness to ill posedness: multiple inverse problems may be moderately ill posedness, as long as other inverse problems are mildly ill posed.

It seems that multiple robustness to ill posedness has not been characterized in previous work on generic functionals of machine learning nuisances \citep{zheng2011cross,chernozhukov2018original,chernozhukov2016locally}. 
 Our multiple robustness to ill posedness generalizes the double robustness to ill posedness known for linear functionals of NPIV \citep{chernozhukov2021simple} such as proxy treatment effects in cross sectional data \citep{kallus2021causal,ghassami2021minimax}. 

\textbf{Bilinear influence functions.} We study causal parameters with bilinear influence functions, e.g. mediated, time varying, and long term effects, since these motivate our interest in nested NPIV. This class of parameters includes causal scalars, e.g. the proxy direct effect (Example~\ref{ex:direct}), and causal functions, e.g. heterogeneous long term effects (Example~\ref{ex:het}). We develop both of these examples to make our general results concrete.

For this class of causal parameters, four nuisances appear in the influence function, which we denote by $(h_1,h_2,h_3,h_4)$. Below, we will see that $(h_1,h_2,h_3,h_4)$ are simple transformations of $(h_0,g_0)$ and their dual analogues. For simplicity, let the argument for $h_j$ be $B_j$, and let its conditional expectation operator be $T_j$.\footnote{If $\mathbb{E}\{h_0(B)|C\}=\mathbb{E}\{g_0(A)|C\}$ and $\mathbb{E}\{g_0(A)|C'\}=\mathbb{E}(Y|C')$, then $(B_1,B_2,B_3,B_4)=(B,A,C',C)$.} We study parameters whose influence functions take the form
$$
\psi(B_1,B_2,B_3,B_4)=h_1(B_1)+h_3(B_3)\{Y-h_2(B_2)\}+h_4(B_4)\{h_2(B_2)-h_1(B_1)\}-\theta_0.
$$

\addtocounter{example}{-7}

\begin{example}[Proxy mediation analysis]
    Recall the definitions of $(h_0,g_0)$ from Section~\ref{sec:examples}.
Now let $h_0'$ be a treatment confounding bridge that solves 
$\mathbb{E}\{h'(X,Z,D,M)|X,D=1,M,W\}=\mathbb{E}\left\{g_0'(X,Z,D=0)\frac{\mathbb{P}(D=0|X,M,W)}{\mathbb{P}(D=1|X,M,W)}|X,D=1,M,W\right\}$, 
where $g_0'$ solves
$\mathbb{E}\{g'(X,Z,D)|X,D=0,W\}=\mathbb{E}\left\{\frac{1}{\mathbb{P}(D=0|X,W)}|X,D=0,W\right\}$. 
Then
$h_1(X,W)=h_0(X,D=0,W)$, 
    $h_2(X,M,W)=g_0(X,D=1,M,W)$, 
    $h_3(X,Z,D,M)=1_{D=1}h_0'(X,Z,D=1,M)$, and   
    $h_4(X,Z,D)=1_{D=0}g_0'(X,Z,D=0)$. 
\end{example}

\addtocounter{example}{2}

\begin{example}[Heterogeneous long term effects]
     Recall the definitions of $(h_0,g_0)$ from Section~\ref{sec:examples}.
Then 
    $h_1(X_1,X_2)=\ell_{\lambda}(X_1)h_0(X_1,X_2,D=d)$ and 
    $h_2(X_1,X_2,M)=\ell_{\lambda}(X_1)g_0(X_1,X_2,M,G=\OBS)$ localize the outcome bridge and outcome mechanism. Now define
    $h_3(X_1,X_2,M,G)=\frac{1_{G=\OBS}}{\mathbb{P}(G=\OBS|X_1,X_2,M)}\frac{\mathbb{P}(D=d|X_1,X_2,M,G=\EXP)\mathbb{P}(G=\EXP|X_1,X_2,M)}{\mathbb{P}(D=d|X_1,X_2,G=\EXP)\mathbb{P}(G=\EXP|X_1,X_2)}$ and  
    $h_4(X_1,X_2,D,G)=)\frac{1_{G=\EXP}1_{D=d}}{\mathbb{P}(D=d|X_1,X_2,G=\EXP)\mathbb{P}(G=\EXP|X_1,X_2)}$, which contain the treatment and selection mechanisms.
\end{example}

For such parameters, we estimate the causal parameter according to a standard procedure.

\begin{algorithm}[Causal parameter]\label{algo:dml}
    Split the sample into $\tr$ and $\te$ folds. Given nested NPIV estimators $(\hat{h}_1,\hat{h}_2,\hat{h}_3,\hat{h}_4)$ estimated from observations in $\tr$, calculate the empirical influence of observation $i\in \te$ as
     $$
    \hpsi_i=\hat{h}_1(B_{1i})+\hat{h}_3(B_{3i})\{Y_i-\hat{h}_2(B_{2i})\}+\hat{h}_4(B_{4i})\{\hat{h}_2(B_{2i})-\hat{h}_1(B_{1i})\}.
    $$
    This process generates a vector $\hpsi\in\R^{n/2}$. Reversing the roles of $\tr$ and $\te$, we generate another such vector. Slightly abusing notation, we concatenate the two to obtain  $\hpsi \in \R^{n}$. Estimate $\htheta=\textsc{mean}(\hpsi)$ and $\hsigma^2=\textsc{var}(\hpsi)$, and return the confidence interval $\textsc{CI}=\htheta \pm 1.96 \hsigma n^{-1/2}$.
\end{algorithm}

\textbf{Multiple robustness to ill posedness.} Algorithm~\ref{algo:dml} is well known in semiparametric theory. Our secondary contribution is to demonstrate that it has stronger properties than previously documented, which help to alleviate the compounding ill posedness of nested NPIV. Similar to many previous works, we leverage Neyman orthogonality.

\begin{assumption}[Neyman orthogonality]\label{assumption:orthogonal}
    (i) For all $\tilde{h}_1\in\mathcal{H}_1$, $\mathbb{E}[\tilde{h}_1(B_1)\{1-h_4(B_4)\}]=0$. 
     (ii) For all $\tilde{h}_2\in\mathcal{H}_2$, $\mathbb{E}[\tilde{h}_2(B_2)\{h_4(B_4)-h_3(B_3)\}]=0$.
     (iii) For all $\tilde{h}_3\in\mathcal{H}_3$, $\mathbb{E}[\tilde{h}_3(B_3)\{Y-h_2(B_2)\}]=0$.
      (iv) For all $\tilde{h}_4\in\mathcal{H}_4$, $\mathbb{E}[\tilde{h}_4(B_4)\{h_2(B_2)-h_1(B_1)\}]=0$.
\end{assumption}

Assumption~\ref{assumption:orthogonal} implies Neyman orthogonality in our bilinear setting. It is straightforward to verify Assumption~\ref{assumption:orthogonal} by the law of iterated expectations, for all of the motivating examples, e.g. mediated, time varying, and long term treatment effects with or without proxies. Finally, we place some weak regularity conditions.

\begin{assumption}[Regularity conditions]\label{assumption:regular}
   (i) The residual variances are bounded: $\mathbb{E}[\{Y-h_2(B_2)\}^2|B_2]\leq \bar{\sigma}_y^2$ and $\mathbb{E}[\{h_2(B_2)-h_1(B_1)\}^2|B_1]\leq \bar{\sigma}_2^2$. (ii) The generalized balancing weights are bounded: $\|h_3\|_{\infty}\leq \bar{h}_3$ and $\|h_4\|_{\infty}\leq \bar{h}_4$. (iii) The generalized balancing weights are censored: $\|\hat{h}_3\|_{\infty}\leq \bar{h}'_3$ and $\|\hat{h}_4\|_{\infty}\leq \bar{h}'_4$. Here, $(\bar{\sigma}_y^2,\bar{\sigma}_2^2,\bar{h}_3,\bar{h}_4,\bar{h}'_3,\bar{h}'_4)$ may be diverging sequences.
\end{assumption}

Assumption~\ref{assumption:regular}(i) is weak and standard. Assumption~\ref{assumption:regular}(ii) encodes a  familiar condition necessary for regular estimation. In Example~\ref{ex:het}, it imposes that the treatment and selection propensity scores are bounded away from zero and one. Assumption~\ref{assumption:regular}(iii) can be achieved by censoring extreme values in way that is asymptotically negligible. For causal functions, $(\bar{\sigma}_y^2,\bar{\sigma}_2^2,\bar{h}_3,\bar{h}_4,\bar{h}'_3,\bar{h}'_4)$ diverge as the bandwidth $\lambda$ vanishes, which our inference result tolerates.

To state our final result, let $\sigma^2$, $\kappa^3$, and $\chi^4$ be the second, third, and fourth moments of $\psi(B_1,B_2,B_3,B_4)$ defined above. Like Assumption~\ref{assumption:regular}, these quantities are fixed for causal scalars and diverging for causal functions,  which our inference result tolerates.

\begin{theorem}[Multiple robustness to ill posedness]\label{theorem:ci}
    Suppose Assumptions~\ref{assumption:orthogonal} and~\ref{assumption:regular} hold, as well as the moment regularity $
\{\left(\kappa/\sigma\right)^3+\chi^2\}n^{-1/2}\rightarrow0.
$ Suppose the following quantities are $o_p(1)$: the individual rates {\small $\left(1+\bar{h}_4/\sigma+\bar{h}_4'/\sigma\right)\|\hat{h}_1-h_1\|$,
     $\left(\bar{h}_3/\sigma+\bar{h}_3'+\bar{h}_4/\sigma+\bar{h}_4'\right)\|\hat{h}_2-h_2\|$,
     $(\bar{h}_4'+\bar{\sigma}_y)\|\hat{h}_3-h_3\|$,
     $\bar{\sigma}_2\|\hat{h}_4-h_4\|$}; and the product rates
\begin{enumerate}
    \item {\small $n^{1/2}\sigma^{-1}\{ \|T_1(\hat{h}_1-h_1)\|\|\hat{h}_4-h_4\| \wedge \|\hat{h}_1-h_1\|\|T_4(\hat{h}_4-h_4)\|\} 
    $},
     \item {\small $n^{1/2}\sigma^{-1}\{\|T_2(\hat{h}_2-h_2)\|\|\hat{h}_3-h_3\| \wedge \|\hat{h}_2-h_2\|\|T_3(\hat{h}_3-h_3)\|\} 
     $},
      \item {\small $n^{1/2}\sigma^{-1}\{ \|T_2(\hat{h}_2-h_2)\|\|\hat{h}_4-h_4\| \wedge \|\hat{h}_2-h_2\|\|T_4(\hat{h}_4-h_4)\|\}
      $}.
\end{enumerate}
Then $
\hat{\theta}\overset{p}{\rightarrow}\theta_0,$ 
$\frac{\sqrt{n}}{\sigma}(\hat{\theta}-\theta_0)\overset{d}{\rightarrow}\mathcal{N}(0,1),$ and 
$
\mathbb{P} \{\theta_0 \in (\hat{\theta}\pm 1.96\hat{\sigma} n^{-1/2})\} \rightarrow 0.95.
$
\end{theorem}

Theorem~\ref{theorem:ci} summarizes nonasymptotic Gaussian approximation and variance estimation for mediated, time varying, and long term effects with generic machine learning in Appendix~\ref{sec:ci}. Whereas previous nonasymptotic results are limited to cross sectional parameters and NPIV \citep{chernozhukov2021simple}, we study short panel parameters and nested NPIV. Theorem~\ref{theorem:ci} holds with or without proxy variables, and applies to nonparametric causal functions. 
Compared to previous asymptotic results, we handle new cases, e.g. the proxy direct effect (Example~\ref{ex:direct}), heterogeneous
long term  treatment effects (Example~\ref{ex:het}), and many more; see Section~\ref{sec:examples}. 

See Appendix~\ref{sec:local} for a more explicit version of Theorem~\ref{theorem:ci} for causal functions, including conditions on how the bandwidth $\lambda$ vanishes. When $\lambda \asymp n^{-1/5}$, $\hat{\theta}(x_1)$ converges to $\theta_0(x_1)$ at the rate $n^{-2/5}$, which is familiar in Nadaraya-Watson estimation.

Our results confer a multiple robustness to ill posedness for functionals of nested NPIV.
In particular, Theorem~\ref{theorem:ci}'s product rate conditions partly ameliorate how the ill posedness of the nested NPIV $h_0$ affects inference for the causal parameter $\theta_0$. Each rate condition multiplies a projected mean square rate with a mean square rate. The former sidesteps the ill posedness of $\hat{h}$, extending classic results for functionals of NPIV \citep{blundell2007semi}. 

Comparing Theorem~\ref{theorem:ci} with previous series analysis, the estimating equation is different, the complexity assumption is more lax, and the rate conditions are more strict. Algorithm~\ref{algo:dml} is constructed from a Neyman orthogonal moment, rather than a plug-in moment \cite[eq. 7]{ai2007estimation}. Theorem~\ref{theorem:ci} does not impose the restrictive Donsker condition \cite[Assumption 4.5]{ai2007estimation}. It requires product conditions involving not only projected mean square rates \cite[Assumption 4.2]{ai2007estimation} but also mean square rates.

\textbf{Ill posedness saved is flexibility earned.} 
Corollary~\ref{cor:L2_compound2} improves rates of convergence for nested NPIV estimators. Theorem~\ref{theorem:ci} refines the nested NPIV rate conditions for inference on the causal parameter. We combine these results to concretely illustrate how some inverse problems may be moderately ill posed, as long as others are sufficiently well posed. By sharpening the dependence on ill posedness, we allow for estimation over more complex function spaces.

Given the compounding ill posedness, it is not obvious that machine learning estimation of nested NPIV could culminate in $n^{-1/2}$ inference for the causal parameter. Indeed, previous machine learning inference results for causal scalars may be pessimistic.

\begin{proposition}[A negative result]\label{prop:negative}
   Some previous rate conditions for causal inference are very general, yet in terms of mean square error only, e.g. \cite{zheng2011cross,chernozhukov2018original} and various works that build on them. The mean square rates of Corollary~\ref{cor:L2_compound2} 
    fail to satisfy such rate conditions.
\end{proposition}

By contrast, we provide a positive end-to-end result, which is useful for short panel data models with proxies and compounding ill posedness. In full generality, $(h_1,h_2,h_3,h_4)$ may each be a nested NPIV.%
\footnote{For several examples, $(h_1,h_3)$ are nested NPIVs while $(h_2,h_4)$ are NPIVs.  Proposition~\ref{prop:robust2} remains the same, replacing $(\vec{\beta}_2,\vec{\beta}_4)$ with $(\beta_2,\beta_4)$ and using fewer source conditions.}
 To lighten notation, write $\vec{\beta}_j=(\beta_{jg},\beta'_{jg},\beta_{jh})$ with smallest entry $\underline{\beta}_j$.

\begin{proposition}[A positive result]\label{prop:robust2}
    Suppose the conditions of Corollary~\ref{cor:L2_compound2} hold for $(\hat{h}_1,\hat{h}_2,\hat{h}_3,\hat{h}_4)$. Write the largest critical radius as $\bar{\delta}_n=\tilde{O}(n^{-\alpha})$, and the source conditions as $(\vec{\beta}_1,\vec{\beta}_2,\vec{\beta}_3,\vec{\beta}_4)$. Set regularizations as in Corollary~\ref{cor:L2_compound2}. Suppose $\sigma\asymp n^{\gamma}$. Then product rate conditions of Theorem~\ref{theorem:ci} are satisfied if
    (i)
    $ \gamma + \alpha \{
\textsc{well}(\underline{\beta}_1) \vee \textsc{well}(\underline{\beta}_4) +1
  \}>1/2;
    $
    (ii)
    $ \gamma + \alpha\{
\textsc{well}(\underline{\beta}_2) \vee \textsc{well}(\underline{\beta}_3) +1
  \}>1/2;
    $
    (iii)
    $
    \gamma + \alpha \{
 \textsc{well}(\underline{\beta}_2) \vee \textsc{well}(\underline{\beta}_4) +1
 \} >1/2.$
\end{proposition}

We interpret these inequalities, and confirm that the set of values $(\alpha,\vec{\beta},\gamma)$ satisfying them is nonempty. 
Each right hand side is a constant, so each inequality is a joint requirement on the critical radius via $\alpha$, the source conditions via $\vec{\beta}$, and the asymptotic variance via $\gamma$.

The quantity $\alpha$ measures the complexity of the function classes. At best, $\alpha=1/2$ for parametric function classes. For nonparametric classes, $\alpha<1/2$.

Each quantity $\textsc{well}(\beta)$ measures the well posedness of an inverse problem. At best, $\beta\geq 1$ and $\textsc{well}(\beta)=1/2$. For severely ill posed inverse problems, $\beta\rightarrow 0$ and $\textsc{well}(\beta)\rightarrow 0$. Our conditions allow the well posedness of some inverse problem to compensate the ill posedness of others. 
Hence Proposition~\ref{prop:robust2} clarifies multiple robustness to ill posedness.
    
    For causal scalars, under mild regularity conditions, $\sigma\asymp 1$ and hence $\gamma=0$. For causal functions, under the regularity conditions given in Appendix~\ref{sec:local}, $\sigma_{\lambda}(x_1)\asymp \lambda^{-1/2}$ and hence for bandwidth $\lambda=n^{-1/5}$ we have $\gamma=1/10$. Echoing the work of \cite{kennedy2023towards} on heterogeneous treatment effects, and many references therein, we derive product rate conditions that are \textit{weaker} for causal functions than for causal scalars. Unlike \cite{kennedy2023towards}, we study causal functions in short panel data e.g. heterogeneous long term  treatment effects.

Finally, we demonstrate that the set of $(\alpha,\vec{\beta},\gamma)$ is non-empty. For simplicity, saturate each source condition with $\beta\geq 1$ and suppose we are studying causal scalars with $\gamma=0$. Then the single sufficient condition is $\alpha>1/3$ for Proposition~\ref{prop:robust2}, which allows many nonparametric function classes. Without Assumption~\ref{assumption:posedness}, the condition is $\alpha>2/5$, i.e. less complexity; see Appendix~\ref{sec:analytic}.

Proposition~\ref{prop:robust2} is a consequence of our two technical innovations. It uses rates for nested NPIV in which the ill posedness does not compound too much (our first technique). It also uses product rate conditions involving projected mean square rates and mean square rates (our second technique). It allows non-Donsker spaces and partly ameliorates ill posedness.

\section{Simulated and real data analysis}\label{sec:experiments}

Our nested NPIV method may improve mean square error and coverage in nonlinear simulations, compared to some previous series methods. Using real data, we demonstrate that our method improves precision for the proxy direct treatment effect of the US Job Corps, compared to a previous parametric method. We also uncover heterogeneity in the long term effects of Project STAR: students with the lowest prior ability benefit the most from small class sizes. 

For brevity, we focus on Algorithms~\ref{algo:simultaneous} and~\ref{algo:dml}. Results for Algorithm~\ref{algo:sequential} are similar.

\textbf{Estimation with general function spaces may reduce mean square error.} Algorithm~\ref{algo:simultaneous} outperforms some benchmarks in mean square error across several nonlinear data generating processes (DGPs) for nested NPIV. For simplicity, we focus on the equal dimensional setting:   $dim(A)=dim(B)=dim(C)=dim(C')=10$.
We fix the initial NPIV $g_0$ as a cubic function, and let $h_0$ be one of four different nonlinear functions, inspired by \cite{dikkala2020minimax}.
Each sample has $n=2000$ observations.

For each of these four variations of the DGP, we implement two versions of our estimator: an RKHS or neural network estimator with $\mathcal{F}=\mathcal{F}'=\mathcal{G}=\mathcal{H}$. Appendices~\ref{sec:details} and~\ref{sec:tuning} give implementation details, including a closed form for the RKHS version and principled tuning. As benchmarks, we also implement nested 2SLS, nested series, and nested series with regularization. 

\begin{table}
    \centering
   \begin{tabular}{cccccc}
        & \multicolumn{3}{c}{Benchmarks} & \multicolumn{2}{c}{Proposals} \tabularnewline
        \cmidrule(lr){2-4}\cmidrule(lr){5-6}
        & 2SLS & series & reg. & RKHS & neural  \tabularnewline
        \hline
        linear & 0.006 & 1$\times 10^5$ & 0.458 & 0.001 & 0.007  \tabularnewline
        piecewise linear & 0.009 & 2$\times 10^5$  & 0.191 & 0.008 & 0.013  \tabularnewline
        sigmoid & 0.006 & 8$\times 10^4$  & 0.076 & 0.005 & 0.011  \tabularnewline
        exponential & 0.080 & 2$\times 10^3$  & 2$\times10^2$ & 0.030 & 0.020 \tabularnewline
        \hline
    \end{tabular}
    \caption{Nested NPIV mean square error simulations}    \label{tab:npiv}
\end{table}

\begin{table}
\begin{subtable}[t]{0.48\textwidth}
         \centering
        \resizebox{\textwidth}{!}{%
        \begin{tabular}{cccccc}
            & \multicolumn{3}{c}{Benchmarks}  & \multicolumn{2}{c}{Proposals}  \tabularnewline
            \cmidrule(lr){2-4}\cmidrule(lr){5-6}
             & 2SLS & series & reg. & RKHS  & neural \tabularnewline
            \hline
            bias & -0.00 & -2$\times 10^2$ & -0.01 & 0.02 & 0.02 \tabularnewline
            variance & 27.24 & 7$\times 10^{11}$ & 2$\times 10^{2}$ & 22.41 & 25.24 \tabularnewline
           coverage & 0.95 & 0.94 & 0.96 & 0.92 & 0.90 \tabularnewline
           length & 0.46 & 5$\times 10^4$ & 1.16 & 0.42 & 0.44 \tabularnewline
            \hline
        \end{tabular}
        }
    \caption{Linear DGP}
\end{subtable}\hspace{\fill} 
\begin{subtable}[t]{0.48\textwidth}
          \centering
        \resizebox{\textwidth}{!}{%
      \begin{tabular}{cccccc}
          & \multicolumn{3}{c}{Benchmarks}  & \multicolumn{2}{c}{Proposals}  \tabularnewline
            \cmidrule(lr){2-4}\cmidrule(lr){5-6}
             & 2SLS & series & reg. & RKHS  & neural \tabularnewline
            \hline
            bias & -0.02 & 8$\times 10^2$ & -0.01 & 0.01 & 0.00 \tabularnewline
            variance & 32.33 & 4$\times 10^{12}$ & 27.34 & 24.68 & 27.74 \tabularnewline
           coverage & 0.94 & 0.92 & 0.95 & 0.92 & 0.93 \tabularnewline
           length & 0.50 & 3$\times 10^5$ & 0.46 & 0.44 & 0.46 \tabularnewline
            \hline
        \end{tabular}
        }
    \caption{Piecewise linear DGP}
\end{subtable}

\medskip 
\begin{subtable}[t]{0.48\textwidth}
          \centering
        \resizebox{\textwidth}{!}{%
        \begin{tabular}{cccccc}
           & \multicolumn{3}{c}{Benchmarks}  & \multicolumn{2}{c}{Proposals}  \tabularnewline
            \cmidrule(lr){2-4}\cmidrule(lr){5-6}
             & 2SLS & series & reg. & RKHS  & neural \tabularnewline
            \hline
            bias & -0.02 & -2$\times 10^2$ & 0.00 & 0.02 & -0.00 \tabularnewline
            variance & 31.52 & 4$\times 10^{11}$ & 34.92 & 25.00 & 27.69 \tabularnewline
           coverage & 0.94 & 0.93 & 0.95 & 0.92 & 0.93 \tabularnewline
           length & 0.49 & 5$\times 10^4$ & 0.52 & 0.44 & 0.46 \tabularnewline
            \hline
        \end{tabular}
        }
    \caption{Sigmoid DGP}
\end{subtable}\hspace{\fill} 
\begin{subtable}[t]{0.48\textwidth}
          \centering
        \resizebox{\textwidth}{!}{%
       \begin{tabular}{cccccc}
            & \multicolumn{3}{c}{Benchmarks}  & \multicolumn{2}{c}{Proposals}  \tabularnewline
            \cmidrule(lr){2-4}\cmidrule(lr){5-6}
             & 2SLS & series & reg. & RKHS  & neural \tabularnewline
            \hline
            bias & 0.07 & 6$\times 10^4$ & -1$\times 10^{10}$ & 0.47 & -0.12 \tabularnewline
            variance & 4$\times 10^3$ & 5$\times 10^{15}$ & 7$\times 10^{27}$ & 1$\times 10^4$ & 81.75 \tabularnewline
           coverage & 0.97 & 0.98 & 0.87 & 0.97 & 0.94 \tabularnewline
           length & 5.72 & 1$\times 10^7$ & 1$\times 10^{13}$ & 6.36 & 0.78 \tabularnewline
            \hline
\end{tabular}

        }
    \caption{Exponential DGP}
\end{subtable}
\caption{Proxy mediation analysis coverage simulations
}\label{tab:cov}
\end{table}

Table~\ref{tab:npiv} summarizes results of nested NPIV simulations. Each row corresponds to a different nonlinear function $h_0$. Each column corresponds to a different estimator. We report the empirical mean square error, averaged across 500 samples. 
Series estimators perform relatively poorly across these DGPs. 
Our proposed RKHS estimator outperforms every benchmark. In the particularly challenging exponential DGP, our proposed neural network estimator performs best.

\textbf{Inference with general function spaces may improve coverage and interval length.} Algorithms~\ref{algo:simultaneous} and~\ref{algo:dml} outperform some benchmarks in confidence interval coverage and length across several nonlinear DGPs for the proxy direct effect, i.e. Example~\ref{ex:direct}. We modify the simulation design of \cite{dukes2023proximal}, introducing the same four nonlinearities studied in Table~\ref{tab:npiv}. All variables are scalars except $X\in\mathbb{R}^2$; see Appendix~\ref{sec:details} for details. As before, each sample has $n=2000$ observations. We implement two versions of our estimator as well as three benchmarks.

Table~\ref{tab:cov} summarizes results of coverage experiments. Each table corresponds to a different nonlinear function $h_0$. Each column corresponds to a different estimator. Each row corresponds to a different performance metric,
averaged across 500 samples.

The various estimators often obtain nominal coverage, but starkly differ in their bias, variance, and interval length. The unregularized series estimator performs relatively poorly across DGPs. In the linear, piecewise linear, and sigmoid DGPs, our proposals consistently have shorter confidence intervals than the benchmarks. In the particularly challenging exponential DGP, only our proposed neural network procedure performs well.

In summary, Algorithms~\ref{algo:simultaneous} and~\ref{algo:dml} work well across nonlinear DGPs and across machine learning function spaces. They repeatedly outperform 2SLS and series benchmarks in some nonlinear, heterogeneous causal models using short panel data and proxy variables.

\textbf{Proxy mediation analysis of US Job Corps.} We turn to a motivating real world application: how to flexibly measure direct treatment effects while using proxies for unobserved confounding (Example~\ref{ex:direct}). We replicate an influential program evaluation \citep{dukes2023proximal} that studies the direct effect of job training on arrests later in life, i.e. the effect that is not through the mechanism of employment. The answer sheds light on the development of worker skills. We extend previous parametric estimation to semiparametric estimation. Our flexible method documents a significant, negative direct effect. These empirical results represent a realistic use case: causal estimation in short panel data models with proxy variables.

\cite{dukes2023proximal} propose a parametric estimator for the proxy direct treatment effect. Continuing an extensive literature on the non-employment effects of the US Job Corps, the treatment $D$ is job training in the year following randomization, the mechanism $M$ is employment two years after randomization, and the outcome $Y$ is arrests four years after randomization. The auxiliary variables $(Z,W)$ are assumed to be relevant to unobserved motivation $U$, and also to satisfy exclusion restrictions: time spent with the Job Corps recruiter $Z$ does not directly cause employment or arrests, and pre-training expectations $W$ are not directly caused by training or employment. The final estimate is a scalar.

\begin{figure}
    \centering
    \includegraphics[width=0.5\linewidth]{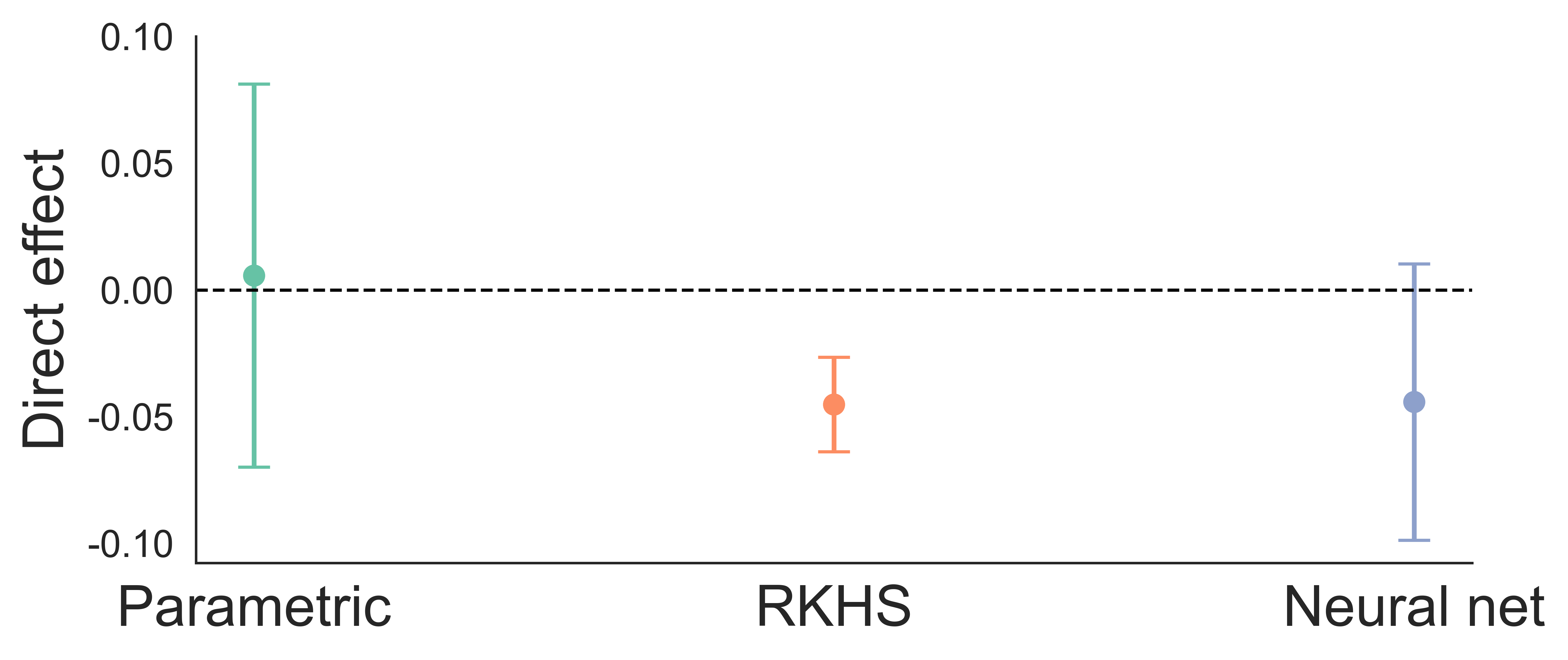}
    \caption{Direct effect of job training on arrests, using proxies for unobserved motivation}\label{fig:job} 
\end{figure}

Figure~\ref{fig:job} compares the previous parametric approach with our semiparametric approach. The previous parametric approach found statistically insignificant effects, possibly due to approximation error of the parametric models for the outcome confounding bridge and treatment confounding bridge. By allowing for flexible nonparametric estimation of the confounding bridges as nested NPIVs, our approach incurs a smaller approximation error. In this setting, it appears to improve statistical precision, suggesting a negative and possibly significant direct effect of job training on arrests. This empirical result is possible due to our new theoretical results for nested NPIV and its functionals.

\textbf{Heterogeneous long term effects of Project STAR.} Finally, we turn to another motivating real world application: how to flexibly measure long term treatment effects by combining short term experimental data with long term observational data (Example~\ref{ex:het}). We replicate an influential program evaluation \citep{athey2020combining} that combines Project STAR experimental data with New York City (NYC) observational data to study the average long term  treatment effect of kindergarten class size on test scores later in life. We ask an additional question: is there meaningful heterogeneity in those long term effects? Our empirical results represent a realistic use case of our proposal. We document substantial heterogeneity: students who have the lowest prior ability benefit the most from small class sizes.

\cite{athey2020combining} propose a parametric estimator for the average long term  treatment effect and validate their results through an intuitive exercise, which we extend. Though Project STAR data include kindergarten class size $D$, elementary school test scores $M$, and middle school test scores $Y$, we suppose that the researcher sees $(D,M)$ but not $Y$. The researcher combines the short term experimental $(D,M)$ from Project STAR with the long term observational $(M,Y)$ from NYC to estimate the average long term  treatment effect. These estimates may be validated by comparing them with the ``oracle'' average long term  treatment effect that an ``oracle'' who sees $(D,Y)$ in the Project STAR data would obtain.

\begin{figure}
   \captionsetup[subfigure]{justification=Centering}
\begin{subfigure}[t]{0.48\textwidth}
         \centering
        \resizebox{\textwidth}{!}{%
       \includegraphics[width=\textwidth]{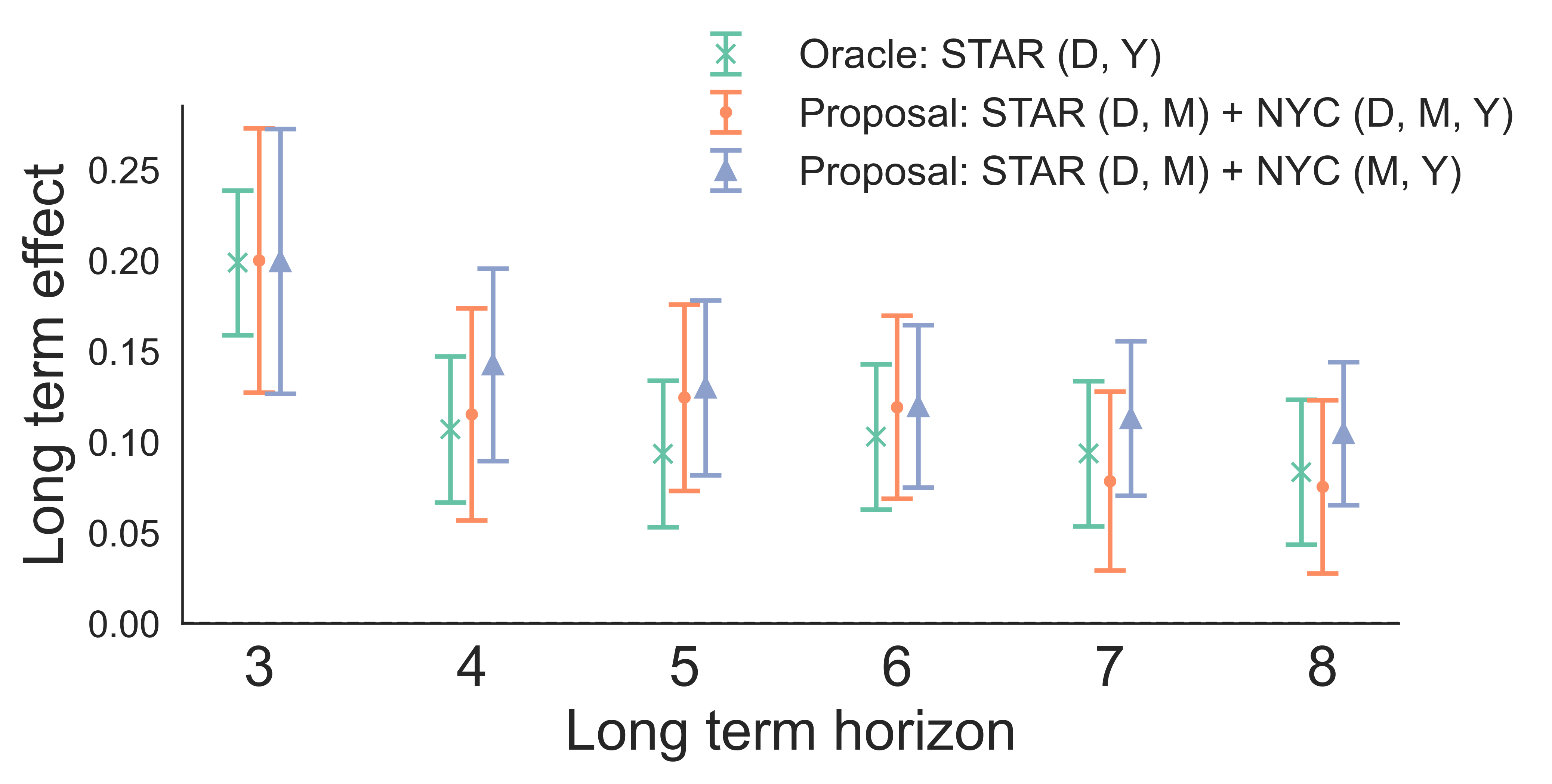}
        }
    \caption{RKHS}
\end{subfigure}\hspace{\fill} 
\begin{subfigure}[t]{0.48\textwidth}
          \centering
        \resizebox{\textwidth}{!}{%
      \includegraphics[width=\textwidth]{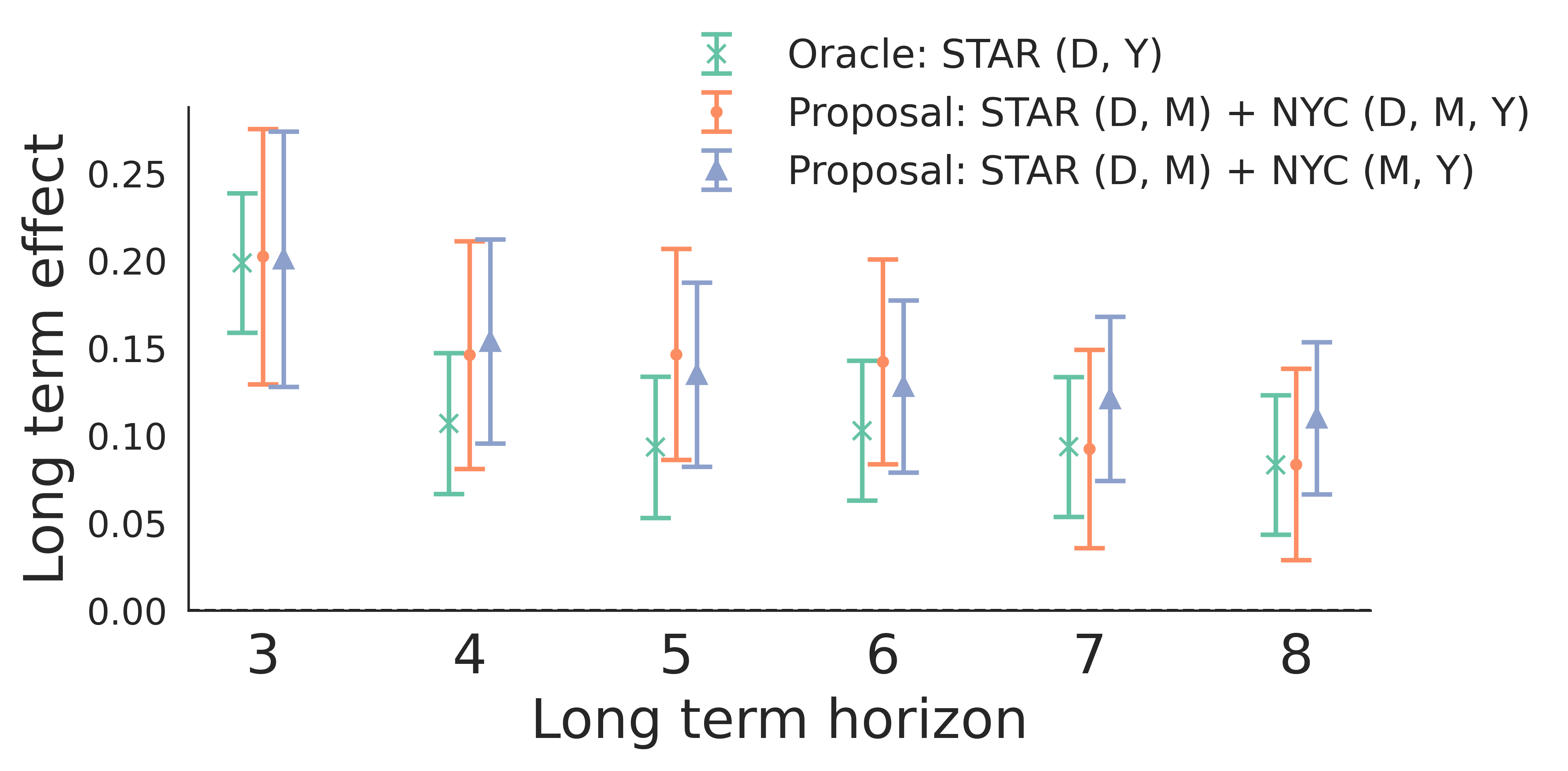}
        }
    \caption{Neural network}
\end{subfigure}
\caption{Average long term  treatment effect over different horizons}\label{fig:star_ate}
\end{figure}

A similar exercise may be conducted in a closely related variation of the problem. In this variation, the researcher sees $(D,M)$ but not $Y$ in Project STAR, and $(D,M,Y)$ in NYC. Both variations of the problem \citep{athey2019surrogate,athey2020combining} belong to the class of parameters we study, as well as their generalizations to causal functions.

Figure~\ref{fig:star_ate} demonstrates that a machine learning approach to average long term  treatment effect estimation performs well in the validation exercise. Following \cite{athey2020combining}, we fix $M$ as third grade test scores and take $Y$ to be third, fourth, fifth, sixth, seventh, or eighth grade test scores. Across choices of $Y$, i.e. across horizons of extrapolation, the average long term  treatment effect estimates recover the oracle estimates.

\begin{figure}
   \captionsetup[subfigure]{justification=Centering}
\begin{subfigure}[t]{0.48\textwidth}
         \centering
        \resizebox{\textwidth}{!}{%
       \includegraphics[width=\textwidth]{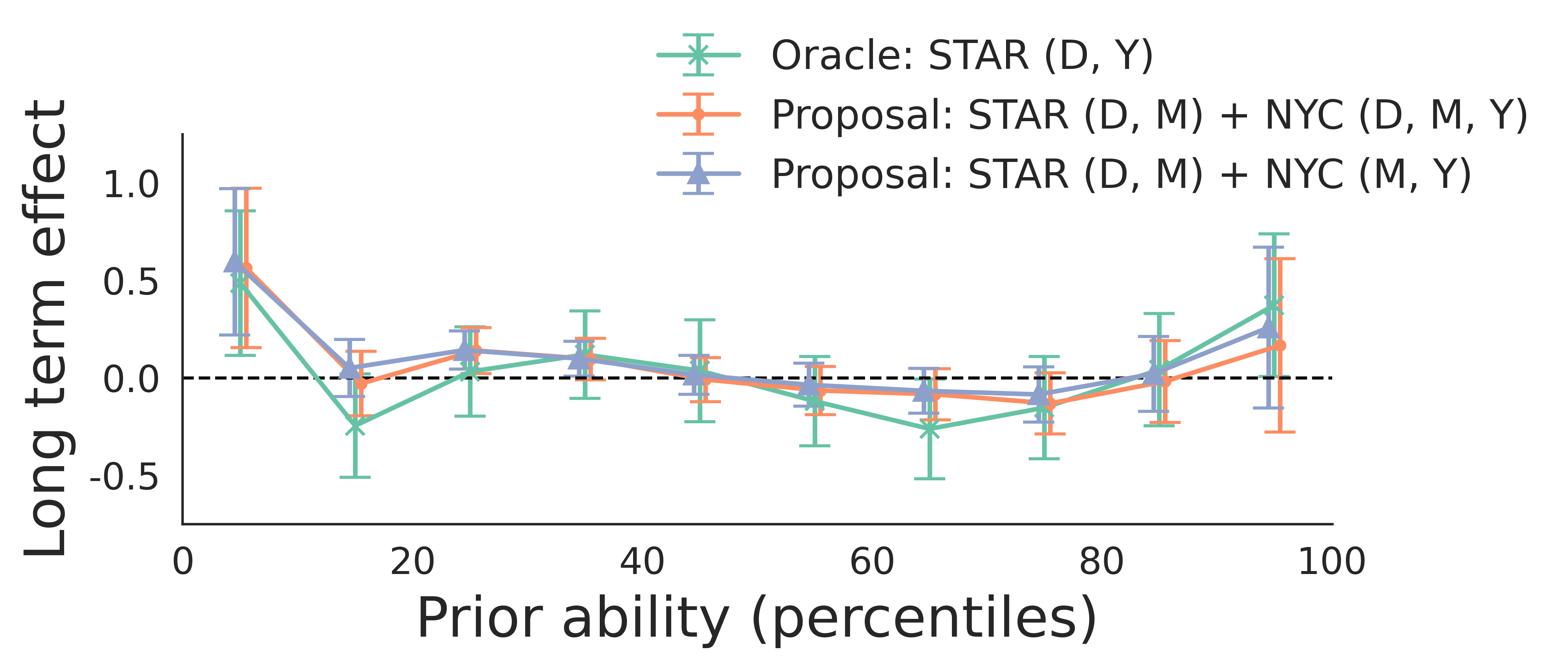}
        }
    \caption{RKHS}
\end{subfigure}\hspace{\fill} 
\begin{subfigure}[t]{0.48\textwidth}
          \centering
        \resizebox{\textwidth}{!}{%
      \includegraphics[width=\textwidth]{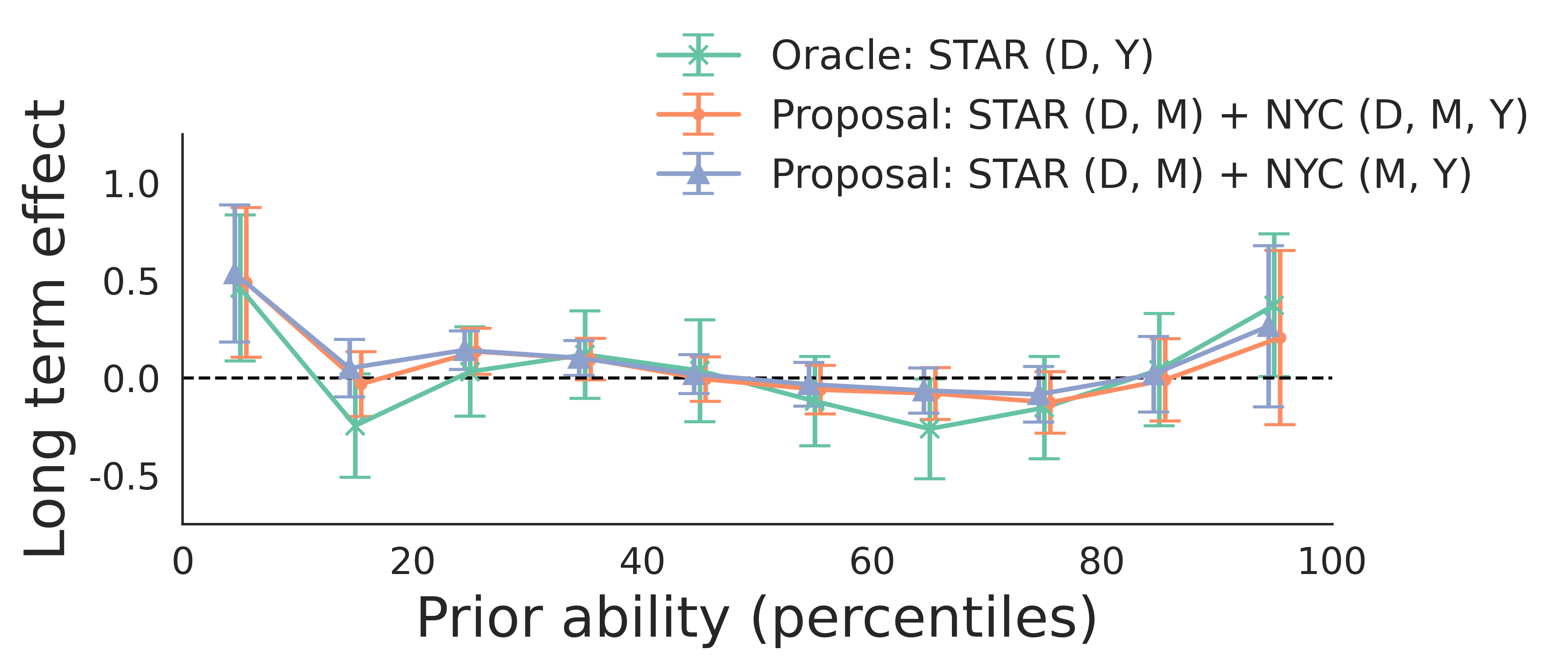}
        }
    \caption{Neural network}
\end{subfigure}
\caption{Heterogeneous long term treatment effects with respect to prior ability}\label{fig:star_cate}
\end{figure}

Figure~\ref{fig:star_cate} goes deeper, from average effects to heterogeneous effects with respect to student aptitude. In particular, we examine heterogeneity with respect to prior ability $X_1$, measured in percentiles before the intervention. 
For simplicity, we continue to fix $M$ as third grade test scores and now fix $Y$ to be seventh grade test scores. 

The students with lowest prior ability benefit the most from enrollment in a small kindergarten class. The results may be statistically significant, with pointwise confidence intervals that exclude zero. 
This empirical insight appears to be new, and it is possible due to our new theoretical results for causal functions such as heterogeneous long term  treatment effects. Appendix~\ref{sec:details} shows similar results across different horizons of extrapolation. 
\section{Discussion}\label{sec:conclusion}

A growing literature identifies parameters in nonlinear, heterogeneous causal models using short panel data and proxy variables. 
These identifications motivate us to study nested nonparametric instrumental variable regression (nested NPIV), which also arises in economic models of simultaneous equations and price systems. 
Our method allows researchers to conduct inference on the new causal models with machine learning, tolerating moderate ill posedness among some inverse problems. 
We provide explicit mean square convergence rates for nested NPIV and introduce two techniques: relative well posedness, and multiple robustness to ill posedness. 
Our new estimators detect direct effects of the US Job Corps with proxies for motivation, and long term effects Project STAR that are heterogeneous by prior ability.

\newpage

\appendix

\section{Additional condition: Relative alignment}\label{sec:extra}

For faster rates, our key assumptions are relative well posedness (Assumption~\ref{assumption:posedness}) and relative completeness (Assumption~\ref{assumption:completeness}). We highlight a consequence of the latter, which allows us to articulate an additional regularity condition that we call relative alignment. Finally, we relate this regularity condition to existing regularity conditions in the NPIV literature.

\subsection{A consequence of relative completeness}

Recall that relative completeness means that $S(g)=0$ implies $T_g(g)=0$.

\begin{lemma}\label{lemma:consequence}
   Under Assumption~\ref{assumption:completeness}, there exists a bounded linear operator $M$ such that $T_g=MS$. 
\end{lemma}

By Lemma~\ref{lemma:consequence}, the nested conditional expectation operator, with instrument $C$, can be expressed as a some operation $M$ applied to the initial conditional expectation operator, with instrument $C'$. In particular, the operation $M$ is  bounded. The intuition is as before:  Assumption~\ref{assumption:completeness} ensures that solving the nested NPIV operator equation does not require information that cannot be learned from the initial NPIV operator equation.

\subsection{Relative alignment}

Having shown that the operation $M$ exists, we use it to articulate a technical regularity condition that rules out sign flipping, from the initial problem to the nested problem. We present a stronger sufficient condition, then a weaker sufficient condition.

\begin{assumption}[Strong relative alignment]\label{assumption:alignment}
Suppose that $T$ preserves alignment relative to $S$ in the following pointwise sense.
    \begin{enumerate}
        \item For all functions of the initial instrument $f,\tilde{f}\in\mathcal{F}'$, $\langle f,\tilde{f} \rangle \geq 0$ implies $\langle M f, M \tilde{f} \rangle\geq 0 $ 
        and $\langle M_hf,M_h \tilde{f} \rangle\geq 0 $, where $M_h=\mu^{1/2}(T_hT_h^*+\mu I)^{-1/2} M$.
        \item For all functions of the nested instrument $f,\tilde{f}\in\mathcal{F}$, 
        $\langle f,\tilde{f} \rangle \geq 0$ implies $\langle f, \{(1-c)I-Q\}\tilde{f} \rangle\geq 0$.  Here, $c$ is a scalar, $I$ is the identity operator, and $Q$ is an operator defined in Appendix~\ref{sec:bias}. Under Assumption~\ref{assumption:posedness}, we later show $(1-c)I-Q$ is positive semidefinite.
    \end{enumerate}
\end{assumption}

When $f=\tilde{f}$, each statement automatically holds. When $f\neq \tilde{f}$, Assumption~\ref{assumption:alignment} imposes that if the functions have a positive angle, then they also have a positive angle after certain transformations involving $S$ and $T$.

As shown in Appendix~\ref{sec:bias}, Assumption~\ref{assumption:alignment} is sufficient to control a few of the terms that appear in the bias of Algorithm~\ref{algo:simultaneous}. Simulations suggest that Assumption~\ref{assumption:alignment} is not necessary. 

In what follows, we relax the strong, pointwise alignment condition (Assumption~\ref{assumption:alignment}) to a weaker, average alignment condition (Assumption~\ref{assumption:weak-alignment}) that rules out sign flipping on average.
Let $\mathcal P'={t\mapsto f_t\in\mathcal F'}$ and
$\mathcal P ={t\mapsto f_t\in\mathcal F}$ be the collections of measurable, square integrable paths on $[0,1]$.

\begin{assumption}[Weak relative alignment]\label{assumption:weak-alignment}
Suppose that $T$ preserves alignment relative to $S$ in the following average sense.

\begin{enumerate}
\item 
For all initial instrument paths $\{f_{(\cdot)},\tilde f_{(\cdot)}\}\in\mathcal P'$ satisfying 
$\langle f_t,\tilde f_t\rangle\ge0$ for almost every  $t\in[0,1]$, we have that 
$\int_0^1\left\langle M f_t, M \tilde{f}_t\right\rangle \mathrm{d} t \geq 0$ and $\int_0^1\left\langle M_h f_t,  M_h \tilde{f}_t\right\rangle \mathrm{d} t \geq 0$, where $M_h$ is defined in Assumption~\ref{assumption:alignment}.
\item For all nested instrument paths $\{f_{(\cdot)},\tilde f_{(\cdot)}\}\in\mathcal P$ satisfying $\langle f_t,\tilde f_t\rangle\ge0$ for almost every  $t\in[0,1]$,  we have that 
$\int_0^1\left\langle f_t,\{(1-c) I-Q\} \tilde{f}_t\right\rangle \mathrm{d} t \geq 0$, where the scalar $c$ and operator $Q$ are as in
Assumption \ref{assumption:alignment}.
\end{enumerate}
\end{assumption}

As before, each statement automatically holds when $f=\tilde{f}$. When $f\neq \tilde{f}$, Assumption~\ref{assumption:weak-alignment} imposes that if the functions have a positive angle on a path, then they also have a positive angle after integration along the path. As shown in Appendix~\ref{sec:bias}, Assumption~\ref{assumption:weak-alignment} is also sufficient to control a few of the terms that appear in the bias of Algorithm~\ref{algo:simultaneous}.\footnote{This technical regularity condition guarantees that $\int_0^t\gamma(s)\partial s>0$ in Lemma~\ref{lemma:alignment}.}

\subsection{Comparison to previous conditions}

Assumption~\ref{assumption:weak-alignment} plays the same role as coercivity, strong curvature, and strong identification conditions in the functional analysis and NPIV literatures. Similar to those works, we use it to guarantee stability of the solution to the inverse problem.

As a first step towards the comparison, we rephrase Assumption~\ref{assumption:weak-alignment} as a kind of positivity on average. Note that, for every bounded operator $B:\mathcal{F} \rightarrow \mathcal{F} $,
$$
\int_0^1\left\langle f_t, B \tilde{f}_t\right\rangle \mathrm{d} t=\int\langle f, B \tilde{f}\rangle d \nu(f, \tilde{f})=\operatorname{trace}\left(B \Gamma \right),\quad \Gamma=\int \tilde{f} \otimes f d \nu(f, \tilde{f}).
$$ 
In the first equality, we introduce $\nu$ as the push-forward of Lebesgue measure on $(0,1)$ by the path $\nu(A):=\int_0^1 1_A\{x(t), y(t)\} \mathrm{d} t$. In the second equality, use outer product notation. Assumption~\ref{assumption:weak-alignment} requires that the eigenvalues are positive on average, rather than positive everywhere.

For further interpretation, simplify the averaging measure $\nu$ to be isotropic, so that $\operatorname{trace}(B \Gamma)=\kappa \operatorname{trace}(B)$ with $\kappa>0$. Then Assumption~\ref{assumption:weak-alignment} reduces to $\operatorname{trace}(B)\geq 0$. In this setting, any coercivity condition that is classically used to guarantee uniqueness of weak solutions
immediately implies weak relative alignment. See, e.g., the Minty-Browder, Babuška-Lax-Milgram, and Lions-Lax-Milgram theorems \cite[Theorem 13.23]{kress1989linear}.

By similar logic, Assumption~\ref{assumption:weak-alignment} resembles conditions in the NPIV literature. See, e.g., the strong curvature requirement on the weak metric \cite[Assumption 3.8(iii)]{ai2007estimation},  and the functional strong identification condition \cite[Definition 1]{bennett2023inference}. Each condition guarantees solvability and stability of the associated inverse problem by imposing some form of positivity.

From a researcher's standpoint, if Assumptions~\ref{assumption:posedness},~\ref{assumption:completeness}, or~\ref{assumption:weak-alignment} is implausible in a particular setting, then we recommend using the rates of Theorem~\ref{theorem:L2} rather than Theorem~\ref{theorem:joint}. Theorem~\ref{theorem:L2} does not require these additional regularity conditions.
\section{Projected rates without source conditions}\label{sec:sourceless}

To clarify the different roles played by Assumptions~\ref{assumption:critical} and~\ref{assumption:source}, we derive projected mean square rates using only Assumption~\ref{assumption:critical}. Assumption~\ref{assumption:source} is essentially a technique to translate projected mean square rates into mean square rates.

\subsection{Beyond ridge regularization}

We study a sequential estimator with regularization besides ridge.

\begin{algorithm}[Nested NPIV with generic regularization]\label{algo:npiv_general}
    Given observations $(A_i,B_i,C_i)$, an initial estimator $\hat{g}$ that may be estimated on the same data, and hyperparameter values $(\mu,\mu'')>0$,
    $$
   \hat{h}=\argmin_{h\in\mathcal{H}}\left[\sup_{f\in\mathcal{F}}\left\{ 2\cdot \textsc{loss}(f,\hat{g},h)-\textsc{penalty} (f,\mu'')\right\}+\textsc{penalty}(h,\mu) \right]$$
    where $\textsc{penalty} (f,\mu'')=\mathbb{E}_n\{f(C)^2\}+\mu''\cdot \|f\|^2_{\mathcal{F}}$ and $\textsc{penalty}(h,\mu)=\mu\cdot \|h\|^2_{\mathcal{H}} $.
\end{algorithm}

Taking $\mu''=0$ and $\|h\|^2_{\mathcal{H}}=\mathbb{E}_n\{h(B)^2\}$ reduces Algorithm~\ref{algo:npiv_general} to Algorithm~\ref{algo:sequential}. More generally,  Algorithm~\ref{algo:npiv_general} allows for complex regularization, e.g. $\ell_1$ norm regularization in sparse linear function spaces, and reproducing kernel Hilbert space (RKHS) norm regularization in RKHSs. Under Assumption~\ref{assumption:critical}, we will prove projected mean square error rates. 
Our analysis of Algorithm~\ref{algo:npiv_general} avoids Assumption~\ref{assumption:source} and accommodates more regularization types. 

\subsection{Projected mean square rate}

\begin{theorem}[Bound for Algorithm~\ref{algo:npiv_general}]\label{theorem:general}
    Suppose Assumption~\ref{assumption:critical} holds for $\mathcal{F}$, $\mathcal{G}$, and $\mathcal{H}\times \mathcal{F}$; and Assumption~\ref{assumption:closed} holds. 
    Further assume $h_0\in \mathcal{H}$ and $\|T(h-h_0)\|^2_{\mathcal{F}}\leq \textsc{lip} \|h-h_0\|^2_{\mathcal{H}}$ for some Lipschitz constant $\textsc{lip}<\infty$.
    Then with probability $1-\zeta$, when $\mu\geq 2\mu'' \cdot  \textsc{lip}$ and $\delta_n=\Omega[\{\log\log(n)+\log(1/\zeta)\}^{1/2}n^{-1/2}]$, we have $\|T(\hat{h}-h_0)\|^2_2=O(R_n)$, where $R_n=\mu\|h_0\|^2_{\mathcal{H}}+\delta_n^2+\|\hat{g}-g_0\|^2_2$.
\end{theorem}

Theorem~\ref{theorem:general} is a projected mean square rate for nested NPIV. It does not require a source condition, and it allows for generic regularization. It generalizes known results for NPIV when $\hat{g}(A)=g_0(A)=Y$ \citep{dikkala2020minimax}. The rate $R_n$ has three terms: bias $\mu\|h_0\|^2_{\mathcal{H}}$, variance $\delta_n^2$, and initial estimation error $\|\hat{g}-g_0\|_2^2$. 
Only the ill posedness of $\hat{g}$ appears in Theorem~\ref{theorem:general}, because the definition of projected mean square error sidesteps the ill posedness of $\hat{h}$.

\begin{corollary}[Rate for Algorithm~\ref{algo:npiv_general}]\label{cor:general}
     Suppose the conditions of Theorem~\ref{theorem:general} hold. Set $
    \mu=O(\delta_n^2)$. Then with probability $1-\zeta$, $\|T(\hat{h}-h_0)\|^2_2=O\left(\delta_n^2+\|\hat{g}-g_0\|^2_2\right)$.
\end{corollary}

Future work may strengthen Theorem~\ref{theorem:general} and Corollary~\ref{cor:general} to mean square rates by placing further approximation assumptions, e.g. a restricted eigenvalue condition \citep{gautier2011high,gautier2018high}.

\section{Proof of Theorem~\ref{theorem:L2}}\label{sec:rate1}

Let $S:\mathcal{G}\rightarrow \mathbb{L}_2 $ be the operator $g\mapsto \E\{g(A)|C' = \cdot\}$. Let $g_0$ be the minimal $\mathbb{L}_2$ norm solution of $S(g)-Y=0$. 

Let $T:\mathcal{H}\times\mathcal{G}\rightarrow \mathbb{L}_2 $ be the operator $(h,g)\mapsto \E\{h(B)-g(A)|C = \cdot\}$. To lighten notation, we abbreviate $Th=T(h,0)$. Let $h_0$ be the minimal $\mathbb{L}_2$ norm solution of $T(h,g_0)=0$.

Let $T^*$ be the adjoint of $T$.

\begin{remark}[AM-GM inequality]
    If $a=O(b\cdot c)$ then $a\leq \frac{b^2}{2}+O(c^2)$.
\end{remark}

\subsection{High probability events}

Consider the space $\mathcal{Q}=\prod_{j=1}^J \mathcal{Q}_j$ of vector valued functions $q(W)=\{q_1(W),...,q_J(W)\}$, where each component is almost surely bounded. Let $\ell\{ W; q(W)\}$ be a loss.

\begin{lemma}[Concentration; Lemma 14 of \cite{foster2023orthogonal}]\label{lemma:concentration}
Suppose Assumption~\ref{assumption:critical} holds for each $\mathcal{Q}_j$. Further suppose $\ell$ is $O(1)$ Lipschitz in its second argument with respect to $\ell_2$ norm. With probability $1-\zeta$, for any fixed $q_0\in\mathcal{Q}$ independent of data and for all $q\in\mathcal{Q}$, when $\delta_n=\Omega[\{J\log\log(n)+\log(1/\zeta)\}^{1/2}n^{-1/2}]$,
$$
\left|\left(\mathbb{E}_n-\mathbb{E}\right)[\ell\{ W ; q(W)\}-\ell\{ W ; q_0(W)\}]\right|=O\left(J \delta_n \sum_{j=1}^J\left\|q_j-q_{0,j}\right\|_2+J \delta_n^2\right).
$$
\end{lemma}

\begin{lemma}[High probability events]\label{lemma:upper_bound_empirical}
    Suppose Assumption~\ref{assumption:critical} holds for $\mathcal{F}$, $\mathcal{G}$, and $\mathcal{H\times \mathcal{F}}$. With probability $1-\zeta$, when $\delta_n=\Omega[\{\log\log(n)+\log(1/\zeta)\}^{1/2}n^{-1/2}]$,
    $$
    \left|\left(\mathbb{E}_n-\mathbb{E}\right)[ 
2\{h(B)-g(A)\}f(C)-f(C)^2
    ]\right|=O\left(\delta_n \left\|g-g_0\right\|_2+\delta_n\|f\|_2+\delta_n^2\right).
    $$
\end{lemma}

\begin{proof}
    We appeal to Lemma~\ref{lemma:concentration} for each term in the empirical process.
    \begin{enumerate}
        \item Consider $\left(\mathbb{E}_n-\mathbb{E}\right)\{h(B)f(C)\}$. Let $q(W)=h(B)f(C)$, $q_0(W)=0$, and $\ell\{W,q(W)\}=h(B)f(C)$, which has derivative $1$ in its second argument. 
        Then $\left|\left(\mathbb{E}_n-\mathbb{E}\right)\{h(B)f(C)-0\}\right|=O\left(\delta_n\|hf\|_2+\delta_n^2\right)=O\left(\delta_n\|f\|_2+\delta_n^2\right)
    $ since $\mathcal{H}$ is almost surely bounded.
        \item Consider $\left(\mathbb{E}_n-\mathbb{E}\right)\{g(A)f(C)\}$. Let $q(W)=\{g(A),f(C)\}$, $q_0(W)=\{g_0(A),0\}$, and $\ell\{W,q(W)\}=g(A)f(C)$, which has derivative $\{f(C),g(A)\}$ in its second argument. 
        Then $\left|\left(\mathbb{E}_n-\mathbb{E}\right)\{g(A)f(C)-0\}\right|=O\left(\delta_n\|g-g_0\|_2+\delta_n\|f\|_2+\delta_n^2\right).
    $ 
        \item Consider $\left(\mathbb{E}_n-\mathbb{E}\right)\{f(C)^2\}$. Let $q(W)=f(C)$, $q_0(W)=0$, and $\ell\{W,q(W)\}=f(C)^2$, which has derivative $2f(C)$ in its second argument. 
        Then $\left|\left(\mathbb{E}_n-\mathbb{E}\right)\{f(C)^2-0\}\right|=O\left(\delta_n\|f\|_2+\delta_n^2\right). \qedhere
    $ 
    \end{enumerate}
\end{proof}

\begin{lemma}[High probability events under weaker conditions]\label{lemma:upper_bound_empirical_weak}
     Suppose Assumption~\ref{assumption:critical} holds for $\mathcal{F}$, $\mathcal{G}$, and $\mathcal{H}$. With probability $1-\zeta$, when $\delta_n=\Omega[\{\log\log(n)+\log(1/\zeta)\}^{1/2}n^{-1/2}]$,  for a data independent hypothesis $h_*\in\mathcal{H}$,
    $$
    \left|\left(\mathbb{E}_n-\mathbb{E}\right)[ 
2\{h(B)-g(A)\}f(C)-f(C)^2
    ]\right|=O\left(\delta_n  \left\|h-h_*\right\|_2 + \delta_n \left\|g-g_0\right\|_2+\delta_n\|f\|_2+\delta_n^2\right).
    $$
\end{lemma}

\begin{proof}
     The proof is identical to Lemma~\ref{lemma:upper_bound_empirical} except for the first empirical process $\left(\mathbb{E}_n-\mathbb{E}\right)\{h(B)f(C)\}$. Let $q(W)=\{h(B),f(C)\}$, $q_0(W)=\{h_*(B),0\}$, and $\ell\{W,q(W)\}=h(B)f(C)$, which has derivative $\{f(C),h(B)\}$ in its second argument. 
        Then by Lemma~\ref{lemma:concentration}, $\left|\left(\mathbb{E}_n-\mathbb{E}\right)\{h(B)f(C)-0\}\right|=O\left(\delta_n\|h-h_*\|_2+\delta_n\|f\|_2+\delta_n^2\right).
    $ 
\end{proof}

Let $\|h\|^2_{2,n}=\mathbb{E}_n\{h(B)^2\}$, $I_n=2\mu(\|h\|^2_{2,n}-\|h_*\|_{2,n}^2)$,  and $I=2\mu(\|h\|^2_2-\|h_*\|_2^2)$ for some data independent $h_*\in\mathcal{H}$.

\begin{lemma}[High probability event for regularization]\label{lemma:bound_regularization}
    Suppose Assumption~\ref{assumption:critical} holds for $\mathcal{H}$. With probability $1-\zeta$, when $\delta_n=\Omega[\{\log\log(n)+\log(1/\zeta)\}^{1/2}n^{-1/2}]$,  for a data independent hypothesis $h_*\in\mathcal{H}$, $|I_n-I|=O(\mu\delta_n \|h-h_*\|_2+\mu\delta_n^2)$.
\end{lemma}

\begin{proof}
    Let $q(W)=h(B)$, $q_0(W)=h_*(B)$, and $\ell\{W,q(W)\}=h(B)^2$, which has derivative $2h(B)$ in its second argument. 
        Then by Lemma~\ref{lemma:concentration}, $|\|h\|^2_{2,n}-\|h_*\|_{2,n}^2-(\|h\|^2_2-\|h_*\|_2^2)|=\left|\left(\mathbb{E}_n-\mathbb{E}\right)\{h(B)^2-h_*(B)^2\}\right|=O\left(\delta_n\|h-h_*\|_2+\delta_n^2\right).
    $ 
\end{proof}


\subsection{Adversarial maximization}

Let $\omega>0$ and define
$
L_{\omega}(h,g)=\sup_{f \in \mathcal{F}} \mathbb{E}\left[2\left\{h(B)-g(A) \right\}f(C)-\omega f(C)^2\right]
$. Let $f_h=T(h-h_0,0)$ and $f_g=T(0,g-g_0)$ so that $f_h+f_g=T(h-h_0,g-g_0)$. 

\begin{lemma}[Maximization identity]\label{lemma:innermax_weak}
If $f_h,f_g\in \mathcal{F}$ for any $h\in\mathcal{H}$ and $g\in\mathcal{G}$, then 
$$
L_{\omega}(h,g) 
        = \omega^{-1}\mathbb{E}\left[2\left\{h(B)-g(A)\right\}(f_h+f_g)(C)- (f_h+f_g)^2(C)\right]=\omega^{-1}\left\|T\left(h-h_0,g-g_0\right)\right\|_2^2.
$$
In particular, since $\mathbb{E}[\{h_0(B)-g_0(A)\}f(C)]=0$,
\begin{align*}
    L_{\omega}(h,g_0) &=  \omega^{-1}\mathbb{E}\left[2\left\{h(B)-h_0(B)\right\}f_h(C)-  f_h^2(C)\right]=\omega^{-1}\left\|T\left(h-h_0,0\right)\right\|_2^2,  \\
L_{\omega}(h_0,g) &= \omega^{-1}\mathbb{E}\left[2\left\{g_0(A)-g(A)\right\}f_g(C)-  f_g^2(C)\right]=\omega^{-1}\left\|T\left(0,g-g_0\right)\right\|_2^2.
\end{align*}
\end{lemma}

\begin{proof}
    By the law of iterated expectations, write $L_{\omega}(h,g)$ equal to
\begin{align*}
&\sup_{f \in \mathcal{F}} \mathbb{E}\left[2\left\{h(B)-h_0(B)+g_0(A)-g(A) \right\}f(C)-\omega f(C)^2\right]  \\
&=\sup_{f \in \mathcal{F}} \mathbb{E}\left\{2\cdot T(h-h_0,g-g_0)f(C)\right\}-\omega \mathbb{E}\{f(C)^2\}
=\sup_{f\in\mathcal{F}}2 \langle T(h-h_0,g-g_0),f \rangle_2-\omega \langle f,f \rangle_2.
\end{align*}
Taking the Gateaux derivative with respect to $f$, we see that the first order condition is $2T(h-h_0,g-g_0)-2\omega f^*=0$. Rearranging, $f^*=\omega^{-1}T(h-h_0,g-g_0)$. Substitute $f^*$ into initial and final expressions in the display, and recall $T(h-h_0,g-g_0)=f_h+f_g$.
\end{proof}

\subsection{Algorithm~\ref{algo:sequential}}

We study the ridge regularized estimator and its population analogue:
\begin{align*}
    \hat{h}&=\arg \min _{h \in \mathcal{H}} \sup_{f \in \mathcal{F}} \mathbb{E}_n\left[2\left\{h(B)-\hat{g}(A)\right\} f(C)-f(C)^2\right]+\mu\E_n\{h(B)^2\}, \\
     h_{\mu}&=\arg \min _{h \in \mathcal{H}} \sup_{f \in \mathcal{F}} \mathbb{E}\left[2\left\{h(B)-g_0(A)\right\} f(C)-f(C)^2\right]+\mu\E\{h(B)^2\}.
\end{align*}

\begin{lemma}[From weak to strong metric]\label{lemma:strong_conv}
    For any $\mu>0$,
    $$
\|T(\hat{h}-h_{2\mu})\|^2_2+2\mu\|\hat{h}-h_{2\mu}\|^2_2 =\|T(\hat{h}-h_0)\|^2_2-\|T(h_{2\mu}-h_0)\|^2_2+2\mu(\|\hat{h}\|^2_2-\|h_{2\mu}\|^2_2).
$$
\end{lemma}

\begin{proof}
To lighten notation, let $h_{(\tau)}=h_{2\mu}+\tau (\hat{h}-h_{2\mu})$.
Define
$
W(\tau)=\|T\{h_{(\tau)}-h_0\}\|^2_2+2\mu\|h_{(\tau)}\|^2_2
$. Clearly $W(\tau)$ is quadratic in $\tau$ and strongly convex. By Lemma~\ref{lemma:innermax_weak},
\begin{align*}
    W(\tau)&=L_1\{h_{(\tau)},g_0\}+2\mu\|h_{(\tau)}\|^2_2 
    =\sup_{f \in \mathcal{F}} \mathbb{E}\left[2\left\{h_{(\tau)}(B)-g_0(A) \right\}f(C)-f(C)^2\right]+2\mu\|h_{(\tau)}\|^2_2
\end{align*}
which is minimized at $\tau=0$ by the definition of $h_{2\mu}$. Therefore by an exact Taylor expansion,
$
 \frac{1}{2}\partial_{\tau}^2 W(0) = \partial_{\tau} W(0)+\frac{1}{2}\partial_{\tau}^2 W(0)  = W(1)-W(0).
$
The derivatives of $W(\tau)$ are
$$
    \partial_{\tau}W(\tau)=2\langle T\{h_{(\tau)}-h_0\},T(\hat{h}-h_{2\mu})\rangle_2+4\mu\langle h_{(\tau)},\hat{h}-h_{2\mu}\rangle_2$$ 
    and 
    $
\partial_{\tau}^2 W(\tau)=2\|T (\hat{h}-h_{2\mu})\|_2^2+4\mu\|\hat{h}-h_{2\mu}\|_2^2.$
  Substituting in $\partial_{\tau}^2 W(0)$, $W(1)$, and $W(0)$ into the Taylor expansion yields the result.
\end{proof}

\begin{lemma}[Relating weak metrics]\label{lemma:main}
    Suppose the conditions of Lemmas~\ref{lemma:innermax_weak} and~\ref{lemma:upper_bound_empirical} hold. With probability $1-\zeta$, when $\delta_n=\Omega[\{\log\log(n)+\log(1/\zeta)\}^{1/2}n^{-1/2}]$, 
    $\|T(\hat{h}-h_0)\|_2^2-\|T(h_*-h_0)\|_2^2\leq 8\|T(h_*-h_0)\|_2^2+2\mu(\|h_*\|^2_{2,n}-\|\hat{h}\|^2_{2,n})+O\{\|\hat{g}-g_0\|_2^2+\delta_n\|T(\hat{h}-h_*)\|_2+\delta_n^2\}.$
\end{lemma}

\begin{proof}
    We proceed in steps.
    \begin{enumerate}
        \item  By Lemma~\ref{lemma:innermax_weak},
    \begin{align*}
        \|T(\hat{h}-h_0)\|_2^2
        &=L_1(\hat{h},g_0) 
        =\mathbb{E}\left[2\left\{\hat{h}(B)-h_0(B)\right\}f_{\hat{h}}(C)-  f_{\hat{h}}(C)^2\right] \\
        &=\mathbb{E}\left[2\left\{\hat{h}(B)-\hat{g}(A)+\hat{g}(A)-g_0(A)\right\}f_{\hat{h}}(C)-  f_{\hat{h}}(C)^2\right].
    \end{align*}
        \item Focusing on the third and fourth term, by Lemma~\ref{lemma:innermax_weak}, Cauchy Schwarz, and AM-GM
        \begin{align*}
            \mathbb{E}\left[2\left\{\hat{g}(A)-g_0(A)\right\}f_{\hat{h}}(C)\right]
            &\leq  2\|\hat{g}-g_0\|_2 \|T(\hat{h}-h_0)\|_2 
            \leq 2\|\hat{g}-g_0\|_2^2+\frac{1}{2}\|T(\hat{h}-h_0)\|^2_2.
        \end{align*}
        \item Focusing on the remaining terms, by Lemma~\ref{lemma:upper_bound_empirical} with probability $1-\zeta$
        \begin{align*}
            &\mathbb{E}\left[2\left\{\hat{h}(B)-\hat{g}(A)\right\}f_{\hat{h}}(C)-  f_{\hat{h}}(C)^2\right]  \\
            &\leq \mathbb{E}_n\left[2\left\{\hat{h}(B)-\hat{g}(A)\right\}f_{\hat{h}}(C)-  f_{\hat{h}}(C)^2\right]  + O\left(\delta_n \left\|\hat{g}-g_0\right\|_2+\delta_n\|f_{\hat{h}}\|_2+\delta_n^2\right).
        \end{align*}
        \begin{enumerate}
            \item For the empirical expectation, by Assumption~\ref{assumption:closed}, the definition of $\hat{h}$, Lemma~\ref{lemma:upper_bound_empirical}, and the AM-GM inequality $O(\delta_n\|f\|_2)\leq \frac{1}{2}\|f\|^2_2+O(\delta_n^2)$, with probability $1-\zeta$,
             {
        \begin{align*}
            &\mathbb{E}_n\left[2\left\{\hat{h}(B)-\hat{g}(A)\right\}f_{\hat{h}}(C)-  f_{\hat{h}}(C)^2\right]
            \leq \sup_{f\in\mathcal{F}}  \mathbb{E}_n\left[2\left\{\hat{h}(B)-\hat{g}(A)\right\}f(C)-  f(C)^2\right] \\
            &\leq \sup_{f\in\mathcal{F}}  \mathbb{E}_n\left[2\left\{h_*(B)-\hat{g}(A)\right\}f(C)-  f(C)^2\right] +\mu(\|h_*\|^2_{2,n}-\|\hat{h}\|^2_{2,n}) \\
            &\leq \sup_{f\in\mathcal{F}}  \mathbb{E}\left[2\left\{h_*(B)-\hat{g}(A)\right\}f(C)-  f(C)^2\right]  +O\left(\delta_n \left\|\hat{g}-g_0\right\|_2+\delta_n\|f\|_2+\delta_n^2\right)  \\
            &\quad +\mu(\|h_*\|^2_{2,n}-\|\hat{h}\|^2_{2,n}) \\
            &\leq \sup_{f\in\mathcal{F}}  \mathbb{E}\left[2\left\{h_*(B)-\hat{g}(A)\right\}f(C)-  \frac{1}{2}f(C)^2\right]  +O\left(\delta_n \left\|\hat{g}-g_0\right\|_2+\delta_n^2\right)+\mu(\|h_*\|^2_{2,n}-\|\hat{h}\|^2_{2,n}).
        \end{align*}
        }
        By triangle inequality, Lemma~\ref{lemma:innermax_weak}, and Jensen's inequality
         {
        \begin{align*}
            &\sup_{f\in\mathcal{F}}  \mathbb{E}\left[2\left\{h_*(B)-\hat{g}(A)\right\}f(C)-  \frac{1}{2}f(C)^2\right] \\
            &\leq \sup_{f\in\mathcal{F}}  \mathbb{E}\left[2\left\{h_*(B)-g_0(A)\right\}f(C)-  \frac{1}{4}f(C)^2\right]
            + \sup_{f\in\mathcal{F}}  \mathbb{E}\left[2\left\{h_0(B)-\hat{g}(A)\right\}f(C)-  \frac{1}{4}f(C)^2\right] \\
            &=L_{1/4}(h_*,g_0)+L_{1/4}(h_0,\hat{g}) 
            =4 \|T(h_*-h_0,0)\|_2^2+4\|T(0,\hat{g}-g_0)\|_2^2 \\
            &\leq 4 \|T(h_*-h_0)\|_2^2+4\|\hat{g}-g_0\|_2^2.
        \end{align*}
        }
        In summary, $\mathbb{E}_n\left[2\left\{\hat{h}(B)-\hat{g}(A)\right\}f_{\hat{h}}(C)-  f_{\hat{h}}(C)^2\right] $ is bounded by
        $$
4 \|T(h_*-h_0)\|_2^2+O\left(\|\hat{g}-g_0\|_2^2+\delta_n^2\right)+\mu(\|h_*\|^2_{2,n}-\|\hat{h}\|^2_{2,n}).
$$
            \item Consider the penultimate term. 
             By Lemma~\ref{lemma:innermax_weak}, triangle inequality, and AM-GM,
        \begin{align*}
              O(\delta_n \|f_{\hat{h}}\|_2)
              &= O\{\delta_n \|T(\hat{h}-h_0)\|_2\} = O\{\delta_n \|T(\hat{h}-h_*)\|_2+ \delta_n\|T(h_*-h_0)\|_2\}\\
              &\leq \frac{1}{2}\|T(h_*-h_0)\|^2_2+O\{\delta_n \|T(\hat{h}-h_*)\|_2+ \delta_n^2\}.
        \end{align*}
        \end{enumerate}
       In summary, we bound $\mathbb{E}\left[2\left\{\hat{h}(B)-\hat{g}(A)\right\}f_{\hat{h}}(C)-  f_{\hat{h}}(C)^2\right]$ by 
      \begin{align*}
           \frac{9}{2} \|T(h_*-h_0)\|_2^2+\mu(\|h_*\|^2_{2,n}-\|\hat{h}\|^2_{2,n})+
        O\{\|\hat{g}-g_0\|_2^2+\delta_n \|T(\hat{h}-h_*)\|_2+\delta_n^2\}.
      \end{align*}
      \item Collecting results,
      \begin{align*}
       \|T(\hat{h}-h_0)\|_2^2
       &\leq \frac{1}{2}\|T(\hat{h}-h_0)\|^2_2+ \frac{9}{2} \|T(h_*-h_0)\|_2^2+\mu(\|h_*\|^2_{2,n}-\|\hat{h}\|^2_{2,n}) \\
           &\quad +
        O\{\|\hat{g}-g_0\|_2^2+\delta_n \|T(\hat{h}-h_*)\|_2+\delta_n^2\}. \qedhere
       \end{align*}
    \end{enumerate}
   
\end{proof}

\begin{lemma}[Relating weak metrics under weaker conditions]\label{lemma:main_weak}
    Suppose the conditions of Lemmas~\ref{lemma:innermax_weak} and~\ref{lemma:upper_bound_empirical_weak} hold. With probability $1-\zeta$, when $\delta_n=\Omega[\{\log\log(n)+\log(1/\zeta)\}^{1/2}n^{-1/2}]$, 
    $\|T(\hat{h}-h_0)\|_2^2-\|T(h_*-h_0)\|_2^2\leq 8\|T(h_*-h_0)\|_2^2+2\mu(\|h_*\|^2_{2,n}-\|\hat{h}\|^2_{2,n})+O\{\|\hat{g}-g_0\|_2^2+\delta_n\|T(\hat{h}-h_*)\|_2+\delta_n\|\hat{h}-h_*\|_2+\delta_n^2\}$.
\end{lemma}

\begin{proof}
    The argument is identical to Lemma~\ref{lemma:main}, using Lemma~\ref{lemma:upper_bound_empirical_weak}  instead of~\ref{lemma:upper_bound_empirical}.
\end{proof}

\begin{lemma}[Regularization bias; Lemma 3 of \cite{bennett2023source}]\label{lemma:bias}
    If Assumption~\ref{assumption:source} holds then $\|h_{\mu}-h_0\|_2^2\leq \|w_h\|^2_2 \mu^{\min(\beta,2)}$ and $\|T(h_{\mu}-h_0)\|_2^2\leq \|w_h\|^2_2 \mu^{\min(\beta+1,2)}$.
\end{lemma}

\begin{proof}
    For completeness, we present the proof in Appendix~\ref{sec:bias}.
\end{proof}

\begin{proof}[Proof of Theorem~\ref{theorem:L2}]
    We consolidate both versions of the result, either placing the stronger assumption on the product space ($\chi=0$, Lemma~\ref{lemma:upper_bound_empirical}) or not ($\chi=1$, Lemma~\ref{lemma:upper_bound_empirical_weak}). Take $h_*=h_{2\mu}$. 
    \begin{enumerate}
        \item By Lemmas~\ref{lemma:strong_conv},~\ref{lemma:main},~\ref{lemma:main_weak}, and~\ref{lemma:bound_regularization}, AM-GM inequality, and $\mu+\chi=O(1)$, we bound the quantity $\|T(\hat{h}-h_{2\mu})\|^2_2+2\mu\|\hat{h}-h_{2\mu}\|^2_2$ by
         { 
         \begin{align*}
             &\|T(\hat{h}-h_0)\|^2_2-\|T(h_{2\mu}-h_0)\|^2_2+2\mu(\|\hat{h}\|^2_2-\|h_{2\mu}\|^2_2) \\
             &\leq 8\|T(h_{2\mu}-h_0)\|_2^2+2\mu(\|h_{2\mu}\|^2_{2,n}-\|\hat{h}\|^2_{2,n})+2\mu(\|\hat{h}\|^2_2-\|h_{2\mu}\|^2_2)  \\
             &\quad +O\{\|\hat{g}-g_0\|_2^2+\delta_n\|T(\hat{h}-h_{2\mu})\|_2+\chi\delta_n\|\hat{h}-h_{2\mu}\|_2+ \delta_n^2\}  \\
             &=8\|T(h_{2\mu}-h_0)\|_2^2-I_n+I+O\{\|\hat{g}-g_0\|_2^2+\delta_n\|T(\hat{h}-h_{2\mu})\|_2+\chi\delta_n\|\hat{h}-h_{2\mu}\|_2+\delta_n^2\} \\
             &\leq 8\|T(h_{2\mu}-h_0)\|_2^2+O\{\|\hat{g}-g_0\|_2^2+\delta_n\|T(\hat{h}-h_{2\mu})\|_2+(\mu+\chi)\delta_n\|\hat{h}-h_{2\mu}\|_2+\delta_n^2\} \\
             &\leq 8\|T(h_{2\mu}-h_0)\|_2^2+\frac{1}{2}\|T(\hat{h}-h_{2\mu})\|^2_2+O\{\|\hat{g}-g_0\|_2^2+(\sqrt{\mu+\chi})^2\delta_n\|\hat{h}-h_{2\mu}\|_2+\delta_n^2\} \\
             &\leq 8\|T(h_{2\mu}-h_0)\|_2^2+\frac{1}{2}\|T(\hat{h}-h_{2\mu})\|^2_2+\frac{\mu+\chi}{2}\|\hat{h}-h_{2\mu}\|^2_2+O(\|\hat{g}-g_0\|_2^2+\delta_n^2).
         \end{align*}
         }
        Rearranging yields
         $$
         \frac{1}{2}\|T(\hat{h}-h_{2\mu})\|^2_2+\frac{3\mu-\chi}{2}\|\hat{h}-h_{2\mu}\|^2_2
         \leq 
         8\|T(h_{2\mu}-h_0)\|_2^2+O(\|\hat{g}-g_0\|_2^2+\delta_n^2),
         $$
         \begin{align*}
           \text{ hence }  \|T(\hat{h}-h_{2\mu})\|^2_2&\leq 16\|T(h_{2\mu}-h_0)\|_2^2+O(\|\hat{g}-g_0\|_2^2+\delta_n^2),\\
             \|\hat{h}-h_{2\mu}\|^2_2&\leq  \frac{16}{3\mu-\chi}\|T(h_{2\mu}-h_0)\|_2^2+\frac{2}{3\mu-\chi}\cdot O(\|\hat{g}-g_0\|_2^2+\delta_n^2).
         \end{align*}
        \item For the weak metric, we use triangle inequality and Lemma~\ref{lemma:bias}:
        \begin{align*}
            \|T(\hat{h}-h_0)\|_2^2&\leq 2 \|T(\hat{h}-h_{2\mu})\|_2^2 + 2 \|T(h_{2\mu}-h_0)\|_2^2
            \leq  18\|T(h_{2\mu}-h_0)\|_2^2+O(\|\hat{g}-g_0\|_2^2+\delta_n^2) \\
            &=O\{\|w_h\|^2_2 \mu^{\min(\beta+1,2)}+\|\hat{g}-g_0\|_2^2+\delta_n^2\}. 
        \end{align*}
        \item For the strong metric, we use triangle inequality and Lemma~\ref{lemma:bias}:
        \begin{align*}
            \|\hat{h}-h_0\|_2^2 &\leq 2 \|\hat{h}-h_{2\mu}\|_2^2 + 2 \|h_{2\mu}-h_0\|_2^2 \\
            &\leq \frac{32}{3\mu-\chi}\|T(h_{2\mu}-h_0)\|_2^2+2 \|h_{2\mu}-h_0\|_2^2+\frac{4}{3\mu-\chi}\cdot O(\|\hat{g}-g_0\|_2^2+\delta_n^2)  \\
            &=O\left\{\frac{\|w_h\|^2_2 \mu^{\min(\beta+1,2)}}{\mu-\chi/3}+\|w_h\|^2_2 \mu^{\min(\beta,2)}  +\frac{\|\hat{g}-g_0\|_2^2+\delta_n^2}{\mu-\chi/3}\right\}.
        \end{align*}
        When $\chi=0$, $\|\hat{h}-h_0\|_2^2=\{\|w_h\|^2_2 \mu^{\min(\beta,1)}+\mu^{-1}\|\hat{g}-g_0\|_2^2+\mu^{-1}\delta_n^2\}$. \qedhere
    \end{enumerate}

\end{proof}

\subsection{Algorithm~\ref{algo:npiv_general}}

We now study
$$
 \hat{h}=\arg \min _{h \in \mathcal{H}} \sup_{f \in \mathcal{F}} \mathbb{E}_n\left[2\left\{h(B)-\hat{g}(A)\right\} f(C)-f(C)^2\right]-\mu''\left\|f\right\|_{\mathcal{F}}^2+\mu\left\|h\right\|_{\mathcal{H}}^2.
$$

\begin{lemma}[Relating weak metrics]\label{lemma:main_general}
    Suppose the conditions of Lemma~\ref{lemma:main} hold and $\|T(h-h_0)\|^2_{\mathcal{F}}\leq \textsc{lip} \|h-h_0\|^2_{\mathcal{H}}$. With probability $1-\zeta$, when $\delta_n=\Omega[\{\log\log(n)+\log(1/\zeta)\}^{1/2}n^{-1/2}]$, 
    $\|T(\hat{h}-h_0)\|_2^2-\|T(h_*-h_0)\|_2^2\leq 8\|T(h_*-h_0)\|_2^2+2\mu(\|h_*\|^2_{\mathcal{H}}-\|\hat{h}\|^2_{\mathcal{H}})+2\mu'' \textsc{lip} \|\hat{h}-h_0\|_{\mathcal{H}}^2+O\{\|\hat{g}-g_0\|_2^2+\delta_n\|T(\hat{h}-h_*)\|_2+\delta_n^2\}.$
\end{lemma}

\begin{proof}
    We proceed in steps similar to Lemma~\ref{lemma:main}.
    \begin{enumerate}
        \item  As before,
$
        \|T(\hat{h}-h_0)\|_2^2=\mathbb{E}\left[2\left\{\hat{h}(B)-\hat{g}(A)+\hat{g}(A)-g_0(A)\right\}f_{\hat{h}}(C)-  f_{\hat{h}}(C)^2\right].
 $
        \item As before,
     $
            \mathbb{E}\left[2\left\{\hat{g}(A)-g_0(A)\right\}f_{\hat{h}}(C)\right]
            \leq 2\|\hat{g}-g_0\|_2^2+\frac{1}{2}\|T(\hat{h}-h_0)\|^2_2.
       $
        \item As before, with probability $1-\zeta$, we bound $\mathbb{E}\left[2\left\{\hat{h}(B)-\hat{g}(A)\right\}f_{\hat{h}}(C)-  f_{\hat{h}}(C)^2\right] $ by 
        \begin{align*}
            \mathbb{E}_n\left[2\left\{\hat{h}(B)-\hat{g}(A)\right\}f_{\hat{h}}(C)-  f_{\hat{h}}(C)^2\right]  + O\left(\delta_n \left\|\hat{g}-g_0\right\|_2+\delta_n\|f_{\hat{h}}\|_2+\delta_n^2\right).
        \end{align*}
        \begin{enumerate}
            \item  Consider the empirical expectation. By Assumption~\ref{assumption:closed}, the definition of $\hat{h}$, Lemma~\ref{lemma:upper_bound_empirical}, and the AM-GM inequality $O(\delta_n\|f\|_2)\leq \frac{1}{2}\|f\|^2_2+O(\delta_n^2)$, with probability $1-\zeta$,
             { 
        \begin{align*}
            &\mathbb{E}_n\left[2\left\{\hat{h}(B)-\hat{g}(A)\right\}f_{\hat{h}}(C)-  f_{\hat{h}}(C)^2\right] -\mu'' \|f_{\hat{h}}\|^2_{\mathcal{F}} \\
            &\leq \sup_{f\in\mathcal{F}}  \mathbb{E}_n\left[2\left\{\hat{h}(B)-\hat{g}(A)\right\}f(C)-  f(C)^2\right]-\mu'' \|f\|^2_{\mathcal{F}} \\
            &\leq \sup_{f\in\mathcal{F}}  \mathbb{E}_n\left[2\left\{h_*(B)-\hat{g}(A)\right\}f(C)-  f(C)^2\right] +\mu(\|h_*\|^2_{\mathcal{H}}-\|\hat{h}\|^2_{\mathcal{H}}) -\mu'' \|f\|^2_{\mathcal{F}}\\
            &\leq \sup_{f\in\mathcal{F}}  \mathbb{E}\left[2\left\{h_*(B)-\hat{g}(A)\right\}f(C)-  f(C)^2\right]  +O\left(\delta_n \left\|\hat{g}-g_0\right\|_2+\delta_n\|f\|_2+\delta_n^2\right) \\
            &\quad +\mu(\|h_*\|^2_{\mathcal{H}}-\|\hat{h}\|^2_{\mathcal{H}})-\mu'' \|f\|^2_{\mathcal{F}} \\
            &\leq \sup_{f\in\mathcal{F}}  \mathbb{E}\left[2\left\{h_*(B)-\hat{g}(A)\right\}f(C)-  \frac{1}{2}f(C)^2\right]  +O\left(\delta_n \left\|\hat{g}-g_0\right\|_2+\delta_n^2\right)+\mu(\|h_*\|^2_{\mathcal{H}}-\|\hat{h}\|^2_{\mathcal{H}}).
        \end{align*}
        }
        As before,
        \begin{align*}
            \sup_{f\in\mathcal{F}}  \mathbb{E}\left[2\left\{h_*(B)-\hat{g}(A)\right\}f(C)-  \frac{1}{2}f(C)^2\right] \leq 4 \|T(h_*-h_0)\|_2^2+4\|\hat{g}-g_0\|_2^2.
        \end{align*}
        Moreover,
        $
        \|f_{\hat{h}}\|^2_{\mathcal{F}}=\|T(\hat{h}-h_0)\|^2_{\mathcal{F}}\leq \textsc{lip} \|\hat{h}-h_0\|_{\mathcal{H}}^2.
        $
        We conclude that the quantity $\mathbb{E}_n\left[2\left\{\hat{h}(B)-\hat{g}(A)\right\}f_{\hat{h}}(C)-  f_{\hat{h}}(C)^2\right] $ is bounded by
        $$
4 \|T(h_*-h_0)\|_2^2+O\left(\|\hat{g}-g_0\|_2^2+\delta_n^2\right)+\mu(\|h_*\|^2_{\mathcal{H}}-\|\hat{h}\|^2_{\mathcal{H}})+\mu'' \textsc{lip} \|\hat{h}-h_0\|_{\mathcal{H}}^2.
$$
            \item Consider the penultimate term. 
             As before,
        $
              O(\delta_n \|f_{\hat{h}}\|_2)
              \leq \frac{1}{2}\|T(h_*-h_0)\|^2_2+O\{\delta_n \|T(\hat{h}-h_*)\|_2+ \delta_n^2\}.
        $
        \end{enumerate}
       In summary, we bound $\mathbb{E}\left[2\left\{\hat{h}(B)-\hat{g}(A)\right\}f_{\hat{h}}(C)-  f_{\hat{h}}(C)^2\right]$ by 
      \begin{align*}
           \frac{9}{2} \|T(h_*-h_0)\|_2^2+\mu(\|h_*\|^2_{\mathcal{H}}-\|\hat{h}\|^2_{\mathcal{H}})+\mu'' \textsc{lip} \|\hat{h}-h_0\|_{\mathcal{H}}^2+
        O\{\|\hat{g}-g_0\|_2^2+\delta_n \|T(\hat{h}-h_*)\|_2+\delta_n^2\}.
      \end{align*}
      \item Collecting results,
      \begin{align*}
       \|T(\hat{h}-h_0)\|_2^2
       &\leq \frac{1}{2}\|T(\hat{h}-h_0)\|^2_2+ \frac{9}{2} \|T(h_*-h_0)\|_2^2+\mu(\|h_*\|^2_{\mathcal{H}}-\|\hat{h}\|^2_{\mathcal{H}})+\mu'' \textsc{lip} \|\hat{h}-h_0\|_{\mathcal{H}}^2 \\
           &\quad +
        O\{\|\hat{g}-g_0\|_2^2+\delta_n \|T(\hat{h}-h_*)\|_2+\delta_n^2\}.
       \end{align*}
       Finally, rearrange as before. \qedhere
    \end{enumerate}
   
\end{proof}

\begin{proof}[Proof of Theorem~\ref{theorem:general}]
    Take $h_*=h_{0}$.  By Lemma~\ref{lemma:main_general},  we bound $\|T(\hat{h}-h_0)\|_2^2$ by
        $$
        2\mu(\|h_0\|^2_{\mathcal{H}}-\|\hat{h}\|^2_{\mathcal{H}})+2\mu'' \textsc{lip} \|\hat{h}-h_0\|_{\mathcal{H}}^2+O\{\|\hat{g}-g_0\|_2^2+\delta_n\|T(\hat{h}-h_0)\|_2+\delta_n^2\}.
        $$
Since $
2\mu'' \textsc{lip} \|\hat{h}-h_0\|_{\mathcal{H}}^2\leq 4\mu'' \textsc{lip} (\|\hat{h}\|_{\mathcal{H}}^2+\|h_0\|_{\mathcal{H}}^2) \leq 2\mu (\|\hat{h}\|_{\mathcal{H}}^2+\|h_0\|_{\mathcal{H}}^2)
$, we have
$$
\|T(\hat{h}-h_0)\|_2^2 \leq 4\mu \|h_0\|^2_{\mathcal{H}}+O\{\|\hat{g}-g_0\|_2^2+\delta_n\|T(\hat{h}-h_0)\|_2+\delta_n^2\}.
$$
By AM-GM inequality,
$$
\|T(\hat{h}-h_0)\|_2^2 \leq 4\mu \|h_0\|^2_{\mathcal{H}}+\frac{1}{2}\|T(\hat{h}-h_0)\|_2^2+ O\{\|\hat{g}-g_0\|_2^2+\delta_n^2\}. \qedhere 
$$
\end{proof}
\section{Proof of Theorem~\ref{theorem:joint}}\label{sec:rate2}

As before, let $S:g\mapsto \E\{g(A)|C' = \cdot\}$ and $T:(h,g)\mapsto \E\{h(B)-g(A)|C = \cdot\}$. Let $(h_0,g_0)$ be the minimal $\mathbb{L}_2$ norm solutions to $S(g)-Y=0$ and $T(h,g)=0$.

\subsection{High probability events}

\begin{lemma}[High probability events]\label{lemma:upper_bound_empirical2}
    Suppose Assumption~\ref{assumption:critical} holds for $\mathcal{F}$, $\mathcal{F}'$, $\mathcal{G}\times\mathcal{F}$, 
 $\mathcal{G}\times\mathcal{F}'$, and $\mathcal{H}\times\mathcal{F}$. With probability $1-\zeta$, when $\delta_n=\Omega[\{\log\log(n)+\log(1/\zeta)\}^{1/2}n^{-1/2}]$,
    \begin{align*}
        &\left|\left(\mathbb{E}_n-\mathbb{E}\right)[ 
2\{g(A)-Y\}f'(C')-f'(C')^2
    ]\right|=O\left(\delta_n\|f'\|_2+\delta_n^2\right), \\
    &\left|\left(\mathbb{E}_n-\mathbb{E}\right)[ 
2\{h(B)-g(A)\}f(C)-f(C)^2
    ]\right|=O\left(\delta_n\|f\|_2+\delta_n^2\right).
    \end{align*}
\end{lemma}

\begin{proof}
    We appeal to Lemma~\ref{lemma:concentration} for each term in the former empirical process, similar to Lemma~\ref{lemma:upper_bound_empirical}.
    \begin{enumerate}
        \item Consider $\left(\mathbb{E}_n-\mathbb{E}\right)\{g(A)f'(C')\}$. Let $q(W)=g(A)f'(C')$, $q_0(W)=0$, and $\ell\{W,q(W)\}=g(A)f'(C')$, which has derivative $1$ in its second argument. 
        Then $\left|\left(\mathbb{E}_n-\mathbb{E}\right)\{g(A)f'(C')-0\}\right|=O\left(\delta_n\|gf'\|_2+\delta_n^2\right)=O\left(\delta_n\|f'\|_2+\delta_n^2\right)
    $ since $\mathcal{G}$ is almost surely bounded.
        \item Consider $\left(\mathbb{E}_n-\mathbb{E}\right)\{Yf'(C')\}$. Let $q(W)=f'(C')$, $q_0(W)=0$, and $\ell\{W,q(W)\}=Yf'(C')$, which has derivative $Y$ in its second argument. 
        Then $\left|\left(\mathbb{E}_n-\mathbb{E}\right)\{Yf'(C')-0\}\right|=O\left(\delta_n\|f'\|_2+\delta_n^2\right).
    $ 
        \item As before, $\left|\left(\mathbb{E}_n-\mathbb{E}\right)\{f'(C')^2-0\}\right|=O\left(\delta_n\|f'\|_2+\delta_n^2\right).
    $     
    \end{enumerate}
    Next we turn to the latter empirical process.
    \begin{enumerate}
        \item As before, $\left|\left(\mathbb{E}_n-\mathbb{E}\right)\{h(B)f(C)-0\}\right|=O\left(\delta_n\|f\|_2+\delta_n^2\right).$
        \item Similarly, $\left|\left(\mathbb{E}_n-\mathbb{E}\right)\{g(A)f(C)-0\}\right|=O\left(\delta_n\|f\|_2+\delta_n^2\right).$
        \item As before, $\left|\left(\mathbb{E}_n-\mathbb{E}\right)\{f(C)^2-0\}\right|=O\left(\delta_n\|f\|_2+\delta_n^2\right). \qedhere
    $ 
    \end{enumerate}
\end{proof}

Let $\|g\|^2_{2,n}=\mathbb{E}_n\{g(A)^2\}$, $I_n'=2\mu'(\|g\|^2_{2,n}-\|g_*\|_{2,n}^2)$,  and $I'=2\mu'(\|g\|^2_2-\|g_*\|_2^2)$ for some data independent $g_*\in\mathcal{G}$.

\begin{lemma}[High probability event for regularization]\label{lemma:bound_regularization2}
    Suppose Assumption~\ref{assumption:critical} holds for $\mathcal{G}$. With probability $1-\zeta$, when $\delta_n=\Omega[\{\log\log(n)+\log(1/\zeta)\}^{1/2}n^{-1/2}]$,  for a data independent hypothesis $g_*\in\mathcal{G}$, $|I'_n-I'|=O(\mu'\delta_n \|g-g_*\|_2+\mu'\delta_n^2)$.
\end{lemma}

\begin{proof}
   The argument is identical to Lemma~\ref{lemma:bound_regularization}.
\end{proof}

\subsection{Adversarial maximization}

Let $\omega>0$,
$
L'_{\omega}(g)=\sup_{f' \in \mathcal{F}'} \mathbb{E}\left[2\left\{g(A)-Y \right\}f'(C')-\omega f'(C')^2\right]
$, and $f'_g=S(g-g_0)$.

\begin{lemma}[Maximization identity]\label{lemma:innermax_weak2}
If $f'_g\in \mathcal{F}'$ for any $g\in\mathcal{G}$, then 
$$
L'_{\omega}(g) 
        = \omega^{-1}\mathbb{E}\left[2\left\{g(A)-Y\right\}f'_g(C')- (f'_g)^2(C')\right]=\omega^{-1}\left\|S\left(g-g_0\right)\right\|_2^2.
$$
\end{lemma}

\begin{proof}
     By the law of iterated expectations, $L'_{\omega}(g)$ equals
\begin{align*}
&\sup_{f \in \mathcal{F}'} \mathbb{E}\left[2\left\{g(A)-g_0(A)\right\}f(C')-\omega f(C')^2\right]
=\sup_{f \in \mathcal{F}'} \mathbb{E}\left\{2\cdot S(g-g_0)f(C')\right\}-\omega \mathbb{E}\{f(C')^2\} \\
&=\sup_{f\in\mathcal{F}'}2 \langle S(g-g_0),f \rangle_2-\omega \langle f,f \rangle_2.
\end{align*}
Taking the Gateaux derivative with respect to $f$, we see that the first order condition is $2S(g-g_0)-2\omega f^*=0$. Rearranging, $f^*=\omega^{-1}S(g-g_0)$. Substitute $f^*$ into initial and final expressions in the display, and recall $S(g-g_0)=f'_g$.
\end{proof}

\subsection{Algorithm~\ref{algo:simultaneous}}

We study Algorithm~\ref{algo:simultaneous} and its population analogue:
\begin{align*}
    (\hat{g},\hat{h})&=\arg \min _{g\in\mathcal{G}, h \in \mathcal{H}} 
    \sup_{f' \in \mathcal{F}} \mathbb{E}_n\left[2\left\{g(A)-Y\right\} f'(C')-f'(C')^2\right]
     +\mu'\E_n\{g(A)^2\} \\
    &\quad +
    \sup_{f \in \mathcal{F}} \mathbb{E}_n\left[2\left\{h(B)-g(A)\right\} f(C)-f(C)^2\right]   
    +\mu\E_n\{h(B)^2\}, \\
    \{g_{(2\mu',2\mu)},h_{(2\mu',2\mu)}\}&=\arg \min _{g\in\mathcal{G}, h \in \mathcal{H}} 
      \sup_{f' \in \mathcal{F}} \mathbb{E}\left[2\left\{g(A)-Y\right\} f'(C')-f'(C')^2\right]+2\mu'\E\{g(A)^2\} \\
      &\quad +
     \sup_{f \in \mathcal{F}} \mathbb{E}\left[2\left\{h(B)-g(A)\right\} f(C)-f(C)^2\right]+2\mu\E\{h(B)^2\}.
\end{align*}

\begin{lemma}[From weak to strong metric]\label{lemma:strong_conv2}
    For any $\mu',\mu>0$,
   \begin{align*}
       &\|S\{\hat{g}-g_{(2\mu',2\mu)}\}\|_2^2
     +2\mu'\|\hat{g}-g_{(2\mu',2\mu)}\|_2^2+\|T \{\hat{h}-h_{(2\mu',2\mu)},\hat{g}-g_{(2\mu',2\mu)}\}\|_2^2
     +2\mu\|\hat{h}-h_{(2\mu',2\mu)}\|_2^2\\
     &=\|S(\hat{g}-g_0)\|^2_2-\|S\{g_{(2\mu',2\mu)}-g_0\}\|^2_2+2\mu'\{\|\hat{g}\|^2_2-\|g_{(2\mu',2\mu)}\|^2_2\}\\
     &+\|T(\hat{h}-h_0,\hat{g}-g_0)\|^2_2-\|T\{h_{(2\mu',2\mu)}-h_0,g_{(2\mu',2\mu)}-g_0\}\|^2_2 +2\mu\{\|\hat{h}\|^2_2-\|h_{(2\mu',2\mu)}\|^2_2\}.
   \end{align*}
\end{lemma}

\begin{proof}
Let $g_{(\tau)}=g_{(2\mu',2\mu)}+\tau \{\hat{g}-g_{(2\mu',2\mu)}\}$, $h_{(\tau)}=h_{(2\mu',2\mu)}+\tau \{\hat{h}-h_{(2\mu',2\mu)}\}$, and
$$
W(\tau)=\|S\{g_{(\tau)}-g_0\}\|^2_2+2\mu'\|g_{(\tau)}\|^2_2+\|T\{h_{(\tau)}-h_0,g_{(\tau)}-g_0\}\|^2_2+2\mu\|h_{(\tau)}\|^2_2.
$$ 
Clearly $W(\tau)$ is quadratic in $\tau$ and strongly convex. By Lemmas~\ref{lemma:innermax_weak} and~\ref{lemma:innermax_weak2},
\begin{align*}
    W(\tau)&=L'_1\{g_{(\tau)}\}+2\mu\|g_{(\tau)}\|^2_2
    +L_1\{h_{(\tau)},g_{(\tau)}\}+2\mu\|h_{(\tau)}\|^2_2  \\
    &=\sup_{f' \in \mathcal{F}} \mathbb{E}\left[2\left\{g_{(\tau)}(A)-Y \right\}f'(C')-f'(C')^2\right]+2\mu'\|g_{(\tau)}\|^2_2 \\
    &+ \sup_{f \in \mathcal{F}} \mathbb{E}\left[2\left\{h_{(\tau)}(B)-g_{(\tau)}(A)\right\}f(C)-f(C)^2\right]+2\mu\|h_{(\tau)}\|^2_2
\end{align*}
which is minimized at $\tau=0$ by the definition of $\{g_{(2\mu',2\mu)},h_{(2\mu',2\mu)}\}$. Therefore by an exact Taylor expansion,
$
 \frac{1}{2}\partial_{\tau}^2 W(0) = \partial_{\tau} W(0)+\frac{1}{2}\partial_{\tau}^2 W(0)  = W(1)-W(0).
$
The derivatives are
\begin{align*}
     \partial_{\tau}W(\tau)
     &=2\langle S\{g_{(\tau)}-g_0\},S\{\hat{g}-g_{(2\mu',2\mu)}\}\rangle_2
     +4\mu'\langle g_{(\tau)},\hat{g}-g_{(2\mu',2\mu)}\rangle_2\\
     &+2\langle T\{h_{(\tau)}-h_0,g_{(\tau)}-g_0\},T\{\hat{h}-h_{(2\mu',2\mu)},\hat{g}-g_{(2\mu',2\mu)}\}\rangle_2
     +4\mu\langle h_{(\tau)},\hat{h}-h_{(2\mu',2\mu)}\rangle_2, \\
     \partial_{\tau}^2 W(\tau)
     &=2\|S\{\hat{g}-g_{(2\mu',2\mu)}\}\|_2^2
     +4\mu'\|\hat{g}-g_{(2\mu',2\mu)}\|_2^2\\
     &+2\|T \{\hat{h}-h_{(2\mu',2\mu)},\hat{g}-g_{(2\mu',2\mu)}\}\|_2^2
     +4\mu\|\hat{h}-h_{(2\mu',2\mu)}\|_2^2.
\end{align*}
  Substituting in $\partial_{\tau}^2 W(0)$, $W(1)$, and $W(0)$ into the Taylor expansion yields the result.
\end{proof}

\begin{lemma}[Relating weak metrics]\label{lemma:main2}
    Suppose the conditions of Lemmas~\ref{lemma:innermax_weak},~\ref{lemma:innermax_weak2}, and~\ref{lemma:upper_bound_empirical2} hold. With probability $1-\zeta$, when $\delta_n=\Omega[\{\log\log(n)+\log(1/\zeta)\}^{1/2}n^{-1/2}]$, 
    \begin{align*}
        &\|S(\hat{g}-g_0)\|^2_2-\|S(g_*-g_0)\|^2_2
        +\|T(\hat{h}-h_0,\hat{g}-g_0)\|^2_2-\|T(h_*-h_0,g_*-g_0)\|^2_2  \\
      &\leq \frac{3}{2}\|S(g_*-g_0)\|^2_2+\frac{3}{2}\|T(h_*-h_0,g_*-g_0)\|^2_2 +2\mu'(\|g_*\|^2_{2,n}-\|\hat{g}\|^2_{2,n})+2\mu(\|h_*\|^2_{2,n}-\|\hat{h}\|^2_{2,n}) \\
      &+O\{\delta_n\|S(\hat{g}-g_*)\|_2+\delta_n\|T(\hat{h}-h_*,\hat{g}-g_*)\|_2+\delta_n^2\}. 
    \end{align*}
\end{lemma}

\begin{proof}
    We proceed in steps.
    \begin{enumerate}
        \item  By Lemma~\ref{lemma:innermax_weak2},
 $
        \|S(\hat{g}-g_0)\|_2^2
        =L'_1(\hat{g}) 
        =\mathbb{E}\left[2\left\{\hat{g}(A)-Y\right\}f'_{\hat{g}}(C')-  f'_{\hat{g}}(C')^2\right].$
    By Lemma~\ref{lemma:innermax_weak},
   $
        \|T(\hat{h}-h_0,\hat{g}-g_0)\|_2^2
        =L_1(\hat{h},\hat{g}) 
        =\mathbb{E}\left[2\left\{\hat{h}(B)-\hat{g}(A)\right\}(f_{\hat{h}}+f_{\hat{g}})(C)-  (f_{\hat{h}}+f_{\hat{g}})^2(C)\right]. 
      $
        \item By Lemma~\ref{lemma:upper_bound_empirical2} with probability $1-\zeta$,
        \begin{align*}
            &\mathbb{E}\left[2\left\{\hat{g}(A)-Y\right\}f'_{\hat{g}}(C')-  f'_{\hat{g}}(C')^2\right]  
            \leq \mathbb{E}_n\left[2\left\{\hat{g}(A)-Y\right\}f'_{\hat{g}}(C')-  f'_{\hat{g}}(C')^2\right]  + O\left(\delta_n\|f'_{\hat{g}}\|_2+\delta_n^2\right),\\
            &\mathbb{E}\left[2\left\{\hat{h}(B)-\hat{g}(A)\right\}(f_{\hat{h}}+f_{\hat{g}})(C)-  (f_{\hat{h}}+f_{\hat{g}})^2(C)\right]  \\
            &\leq \mathbb{E}_n\left[2\left\{\hat{h}(B)-\hat{g}(A)\right\}(f_{\hat{h}}+f_{\hat{g}})(C)-  (f_{\hat{h}}+f_{\hat{g}})^2(C)\right]  + O\left(\delta_n\|f_{\hat{h}}+f_{\hat{g}}\|_2+\delta_n^2\right).
        \end{align*}

        \item By Assumptions~\ref{assumption:closed} and~\ref{assumption:closed2},
        \begin{align*}
            &\mathbb{E}_n\left[2\left\{\hat{g}(A)-Y\right\}f'_{\hat{g}}(C')-  f'_{\hat{g}}(C')^2\right]
            \leq \sup_{f'\in\mathcal{F}'}   \mathbb{E}_n\left[2\left\{\hat{g}(A)-Y\right\}f'(C')-  f'(C')^2\right] \\
            &\mathbb{E}_n\left[2\left\{\hat{h}(B)-\hat{g}(A)\right\}(f_{\hat{h}}+f_{\hat{g}})(C)-  (f_{\hat{h}}+f_{\hat{g}})^2(C)\right]
            \leq \sup_{f\in\mathcal{F}}\mathbb{E}_n\left[2\left\{\hat{h}(B)-\hat{g}(A)\right\}f(C)-  f(C)^2\right].
        \end{align*}

        \item By the definition of $(\hat{g},\hat{h})$, Lemma~\ref{lemma:upper_bound_empirical2}, AM-GM inequality, and Lemmas~\ref{lemma:innermax_weak} and~\ref{lemma:innermax_weak2}, with probability $1-\zeta$,
        \begin{align*}
            &\sup_{f'\in\mathcal{F}'}   \mathbb{E}_n\left[2\left\{\hat{g}(A)-Y\right\}f'(C')-  f'(C')^2\right] + \sup_{f\in\mathcal{F}}\mathbb{E}_n\left[2\left\{\hat{h}(B)-\hat{g}(A)\right\}f(C)-  f(C)^2\right] \\
            &\leq \sup_{f'\in\mathcal{F}'}   \mathbb{E}_n\left[2\left\{g_*(A)-Y\right\}f'(C')-  f'(C')^2\right] + \sup_{f\in\mathcal{F}}\mathbb{E}_n\left[2\left\{h_*(B)-g_*(A)\right\}f(C)-  f(C)^2\right] \\
            &\quad + \mu'(\|g_*\|^2_{2,n}-\|\hat{g}\|^2_{2,n})+\mu(\|h_*\|^2_{2,n}-\|\hat{h}\|^2_{2,n}) \\
            &\leq \sup_{f'\in\mathcal{F}'}   \mathbb{E}\left[2\left\{g_*(A)-Y\right\}f'(C')-  f'(C')^2\right] + \sup_{f\in\mathcal{F}}\mathbb{E}\left[2\left\{h_*(B)-g_*(A)\right\}f(C)-  f(C)^2\right] \\
            &\quad + \mu'(\|g_*\|^2_{2,n}-\|\hat{g}\|^2_{2,n})+\mu(\|h_*\|^2_{2,n}-\|\hat{h}\|^2_{2,n}) +O(\delta_n\|f'\|_2+\delta_n\|f\|_2+\delta_n^2) \\
             &\leq \sup_{f'\in\mathcal{F}'}   \mathbb{E}\left[2\left\{g_*(A)-Y\right\}f'(C')-  \frac{1}{2}f'(C')^2\right] + \sup_{f\in\mathcal{F}}\mathbb{E}\left[2\left\{h_*(B)-g_*(A)\right\}f(C)-  \frac{1}{2}f(C)^2\right] \\
            &\quad + \mu'(\|g_*\|^2_{2,n}-\|\hat{g}\|^2_{2,n})+\mu(\|h_*\|^2_{2,n}-\|\hat{h}\|^2_{2,n}) +O(\delta_n^2) \\
            &=L'_{1/2}(g_*) + L_{1/2}(h_*,g_*) + \mu'(\|g_*\|^2_{2,n}-\|\hat{g}\|^2_{2,n})+\mu(\|h_*\|^2_{2,n}-\|\hat{h}\|^2_{2,n}) +O(\delta_n^2) \\
            &=2\|S(g_*-g_0)\|_2^2+2\|T(h_*-h_0,g_*-g_0)\|_2^2+ \mu'(\|g_*\|^2_{2,n}-\|\hat{g}\|^2_{2,n})+\mu(\|h_*\|^2_{2,n}-\|\hat{h}\|^2_{2,n}) +O(\delta_n^2).
        \end{align*}
\item By Lemmas~\ref{lemma:innermax_weak} and~\ref{lemma:innermax_weak2}, triangle inequality, and AM-GM inequality,
\begin{align*}
              O(\delta_n \|f'_{\hat{g}}\|_2)
              &= O\{\delta_n \|S(\hat{g}-g_0)\|_2\} = O\{\delta_n \|S(\hat{g}-g_*)\|_2+ \delta_n\|S(g_*-g_0)\|_2\}\\
              &\leq \frac{1}{2}\|S(g_*-g_0)\|^2_2+O\{\delta_n \|S(\hat{g}-g_*)\|_2+ \delta_n^2\}, \\
               O(\delta_n \|f_{\hat{h}}+f_{\hat{g}}\|_2)
              &= O\{\delta_n \|T(\hat{h}-h_0,\hat{g}-g_0)\|_2\} \\
              &= O\{\delta_n \|T(\hat{h}-h_*,\hat{g}-g_*)\|_2+ \delta_n\|T(h_*-h_0,g_*-g_0)\|_2\}\\
              &\leq \frac{1}{2}\|T(h_*-h_0,g_*-g_0)\|^2_2+O\{\delta_n \|T(\hat{h}-h_*,\hat{g}-g_*)\|_2+ \delta_n^2\}.
        \end{align*}
        \item Collecting results,
        \begin{align*}
            &\|S(\hat{g}-g_0)\|_2^2 + \|T(\hat{h}-h_0,\hat{g}-g_0)\|_2^2 \\
            &\leq 2\|S(g_*-g_0)\|_2^2+2\|T(h_*-h_0,g_*-g_0)\|_2^2+ \mu'(\|g_*\|^2_{2,n}-\|\hat{g}\|^2_{2,n})+\mu(\|h_*\|^2_{2,n}-\|\hat{h}\|^2_{2,n}) +O(\delta_n^2) \\
            &\quad +\frac{1}{2}\|S(g_*-g_0)\|^2_2+O\{\delta_n \|S(\hat{g}-g_*)\|_2+ \delta_n^2\} \\
            &\quad +\frac{1}{2}\|T(h_*-h_0,g_*-g_0)\|^2_2+O\{\delta_n \|T(\hat{h}-h_*,\hat{g}-g_*)\|_2+ \delta_n^2\} \\
            &=\frac{5}{2}\|S(g_*-g_0)\|_2^2+\frac{5}{2}\|T(h_*-h_0,g_*-g_0)\|_2^2+ \mu'(\|g_*\|^2_{2,n}-\|\hat{g}\|^2_{2,n})+\mu(\|h_*\|^2_{2,n}-\|\hat{h}\|^2_{2,n}) \\
            &\quad+O(\delta_n \|S(\hat{g}-g_*)\|_2+\delta_n \|T(\hat{h}-h_*,\hat{g}-g_*)\|_2+\delta_n^2) \qedhere.
        \end{align*}
    \end{enumerate}
   
\end{proof}

\begin{lemma}[Regularization bias]\label{lemma:bias3}
    Suppose Assumptions~\ref{assumption:source},~\ref{assumption:source2},~\ref{assumption:posedness},~\ref{assumption:completeness},~\ref{assumption:source3}, and~\ref{assumption:alignment} hold. Then 
    \begin{align*}
    \|h_{\mu',\mu}-h_0\|_2^2&=O\left\{\|w_h\|_2^2 \mu^{\min(\beta_h,1)} +\|w_g\|^2_2\mu^{-1}(\mu')^{\min(\beta_g+1,2)} \right\}\\
  \|T_h(h_{\mu',\mu}-h_0)\|_2^2&=O\left\{\|w_h\|_2^2 \mu^{\min(\beta_h+1,2)}+ \|w_g\|^2_2(\mu')^{\min(\beta_g+1,2)} \right\}\\
  \|g_{\mu',\mu}-g_0\|_2^2&=O\left\{\|w_h\|_2^2 (\mu')^{-1}\mu^{\min(\beta_h+1,2)}+\|w_g'\|^2_2(\mu')^{\min(\beta_g',1)} \right\}\\
   \|S(g_{\mu',\mu}-g_0)\|_2^2&=O\left\{\|w_h\|_2^2 \mu^{\min(\beta_h+1,2)}+\|w_g'\|^2_2(\mu')^{\min(\beta_g'+1,2)}\right\} \\
   \|T_g(g_{\mu',\mu}-g_0)\|_2^2&=O\left\{\|w_h\|_2^2 \mu^{\min(\beta_h+1,2)}+ \|w_g\|^2_2(\mu')^{\min(\beta_g+1,2)} \right\}.
\end{align*}
\end{lemma}

\begin{proof}
Appendix~\ref{sec:bias} provides the proof, which is quite involved and invokes the assumption discussed in Appendix~\ref{sec:extra}. Our argument develops what appear to be new techniques for the nested NPIV problem. 
\end{proof}

\begin{proof}[Proof of Theorem~\ref{theorem:joint}]
   Take $(g_*,h_*)=\{g_{(2\mu',2\mu)},h_{(2\mu',2\mu)}\}$.
    \begin{enumerate}
        \item By Lemmas~\ref{lemma:strong_conv2},~\ref{lemma:main2},~\ref{lemma:bound_regularization} and~\ref{lemma:bound_regularization2}, AM-GM inequality, and $\mu',\mu=O(1)$, we bound
\begin{align*}
     &\|S\{\hat{g}-g_{(2\mu',2\mu)}\}\|_2^2
     +2\mu'\|\hat{g}-g_{(2\mu',2\mu)}\|_2^2+\|T \{\hat{h}-h_{(2\mu',2\mu)},\hat{g}-g_{(2\mu',2\mu)}\}\|_2^2
     +2\mu\|\hat{h}-h_{(2\mu',2\mu)}\|_2^2\\
     &=\|S(\hat{g}-g_0)\|^2_2-\|S\{g_{(2\mu',2\mu)}-g_0\}\|^2_2+2\mu'\{\|\hat{g}\|^2_2-\|g_{(2\mu',2\mu)}\|^2_2\}\\
     &+\|T(\hat{h}-h_0,\hat{g}-g_0)\|^2_2-\|T\{h_{(2\mu',2\mu)}-h_0,g_{(2\mu',2\mu)}-g_0\}\|^2_2 +2\mu\{\|\hat{h}\|^2_2-\|h_{(2\mu',2\mu)}\|^2_2\} \\
     &\leq 
\frac{3}{2}\|S\{g_{(2\mu',2\mu)}-g_0\}\|^2_2+\frac{3}{2}\|T\{h_{(2\mu',2\mu)}-h_0,g_{(2\mu',2\mu)}-g_0\}\|^2_2 \\
&+2\mu'\{\|g_{(2\mu',2\mu)}\|^2_{2,n}-\|\hat{g}\|^2_{2,n}+\|\hat{g}\|^2_2-\|g_{(2\mu',2\mu)}\|^2_2\} \\
&+2\mu\{\|h_{(2\mu',2\mu)}\|^2_{2,n}-\|\hat{h}\|^2_{2,n}+\|\hat{h}\|^2_2-\|h_{(2\mu',2\mu)}\|^2_2\} \\
      &+O[\delta_n\|S\{\hat{g}-g_{(2\mu',2\mu)}\}\|_2+\delta_n\|T\{\hat{h}-h_{(2\mu',2\mu)},\hat{g}-g_{(2\mu',2\mu)}\}\|_2+\delta_n^2] \\
      &= \frac{3}{2}\|S\{g_{(2\mu',2\mu)}-g_0\}\|^2_2+\frac{3}{2}\|T\{h_{(2\mu',2\mu)}-h_0,g_{(2\mu',2\mu)}-g_0\}\|^2_2 \\
&-I_n'+I'-I_n+I+O[\delta_n\|S\{\hat{g}-g_{(2\mu',2\mu)}\}\|_2+\delta_n\|T\{\hat{h}-h_{(2\mu',2\mu)},\hat{g}-g_{(2\mu',2\mu)}\}\|_2+\delta_n^2] \\
      &=\frac{3}{2}\|S\{g_{(2\mu',2\mu)}-g_0\}\|^2_2+\frac{3}{2}\|T\{h_{(2\mu',2\mu)}-h_0,g_{(2\mu',2\mu)}-g_0\}\|^2_2 \\
&+O\bigg[\sqrt{\mu'}^2\delta_n \|\hat{g}-g_{(2\mu',2\mu)}\|_2+\sqrt{\mu}^2\delta_n \|\hat{h}-h_{(2\mu',2\mu)}\|_2\\
&+\delta_n\|S\{\hat{g}-g_{(2\mu',2\mu)}\}\|_2+\delta_n\|T\{\hat{h}-h_{(2\mu',2\mu)},\hat{g}-g_{(2\mu',2\mu)}\}\|_2+\delta_n^2\bigg] \\
&\leq \frac{3}{2}\|S\{g_{(2\mu',2\mu)}-g_0\}\|^2_2+\frac{3}{2}\|T\{h_{(2\mu',2\mu)}-h_0,g_{(2\mu',2\mu)}-g_0\}\|^2_2 \\
&+\frac{\mu'}{2}\|\hat{g}-g_{(2\mu',2\mu)}\|^2_2
+\frac{\mu}{2}\|\hat{h}-h_{(2\mu',2\mu)}\|^2_2 \\
&+\frac{1}{2}\|S\{\hat{g}-g_{(2\mu',2\mu)}\}\|^2_2
+\frac{1}{2}\|T\{\hat{h}-h_{(2\mu',2\mu)},\hat{g}-g_{(2\mu',2\mu)}\}\|^2_2+O(\delta_n^2).
\end{align*}
Rearranging yields
\begin{align*}
     &\frac{1}{2}\|S\{\hat{g}-g_{(2\mu',2\mu)}\}\|_2^2
     +\frac{1}{2}\|T \{\hat{h}-h_{(2\mu',2\mu)},\hat{g}-g_{(2\mu',2\mu)}\}\|_2^2 \\
     &+\frac{3\mu'}{2}\|\hat{g}-g_{(2\mu',2\mu)}\|_2^2
     +\frac{3\mu}{2}\|\hat{h}-h_{(2\mu',2\mu)}\|_2^2\\
     &\leq 
     \frac{3}{2}\|S\{g_{(2\mu',2\mu)}-g_0\}\|^2_2+\frac{3}{2}\|T\{h_{(2\mu',2\mu)}-h_0,g_{(2\mu',2\mu)}-g_0\}\|^2_2+O(\delta_n^2).
\end{align*}
Therefore $\|S\{\hat{g}-g_{(2\mu',2\mu)}\}\|_2^2$ and $\|T \{\hat{h}-h_{(2\mu',2\mu)},\hat{g}-g_{(2\mu',2\mu)}\}\|_2^2$ are each bounded by
\begin{align*}
    & 3\|S\{g_{(2\mu',2\mu)}-g_0\}\|^2_2+3\|T\{h_{(2\mu',2\mu)}-h_0,g_{(2\mu',2\mu)}-g_0\}\|^2_2+O(\delta_n^2);\\
    &\|\hat{g}-g_{(2\mu',2\mu)}\|_2^2\leq \|S\{g_{(2\mu',2\mu)}-g_0\}\|^2_2/\mu'+ \|T\{h_{(2\mu',2\mu)}-h_0,g_{(2\mu',2\mu)}-g_0\}\|^2_2/\mu'
    +O(\delta_n^2/\mu');\\
    &\|\hat{h}-h_{(2\mu',2\mu)}\|_2^2 \leq \mu^{-1} \|S\{g_{(2\mu',2\mu)}-g_0\}\|^2_2+ \mu^{-1} \|T\{h_{(2\mu',2\mu)}-h_0,g_{(2\mu',2\mu)}-g_0\}\|^2_2+O(\mu^{-1}\delta_n^2).
\end{align*}
        \item For the weak metric result, we use triangle inequality and Lemma~\ref{lemma:bias3}. For $R_n=\|w_h\|_2^2 \mu^{\min(\beta_h+1,2)}+\|w_g'\|^2_2(\mu')^{\min(\beta_g'+1,2)}+\|w_g\|^2_2(\mu')^{\min(\beta_g+1,2)} +\delta_n^2$,
        \begin{align*}
            &\|S(\hat{g}-g_0)\|_2^2 \leq 2 \|S\{\hat{g}-g_{(2\mu',2\mu)}\}\|_2^2 + 2 \|S\{g_{(2\mu',2\mu)}-g_0\}\|_2^2 \\
            &\leq  8\|S\{g_{(2\mu',2\mu)}-g_0\}\|^2_2+6\|T\{h_{(2\mu',2\mu)}-h_0,g_{(2\mu',2\mu)}-g_0\}\|^2_2+O(\delta_n^2) \\
            &=O\{\|w_h\|_2^2 \mu^{\min(\beta_h+1,2)}+\|w_g'\|^2_2(\mu')^{\min(\beta_g'+1,2)}+\|w_g\|^2_2(\mu')^{\min(\beta_g+1,2)} +\delta_n^2\}
            =O(R_n),
            \\
              &\|T(\hat{h}-h_0,\hat{g}-g_0)\|_2^2 \leq 2 \|T \{\hat{h}-h_{(2\mu',2\mu)},\hat{g}-g_{(2\mu',2\mu)}\}\|_2^2 + 2 \|T \{h_{(2\mu',2\mu)}-h_0,g_{(2\mu',2\mu)}-g_0\}\|_2^2 \\
            &\leq  6\|S\{g_{(2\mu',2\mu)}-g_0\}\|^2_2+8\|T\{h_{(2\mu',2\mu)}-h_0,g_{(2\mu',2\mu)}-g_0\}\|^2_2+O(\delta_n^2)
            =O(R_n).
        \end{align*}
        \item For the strong metric result, we use triangle inequality and Lemma~\ref{lemma:bias3}:
           \begin{align*}
            &\|\hat{g}-g_0\|_2^2 \leq 2 \|\hat{g}-g_{(2\mu',2\mu)}\|_2^2 + 2 \|g_{(2\mu',2\mu)}-g_0\|_2^2 \\
            &\leq 2(\mu')^{-1} \|S\{g_{(2\mu',2\mu)}-g_0\}\|^2_2+ 2(\mu')^{-1} \|T\{h_{(2\mu',2\mu)}-h_0,g_{(2\mu',2\mu)}-g_0\}\|^2_2 \\
            &+2 \|g_{(2\mu',2\mu)}-g_0\|_2^2
            +O\{(\mu')^{-1}\delta_n^2\} \\
            &=O\left\{(\mu')^{-1} R_n+\|w_h\|_2^2 (\mu')^{-1}\mu^{\min(\beta_h+1,2)}+\|w_g'\|^2_2(\mu')^{\min(\beta_g',1)} +(\mu')^{-1}\delta_n^2\right\} \\
            &=O\left\{(\mu')^{-1} R_n\right\}, \\
             &\|\hat{h}-h_0\|_2^2 \leq 2 \|\hat{h}-h_{(2\mu',2\mu)}\|_2^2 + 2 \|h_{(2\mu',2\mu)}-h_0\|_2^2 \\
            &\leq 2\mu^{-1} \|S\{g_{(2\mu',2\mu)}-g_0\}\|^2_2+ 2\mu^{-1} \|T\{h_{(2\mu',2\mu)}-h_0,g_{(2\mu',2\mu)}-g_0\}\|^2_2 \\
            &+2 \|h_{(2\mu',2\mu)}-h_0\|_2^2
            +O(\mu^{-1}\delta_n^2) \\
            &=O\left\{\mu^{-1} R_n+\|w_h\|_2^2 \mu^{\min(\beta_h,1)}
            +\|w_g\|^2_2\mu^{-1}(\mu')^{\min(\beta_g+1,2)}
            +\mu^{-1}\delta_n^2\right\}
            =O(\mu^{-1} R_n).
        \end{align*}
        \item In particular, when $\mu=\mu'$, $\|\hat{g}-g_0\|_2^2=O(\mu^{-1}R_n) $ and $\|\hat{h}-h_0\|_2^2=O(\mu^{-1}R_n)$ where
        $$
            R_n=\|w_h\|_2^2 \mu^{\min(\beta_h+1,2)}+\|w_g'\|^2_2\mu^{\min(\beta_g'+1,2)}+\|w_g\|^2_2\mu^{\min(\beta_g+1,2)} +\delta_n^2. \qedhere
       $$
    \end{enumerate}

\end{proof}

\section{New techniques to control bias}\label{sec:bias}

We control the bias of Algorithm~\ref{algo:sequential} using standard techniques. However, to control the bias of Algorithm~\ref{algo:simultaneous}, we develop what appear to be new techniques. To lighten notation, we abbreviate $\mu I=\mu$ and $\mu' I=\mu'$ when it is clear from context.

\subsection{Algorithm~\ref{algo:sequential}}

\begin{lemma}[Bias algebra; c.f. Lemma 3 of \cite{bennett2023source}]\label{lemma:algebra}
    For any regularization parameter value $\mu>0$, source parameter value $\beta>0$,  and singular values $(\sigma_j)$ satisfying $\sup_j \sigma_j\leq 1$, we have that
   $\mu^2\sup_j \frac{\sigma_j^{2\beta}}{(\sigma_j^2+\mu)^2} \leq \mu^{\min(\beta,2)}$ 
    and $\mu^2\sup_j \frac{\sigma_j^{2(\beta+1)}}{(\sigma_j^2+\mu)^2} \leq \mu^{\min(\beta+1,2)}$.
\end{lemma}

\begin{proof}
We state the proof for completeness.
\begin{enumerate}
    \item For the first result, if $\beta \geq 2$ then it suffices to show $\sup_j \frac{\sigma_j^{2\beta}}{(\sigma_j^2+\mu)^2} \leq 1$. Clearly $\frac{\sigma_j^{2\beta}}{(\sigma_j^2+\mu)^2} \leq \frac{\sigma_j^{2\beta}}{\sigma_j^4}=\sigma_j^{2(\beta-2)}\leq 1$ since $\beta\geq 2$ and $\sup_j \sigma_j\leq 1$. If $\beta<2$ then it suffices to show $f(x)=\mu^2 \frac{x^{\beta}}{(x+\mu)^2}\leq \mu^{\beta}$. The first order condition yields $x_*=\frac{\beta\mu}{2-\beta}$ and $f(x_*)=\frac{1}{4}(2-\beta)^{2-\beta}\beta^{\beta}\mu^{\beta}\leq \mu^{\beta}$ since $\beta<2$.
    \item For the second result, if $\beta+1 \geq 2$ then it suffices to show $\sup_j \frac{\sigma_j^{2(\beta+1)}}{(\sigma_j^2+\mu)^2} \leq 1$. Clearly $\frac{\sigma_j^{2(\beta+1)}}{(\sigma_j^2+\mu)^2} \leq \frac{\sigma_j^{2(\beta+1)}}{\sigma_j^4}=\sigma_j^{2(\beta+1-2)}\leq 1$ since $\beta \geq 1$ and $\sup_j \sigma_j\leq 1$. If $\beta+1<2$ then it suffices to show $f(x)=\mu^2 \frac{x^{(\beta+1)}}{(x+\mu)^2}\leq \mu^{\beta+1}$. The first order condition yields $x_*=\frac{(\beta+1)\mu}{1-\beta}$ and $f(x_*)=\frac{1}{4}(1-\beta)^{1-\beta}(\beta+1)^{\beta+1}\mu^{\beta+1}\leq \mu^{\beta+1}$ since $\beta<1$. \qedhere
\end{enumerate}
\end{proof}

\begin{proof}[Proof of Lemma~\ref{lemma:bias}]
    We state the proof for completeness. To lighten notation, we study $h_{\mu}=h_*$ and abbreviate $T_h=T$. 
    \begin{enumerate}
        \item  By Lemma~\ref{lemma:innermax_weak},
    $
    h_*=\argmin_{h\in\mathcal{H}} L_1(h,g_0)+\mu\|h\|_2^2=\|T(h-h_0)\|^2_2+\mu\|h\|_2^2.
    $
    Taking the Gateaux derivative, the first order condition yields $2T^*T(h_*-h_0)+2\mu h_*=0$ and hence $h_*=(T^*T+\mu)^{-1}(T^*T)h_0$. Hence
    $$
    h_*-h_0=(T^*T+\mu)^{-1}\{(T^*T)-(T^*T+\mu)\}h_0=-\mu(T^*T+\mu)^{-1} h_0.
    $$
   Using $h_0=(T^*T)^{\beta/2}w_0$,
    \begin{align*}
        \|h_*-h_0\|_2^2&=\|-\mu(T^*T+\mu)^{-1}  (T^*T)^{\beta/2}w_0\|^2_2\leq \mu^2 \|(T^*T+\mu)^{-1}  (T^*T)^{\beta/2}\|^2_{\op}\cdot \|w_0\|_2^2, \\
        \|T(h_*-h_0)\|_2^2&=\|-\mu T(T^*T+\mu)^{-1}  (T^*T)^{\beta/2}w_0\|^2_2\leq \mu^2 \|T(T^*T+\mu)^{-1}  (T^*T)^{\beta/2}\|^2_{\op}\cdot \|w_0\|_2^2.
    \end{align*}
        \item By Lemma~\ref{lemma:algebra}, $\mu^2 \|(T^*T+\mu)^{-1}  (T^*T)^{\beta/2}\|^2_{\op}=\mu^2\sup_j \frac{\sigma_j^{2\beta}}{(\sigma_j^2+\mu)^2} \leq \mu^{\min(\beta,2)}$ and $\mu^2 \|T(T^*T+\mu)^{-1}  (T^*T)^{\beta/2}\|^2_{\op}=\mu^2\sup_j \frac{\sigma_j^{2(\beta+1)}}{(\sigma_j^2+\mu)^2} \leq \mu^{\min(\beta+1,2)}$. \qedhere
    \end{enumerate}
\end{proof}

\subsection{Algorithm~\ref{algo:simultaneous}}

To lighten notation, we write $(g_*,h_*)=(g_{\mu',\mu},h_{\mu',\mu})$. We also abbreviate $T_h=T(h,0)$ and $T_g=T(0,g)$, so that $T(h,g)=T_h(h)+T_g(g)$.

\begin{lemma}[Rewriting regularization bias]\label{lemma:bias2}
We have
   $$
     g_*-g_0=(I-A_{G}A_H)^{-1}(\tilde{g}-A_G\tilde{h}),\quad 
    h_*-h_0=(I-A_{H}A_G)^{-1}(\tilde{h}-A_H\tilde{g})
$$
where 
\begin{align*}
    &\tilde{g}=-\mu'(S^*S+T_g^*T_g+\mu')^{-1}g_0,\quad 
    \tilde{h}=-\mu(T_h^*T_h+\mu)^{-1} h_0,\\
    &A_G=(S^*S+T_g^*T_g+\mu')^{-1}T_g^*T_h,\quad 
  A_H=(T_h^*T_h+\mu)^{-1}T_h^*T_g.\end{align*}
If Assumptions~\ref{assumption:source},~\ref{assumption:source2}, and~\ref{assumption:source3} hold then
\begin{align*}
    \tilde{g}&=-\mu'(S^*S+T_g^*T_g+\mu')^{-1}(S^*S)^{\beta'_g/2}w_g'
    =-\mu'(S^*S+T_g^*T_g+\mu')^{-1}(T_g^*T_g)^{\beta_g/2}w_g\\
    \tilde{h}&=-\mu(T_h^*T_h+\mu)^{-1} (T_h^*T_h)^{\beta_h/2}w_h.
\end{align*}
\end{lemma}

\begin{proof}
    By Lemmas~\ref{lemma:innermax_weak} and~\ref{lemma:innermax_weak2},
    \begin{align*}
         &(g_*,h_*)=\argmin_{g\in\mathcal{G},h\in\mathcal{H}} L'_1(g)+L_1(h,g)+\mu'\|g\|_2^2+\mu\|h\|_2^2 \\
         &=\|S(g-g_0)\|^2_2+\|T(h-h_0,g-g_0)\|^2_2+\mu'\|g\|_2^2+\mu\|h\|_2^2 \\
         &=\|S(g-g_0)\|^2_2+\|T_h(h-h_0)\|^2_2+\|T_g(g-g_0)\|^2_2+2\langle T_h(h-h_0),T_g(g-g_0) \rangle +\mu'\|g\|_2^2+\mu\|h\|_2^2.
    \end{align*}
Taking the Gateaux derivative, the first order conditions yield, after dividing by two,
$$
    (S^*S+T_g^*T_g)(g_*-g_0)+T_g^*T_h(h_*-h_0)+\mu'g_*=0,\quad 
    T_h^*T_h(h_*-h_0)+T_h^*T_g(g_*-g_0)+\mu h_*=0.
$$
Rearranging each expression,
\begin{align*}
    g_*&=(S^*S+T_g^*T_g+\mu')^{-1}\{(S^*S+T_g^*T_g) g_0-T_g^*T_h(h_*-h_0)\},  \\
    h_*&=(T_h^*T_h+\mu)^{-1}\{T_h^*T_h h_0-T_h^*T_g(g_*-g_0)\};\\
  g_*-g_0&=(S^*S+T_g^*T_g+\mu')^{-1}\{-\mu' g_0-T_g^*T_h(h_*-h_0)\}=\tilde{g}-A_{G}(h_*-h_0), \\
    h_*-h_0&=(T_h^*T_h+\mu)^{-1}\{-\mu  h_0-T_h^*T_g(g_*-g_0)\}=\tilde{h}-A_{H}(g_*-g_0).
\end{align*}
Combining the final two expressions yields the former the result. The latter result is immediate.
\end{proof}

For operators $A$ and $B$, we write $A\geq B$ in the Loewner sense when $A-B$ is positive semidefinite. 

\begin{lemma}[Rewriting pre-factors]\label{lemma:pre}
We have
\begin{align*}
(I-A_GA_H)^{-1}=(S^*S+R+\mu')^{-1}(S^*S+T_g^*T_g+\mu') \\
    (I-A_HA_G)^{-1}=\{T_h^*(I-Q)T_h+\mu\}^{-1}(T_h^*T_h+\mu).
\end{align*}
for some $R\geq 0$ and some $Q\geq 0$ satisfying $I-Q\geq cI$ with $c=\|(S^*S+T_g^*T_g+\mu')^{1/2}(S^*S+\mu')^{-1} (S^*S+T_g^*T_g+\mu')^{1/2}\|_{\op}^{-1}$.
\end{lemma}

\begin{proof}
   Since $B^{-1}=(AB)^{-1}A$, by Lemma~\ref{lemma:pass},
    \begin{align*}
        (I-A_GA_H)^{-1}
        &=\{I-(S^*S+T_g^*T_g+\mu')^{-1}T_g^*T_h(T_h^*T_h+\mu)^{-1}T_h^*T_g\}^{-1}\\
        &=\{S^*S+T_g^*T_g+\mu'-T_g^*T_h(T_h^*T_h+\mu)^{-1}T_h^*T_g\}^{-1}(S^*S+T_g^*T_g+\mu') \\
        &=\{S^*S+T_g^* I T_g+\mu'-T_g^*T_hT_h^*(T_hT_h^*+\mu)^{-1}T_g\}^{-1}(S^*S+T_g^*T_g+\mu') \\
        &=\{S^*S+T_g^*\mu (T_hT_h^*+\mu)^{-1}T_g+\mu'\}^{-1}(S^*S+T_g^*T_g+\mu').
    \end{align*}
    We take $R=T_g^*\mu (T_hT_h^*+\mu)^{-1}T_g\geq 0$. Similarly,
    \begin{align*}
        (I-A_HA_G)^{-1}
        &=\{I-(T_h^*T_h+\mu)^{-1}T_h^*T_g (S^*S+T_g^*T_g+\mu')^{-1}T_g^*T_h\}^{-1} \\
        &=\{T_h^*T_h+\mu-T_h^*T_g (S^*S+T_g^*T_g+\mu')^{-1}T_g^*T_h\}^{-1}(T_h^*T_h+\mu)\\
        &=\{T_h^*(I-Q)T_h+\mu\}^{-1}(T_h^*T_h+\mu)
    \end{align*}
    where $Q=T_g (S^*S+T_g^*T_g+\mu')^{-1}T_g^*\geq 0$. 
    
    Finally, we characterize $c$. By \citet[Theorem IX.4.2]{bhatia2013matrix}, $\sigma_j(A^*B)\leq \frac{1}{2}\sigma_j(AA^*+BB^*)$. Taking $A^*=T_g (S^*S+T_g^*T_g+\mu')^{-1/2}$ and $B=(S^*S+T_g^*T_g+\mu')^{-1/2}T_g^*$, we have that for each $j$,
 \begin{align*}
     \sigma_j(Q)
     &\leq \sigma_j\{(S^*S+T_g^*T_g+\mu')^{-1/2}T_g^*T_g (S^*S+T_g^*T_g+\mu')^{-1/2}\} \\
     &=\sigma_j\{I-(S^*S+T_g^*T_g+\mu')^{-1/2}(S^*S+\mu') (S^*S+T_g^*T_g+\mu')^{-1/2}\}\\
     &=\sigma_j[I-\{(S^*S+T_g^*T_g+\mu')^{1/2}(S^*S+\mu')^{-1} (S^*S+T_g^*T_g+\mu')^{1/2}\}^{-1}] \\
     &\leq \sigma_1[I-\{(S^*S+T_g^*T_g+\mu')^{1/2}(S^*S+\mu')^{-1} (S^*S+T_g^*T_g+\mu')^{1/2}\}^{-1}] \\
     &=1-\frac{1}{\sigma_1\{(S^*S+T_g^*T_g+\mu')^{1/2}(S^*S+\mu')^{-1} (S^*S+T_g^*T_g+\mu')^{1/2}\}} 
     =1-c
 \end{align*}
for $c=\|(S^*S+T_g^*T_g+\mu')^{1/2}(S^*S+\mu')^{-1} (S^*S+T_g^*T_g+\mu')^{1/2}\|_{\op}^{-1}$. Therefore $Q\leq (1-c)I$ in the Loewner sense.
\end{proof}

\begin{assumption}[High level conditions]\label{assumption:high}
    Suppose that the following conditions hold:
    \begin{enumerate}
    \item $\|T_g(S^*S+T_g^*T_g+\mu')^{-1}T_g^*\|_{\op}=O(1)$;
    \item  $\|T_g(S^*S+T_g^*T_g+\mu')^{-1}\|_{\op}=O\{(\mu')^{-1/2}\}$;
    \item $\|T_h\{T_h^*(I-Q)T_h+\mu\}^{-1}T_h^*\|_{\op}=O(1)$; 
    \item $\|T_h\{T_h^*(I-Q)T_h+\mu\}^{-1}\|_{\op} =O(\mu^{-1/2})$; 
    \item $\|\{T_h^*(I-Q)T_h+\mu\}^{-1}T_h^*\|_{\op} =O (\mu^{-1/2})$; 
    \item  $\|\{T_h^*(I-Q)T_h+\mu\}^{-1}\|_{\op} =O(\mu^{-1})$;   
    \item $\|S(S^*S+R+\mu')^{-1}S^*\|_{\op}=O(1)$; 
    \item $\|T_g(S^*S+R+\mu')^{-1}T_g^*\|_{\op}=O(1)$; 
    \item $\|S(S^*S+R+\mu')^{-1}\|_{\op}=O\{(\mu')^{-1/2}\}$; 
    \item $\|(S^*S+R+\mu')^{-1}S^*\|_{\op}=O\{(\mu')^{-1/2}\}$; 
    \item $\|T_g(S^*S+R+\mu')^{-1}\|_{\op}=O\{(\mu')^{-1/2}\}$; 
    \item $\|(S^*S+R+\mu')^{-1}T_g^*\|_{\op}=O\{(\mu')^{-1/2}\}$; 
    \item $\|S(S^*S+R+\mu')^{-1}T_g^*\|_{\op}=O(1)$; 
    \item $\|(S^*S+R+\mu')^{-1}\|_{\op} =O\{(\mu')^{-1}\}$. 
\end{enumerate}
\end{assumption}

\begin{lemma}\label{lemma:high}
      Suppose Assumptions~\ref{assumption:source},~\ref{assumption:source2},~\ref{assumption:source3}, and~\ref{assumption:high} hold. Then 
    \begin{align*}
    \|h_*-h_0\|_2^2&=O\left\{\|w_h\|_2^2 \mu^{\min(\beta_h,1)} +\|w_g\|^2_2\mu^{-1}(\mu')^{\min(\beta_g+1,2)} \right\}\\
  \|T_h(h_*-h_0)\|_2^2&=O\left\{\|w_h\|_2^2 \mu^{\min(\beta_h+1,2)}+ \|w_g\|^2_2(\mu')^{\min(\beta_g+1,2)} \right\}\\
  \|g_*-g_0\|_2^2&=O\left\{\|w_h\|_2^2 (\mu')^{-1}\mu^{\min(\beta_h+1,2)}+\|w_g'\|^2_2(\mu')^{\min(\beta_g',1)} \right\}\\
   \|S(g_*-g_0)\|_2^2&=O\left\{\|w_h\|_2^2 \mu^{\min(\beta_h+1,2)}+\|w_g'\|^2_2(\mu')^{\min(\beta_g'+1,2)}\right\} \\
   \|T_g(g_*-g_0)\|_2^2&=O\left\{\|w_h\|_2^2 \mu^{\min(\beta_h+1,2)}+ \|w_g\|^2_2(\mu')^{\min(\beta_g+1,2)} \right\}.
\end{align*}
\end{lemma}

\begin{proof}
To lighten notation, we abbreviate $\|\cdot\|=\|\cdot\|_{\op}$. 
\begin{enumerate}
    \item  By Lemma~\ref{lemma:bias2},  $\tilde{g}=-G'w_g'=-Gw_g$ and $\tilde{h}=-Hw_h$ where
    \begin{align*}
        &G'=\mu'(S^*S+T_g^*T_g+\mu')^{-1}(S^*S)^{\beta'_g/2},\quad G=\mu'(S^*S+T_g^*T_g+\mu')^{-1}(T_g^*T_g)^{\beta_g/2} \\
   &H=\mu(T_h^*T_h+\mu)^{-1} (T_h^*T_h)^{\beta_h/2}.
    \end{align*}

 \item To bound $\|h_*-h_0\|^2_2$, it suffices to control $\|(I-A_HA_G)^{-1}H\|$ and $\|(I-A_HA_G)^{-1}A_HG\|$ by Lemma~\ref{lemma:bias2} and the characterization above. By Lemma~\ref{lemma:pre} and triangle inequality,
\begin{align*}
&\|(I-A_HA_G)^{-1}H\|
=\|\{T_h^*(I-Q)T_h+\mu\}^{-1}\mu(T_h^*T_h)^{\beta_h/2}\| \\
&=\|\{T_h^*(I-Q)T_h+\mu\}^{-1}(T_h^*T_h+\mu)(T_h^*T_h+\mu)^{-1}\mu(T_h^*T_h)^{\beta_h/2}\| \\
&\leq 
\|\{T_h^*(I-Q)T_h+\mu\}^{-1}T_h^*T_h(T_h^*T_h+\mu)^{-1}\mu(T_h^*T_h)^{\beta_h/2}\| \\
&\quad + \|\{T_h^*(I-Q)T_h+\mu\}^{-1}\mu(T_h^*T_h+\mu)^{-1}\mu(T_h^*T_h)^{\beta_h/2}\| \\
&\leq 
\|\{T_h^*(I-Q)T_h+\mu\}^{-1}T_h^*\|\|T_h(T_h^*T_h+\mu)^{-1}\mu(T_h^*T_h)^{\beta_h/2}\| \\
&\quad + \mu\|\{T_h^*(I-Q)T_h+\mu\}^{-1}\|\|(T_h^*T_h+\mu)^{-1}\mu(T_h^*T_h)^{\beta_h/2}\| \\
&=O\left\{\mu^{-1/2} \|T_h(T_h^*T_h+\mu)^{-1}\mu(T_h^*T_h)^{\beta_h/2}\|+ \|(T_h^*T_h+\mu)^{-1}\mu(T_h^*T_h)^{\beta_h/2}\| \right\}
\end{align*}
where in the last step we use Assumptions~\ref{assumption:high}.5 and~\ref{assumption:high}.6. After squaring, Lemma~\ref{lemma:algebra} gives the first term in the bound.

Next we turn to
\begin{align*}
&\|(I-A_HA_G)^{-1}A_HG\|
=\|\{T_h^*(I-Q)T_h+\mu\}^{-1}T_h^*T_g\mu'(S^*S+T_g^*T_g+\mu')^{-1}(T_g^*T_g)^{\beta_g/2}\| \\
&\leq \|\{T_h^*(I-Q)T_h+\mu\}^{-1}T_h^*\|\|T_g\mu'(S^*S+T_g^*T_g+\mu')^{-1}(T_g^*T_g)^{\beta_g/2}\|.
\end{align*}
The former factor is $O(\mu^{-1/2})$ by Assumption~\ref{assumption:high}.5. We bound the latter factor as
\begin{align*}
    &\|T_g(S^*S+T_g^*T_g+\mu')^{-1}\mu'(T_g^*T_g)^{\beta_g/2}\| \\
    &= \|T_g(S^*S+T_g^*T_g+\mu')^{-1}(T_g^*T_g+\mu')(T_g^*T_g+\mu')^{-1}\mu'(T_g^*T_g)^{\beta_g/2}\| \\
    &\leq \|T_g(S^*S+T_g^*T_g+\mu')^{-1}T_g^*T_g(T_g^*T_g+\mu')^{-1}\mu'(T_g^*T_g)^{\beta_g/2}\|
    \\
    &\quad +
    \|T_g(S^*S+T_g^*T_g+\mu')^{-1}\mu'(T_g^*T_g+\mu')^{-1}\mu'(T_g^*T_g)^{\beta_g/2}\| \\
    &\leq 
    \|T_g(S^*S+T_g^*T_g+\mu')^{-1}T_g^*\|\|T_g(T_g^*T_g+\mu')^{-1}\mu'(T_g^*T_g)^{\beta_g/2}\|
    \\
    &\quad +
    \mu'\|T_g(S^*S+T_g^*T_g+\mu')^{-1}\|\|(T_g^*T_g+\mu')^{-1}\mu'(T_g^*T_g)^{\beta_g/2}\|\\
    &=O\left\{\|T_g(T_g^*T_g+\mu')^{-1}\mu' (T_g^*T_g)^{\beta_g/2}\|
    +(\mu')^{1/2}\|(T_g^*T_g+\mu')^{-1}\mu'(T_g^*T_g)^{\beta_g/2}\|
    \right\}
\end{align*}
where in the last step we use Assumptions~\ref{assumption:high}.1 and~\ref{assumption:high}.2.  After squaring, Lemma~\ref{lemma:algebra} gives the second term in the bound.
    \item The argument for $ \|T_h(h_*-h_0)\|_2^2$ is similar. First we control
       \begin{align*}
&\|T_h(I-A_HA_G)^{-1}H\|
=\|T_h\{T_h^*(I-Q)T_h+\mu\}^{-1}\mu(T_h^*T_h)^{\beta_h/2}\| \\
&=\|T_h\{T_h^*(I-Q)T_h+\mu\}^{-1}(T_h^*T_h+\mu)(T_h^*T_h+\mu)^{-1}\mu(T_h^*T_h)^{\beta_h/2}\| \\
&\leq 
\|T_h\{T_h^*(I-Q)T_h+\mu\}^{-1}T_h^*T_h(T_h^*T_h+\mu)^{-1}\mu(T_h^*T_h)^{\beta_h/2}\|\\
&\quad +
\|T_h\{T_h^*(I-Q)T_h+\mu\}^{-1}\mu(T_h^*T_h+\mu)^{-1}\mu(T_h^*T_h)^{\beta_h/2}\| \\
&\leq \|T_h\{T_h^*(I-Q)T_h+\mu\}^{-1}T_h^*\|\|T_h(T_h^*T_h+\mu)^{-1}\mu(T_h^*T_h)^{\beta_h/2}\|\\
&\quad +
\mu\|T_h\{T_h^*(I-Q)T_h+\mu\}^{-1}\|\|(T_h^*T_h+\mu)^{-1}\mu(T_h^*T_h)^{\beta_h/2}\| \\
&=O\left\{\|T_h(T_h^*T_h+\mu)^{-1}\mu(T_h^*T_h)^{\beta_h/2}\|+\mu^{1/2}\|(T_h^*T_h+\mu)^{-1}\mu(T_h^*T_h)^{\beta_h/2}\|\right\}
\end{align*}
where in the last step we use Assumptions~\ref{assumption:high}.3 and~\ref{assumption:high}.4. After squaring, Lemma~\ref{lemma:algebra} gives the first term in the bound.  
Next we study
\begin{align*}
&\|T_h(I-A_HA_G)^{-1}A_HG\|\\
&=\|T_h\{T_h^*(I-Q)T_h+\mu\}^{-1}T_h^*T_g\mu'(S^*S+T_g^*T_g+\mu')^{-1}(T_g^*T_g)^{\beta_g/2}\| \\
&\leq \|T_h\{T_h^*(I-Q)T_h+\mu\}^{-1}T_h^*\|\|T_g(S^*S+T_g^*T_g+\mu')^{-1}\mu'(T_g^*T_g)^{\beta_g/2}\|.
\end{align*}
The former factor is $O(1)$ by Assumption~\ref{assumption:high}.3. We bound the latter factor as above, which gives the second term in the bound.
    \item To bound $\|g_*-g_0\|^2_2$,  it suffices to control $\|(I-A_GA_H)^{-1}G'\|$ and $\|(I-A_GA_H)^{-1}A_GH\|$ by Lemma~\ref{lemma:bias2} and the characterization above. By Lemma~\ref{lemma:pre} and triangle inequality,
    \begin{align*}
&\|(I-A_GA_H)^{-1}G'\|=\|(S^*S+R+\mu')^{-1}\mu'(S^*S)^{\beta'_g/2}\| \\
&=\|(S^*S+R+\mu')^{-1}(S^*S+\mu')(S^*S+\mu')^{-1}\mu'(S^*S)^{\beta'_g/2}\| \\
&\leq \|(S^*S+R+\mu')^{-1}S^*S(S^*S+\mu')^{-1}\mu'(S^*S)^{\beta'_g/2}\| \\
&\quad + \|(S^*S+R+\mu')^{-1}\mu'(S^*S+\mu')^{-1}\mu'(S^*S)^{\beta'_g/2}\| \\
&\leq \|(S^*S+R+\mu')^{-1}S^*\|\|S(S^*S+\mu')^{-1}\mu'(S^*S)^{\beta'_g/2}\| \\
&\quad + \mu'\|(S^*S+R+\mu')^{-1}\|\|(S^*S+\mu')^{-1}\mu'(S^*S)^{\beta'_g/2}\| \\
&=O\left\{(\mu')^{-1/2}\|S(S^*S+\mu')^{-1}\mu'(S^*S)^{\beta'_g/2}\|
+\|(S^*S+\mu')^{-1}\mu'(S^*S)^{\beta'_g/2}\|
\right\}
\end{align*}
where in the last step we use Assumptions~\ref{assumption:high}.10 and~\ref{assumption:high}.14. After squaring, Lemma~\ref{lemma:algebra} gives the second term in the bound.  

Next we turn to
\begin{align*}
&\|(I-A_GA_H)^{-1}A_GH\| 
=\|(S^*S+R+\mu')^{-1}T_g^*T_h \mu(T_h^*T_h+\mu)^{-1} (T_h^*T_h)^{\beta_h/2}\| \\
&\leq \|(S^*S+R+\mu')^{-1}T_g^*\|\|T_h \mu(T_h^*T_h+\mu)^{-1} (T_h^*T_h)^{\beta_h/2}\| \\
&=O\left\{\right (\mu')^{-1/2} \|T_h \mu(T_h^*T_h+\mu)^{-1} (T_h^*T_h)^{\beta_h/2}\| \}
    \end{align*}
    where in the last line we use Assumption~\ref{assumption:high}.12. After squaring, Lemma~\ref{lemma:algebra} gives the first term in the bound.  
    
    \item The argument for $ \|S(g_*-g_0)\|_2^2$ is similar. First we control
       \begin{align*}
  &\|S(I-A_GA_H)^{-1}G'\|=\|S(S^*S+R+\mu')^{-1}\mu'(S^*S)^{\beta'_g/2}\| \\
&=\|S(S^*S+R+\mu')^{-1}(S^*S+\mu')(S^*S+\mu')^{-1}\mu'(S^*S)^{\beta'_g/2}\| \\
&\leq 
\|S(S^*S+R+\mu')^{-1}S^*S(S^*S+\mu')^{-1}\mu'(S^*S)^{\beta'_g/2}\| \\
&\quad + \|S(S^*S+R+\mu')^{-1}\mu'(S^*S+\mu')^{-1}\mu'(S^*S)^{\beta'_g/2}\| \\
&\leq 
\|S(S^*S+R+\mu')^{-1}S^*\|\|S(S^*S+\mu')^{-1}\mu'(S^*S)^{\beta'_g/2}\|
\\
&\quad +
\mu'\|S(S^*S+R+\mu')^{-1}\|\|(S^*S+\mu')^{-1}\mu'(S^*S)^{\beta'_g/2}\| \\
&=O\left\{\|S(S^*S+\mu')^{-1}\mu'(S^*S)^{\beta'_g/2}\|
+(\mu')^{1/2}\|(S^*S+\mu')^{-1}\mu'(S^*S)^{\beta'_g/2}\|
\right\}
\end{align*}
where in the last line we use Assumptions~\ref{assumption:high}.7 and~\ref{assumption:high}.9. After squaring, Lemma~\ref{lemma:algebra} gives the second term in the bound.  
Next we turn to
\begin{align*}
&\|S(I-A_GA_H)^{-1}A_GH\|
=\|S(S^*S+R+\mu')^{-1}T_g^*T_h \mu(T_h^*T_h+\mu)^{-1} (T_h^*T_h)^{\beta_h/2}\| \\
&\leq \|S(S^*S+R+\mu')^{-1}T_g^*\|\|T_h \mu(T_h^*T_h+\mu)^{-1} (T_h^*T_h)^{\beta_h/2}\| \\
&=O\{\|T_h \mu(T_h^*T_h+\mu)^{-1} (T_h^*T_h)^{\beta_h/2}\|\}
 \end{align*}
 where in the last line we use Assumption~\ref{assumption:high}.13. After squaring, Lemma~\ref{lemma:algebra} gives the first term in the bound.  
 
 \item The argument for $\|T_g(g_*-g_0)\|_2^2$ is also similar. First we control
       \begin{align*}
       &\|T_g(I-A_GA_H)^{-1}G\|=
       \|T_g(S^*S+R+\mu')^{-1}\mu'(T_g^*T_g)^{\beta_g/2}\|
       \\
       &=\|T_g(S^*S+R+\mu')^{-1}(T_g^*T_g+\mu')(T_g^*T_g+\mu')^{-1}\mu'(T_g^*T_g)^{\beta_g/2}\| \\
       &\leq 
       \|T_g(S^*S+R+\mu')^{-1}T_g^*T_g(T_g^*T_g+\mu')^{-1}\mu'(T_g^*T_g)^{\beta_g/2}\| \\
       &\quad +
       \|T_g(S^*S+R+\mu')^{-1}\mu'(T_g^*T_g+\mu')^{-1}\mu'(T_g^*T_g)^{\beta_g/2}\| \\
       &\leq 
       \|T_g(S^*S+R+\mu')^{-1}T_g^*\|\|T_g(T_g^*T_g+\mu')^{-1}\mu'(T_g^*T_g)^{\beta_g/2}\|\\
       &\quad +
       \mu'\|T_g(S^*S+R+\mu')^{-1}\|\|(T_g^*T_g+\mu')^{-1}\mu'(T_g^*T_g)^{\beta_g/2}\|\\
       &=O\left\{\|T_g(T_g^*T_g+\mu')^{-1}\mu'(T_g^*T_g)^{\beta_g/2}\|+(\mu')^{1/2}\|(T_g^*T_g+\mu')^{-1}\mu'(T_g^*T_g)^{\beta_g/2}\|\right\}
       \end{align*}
       where in the last line we use Assumptions~\ref{assumption:high}.8 and~\ref{assumption:high}.11. After squaring, Lemma~\ref{lemma:algebra} gives the second term in the bound.  
       
       Finally, we turn to
\begin{align*}
       &\|T_g(I-A_GA_H)^{-1}A_GH\|
=\|T_g(S^*S+R+\mu')^{-1}T_g^*T_h \mu(T_h^*T_h+\mu)^{-1} (T_h^*T_h)^{\beta_h/2}\| \\
&\leq \|T_g(S^*S+R+\mu')^{-1}T_g^*\|\|T_h \mu(T_h^*T_h+\mu)^{-1} (T_h^*T_h)^{\beta_h/2}\| \\
&=O\left\{\|T_h \mu(T_h^*T_h+\mu)^{-1} (T_h^*T_h)^{\beta_h/2}\|\right\}
    \end{align*}
    where in the last line we use Assumption~\ref{assumption:high}.8. After squaring, Lemma~\ref{lemma:algebra} gives the first term in the bound.  \qedhere
\end{enumerate}
  
\end{proof}

\begin{proof}[Proof of Lemma~\ref{lemma:bias3}]
    The result is an immediate consequence of Lemmas~\ref{lemma:high} and~\ref{lemma:verify}. The proof of Lemma~\ref{lemma:verify} is given below.
\end{proof}

\subsection{Useful operator theory}

\begin{lemma}[Pass through]\label{lemma:pass}
    For any operators $U$ and $V$ and any scalar $\mu$, $(UV+\mu)^{-1}U=U(VU+\mu)^{-1}$ when the products are well defined and the inverses exist.
\end{lemma}
\begin{proof}
    Write $U(VU+\mu)=(UV+\mu)U$.
\end{proof}

\begin{lemma}[Discard denominator]\label{lemma:denom}
    Suppose $A$ is positive definite and $B$ is positive semidefinite. Then for any operator $C$,
    $
    \|C^*(A+B)^{-1}C\|_{\op}\leq \|C^*A^{-1}C\|_{\op}$.
\end{lemma}

\begin{proof}
    In the Loewner sense, $A+B\geq A >0$, so $0<(A+B)^{-1}\leq A^{-1}$ by \citet[Proposition V.1.6]{bhatia2013matrix}. Therefore $A^{-1}-(A+B)^{-1} \geq 0$ and in particular $C^*\{A^{-1}-(A+B)^{-1}\}C \geq 0$ using \citet[Lemma V.1.5]{bhatia2013matrix}.
\end{proof}

\begin{lemma}[Relative completeness]\label{lemma:completeness}
   Suppose $A$ and $B$ are bounded linear operators between Hilbert spaces: $A\in \mathcal{L}(X,Y)$, and $B\in\mathcal{L}(X,Z)$. If $A(x)=0$ implies $B(x)=0$, and $A$ has closed range, then there exists a linear operator $M$ such that $B=MA$ and $\|M\|_{\op}=O(1)$.
\end{lemma}
\begin{proof}
Decompose $X=\textsc{null}(A) \oplus\{\textsc{null}(A)\}^{\perp}$, and consider the restricted operator $A_0: \{\textsc{null}(A)\}^{\perp} \rightarrow  \textsc{range}(A)$ by $ A_0(x)=A(x)$. This is a bijective operator between Banach spaces, since we suppose that $\textsc{range}(A)$ is closed. An application of the bounded inverse theorem states that the inverse $
A_0^{-1}:  \textsc{range}(A) \rightarrow \{\textsc{null}(A)\}^{\perp}$ is bounded. In particular, there exists $C>0$ such that
$
\|A_0^{-1}(y)\| \leq C\|y\|$, for all $y \in \textsc{range}(A).
$
Equivalently, $A$ is bounded below in $\{\textsc{null}(A)\}^{\perp}$: $\|x\| \leq C\|A(x)\|$ for all $x\in \{\textsc{null}(A)\}^{\perp}$. Thus we can write $\gamma:=\inf_{{x \in\{\textsc{null}(A)\}^{\perp} \;,\;\|x\|=1}}\|A(x)\|>0$.

This enables us to construct a bounded left inverse of $A$ in its range. Define $L: \textsc{range}(A) \rightarrow\{\textsc{null}(A)\}^{\perp}$ as $L\{A(x)\}=x$ with $x \in\{\textsc{null}(A)\}^{\perp}$. This is a well-defined and bounded operator:
if $A(x_1)=A(x_2)$ with $x_1,x_2 \in\{\textsc{null}(A)\}^{\perp}$, then $A\left(x_1-x_2\right)=0$, so $x_1-x_2 \in \textsc{null}(A) \cap\{\textsc{null}(A)\}^{\perp}=\{0\}$; hence $x_1=x_2$.
To establish the boundedness condition, take $y=A(x)$ with $x \in\{\textsc{null}(A)\}^{\perp}$, then
$\|L (y)\|=\|x\| \leq \frac{1}{\gamma}\|A(x)\|=\frac{1}{\gamma}\|y\|$.

Now decompose $Y = \textsc{range}(A) \oplus\{\textsc{range}(A)\}^\perp $, and we extend $L$ to all $Y$ by letting $\tilde{L}(y_R+y_{R^{\perp}})=L (y_R)$ with $y_R \in \textsc{range}(A)$ and  $y_{R^{\perp}} \in \{\textsc{range}(A)\}^{\perp}$. Such an extension preserves the same operator norm: $
\|\tilde{L}\|_{\op}=\|L\|_{\op} \leq \frac{1}{\gamma}$.

Finally, set $M:=B \tilde{L} \in \mathcal{L}(Y, Z)$. Such an $M$ satisfies $\|M\|_{\op}\leq\|B\|_{\op}\|\tilde L\|_{\op}\leq \|B\|_{\op}\gamma^{-1}$. Moreover, $B = MA$. To prove the latter, take any $x \in X$ and represent it as $x=u+v$ with $u \in \textsc{null}(A)$, $v \in\{\textsc{null}(A)\}^{\perp}$. Then
$$
M A(x)=B \tilde{L} A(u+v)=B \tilde{L}\{A (v)\}=B LA (v) = B (v)= B(v) + B(u)= B(x)
$$
by $M=B \tilde{L}$ and $x=(u+v)$; $A (u)=0$; $A (v) \in \textsc{range}(A)$; $L\{A (v)\}=v$; and $B(u)=0$ by the hypothesized property that $A(u)=0$ implies $B(u)=0$.
\end{proof}

The operator norm $\|M\|_{\op}$ can be interpreted as a relative condition number between the operators $A$ and $B$: it is bounded by $\|L\|_{\op}\|B\|_{\op}$, where $L$ is the left inverse of $A$.

\begin{lemma}[Relative alignment]\label{lemma:alignment}
 Suppose $(A,B,C)$ are operators and $\mu>0$ is a scalar, where $(A,B,A-B)$ are positive semidefinite. If $\langle x,y\rangle\geq 0$ implies $\langle x,(A-B)y\rangle\geq 0$ then $\|C (C^* A C + \mu I)^{-1}\|_{\op}  \leq \|C (C^*B C + \mu I)^{-1}\|_{\op}$.
\end{lemma}

\begin{proof}

Define the function $f:[0,1]\rightarrow \mathbb{R}$ by
$
f(t) = \| C [C^* \{tA+(1-t)B\} C + \mu I]^{-1} \|_{\op},
$
which is differentiable.\footnote{Since $f$ is 1-Lipschitz, it is differentiable almost everywhere in the interval $[0,1]$ by Rademacher's theorem. The differential equality we obtain only needs to hold almost everywhere to establish our desired bound.} We prove that $f(1) \leq f(0)$ by controlling the derivative of $f(t)$.
\begin{enumerate}
    \item Define the operator-valued function
$
W(t) = C [C^* \{tA+(1-t)B\} C + \mu I]^{-1}.
$
Then
$
f(t) = \| W(t) \|_{\op} = \sigma_{\max}\{W(t)\},
$
where $\sigma_{\max}\{W(t)\}$ is the largest singular value of $W(t)$.
    \item The derivative of $f(t)$ is
$$
    \frac{\partial}{\partial t} f(t) =  \frac{\partial}{\partial t} \sigma_{\max}\{W(t)\} 
    = \frac{\partial}{\partial t} \sup_{\|u\|\leq 1,\|v\|\leq 1} \langle u , W(t)v\rangle = \langle u(t) , \left\{\frac{\partial}{\partial t} W(t)\right\}v(t)\rangle
$$
where $u(t)$ and $v(t)$ are the initial left and right  singular functions of $W(t)$, i.e.
$$
W(t)v(t) = \sigma_{\max}\{W(t)\}u(t),\quad u(t)^{*}W(t) = \sigma_{\max}\{W(t)\}v(t)^{*}.
$$
The final equality appeals to \citet[Theorem 4.1]{Bonnans}. In particular, the derivative formula applies their result on the closed unit ball endowed with the weak topology, which is compact by Banach-Alaoglu theorem. The continuity conditions are satisfied because $W(t)$ and $\frac{\partial}{\partial t} W(t)$ are compact operators.
    \item Next we compute the derivative of $W(t)$. 
     Recall that $
\frac{\partial}{\partial t} \{X^{-1}(t)\} = -X^{-1}(t) \frac{\partial X(t)}{\partial t} X^{-1}(t)
$. Set $X(t)=C^* \{tA+(1-t)B\} C + \mu I$. Then  $\frac{\partial}{\partial t} X(t) = C^* (A-B) C$ and therefore 
\begin{align*}
    \frac{\partial}{\partial t} W(t) 
    &= \frac{\partial}{\partial t} CX^{-1}(t) 
    =-CX^{-1}(t) \frac{\partial X(t)}{\partial t} X^{-1}(t) \\
    &= - C [C^* \{tA+(1-t)B\} C + \mu I]^{-1} C^*  (A-B)C [C^* \{tA+(1-t)B\} C + \mu I]^{-1} \\
    &= -W(t)C^* (A-B) W(t).
\end{align*}
    \item Collecting results, the desired derivative is
\begin{align*}
    &\frac{\partial}{\partial t} f(t) 
    = \langle u(t) , \left\{\frac{\partial}{\partial t} W(t)\right\}v(t)\rangle
    = -\langle u(t) , W(t)C^* (A-B) W(t)v(t)\rangle \\
    &= -\langle C W(t)^*u(t) ,(A-B) W(t)v(t)\rangle 
    = -\langle C\sigma_{\max}\{W(t)\}v(t) , (A-B)  \sigma_{\max}\{W(t)\} u(t)\rangle \\
    &= -\sigma_{\max}\{W(t)\}^2 \cdot \langle C v(t) ,  (A-B)  u(t)\rangle 
    = -f(t)^2 \cdot \langle C v(t) ,  (A-B) u(t)\rangle 
    =-f(t)^2 \gamma(t)
\end{align*}
with $\gamma(t) := \langle C v(t) ,  (A-B) u(t)\rangle$.
\item We solve the differential equation $
\frac{1}{f(s)^2} \frac{\partial f(s)}{\partial s} = -\gamma(s).$
Integrating both sides,
$$
 -\int_0^t \gamma(s) \partial s=\int_0^t \frac{1}{f(s)^2}  \partial f(s) = -\frac{1}{f(t)} + \frac{1}{f(0)}.
$$
Therefore, 
$f(t)  = \frac{1}{\frac{1}{f(0)} +  \int_0^{t}\gamma(s)\,\partial s}$ and hence $f(1)  = \frac{1}{\frac{1}{f(0)} +  \int_0^{1}\gamma(s)\,\partial s}$.
\item We argue that relative alignment of $A-B$ implies $\gamma(s)\geq 0$. Note that
\begin{align*}
    \langle C v(t) , u(t)\rangle &= \langle u(t), C v(t) \rangle= u(t)^* C X^{-1}(t)X(t)v(t) \\
    &=u(t)^* W(t)X(t) v(t) 
    =\sigma_{\max}\{W(t)\}v(t)^*X(t)v(t) \geq 0
\end{align*}
because $X(t)$ is positive semidefinite as the sum of three positive semidefinite operators. Therefore, by hypothesis, $\gamma(t)=\langle C v(t) , (A-B) u(t)\rangle\geq 0$.

\item In summary, we have shown $f(1)  = \frac{1}{\frac{1}{f(0)} +  \int_0^{1}\gamma(s)\,\partial s} \leq  \frac{1}{\frac{1}{f(0)}}=f(0)$. \qedhere 
\end{enumerate}
\end{proof}

In the final step, we use $\int_0^{1}\gamma(s)\,\partial s \geq 0$. Assumption~\ref{assumption:alignment} is a strong sufficient condition that guarantees $\gamma(s)\geq 0$. Assumption~\ref{assumption:weak-alignment} is a weaker condition that guarantees $\int_0^{1}\gamma(s)\,\partial s \geq 0$.

\subsection{Verifying high level conditions}

\begin{proof}[Proof of Lemma~\ref{lemma:consequence}]
    The result is immediate from Lemma~\ref{lemma:completeness}.
\end{proof}

\begin{lemma}\label{lemma:alignment_interp}
    Assumption~\ref{assumption:alignment} holds if and only if
\begin{enumerate}
        \item $\langle x,y \rangle \geq 0$ implies $\langle x, M^*M y \rangle\geq 0 $;
         \item $\langle x,y \rangle \geq 0$ implies $\langle x, M^*\mu(T_hT_h^*+\mu)^{-1}M y \rangle\geq 0 $; 
        \item  $\langle x,y \rangle \geq 0$ implies $\langle x, (I-Q-cI) y \rangle\geq 0 $.
    \end{enumerate}
    A similar statement holds for Assumption~\ref{assumption:weak-alignment}.
\end{lemma}

\begin{proof}
 The first and second results use the definition of the adjoint. The third uses $
        \langle x, (I-Q-cI) y \rangle=(1-c)\langle x, y \rangle - \langle x, Q y \rangle
        $.
\end{proof}

\begin{lemma}[Verifying high level conditions]\label{lemma:verify}
  Assumptions~\ref{assumption:posedness},~\ref{assumption:completeness}, and~\ref{assumption:alignment} imply Assumption~\ref{assumption:high}. The same is true replacing Assumption~\ref{assumption:alignment} with Assumption~\ref{assumption:weak-alignment}.   
\end{lemma}

\begin{proof}
We state the argument under Assumption~\ref{assumption:alignment}; the argument under Assumption~\ref{assumption:weak-alignment} is identical.
Again we abbreviate $\|\cdot\|=\|\cdot\|_{\op}$. 

By Lemma~\ref{lemma:pre} and \citet[Theorem IX.2.1]{bhatia2013matrix}, 
\begin{align*}
    c^{-1}
    &=\|(S^*S+T_g^*T_g+\mu')^{1/2}(S^*S+\mu')^{-1} (S^*S+T_g^*T_g+\mu')^{1/2}\|\\
    &\leq \|(S^*S+T_g^*T_g+\mu')^{1/2}(S^*S+\mu')^{-1/2}\| \|(S^*S+\mu')^{-1/2} (S^*S+T_g^*T_g+\mu')^{1/2}\|\\
    &=\|(S^*S+\mu')^{-1/2} (S^*S+T_g^*T_g+\mu')^{1/2}\|^2
     \leq \|(S^*S+\mu')^{-1} (S^*S+T_g^*T_g+\mu')\| \\
    &= \|I+(S^*S+\mu')^{-1}T_g^*T_g\|
    \leq 1+\|(S^*S+\mu')^{-1}T_g^*T_g\| 
    =O(1)
\end{align*}
where in the last line we use Assumption~\ref{assumption:posedness}.
\begin{enumerate}
    \item By Lemma~\ref{lemma:denom}, $\|T_g(S^*S+T_g^*T_g+\mu')^{-1}T_g^*\|\leq \|T_g(T_g^*T_g+\mu')^{-1}T_g^*\|\leq 1$. 
    
\item Under Assumption~\ref{assumption:completeness}, Lemma~\ref{lemma:consequence} implies 
\begin{align*}
  \|T_g(S^*S+T_g^*T_g+\mu)^{-1}\| &= \|MS\{S^*(I+M^*M)S+\mu\}^{-1}\| \leq \|M\| \|S\{S^*(I+M^*M)S+\mu\}^{-1}\|.
\end{align*}
Focusing on the latter factor, Assumption~\ref{assumption:alignment}.1 and Lemmas~\ref{lemma:alignment} and~\ref{lemma:alignment_interp} imply   
$$\|S\{S^*(I+M^*M)S+\mu\}^{-1}\|\leq \|S(S^*S+\mu)^{-1}\|= O(\mu^{-1/2}).$$ 

\item Since $T_h^*(I-Q)T_h\geq cT_h^*T_h$ by Lemma~\ref{lemma:pre}, an argument similar to Lemma \ref{lemma:denom} gives 
$$\|T_h\{T_h^*(I-Q)T_h+\mu\}^{-1}T_h^*\|\leq \|T_h(cT_h^*T_h+\mu)^{-1}T_h^*\|= c^{-1}\|T_h(T_h^*T_h+\mu/c)^{-1}T_h^*\|\leq c^{-1}.$$

\item Assumption~\ref{assumption:alignment}.2 and Lemmas~\ref{lemma:alignment} and~\ref{lemma:alignment_interp} imply   
\begin{align*}
  \|T_h\{T_h^*(I-Q)T_h+\mu\}^{-1}\| &\leq  \|T_h(cT_h^*T_h+\mu)^{-1}\|=c^{-1}\|T_h(T_h^*T_h+\mu/c)^{-1}\|= O(\mu^{-1/2}).
\end{align*}

\item See result 4. 

\item Since $T_h^*(I-Q)T_h\geq 0$ by Lemma~\ref{lemma:pre}, $\|\{T_h^*(I-Q)T_h+\mu\}^{-1}\|\leq \mu^{-1}$. 

\item Since $R\geq 0$ by Lemma~\ref{lemma:pre}, Lemma~\ref{lemma:denom} gives $\|S(S^*S+R+\mu')^{-1}S^*\|\leq \|S(S^*S+\mu')^{-1}S^*\|\leq 1.$ 

\item Under Assumption~\ref{assumption:completeness},  Lemma~\ref{lemma:consequence} implies 
$$\|T_g(S^*S+R+\mu')^{-1}T_g^*\|=\|MS(S^*S+R+\mu')^{-1}S^*M^*\|\leq \|M\|\|S(S^*S+R+\mu')^{-1}S^*\|\|M^*\|.$$
Then appeal to result 7. 

\item Under Assumption~\ref{assumption:completeness},  Lemmas~\ref{lemma:pre} and~\ref{lemma:consequence} imply that
$$R =T_g^*\mu(T_hT_h^*+\mu)^{-1}T_g= S^*M^*\mu(T_hT_h^*+\mu)^{-1}MS.$$ 
Assumption~\ref{assumption:alignment}.1 and Lemmas~\ref{lemma:alignment} and~\ref{lemma:alignment_interp} imply  
\begin{align*}
    \|S(S^*S+R+\mu')^{-1}\| &= \|S[S^*\left\{I+M^*\mu(T_hT_h^*+\mu)^{-1}M\right\}S+\mu']^{-1}\|\\
    &\leq \|S(S^*S+\mu')^{-1}\|= O\{(\mu')^{-1/2}\}.
\end{align*}

\item See result 9. 

\item Under Assumption~\ref{assumption:completeness}, Lemma~\ref{lemma:consequence} implies  $\|T_g(S^*S+R+\mu')^{-1}\| = \|MS(S^*S+R+\mu')^{-1}\|\leq \|M\|\|S(S^*S+R+\mu')^{-1}\|$.  Then appeal to result 9. 

\item See result 11. 

\item Under Assumption~\ref{assumption:completeness}, Lemma~\ref{lemma:consequence} implies 
\begin{align*}
    &\|S(S^*S+R+\mu')^{-1}T_g^*\| = \|S(S^*S+R+\mu')^{-1}S^*M^*\|\leq  \|S(S^*S+R+\mu')^{-1}S^*\|\|M^*\|.
\end{align*}
Then appeal to result 7. 

\item Since $S^*S+R\geq 0$ by Lemma~\ref{lemma:pre}, $\|(S^*S+R+\mu')^{-1}\|\leq (\mu')^{-1}$. \qedhere
\end{enumerate}
\end{proof}

\section{Proof of Theorem~\ref{theorem:ci}}\label{sec:ci}

Theorem~\ref{theorem:ci} summarizes nonasymptotic Gaussian approximation and variance estimation results, which we now state and prove. Our nonasymptotic results apply to casual scalars and causal functions, though we defer details on the latter to Appendix~\ref{sec:local}.

\subsection{Nonasymptotic refinements}

We pointwise approximate causal functions such as Example~\ref{ex:het}, taking the limit where the bandwidth $\lambda$ of the weighting $\ell_{\lambda}$ vanishes. Formally, we write $\theta_0(v)=\lim_{\lambda\rightarrow 0}\theta_{\lambda}(v)$, where $v=x_1$ in Example~\ref{ex:het}. For each $\theta_{\lambda}(v)$, we write the influence function as
\begin{align*}
    \psi_{\lambda}(B_1,B_2,B_3,B_4)&=\ell_{\lambda}(V)[h_1(B_1)+h_3(B_3)\{Y-h_2(B_2)\}+h_4(B_4)\{h_2(B_2)-h_1(B_1)\}]-\theta_{\lambda}(v)\\
    &=h_{1,\lambda}(B_1)+h_3(B_3)\{Y_{\lambda}-h_{2,\lambda}(B_2)\}+h_4(B_4)\{h_{2,\lambda}(B_2)-h_{1,\lambda}(B_1)\}-\theta_{\lambda}(v).
\end{align*}
By construction, the first moment of $\psi_{\lambda}$ is zero, and the second, third, and fourth moments $(\sigma_{\lambda}^2,\kappa_{\lambda}^3,\chi_{\lambda}^4)$ diverge as $\lambda\downarrow 0$. Finally, we write the pointwise approximation error as $\Delta_{\lambda}(v)=n^{1/2}\sigma_{\lambda}^{-1}|\theta_{\lambda}(v)-\theta_0(v)|$. 

For causal scalars, $\ell_{\lambda}(V)=1$, $(\sigma^2,\kappa^3,\chi^4)$ are fixed constants, and there is no pointwise approximation error.

\begin{theorem}[Finite sample Gaussian approximation for Algorithm~\ref{algo:dml}]\label{theorem:gaussian}
    Suppose Assumptions~\ref{assumption:orthogonal},~\ref{assumption:regular}(i), and~\ref{assumption:regular}(ii) hold. Then with probability $1-\epsilon$, $
\sup_{z\in\mathbb{R}} \left| \mathbb{P} \left\{\frac{n^{1/2}}{\sigma}(\hat{\theta}-\theta_0)\leq z\right\}-\Phi(z)\right|\leq 0.4748\left(\frac{\kappa}{\sigma}\right)^3 n^{-1/2}+\frac{\Delta}{(2\pi)^{1/2}}+\epsilon,
$ where $\Phi$ is the standard Gaussian distribution and 
{\small 
\begin{align*}
    \Delta&=\frac{7L}{2\epsilon  \sigma}
    \bigg\{
    (1+\bar{h}_4)\|\hat{h}_1-h_1\|_2
    +(\bar{h}_3+\bar{h}_4)\|\hat{h}_2-h_2\|_2
    +\bar{\sigma}_y\|\hat{h}_3-h_3\|_2
    +\bar{\sigma}_2\|\hat{h}_4-h_4\|_2 \\
    &\quad +n^{1/2}\|\hat{h}_1-h_1\|_2\|\hat{h}_4-h_4\|_2
    +n^{1/2}\|\hat{h}_2-h_2\|_2\|\hat{h}_3-h_3\|_2
    +n^{1/2}\|\hat{h}_2-h_2\|_2\|\hat{h}_4-h_4\|_2
    \bigg\}.
\end{align*}
}
    If in addition Assumption~\ref{assumption:regular}(iii) holds, then the same holds updating $\Delta$ to be
    {\small
\begin{align*}
    \Delta&=\frac{4 L}{\epsilon^{1/2} \sigma}
\bigg\{
(1+\bar{h}_4+\bar{h}_4')\|\hat{h}_1-h_1\|_2
+(\bar{h}_3+\bar{h}_3'+\bar{h}_4+\bar{h}_4')\|\hat{h}_2-h_2\|_2
 +\bar{\sigma}_y\|\hat{h}_3-h_3\|_2
+\bar{\sigma}_2\|\hat{h}_4-h_4\|_2
\bigg\}\\
&\quad +\frac{1}{2\sigma}\bigg\{
n^{1/2}\|T_1(\hat{h}_1-h_1)\|_2\|\hat{h}_4-h_4\|_2 \wedge n^{1/2}\|\hat{h}_1-h_1\|_2\|T_4(\hat{h}_4-h_4)\|_2\\
&\quad\quad+
n^{1/2}\|T_2(\hat{h}_2-h_2)\|_2\|\hat{h}_3-h_3\|_2 \wedge n^{1/2}\|\hat{h}_2-h_2\|_2\|T_3(\hat{h}_3-h_3)\|_2 \\
&\quad\quad+
n^{1/2}\|T_2(\hat{h}_2-h_2)\|_2\|\hat{h}_4-h_4\|_2 \wedge n^{1/2}\|\hat{h}_2-h_2\|_2\|T_4(\hat{h}_4-h_4)\|_2
\bigg\}.
\end{align*}
}

For causal functions, the same holds replacing $(\hat{\theta},\theta_0,\Delta)$ with $\{\hat{\theta}_{\lambda}(v),\theta_0(v),\Delta+\Delta_{\lambda}(v)\}$ and indexing several quantities by $\lambda$; see Appendix~\ref{sec:local}.
\end{theorem}

\begin{theorem}[Finite sample variance estimation for Algorithm~\ref{algo:dml}]\label{theorem:variance}
    Suppose Assumptions~\ref{assumption:regular}(i) and~(iii) hold. Then with probability $1-\epsilon'$, $
|\hat{\sigma}^2-\sigma^2|\leq \Delta'+2(\Delta')^{1/2}\{(\Delta'')^{1/2}+\sigma\}+\Delta'',
$ where $\Delta''=\left(\frac{2}{\epsilon'}\right)^{1/2}\chi^2 n^{-1/2}$ and
{ \small
$$
\Delta'=7(\hat{\theta}-\theta_0)^2+\frac{84 L}{\epsilon'}\left[
   \|\hat{h}_1-h_1\|_2^2
    +\{(\bar{h}_3')^2+(\bar{h}_4')^2\}\|\hat{h}_2-h_2\|_2^2
    +\{(\bar{h}_4')^2+\bar{\sigma}_y^2\}\|\hat{h}_3-h_3\|_2^2
    +\bar{\sigma}_2^2\|\hat{h}_4-h_4\|_2^2
    \right].
$$
}

For causal functions, the same holds, indexing several quantities by $\lambda$; see Appendix~\ref{sec:local}.
\end{theorem}

Theorems~\ref{theorem:gaussian} and~\ref{theorem:variance} give nonasymptotic Gaussian approximation and variance estimation for mediated, time varying, and long term treatment effects with generic machine learning. The results hold with or without proxy variables, and apply to nonparametric causal functions. 

\begin{proof}[Proof of Theorem~\ref{theorem:ci}]
        By Theorem~\ref{theorem:gaussian}, $\hat{\theta}\overset{p}{\rightarrow}\theta_0$ and 
$\lim_{n\rightarrow\infty} \mathbb{P}\left\{\theta_0 \in  \left(\hat{\theta}\pm 1.96\frac{\sigma}{n^{1/2}}\right)\right\}=0.95.$
For the desired result, it suffices that $\hat{\sigma}^2\overset{p}{\rightarrow}\sigma^2$, which follows from Theorem~\ref{theorem:variance}.
\end{proof}

\subsection{Neyman orthogonality}\label{sec:neyman}

In this appendix, we expand the notation to eliminate some subscripts. We denote the norm $\mathcal{R}(h)=\|h-h_0\|_2^2$ and $\mathcal{P}(h)=\|T(h-h_0)\|_2^2$, where the operator $T$ is relative to the definition of $h_0$. We write the nuisances as $(\nu_0,\delta_0,\alpha_0,\eta_0)=(h_1,h_2,h_3,h_4)$. Let $W$ concatenate all of the random variables in an observation. Let $\psi_0(w)=\psi(w,\theta_0,\nu_0,\delta_0,\alpha_0,\eta_0)$ where 
$$
\psi(w,\theta,\nu,\delta,\alpha,\eta)=\nu(w)+\alpha(w)\{y-\delta(w)\}+\eta(w)\{\delta(w)-\nu(w)\}-\theta
$$
and we suppress the indexing by $\lambda$ for causal functions.

Let $s(w),t(w),u(w),v(w)$ be functions and let $\tau,\zeta \in  \mathbb{R}$ be scalars. The Gateaux derivative of $\psi(w,\theta,\nu,\delta,\alpha,\eta)$ with respect to its argument $\nu$ in the direction $s$ is
$
\{\partial_{\nu} \psi(w,\theta,\nu,\delta,\alpha,\eta)\}(s)=\frac{\partial}{\partial \tau} \psi(w,\theta,\nu+\tau s,\delta,\alpha,\eta)|_{\tau=0}.
$
The cross derivative of $\psi(w,\theta,\nu,\delta,\alpha,\eta)$ with respect to its arguments $(\nu,\delta)$ in the directions $(s,t)$ is
$
\{\partial^2_{\nu,\delta} \psi(w,\theta,\nu,\delta,\alpha,\eta)\}(s,t)=\frac{\partial^2}{\partial \tau \partial \zeta} \psi(w,\theta,\nu+\tau s,\delta+\zeta t, \alpha, \eta)|_{\tau=0,\zeta=0}.
$

\begin{lemma}[Calculation of derivatives]\label{lemma:deriv}
The first derivatives are 
$\{\partial_{\nu} \psi(w,\theta,\nu,\delta,\alpha,\eta)\}(s)
=s(w)\{1-\eta(w)\}$, 
$\{\partial_{\delta} \psi(w,\theta,\nu,\delta,\alpha,\eta)\}(t)
    =t(w)\{\eta(w)-\alpha(w)\}$, 
    $\{\partial_{\alpha} \psi(w,\theta,\nu,\delta,\alpha,\eta)\}(u)=u(w)\{y-\delta(w)\}$, and 
    $\{\partial_{\eta} \psi(w,\theta,\nu,\delta,\alpha,\eta)\}(v)=v(w)\{\delta(w)-\nu(w)\}$.
The second derivatives are $\{\partial^2_{\nu,\delta} \psi(w,\theta,\nu,\delta,\alpha,\eta)\}(s,t)=0$, $\{\partial^2_{\nu,\alpha} \psi(w,\theta,\nu,\delta,\alpha,\eta)\}(s,u)=0$, $ \{\partial^2_{\nu,\eta} \psi(w,\theta,\nu,\delta,\alpha,\eta)\}(s,v)=-v(w)s(w)$, $\{\partial^2_{\delta,\alpha} \psi(w,\theta,\nu,\delta,\alpha,\eta)\}(t,u)=-u(w)t(w)$, $\{\partial^2_{\delta,\eta} \psi(w,\theta,\nu,\delta,\alpha,\eta)\}(t,v)=v(w)t(w)$, and $\{\partial^2_{\alpha,\eta} \psi(w,\theta,\nu,\delta,\alpha,\eta)\}(u,v)=0$.
\end{lemma}

\begin{proof}
    The result is immediate from the definition of Gateaux differentiation.
\end{proof}

\begin{lemma}[Neyman orthogonality]\label{lemma:neyman}
If Assumption~\ref{assumption:orthogonal} holds then $\psi$ is Neyman orthogonal with respect to $(\nu,\delta,\alpha,\eta)$.
\end{lemma}

\begin{proof}
    The result is immediate from the first derivatives in Lemma~\ref{lemma:deriv}.
\end{proof}

\begin{lemma}[Verifying Neyman orthogonality]
    Examples~\ref{ex:direct} and~\ref{ex:het} are Neyman orthogonal. So are the other examples, using the efficient influence functions derived in other works.
\end{lemma}

\begin{proof}
    By the law of iterated expectations, it is straightforward to verify Assumption~\ref{assumption:orthogonal} for each example. By Lemma~\ref{lemma:neyman}, this suffices for Neyman orthogonality.
\end{proof}

\subsection{Gaussian approximation}

Partition the observations into $L$ folds. Denote the $\ell$th fold by $I_{\ell}$. Train $(\hat{\nu}_{\ell},\hat{\delta}_{\ell},\hat{\alpha}_{\ell},\hat{\eta}_{\ell})$ on observations in the complement of $I_{\ell}$, i.e. $I_{\ell}^c$. Let $n_{\ell}=|I_{\ell}|=n/L$ be the number of observations in $I_{\ell}$. Denote by $\mathbb{E}_{\ell}(\cdot)=n_{\ell}^{-1}\sum_{i\in I_{\ell}}(\cdot)$ the average over observations in $I_{\ell}$. Denote by $\mathbb{E}_n(\cdot)=n^{-1}\sum_{i=1}^n(\cdot)$ the average over all observations in the sample.

We define the foldwise target as $\hat{\theta}_{\ell}=\mathbb{E}_{\ell} [\hat{\nu}_{\ell}(W)+\hat{\alpha}_{\ell}(W)\{Y-\hat{\delta}_{\ell}(W)\}+\hat{\eta}_{\ell}(W)\{\hat{\delta}_{\ell}(W)-\hat{\nu}_{\ell}(W)\}]$. We define the foldwise oracle as $\bar{\theta}_{\ell}=\mathbb{E}_{\ell} [\nu_0(W)+\alpha_0(W)\{Y-\delta_0(W)\}+\eta_0(W)\{\delta_0(W)-\nu_0(W)\}]$. We define the overall target as $
    \hat{\theta}=\frac{1}{L}\sum_{\ell=1}^L \hat{\theta}_{\ell}$. We define the overall oracle as 
$
    \bar{\theta}=\frac{1}{L}\sum_{\ell=1}^L \bar{\theta}_{\ell}.
$
Finally, let $(\bar{\alpha},\bar{\eta},\bar{\alpha}',\bar{\eta}',\bar{\sigma}_1,\bar{\sigma}_2)=(\bar{h}_3,\bar{h}_4,\bar{h}_3',\bar{h}_4',\bar{\sigma}_y,\bar{\sigma}_2)$.

\begin{lemma}[Taylor expansion]\label{lemma:Taylor}
Let $s=\hat{\nu}_{\ell}-\nu_0$, $t=\hat{\delta}_{\ell}-\delta_0$, $ u= \hat{\alpha}-\alpha_0$, and $v=\hat{\eta}-\eta_0$. Then $n_{\ell}^{1/2}(\hat{\theta}_{\ell}-\bar{\theta}_{\ell})=\sum_{j=1}^{7} \Delta_{j{\ell}}$ where the first derivative terms are $\Delta_{1{\ell}}=n_{\ell}^{1/2}\mathbb{E}_{\ell}[s(W)\{1-\eta_0(W)\}]$, $ \Delta_{2{\ell}}=n_{\ell}^{1/2}\mathbb{E}_{\ell}[t(W)\{\eta_0(W)-\alpha_0(W)\}]$, $\Delta_{3{\ell}}=n_{\ell}^{1/2}\mathbb{E}_{\ell}[u(W)\{Y-\delta_0(W)\}]$, $\Delta_{4{\ell}}=n_{\ell}^{1/2}\mathbb{E}_{\ell}[v(W)\{\delta_0(W)-\nu_0(W)\}]$, 
and the second derivative terms are $ \Delta_{5{\ell}}=\frac{n_{\ell}^{1/2}}{2}\mathbb{E}_{\ell} \{-s(W)v(W)\}$, $ \Delta_{6{\ell}}=\frac{n_{\ell}^{1/2}}{2}\mathbb{E}_{\ell} \{-t(W)u(W)\}$, $\Delta_{7{\ell}}=\frac{n_{\ell}^{1/2}}{2}\mathbb{E}_{\ell} \{t(W)v(W)\}$.
\end{lemma}

\begin{proof}
An exact Taylor expansion gives $\psi(w,\theta_0,\hat{\nu}_{\ell},\hat{\delta}_{\ell},\hat{\alpha}_{\ell},\hat{\eta}_{\ell})-\psi_0(w)$ equal to 
\begin{align*}
&\{\partial_{\nu} \psi_0(w)\}(s)
+\{\partial_{\delta} \psi_0(w)\}(t)
+\{\partial_{\alpha} \psi_0(w)\}(u)
+\{\partial_{\eta} \psi_0(w)\}(v) \\
&+\frac{1}{2}\{\partial^2_{\nu,\delta} \psi_0(w)\}(s,t)
+\frac{1}{2}\{\partial^2_{\nu,\alpha} \psi_0(w)\}(s,u)
+\frac{1}{2}\{\partial^2_{\nu,\eta} \psi_0(w)\}(s,v)\\
&+\frac{1}{2}\{\partial^2_{\delta,\alpha} \psi_0(w)\}(t,u)
+\frac{1}{2}\{\partial^2_{\delta,\eta} \psi_0(w)\}(t,v)
+\frac{1}{2}\{\partial^2_{\alpha,\eta} \psi_0(w)\}(u,v).
\end{align*}
Averaging over observations in $I_{\ell}$, $\hat{\theta}_{\ell}-\bar{\theta}_{\ell}
    =\mathbb{E}_{\ell}\{\psi(W,\theta_0,\hat{\nu}_{\ell},\hat{\delta}_{\ell},\hat{\alpha}_{\ell},\hat{\eta}_{\ell})\}-\mathbb{E}_{\ell}\{\psi_0(W)\}$. Finally appeal to Lemma~\ref{lemma:deriv}.
\end{proof}

\begin{lemma}[Residuals]\label{lemma:resid}
Suppose the conditions of Theorem~\ref{theorem:gaussian} hold. Then with probability $1-\epsilon/L$, the first derivative terms have the bounds $|\Delta_{1\ell}|\leq t_1=\left(\frac{7L}{\epsilon}\right)^{1/2}(1+\bar{\eta})\{\mathcal{R}(\hat{\nu}_{\ell})\}^{1/2}$, $|\Delta_{2\ell}|\leq t_2=\left(\frac{7L}{\epsilon}\right)^{1/2}(\bar{\alpha}+\bar{\eta})\{\mathcal{R}(\hat{\delta}_{\ell})\}^{1/2}$, $|\Delta_{3\ell}|\leq t_3=\left(\frac{7L}{\epsilon}\right)^{1/2}\bar{\sigma}_1\{\mathcal{R}(\hat{\alpha}_{\ell})\}^{1/2}$, $|\Delta_{4\ell}|\leq t_4=\left(\frac{7L}{\epsilon}\right)^{1/2}\bar{\sigma}_2\{\mathcal{R}(\hat{\eta}_{\ell})\}^{1/2} $  while the second derivative terms have the bounds $ |\Delta_{5\ell}|\leq t_5= \frac{7L^{1/2}}{2\epsilon}\{n\mathcal{R}(\hat{\nu}_{\ell})\mathcal{R}(\hat{\eta}_{\ell})\}^{1/2}$, $ |\Delta_{6\ell}|\leq t_6= \frac{7L^{1/2}}{2\epsilon}\{n\mathcal{R}(\hat{\delta}_{\ell})\mathcal{R}(\hat{\alpha}_{\ell})\}^{1/2}$, $|\Delta_{7\ell}|\leq t_7= \frac{7L^{1/2}}{2\epsilon}\{n\mathcal{R}(\hat{\delta}_{\ell})\mathcal{R}(\hat{\eta}_{\ell})\}^{1/2}$.
\end{lemma}

\begin{proof}
For simplicity, we focus on one first derivative term and one second derivative term; the rest are similar.
\begin{enumerate}
    \item Markov inequality implies $ \mathbb{P}(|\Delta_{1\ell}|>t_1)\leq \frac{\mathbb{E}(\Delta^2_{1\ell})}{t_1^2}$ and $ \mathbb{P}(|\Delta_{5\ell}|>t_5)\leq \frac{\mathbb{E}(|\Delta_{5\ell}|)}{t_5}$.
    \item The law of iterated expectations implies $ \mathbb{E}(\Delta^2_{1\ell})=\mathbb{E}\{\mathbb{E}(\Delta^2_{1\ell}\mid I^c_{\ell})\}$ and $  \mathbb{E}(|\Delta_{5\ell}|)=\mathbb{E}\{\mathbb{E}(|\Delta_{5\ell}|\mid I^c_{\ell})\}$.
   
    \item 
    Conditional on $I_{\ell}^c$, $(s,t,u,v)$ are nonrandom. Moreover, observations within fold $I_{\ell}$ are independent and identically distributed. Hence by Assumption~\ref{assumption:orthogonal}, $\mathbb{E}(\Delta^2_{1\ell}\mid I^c_{\ell})$ equals
    \begin{align*}
        &\mathbb{E} \left([n_{\ell}^{1/2}\mathbb{E}_{\ell}\{s(W)-s(W)\eta_0(W)\}\}]^2 \mid I^c_{\ell}\right) \\
        &=\mathbb{E} \left[ \frac{n_{\ell}}{n^2_{\ell}} \sum_{i,j\in I_{\ell}} \{s(W_i)-s(W_i)\eta_0(W_i)\}\{s(W_j)-s(W_j)\eta_0(W_j)\} \mid I^c_{\ell}\right] \\
        &= \frac{n_{\ell}}{n^2_{\ell}} \sum_{i,j\in I_{\ell}}\mathbb{E} \left[ \{s(W_i)-s(W_i)\eta_0(W_i)\}\{s(W_j)-s(W_j)\eta_0(W_j)\} \mid I^c_{\ell}\right] \\
        &= \frac{n_{\ell}}{n^2_{\ell}} \sum_{i\in I_{\ell}}\mathbb{E} \left[ \{s(W_i)-s(W_i)\eta_0(W_i)\}^2 \mid I^c_{\ell}\right] 
        =\mathbb{E}[s(W)^2\{1-\eta_0(W)\}^2\mid I^c_{\ell}] 
        \leq (1+\bar{\eta})^2 \mathcal{R}(\hat{\nu}_{\ell}).
    \end{align*}
    By Cauchy Schwarz, $ \mathbb{E}(|\Delta_{5\ell}|\mid I^c_{\ell})
        =\frac{n_{\ell}^{1/2}}{2}  \mathbb{E}\{|-s(W)v(W)|\mid I^c_{\ell}\} $ is bounded by
    \begin{align*}
   \frac{n_{\ell}^{1/2}}{2} [\mathbb{E}\{s(W)^2\mid I^c_{\ell}\}]^{1/2} [\mathbb{E}\{v(W)^2\mid I^c_{\ell}\}]^{1/2} 
   =\frac{n_{\ell}^{1/2}}{2}  \{\mathcal{R}(\hat{\nu}_{\ell})\}^{1/2}\{\mathcal{R}(\hat{\eta}_{\ell})\}^{1/2}.
    \end{align*}
    
    \item Collecting results gives $ \mathbb{P}(|\Delta_{1\ell}|>t_1)\leq \frac{(1+\bar{\eta})^2 \mathcal{R}(\hat{\nu}_{\ell})}{t_1^2}=\frac{\epsilon}{7L}$ and $ \mathbb{P}(|\Delta_{5\ell}|>t_5)\leq \frac{n_{\ell}^{1/2} \{\mathcal{R}(\hat{\nu}_{\ell})\}^{1/2}\{\mathcal{R}(\hat{\eta}_{\ell})\}^{1/2}}{2t_5}=\frac{\epsilon}{7L}$. Therefore with probability $1-\epsilon/L$, $|\Delta_{1\ell}|\leq t_1=\left(\frac{7L}{\epsilon}\right)^{1/2}(1+\bar{\eta})\{\mathcal{R}(\hat{\nu}_{\ell})\}^{1/2}$ and  $|\Delta_{5\ell}|\leq t_5=\frac{7L}{2\epsilon}n_{\ell}^{1/2}\{\mathcal{R}(\hat{\nu}_{\ell})\}^{1/2}\{\mathcal{R}(\hat{\eta}_{\ell})\}^{1/2}$ and similarly for all $j\in\{1,...,7\}$.  \qedhere 
\end{enumerate}
\end{proof}

\begin{lemma}[Residuals: Alternative path]\label{lemma:resid_alt}
Suppose the conditions of Theorem~\ref{theorem:gaussian} hold. 
Then with probability $1-\epsilon/L$, the first derivative terms have the bounds of Lemma~\ref{lemma:resid}, while the second derivative terms have the bounds 
\begin{align*}
    |\Delta_5| &\leq t_5= \left(\frac{7L}{4\epsilon}\right)^{1/2}(\bar{\eta}+\bar{\eta}')\{\mathcal{R}(\hat{\nu}_{\ell})\}^{1/2} +(4L)^{-1/2} [\{n\mathcal{P}(\hat{\nu}_{\ell})\mathcal{R}(\hat{\eta}_{\ell})\}^{1/2} \wedge \{n\mathcal{R}(\hat{\nu}_{\ell})\mathcal{P}(\hat{\eta}_{\ell})\}^{1/2}],\\
    |\Delta_6| &\leq t_6= \left(\frac{7L}{4\epsilon}\right)^{1/2}(\bar{\alpha}+\bar{\alpha}')\{\mathcal{R}(\hat{\delta}_{\ell})\}^{1/2} +(4L)^{-1/2} [\{n\mathcal{P}(\hat{\delta}_{\ell})\mathcal{R}(\hat{\alpha}_{\ell})\}^{1/2} \wedge \{n\mathcal{R}(\hat{\delta}_{\ell})\mathcal{P}(\hat{\alpha}_{\ell})\}^{1/2}], \\
    |\Delta_7| &\leq t_7= \left(\frac{7L}{4\epsilon}\right)^{1/2}(\bar{\eta}+\bar{\eta}')\{\mathcal{R}(\hat{\delta}_{\ell})\}^{1/2} +(4L)^{-1/2} [\{n\mathcal{P}(\hat{\delta}_{\ell})\mathcal{R}(\hat{\eta}_{\ell})\}^{1/2} \wedge \{n\mathcal{R}(\hat{\delta}_{\ell})\mathcal{P}(\hat{\eta}_{\ell})\}^{1/2}].
\end{align*}
\end{lemma}

\begin{proof}
See Lemma~\ref{lemma:resid} for $(t_1,t_2,t_3,t_4)$. We focus on $t_5$; $(t_6,t_7)$ are similar.
\begin{enumerate}
\item 
    Write $2\Delta_{5\ell}=n_{\ell}^{1/2}\mathbb{E}_{\ell} \{-s(W)v(W)\}=\Delta_{5'\ell}+\Delta_{5''\ell}$ where $\Delta_{5'\ell}=n_{\ell}^{1/2}\mathbb{E}_{\ell} [-s(W)v(W)+\mathbb{E}\{s(W)v(W)\mid I^c_{\ell}\}]$ and $  \Delta_{5''\ell}= n_{\ell}^{1/2} \mathbb{E}\{-s(W)v(W)\mid I^c_{\ell}\}$.

\item Consider the former term. By Markov inequality, $\mathbb{P}(|\Delta_{5'\ell}|>t)\leq \frac{\mathbb{E}(\Delta^2_{5'\ell})}{t^2}$. By the law of iterated expectations,
$      \mathbb{E}(\Delta^2_{5'\ell})=\mathbb{E}\{\mathbb{E}(\Delta^2_{5'\ell}\mid I^c_{\ell})\}.
$
We bound the conditional moment. Conditional on $I_{\ell}^c$, $(s,t,u,v)$ are nonrandom. Moreover, observations within fold $I_{\ell}$ are independent and identically distributed. Since $\Delta_{5'\ell}$ has conditional mean zero by construction, $\mathbb{E}(\Delta^2_{5'\ell}\mid I^c_{\ell})$ equals
    \begin{align*}
  & \mathbb{E}\left\{\left(n_{\ell}^{1/2}\mathbb{E}_{\ell} [-s(W)v(W)+\mathbb{E}\{s(W)v(W)\mid I^c_{\ell}\}]\right)^2 \mid I^c_{\ell}\right\} \\
        &= \mathbb{E}\left( \frac{n_{\ell}}{n^2_{\ell}}\sum_{i,j \in I_{\ell}} [-s(W_i)v(W_i)+\mathbb{E}\{s(W_i)v(W_i)\mid I^c_{\ell}\}][-s(W_j)v(W_j)+\mathbb{E}\{s(W_j)v(W_j)\mid I^c_{\ell}\}] \mid I^c_{\ell}\right) \\
        &= \frac{n_{\ell}}{n^2_{\ell}}\sum_{i,j \in I_{\ell}} \mathbb{E}\left(  [-s(W_i)v(W_i)+\mathbb{E}\{s(W_i)v(W_i)\mid I^c_{\ell}\}][-s(W_j)v(W_j)+\mathbb{E}\{s(W_j)v(W_j)\mid I^c_{\ell}\}] \mid I^c_{\ell}\right) \\
        &= \frac{n_{\ell}}{n^2_{\ell}}\sum_{i \in I_{\ell}} \mathbb{E}\left(  [-s(W_i)v(W_i)+\mathbb{E}\{s(W_i)v(W_i)\mid I^c_{\ell}\}]^2 \mid I^c_{\ell}\right) \\
        &=\mathbb{E}([s(W)v(W)-\mathbb{E}\{s(W)v(W)\mid I^c_{\ell}\} ]^2\mid I^c_{\ell}) 
        \leq \mathbb{E} \{ s(W)^2v(W)^2\mid I^c_{\ell}\}
        \leq (\bar{\eta}+\bar{\eta}')^2\mathcal{R}(\hat{\nu}_{\ell}).
    \end{align*}
    Collecting results gives $\mathbb{P}(|\Delta_{5'\ell}|>t)\leq \frac{(\bar{\eta}+\bar{\eta}')^2\mathcal{R}(\hat{\nu}_{\ell})}{t^2}=\frac{\epsilon}{7L}$.
    Therefore with probability $1-3\epsilon/(7L)$, $|\Delta_{5'\ell}|\leq t=\left(\frac{7L}{\epsilon}\right)^{1/2}(\bar{\eta}+\bar{\eta}')\{\mathcal{R}(\hat{\nu}_{\ell})\}^{1/2}$ and similarly for $j\in\{5,6,7\}$.

\item Consider the latter term. Specializing to nonparametric confounding bridges, if $\mathbb{E}\{h_0(B)|C\}=\mathbb{E}\{g_0(A)|C\}$ and $\mathbb{E}\{g_0(A)|C'\}=\mathbb{E}(Y|C')$, then the arguments of $(\nu,\delta,\alpha,\eta)$, and hence $(s,t,u,v)$, are $(B,A,C',C)$, respectively. Therefore $\mathbb{E}\{-s(W)v(W)\mid I^c_{\ell}\}=\mathbb{E}[  \mathbb{E}\{-s(B)\mid C, I^c_{\ell}\}v(C) \mid I^c_{\ell}] $ is bounded by
\begin{align*}
 \{\mathbb{E}( [\mathbb{E}\{s(B)\mid C, I^c_{\ell}\}]^2 \mid I^c_{\ell})\}^{1/2}[\mathbb{E}\{v(C)^2\mid I^c_{\ell}\}]^{1/2} 
   =\{\mathcal{P}(\hat{\nu}_{\ell})\}^{1/2}\{\mathcal{R}(\hat{\eta}_{\ell})\}^{1/2}.
\end{align*}
Hence
$
\Delta_{5''\ell} \leq n_{\ell}^{1/2}\{\mathcal{P}(\hat{\nu}_{\ell})\}^{1/2}\{\mathcal{R}(\hat{\eta}_{\ell})\}^{1/2}=L^{-1/2}\{n\mathcal{P}(\hat{\nu}_{\ell})\mathcal{R}(\hat{\eta}_{\ell})\}^{1/2}.
$
Likewise
$
    \Delta_{5''\ell} 
    \leq n_{\ell}^{1/2}\{\mathcal{R}(\hat{\nu}_{\ell})\}^{1/2}\{\mathcal{P}(\hat{\eta}_{\ell})\}^{1/2}=L^{-1/2}\{n\mathcal{R}(\hat{\nu}_{\ell})\mathcal{P}(\hat{\eta}_{\ell})\}^{1/2}.
$

    \item Combining terms yields the desired result. \qedhere
\end{enumerate}
\end{proof}

\begin{lemma}[Oracle approximation]\label{lemma:Delta}
Suppose the conditions of Theorem~\ref{theorem:gaussian} hold. Then with probability $1-\epsilon$, $\frac{n^{1/2}}{\sigma}|\hat{\theta}-\bar{\theta}|\leq \Delta$ where $\Delta$ equals
\begin{align*}
 &\frac{7L}{2\epsilon  \sigma}
    \bigg[
    (1+\bar{\eta})\{\mathcal{R}(\hat{\nu}_{\ell})\}^{1/2}
    +(\bar{\alpha}+\bar{\eta})\{\mathcal{R}(\hat{\delta}_{\ell})\}^{1/2}
    +\bar{\sigma}_1\{\mathcal{R}(\hat{\alpha}_{\ell})\}^{1/2}
    +\bar{\sigma}_2\{\mathcal{R}(\hat{\eta}_{\ell})\}^{1/2} \\
    &\quad +\{n\mathcal{R}(\hat{\nu}_{\ell})\mathcal{R}(\hat{\eta}_{\ell})\}^{1/2}
    +\{n\mathcal{R}(\hat{\delta}_{\ell})\mathcal{R}(\hat{\alpha}_{\ell})\}^{1/2}
    +\{n\mathcal{R}(\hat{\delta}_{\ell})\mathcal{R}(\hat{\eta}_{\ell})\}^{1/2}
    \bigg].
\end{align*}
\end{lemma}

\begin{proof}
We proceed in steps.
\begin{enumerate}
    \item By Lemma~\ref{lemma:Taylor}, write
    $
    n^{1/2}(\hat{\theta}-\bar{\theta})
    =\frac{n^{1/2}}{n_{\ell}^{1/2}}\frac{1}{L} \sum_{\ell=1}^L n_{\ell}^{1/2} (\hat{\theta}_{\ell}-\bar{\theta}_{\ell}) 
    = L^{1/2}\frac{1}{L} \sum_{\ell=1}^L \sum_{j=1}^3 \Delta_{j\ell}.
$

    \item 
    Define the events
$
\mathcal{E}_{\ell}=\{\text{for all } j \in \{1,...,7\},\; |\Delta_{j\ell}|\leq t_j\}$, $ \mathcal{E}=\cap_{\ell=1}^L \mathcal{E}_{\ell}$, and $\mathcal{E}^c=\cup_{\ell=1}^L \mathcal{E}^c_{\ell}.
$
Hence by the union bound and Lemma~\ref{lemma:resid},
$
\mathbb{P}(\mathcal{E}^c)\leq \sum_{\ell=1}^L \mathbb{P}(\mathcal{E}^c_{\ell}) \leq L\frac{\epsilon}{L}=\epsilon.
$
    \item     Therefore with probability $1-\epsilon$,
\begin{align*}
 n^{1/2}|\hat{\theta}-\bar{\theta}|&\leq L^{1/2}\frac{1}{L} \sum_{\ell=1}^L \sum_{j=1}^7 |\Delta_{jk}| \leq L^{1/2}\frac{1}{L} \sum_{\ell=1}^L \sum_{j=1}^7 t_j 
 =L^{1/2}\sum_{j=1}^7 t_j.
\end{align*}
Finally, we simplify $(t_j)$. Note that $7^{1/2}<7/2$ and that for $\epsilon\leq 1$, $\epsilon^{-1/2}\leq \epsilon^{-1}$. \qedhere
\end{enumerate}
\end{proof}

\begin{lemma}[Oracle approximation: Alternative path]\label{lemma:Delta_alt}
Suppose the conditions of Theorem~\ref{theorem:gaussian} hold. Then with probability $1-\epsilon$, $\frac{n^{1/2}}{\sigma}|\hat{\theta}-\bar{\theta}|\leq \Delta$ where $\Delta$ equals
\begin{align*}
   &\frac{4 L}{\epsilon^{1/2} \sigma}
\bigg[
(1+\bar{\eta}+\bar{\eta}')\{\mathcal{R}(\hat{\nu}_{\ell})\}^{1/2}
+(\bar{\alpha}+\bar{\alpha}'+\bar{\eta}+\bar{\eta}')\{\mathcal{R}(\hat{\delta}_{\ell})\}^{1/2}
 +\bar{\sigma}_1\{\mathcal{R}(\hat{\alpha}_{\ell})\}^{1/2}
+\bar{\sigma}_2\{\mathcal{R}(\hat{\eta}_{\ell})\}^{1/2}
\bigg]\\
 &+\frac{1}{2\sigma}\bigg[
\{n\mathcal{P}(\hat{\nu}_{\ell})\mathcal{R}(\hat{\eta}_{\ell})\}^{1/2} \wedge \{n\mathcal{R}(\hat{\nu}_{\ell})\mathcal{P}(\hat{\eta}_{\ell})\}^{1/2}
+
\{n\mathcal{P}(\hat{\delta}_{\ell})\mathcal{R}(\hat{\alpha}_{\ell})\}^{1/2} \wedge \{n\mathcal{R}(\hat{\delta}_{\ell})\mathcal{P}(\hat{\alpha}_{\ell})\}^{1/2} \\
&\quad\quad+
\{n\mathcal{P}(\hat{\delta}_{\ell})\mathcal{R}(\hat{\eta}_{\ell})\}^{1/2} \wedge \{n\mathcal{R}(\hat{\delta}_{\ell})\mathcal{P}(\hat{\eta}_{\ell})\}^{1/2}
\bigg].
\end{align*}
\end{lemma}

\begin{proof}
As in Lemma~\ref{lemma:Delta}, Lemmas~\ref{lemma:Taylor} and~\ref{lemma:resid_alt} imply that with probability $1-\epsilon$, $
 n^{1/2}|\hat{\theta}-\bar{\theta}|\leq L^{1/2}\sum_{j=1}^3 t_j.
$
Note $7^{1/2}+(7/4)^{1/2}<4$ when combining terms.
\end{proof}

\begin{proof}[Proof of Theorem~\ref{theorem:gaussian}]
The steps of \citet[Theorem 1]{chernozhukov2021simple} generalize to our setting, using our new $\Delta$ defined in Lemmas~\ref{lemma:Delta} and~\ref{lemma:Delta_alt}.
\end{proof}

\subsection{Variance estimation}

Recall that $\mathbb{E}_{\ell}(\cdot)=n_{\ell}^{-1}\sum_{i\in I_{\ell}}(\cdot)$ means the average over observations in $I_{\ell}$ and $\mathbb{E}_n(\cdot)=n^{-1}\sum_{i=1}^n(\cdot)$ means the average over all observations in the sample. For $i \in  I_{\ell}$, define $ \psi_0(W_i)=\psi(W_i,\theta_0,\nu_0,\delta_0,\alpha_0,\eta_0)$ and $
    \hat{\psi}(W_i)=\psi(W_i,\hat{\theta},\hat{\nu}_{\ell},\hat{\delta}_{\ell},\hat{\alpha}_{\ell},\hat{\eta}_{\ell})$.

    \begin{lemma}[Foldwise second moment]\label{lemma:foldwise2}
$
\mathbb{E}_{\ell}[\{\hat{\psi}(W)-\psi_0(W)\}^2]\leq 7\left\{(\hat{\theta}-\theta_0)^2+\sum_{j=8}^{13} \Delta_{j\ell}\right\},
$
where
$    \Delta_{8\ell} =\mathbb{E}_{\ell}\{s(W_i)^2\}$, $
    \Delta_{9\ell}=\mathbb{E}_{\ell}[u(W_i)^2\{Y-\delta_0(W_i)\}^2]$, $
    \Delta_{10\ell}=\mathbb{E}_{\ell}[v(W_i)^2\{\delta_0(W_i)-\nu_0(W_i)\}^2]$, $
     \Delta_{11\ell} =\mathbb{E}_{\ell}\{\hat{\alpha}_{\ell}(W_i)^2t(W_i)^2\}$, $
    \Delta_{12\ell}=\mathbb{E}_{\ell}\{\hat{\eta}_{\ell}(W_i)^2t(W_i)^2\}$, $
    \Delta_{13\ell}=\mathbb{E}_{\ell}\{\hat{\eta}_{\ell}(W_i)^2u(W_i)^2\}$.

\end{lemma}

\begin{proof}
Write $\hat{\psi}(W_i)-\psi_0(W_i)$ equal to
\begin{align*}
  &\hat{\nu}_{\ell}(W_i)+\hat{\alpha}_{\ell}(W_i)\{Y_i-\hat{\delta}_{\ell}(W_i)\}+\hat{\eta}_{\ell}(W_i)\{\hat{\delta}_{\ell}(W_i)-\hat{\nu}_{\ell}(W_i)\}-\hat{\theta}\\
    &\quad -\left[\nu_0(W)+\alpha_0(W_i)\{Y_i-\delta_0(W_i)\}+\eta_0(W_i)\{\delta_0(W_i)-\nu_0(W_i)\}-\theta_0\right] \\
    &\quad\pm \hat{\alpha}_{\ell}\{Y-\delta_0(W_i)\}\pm \hat{\eta}_{\ell}\{\delta_0(W_i)-\nu_0(W_i)\}\\
    &=(\theta_0-\hat{\theta})+s(W_i)+u(W_i)\{Y-\delta_0(W_i)\}+v(W_i)\{\delta_0(W_i)-\nu_0(W_i)\}\\
    &\quad -\hat{\alpha}_{\ell}(W_i)t(W_i)+\hat{\eta}_{\ell}(W_i)t(W_i)-\hat{\eta}_{\ell}(W_i)u(W_i).
\end{align*}
Apply parallelogram law across the seven terms, and take $\mathbb{E}_{\ell}(\cdot)$ of both sides.
\end{proof}

\begin{lemma}[Residuals]\label{lemma:resid2}
Suppose the conditions of Theorem~\ref{theorem:variance} hold. Then with probability $1-\epsilon'/(2L)$,
$
    |\Delta_{8\ell}|\leq t_8=\frac{12L}{\epsilon'}\mathcal{R}(\hat{\nu}_{\ell})$, $
    |\Delta_{9\ell}|\leq t_9=\frac{12L}{\epsilon'}\bar{\sigma}_1^2 \mathcal{R}(\hat{\alpha}_{\ell})$, $
    |\Delta_{10\ell}|\leq t_{10}=\frac{12L}{\epsilon'}\bar{\sigma}_2^2 \mathcal{R}(\hat{\eta}_{\ell})$, $
     |\Delta_{11\ell}|\leq t_{11}=\frac{12L}{\epsilon'}(\bar{\alpha}')^2 \mathcal{R}(\hat{\delta}_{\ell})$, $
    |\Delta_{12\ell}|\leq t_{12}=\frac{12L}{\epsilon'}(\bar{\eta}')^2 \mathcal{R}(\hat{\delta}_{\ell})$, $
    |\Delta_{13\ell}|\leq t_{13}=\frac{12L}{\epsilon'}(\bar{\eta}')^2 \mathcal{R}(\hat{\alpha}_{\ell}).$
\end{lemma}

\begin{proof}
The steps are analogous to Lemma~\ref{lemma:resid}.
\end{proof}

\begin{lemma}[Oracle approximation]\label{lemma:Delta2}
Suppose the conditions of Lemma~\ref{lemma:resid2} hold. Then with probability $1-\epsilon'/2$, $\mathbb{E}_n[\{\hat{\psi}(W)-\psi_0(W)\}^2]\leq \Delta'$ where 
\begin{align*}
    \Delta'=7(\hat{\theta}-\theta_0)^2+\frac{84 L}{\epsilon'}\left[
    \mathcal{R}(\hat{\nu}_{\ell})
    +\{(\bar{\alpha}')^2+(\bar{\eta}')^2\}\mathcal{R}(\hat{\delta}_{\ell})
    +\{(\bar{\eta}')^2+\bar{\sigma}_1^2\}\mathcal{R}(\hat{\alpha}_{\ell})
    +\bar{\sigma}_2^2\mathcal{R}(\hat{\eta}_{\ell})
    \right].
\end{align*}
\end{lemma}

\begin{proof}
The steps are analogous to Lemma~\ref{lemma:Delta}, appealing to Lemmas~\ref{lemma:foldwise2} and~\ref{lemma:resid2}.
\end{proof}

\begin{lemma}[Markov inequality]\label{lemma:other_half}
If $\chi<\infty$, then with probability $1-\epsilon'/2$
$
|\mathbb{E}_n\{\psi_0(W)^2\}-\sigma^2|\leq \Delta''=\left(\frac{2}{\epsilon'}\right)^{1/2}\frac{\chi^2}{n^{1/2}}.
$
\end{lemma}

\begin{proof}
The steps of \citet[Proposition S11]{chernozhukov2021simple} generalize to our setting, using our new moments.
\end{proof}

\begin{proof}[Proof of Theorem~\ref{theorem:variance}]
The steps of \citet[Theorem 3]{chernozhukov2021simple} generalize to our setting, using our new $(\Delta',\Delta'')$ defined in Lemmas~\ref{lemma:Delta2} and~\ref{lemma:other_half}, respectively.
\end{proof}

\section{Extension to causal functions}\label{sec:local}

We revisit causal functions to clarify how the main inference result encompasses them. We pointwise approximate the causal function $\theta_0(v)$, e.g. the heterogeneous long term effect for the subpopulation with $V=v$, by the local functional $\theta_{\lambda}(v)$, e.g. the heterogeneous long term effect for the subpopulation with $V$ close to $v$. We emphasize which quantities depend on the bandwidth by indexing by $\lambda$, e.g. $\sigma_{\lambda}^2$ is the variance of the approximating Gaussian when $\lambda>0$. Moreover, we write $h_{j,\lambda}=\ell_{\lambda}h_j$ where $h_{j,\lambda}$ is the nuisance of the local functional and $h_j$ is the nuisance of the corresponding global functional. For example, $h_{j,\lambda}$ is the nuisance for the (approximate) heterogeneous long term effect, while $h_j$ is the nuisance for the average long term effect.

\subsection{Main result}

\begin{theorem}[Key quantities for causal functions in Algorithm~\ref{algo:dml}]\label{theorem:local}
    Suppose that Assumption~\ref{assumption:orthogonal} holds. Suppose that bounded balancing weight, residual variance, density, derivative, and kernel conditions hold, which are defined below. Then for the local functional $\theta_{\lambda}$, suppressing the index $v$,
$
\kappa_{\lambda}/\sigma_{\lambda} \lesssim \lambda^{-1/6}$, $\sigma_{\lambda} \asymp \lambda^{-1/2}$, $\kappa_{\lambda}\lesssim \lambda^{-2/3}$, $\chi_{\lambda}\lesssim \lambda^{-3/4},
$
 and 
$
\bar{\sigma}_{1,\lambda}\lesssim \lambda^{-1}\bar{\sigma}_{1}$, $\bar{\sigma}_{2,\lambda}\lesssim \lambda^{-1}\bar{\sigma}_{2} $, $\Delta_{\lambda} \lesssim n^{1/2} \lambda^{s+1/2},
$
where $s$ is the order of differentiability defined below. Moreover,
$
\|\hat{h}_{1,\lambda}-h_{1,\lambda}\|_2\lesssim \lambda^{-1}\|\hat{h}_1-h_1\|_2$, $\|T_1(\hat{h}_{1,\lambda}-h_{1,\lambda})\|_2\lesssim \lambda^{-1} \|T_1(\hat{h}_1-h_1)\|_2$, $\|\hat{h}_{2,\lambda}-h_{2,\lambda}\|_2\lesssim \lambda^{-1}\|\hat{h}_2-h_2\|_2$, $\|T_2(\hat{h}_{2,\lambda}-h_{2,\lambda})\|_2\lesssim \lambda^{-1} \|T_2(\hat{h}_2-h_2)\|_2.
$
\end{theorem}

\begin{corollary}[Multiple robustness to ill posedness: Causal functions]\label{cor:ci_local}
    Suppose Assumptions~\ref{assumption:orthogonal} and~\ref{assumption:regular} hold as well as the regularity conditions of Theorem~\ref{theorem:local}. 
Finally assume the following are $o_p(1)$: the bandwidth rates $
n^{-1/2}\lambda^{-3/2}$ and $n^{1/2}\lambda^{s+1/2}
$;
the individual rates {\small $(\lambda^{-1}+\lambda^{-1/2}\bar{h}_4+\lambda^{-1/2}\bar{h}_4')\|\hat{h}_1-h_1\|_2$,
   $(\lambda^{-1/2}\bar{h}_3+\lambda^{-1}\bar{h}_3'+\lambda^{-1/2}\bar{h}_4+\lambda^{-1}\bar{h}_4')\|\hat{h}_2-h_2\|_2$,
    $(\bar{h}_4'+\lambda^{-1}\bar{\sigma}_y\|\hat{h}_3-h_3\|_2$,
 $\lambda^{-1}\bar{\sigma}_2\|\hat{h}_4-h_4\|_2$}; and the product rates
\begin{enumerate}
    \item {\small $\lambda^{-1/2}n^{1/2}\{\|\hat{h}_1-h_1\|_2\|\hat{h}_4-h_4\|_2 \wedge \|T_1(\hat{h}_1-h_1)\|_2\|\hat{h}_4-h_4\|_2 \wedge \|\hat{h}_1-h_1\|_2\|T_4(\hat{h}_4-h_4)\|_2\} $};
     \item {\small $\lambda^{-1/2}n^{1/2}\{\|\hat{h}_2-h_2\|_2\|\hat{h}_3-h_3\|_2 \wedge\|T_2(\hat{h}_2-h_2)\|_2\|\hat{h}_3-h_3\|_2 \wedge \|\hat{h}_2-h_2\|_2\|T_3(\hat{h}_3-h_3)\|_2\} $};
      \item {\small $\lambda^{-1/2}n^{1/2}\{\|\hat{h}_2-h_2\|_2\|\hat{h}_4-h_4\|_2 \wedge \|T_2(\hat{h}_2-h_2)\|_2\|\hat{h}_4-h_4\|_2 \wedge \|\hat{h}_2-h_2\|_2\|T_4(\hat{h}_4-h_4)\|_2\} $}.
\end{enumerate}
Then
$
\hat{\theta}_{\lambda}\overset{p}{\rightarrow}\theta_0$, $ \frac{\sqrt{n}}{\sigma_{\lambda}}(\hat{\theta}_{\lambda}-\theta_0)\overset{d}{\rightarrow}\mathcal{N}(0,1)$, and $
\mathbb{P} \{\theta_0 \in  (\hat{\theta}_{\lambda}\pm 1.96\hat{\sigma}_{\lambda} n^{-1/2} )\}\rightarrow 0.95.
$
\end{corollary}

\subsection{Extended notation}

While proving this result, we use the notation of Appendix~\ref{sec:neyman} to eliminate some subscripts. We emphasize which quantities are diverging sequences for local functionals by indexing with the bandwidth $\lambda$. We study a function of the variable $V\subset W$, i.e. $\theta_0(v)$, which we approximate with $\theta_{0,\lambda}(v)=\mathbb{E}\{\ell_{\lambda}(V)\nu_0(W)\}=\mathbb{E}\{\nu_{0,\lambda}(W)\}$. Denote the localized moment function
$$
\psi_{\lambda}(W,\theta_{\lambda},\nu_{\lambda},\delta_{\lambda},\alpha,\eta)= \nu_{\lambda}(W)+\alpha(W)\{Y_{\lambda}-\delta_{\lambda}(W)\}+\eta(W)\{\delta_{\lambda}(W)-\nu_{\lambda}(W)\}-\theta_{\lambda}
$$
where $\nu_{\lambda}(W)=\ell_{\lambda}(V) \nu(W)$, $\delta_{\lambda}(W)=\ell_{\lambda}(V) \delta(W)$, and $  Y_{\lambda}= \ell_{\lambda}(V)Y$. To use Theorem~\ref{theorem:gaussian}, we reduce the rates for the localized nuisances $(\hat{\nu}_{\lambda},\hat{\delta}_{\lambda})$ to the rates for the global nuisances $(\hat{\nu},\hat{\delta})$. Intuitively, we expect the former to be slower than the latter.

The moments of the localized moment function are
$
0=\mathbb{E}\{\psi_{0,\lambda}(W)\}$, $\sigma_{\lambda}^2=\mathbb{E}\{\psi_{0,\lambda}(W)^2\}$, $\kappa_{\lambda}^3=\mathbb{E}\{|\psi_{0,\lambda}(W)|^3\}$, and $\chi_{\lambda}^4=\mathbb{E}\{\psi_{0,\lambda}(W)^4\}.
$
The moments $(\sigma_{\lambda},\kappa_{\lambda},\chi_{\lambda})$ are indexed by $\lambda$, so a complete analysis must also characterize how these parameters diverge as the bandwidth $\lambda$ vanishes. Doing so will verify the regularity condition on moments and also pin down the nonparametric rate of Gaussian approximation $\sigma_{\lambda} n^{-1/2}$.

Finally, the residual variances must also be updated. With localization, they become
$
\mathbb{E}[\{Y_{\lambda}-\delta_{0,\lambda}(W)\}^2 \mid W ]\leq \bar{\sigma}_{1,\lambda}^2$ and $\mathbb{E}[\{\delta_{0,\lambda}(W)-\nu_{0,\lambda}(W_1)\}^2 \mid W_1 ]\leq \bar{\sigma}_{2,\lambda}^2.
$
A complete analysis must also characterize how these parameters diverge as bandwidth $\lambda$ vanishes.

We restate the conclusions of Theorem~\ref{theorem:local} that we wish to prove in this alternative notation. Suppose that the global residual variances are finite. Suppose bounded balancing weight, residual, density, derivative, and kernel conditions hold. Then for local functionals, 
$
\kappa_{\lambda}/\sigma_{\lambda} \lesssim \lambda^{-1/6}$, $\sigma_{\lambda} \asymp \lambda^{-1/2}$, $\kappa_{\lambda}\lesssim \lambda^{-2/3}$, $\chi_{\lambda}\lesssim \lambda^{-3/4}
$
and
$
\bar{\sigma}_{1,\lambda}\lesssim \lambda^{-1}\bar{\sigma}_{1}$, $\bar{\sigma}_{2,\lambda}\lesssim \lambda^{-1}\bar{\sigma}_{2} $, $\Delta_{\lambda} \lesssim n^{1/2} \lambda^{s+1/2}
$
where $s$ is the order of differentiability. Moreover,
$
\mathcal{R}(\hat{\nu}_{\ell,\lambda})\lesssim \lambda^{-2} \mathcal{R}(\hat{\nu}_{\ell})$, $\mathcal{P}(\hat{\nu}_{\ell,\lambda})\lesssim \lambda^{-2} \mathcal{P}(\hat{\nu}_{\ell})$, $\mathcal{R}(\hat{\delta}_{\ell,\lambda})\lesssim \lambda^{-2} \mathcal{R}(\hat{\delta}_{\ell})$, $\mathcal{P}(\hat{\delta}_{\ell,\lambda})\lesssim \lambda^{-2} \mathcal{P}(\hat{\delta}_{\ell}).
$

\subsection{Oracle moments}

To lighten notation, we write the local weighting as $\ell=\ell_{\lambda}$. We also suppress the arguments of functions and define $
    U_0=\nu_0-\mathbb{E}(\nu_0)$, $
    U_1=Y-\delta_0$, $
    U_2=\delta_0-\nu_0$
so that
$
\psi_{0,\lambda}=\ell\cdot  (U_0+\alpha_0U_1+\eta_0U_2).
$
Finally, we lighten notation by defining
$
\|W\|_{\mathbb{P},q}=\{\mathbb{E}(W^q)\}^{1/q}.
$

\begin{lemma}[Oracle moments for local functionals]\label{lemma:local}
Suppose there exist 
$$(\underline \alpha, \bar{\alpha},\underline \eta, \bar{\eta}, \underline{\sigma}_0, \bar{\sigma}_0, \underline{\sigma}_1, \bar{\sigma}_1,\underline{\sigma}_2, \bar{\sigma}_2,\underline f, \bar f, \bar f', \lambda_0)$$ bounded away from zero and above such that the following conditions hold.
\begin{enumerate}
\item Control of balancing weights: $\underline{\alpha}\leq \|\alpha_0\|_{\infty}\leq \bar{\alpha}$, $
    \underline{\eta}\leq \|\eta_0\|_{\infty}\leq \bar{\eta}$.
\item Control of residual moments: for $q\in\{2,3,4\}$, $\underline{\sigma}_0\leq \|U_0|V\|_{\mathbb{P},q} \leq \bar{\sigma}_0$, $\underline{\sigma}_1\leq \|U_1|W\|_{\mathbb{P},q} \leq \bar{\sigma}_1$, $\underline{\sigma}_2\leq \|U_2|W_1\|_{\mathbb{P},q} \leq \bar{\sigma}_2.$
    \item Bounded density: the density $f_V$ obeys, for all $v'\in N_{\lambda_0}(v)=(v':|v'-v|\leq \lambda_0)$,
    $
\underline f \leq f_V(v') \leq \bar f$ and $|\partial f_V(v')| \leq \bar f'.
    $
\end{enumerate}
Then 
$
\frac{\kappa_{\lambda}}{\sigma_{\lambda}} \lesssim \lambda^{-1/6}$, $\sigma_{\lambda} \asymp \lambda^{-1/2}$, $\kappa_{\lambda}\lesssim \lambda^{-2/3} $, $\chi_{\lambda}\lesssim \lambda^{-3/4}.
$
\end{lemma}

\begin{proof}
We extend \citet[Lemma 3.4]{chernozhukov2018global}. We proceed in steps.
\begin{enumerate}
    \item Observe that $\sigma_{\lambda}^2
     =\mathbb{E}\{\ell^2\cdot (U_0^2+\alpha_0^2U_1^2+\eta_0^2U_2^2+2\alpha_0U_0U_1+2\eta_0U_0U_2+2\alpha_0\eta_0U_1U_2)\}$ equals $\mathbb{E}\{\ell^2\cdot (U_0^2+\alpha_0^2U_1^2+\eta_0^2U_2^2)\}$ by Assumption~\ref{assumption:orthogonal}. In particular, we use
$
      \mathbb{E}[\{\alpha(W)-\alpha_0(W)\}U_1]=0$ and $\mathbb{E}[\{\eta(W)-\eta_0(W)\}U_2]=0.
$
Hence
$$
(\underline{\sigma}_0^2+\underline{\alpha}^2\underline{\sigma}_1^2 +\underline{\eta}^2\underline{\sigma}_2^2)\|\ell\|_{\mathbb{P},2}
\leq  \sigma_{\lambda}^2\leq  
(\bar{\sigma}_0^2+\bar{\alpha}^2\bar{\sigma}_1^2 +\bar{\eta}^2\bar{\sigma}_2^2)\|\ell\|_{\mathbb{P},2}.
$$
Moreover, $
  \|\psi_{0,\lambda}\|_{\mathbb{P},q} \leq   (\bar{\sigma}_0+\bar{\alpha}\bar{\sigma}_1 +\bar{\eta}\bar{\sigma}_2)\|\ell\|_{\mathbb{P},q}.
$
In summary,
$
 \| \ell\|_{\mathbb{P},2}  \lesssim \sigma_{\lambda} \lesssim \| \ell\|_{\mathbb{P},2}$ and $
\|\psi_{0,\lambda}\|_{\mathbb{P},q}
\lesssim  \| \ell \|_{\mathbb{P},q}.
$
    
    \item 
    Consider the change of variables $u=(v'-v)/\lambda$ so that $\mathrm{d} u  = \lambda^{-1} \mathrm{d} v'$. Hence
    \begin{align*}
        \| \ell\|^q_{\mathbb{P},q} \omega^q &=
        \| \ell  \omega\|^q_{\mathbb{P},q} 
        = \int \lambda^{-q}\left|K \left(\frac{v'-v}{\lambda}\right)\right|^q f_V(v') \mathrm{d} v' 
        =  \int \lambda^{-(q-1) }|K (u)|^q f_V(v - u \lambda)  \mathrm{d} u .
    \end{align*}
It follows that
$
 \lambda^{-(q - 1)/q}  \underline f^{1/q}  \left(\int |K|^q\right)^{1/q} 
 \leq \| \ell\|_{\mathbb{P},q} \omega 
 \leq  \lambda^{-(q - 1)/q}  \bar f^{1/q}  \left(\int |K|^q\right)^{1/q}.
$
Further, we have that
$
\omega  = \int  \lambda^{-1} K\left(\frac{v'-v}{\lambda}\right) f_V(v') \mathrm{d} v' = \int   K(u) f_V(v- u \lambda) \mathrm{d} u
$
and 
$
\int   K(u) f_V(v-0u) \mathrm{d} u=\int   K(u) f_V(v) \mathrm{d} u=f_V(v).
$
Using the Taylor expansion in $\lambda$ around $\lambda=0$ and the Holder inequality, there exist some $\tilde{\lambda}\in[0,\lambda]$ such that
$$
|\omega - f_V(v)| =  \left|    \lambda \int   K(u) \partial_v f_V(v- u \tilde{\lambda}) u \mathrm{d} u \right |  \leq   \lambda \bar f' \int |u|| K(u)| du.
$$
Hence there exists some $\lambda_1\in(\lambda,\lambda_0)$ depending only on $(K, \bar f', \underline{f}, \bar f)$ such that
$ \underline{f}/2  \leq \omega \leq 2 \bar f.$
In summary,
$$
 \lambda^{- (q - 1)/q}  \underline f^{1/q}  \left(\int |K|^q\right)^{1/q} \frac{1}{2 \bar f} \leq \| \ell\|_{\mathbb{P},q}  \leq  \lambda^{-(q - 1)/q}  \bar f^{1/q}  \left(\int |K|^q\right)^{1/q} \frac{2}{\underline{f}}
$$  
which implies
$
 \lambda^{- (q - 1)/q}  \lesssim \| \ell\|_{\mathbb{P},q}  \lesssim  \lambda^{-(q - 1)/q}.
$
    \item 
For all $\lambda < \lambda_1$,
$
\sigma_{\lambda} \asymp \| \ell\|_{\mathbb{P},2}$, $\|\psi_{0,\lambda}\|_{\mathbb{P},2}\lesssim \|\ell\|_{\mathbb{P},q}$, $\|\ell\|_{\mathbb{P},q}\asymp \lambda^{-(q-1)/q}.
$ \qedhere
\end{enumerate}
\end{proof}

\subsection{Residual variances and mean square rates}

\begin{lemma}[Residual variance for local functionals]\label{lemma:bounded_RR_local}
Suppose there exist $(\bar{\sigma}_1,\bar{\sigma}_2,\underline f, \bar f, \bar f', \lambda_0,\bar{K})$ bounded away from zero and above such that the following conditions hold.
\begin{enumerate}
  \item Bounded residual variance: 
  $
   \|U_1|W\|_{\mathbb{P},2} \leq \bar{\sigma}_1$, $\|U_2|W_1\|_{\mathbb{P},2} \leq \bar{\sigma}_2.
  $
    \item Bounded density: the density $f_V$ obeys, for all $v'\in N_{\lambda_0}(v)$,
    $
  0< \underline f \leq f_V(v') \leq \bar f$, $|\partial f_V(v')| \leq \bar f'.
    $
    \item Bounded kernel: $|K(u)|\leq \bar{K}$.
\end{enumerate}
Then
$
\bar{\sigma}_{1,\lambda}\lesssim \lambda^{-1}\bar{\sigma}_{1}$ and $\bar{\sigma}_{2,\lambda}\lesssim \lambda^{-1}\bar{\sigma}_{2}.
$
\end{lemma}

\begin{proof}
Write
$
\bar{\sigma}_{1,\lambda}=\|\ell \cdot  U_1 |W\|_{\mathbb{P},2} \leq \|\ell\|_{\infty}\| U_1 |W\|_{\mathbb{P},2}.
$
By the proof of Lemma~\ref{lemma:local},
$
\|\ell\|_{\infty}=\left\| \frac{1}{\lambda \omega}K\left(\frac{v'-v}{\lambda}\right) \right\|_{\infty}\leq \bar{K}\frac{1}{\lambda \omega}\leq \bar{K}\frac{2}{\lambda \underline f}.
$
Therefore
$
\bar{\sigma}_{1,\lambda}\leq \bar{K}\frac{2}{\lambda \underline f} \cdot \bar{\sigma}_{\lambda} \lesssim \lambda^{-1}\bar{\sigma}_{1}.
$
The argument for $\bar{\sigma}_{2,\lambda}$ is identical.
\end{proof}

A natural choice of estimator $\hat{\nu}_{\lambda}$ for $\nu_{0,\lambda}$ is the localization $\ell_{\lambda}$ times an estimator $\hat{\nu}$ for $\nu_0$. We prove that this choice translates global nuisance parameter rates into local nuisance parameter rates under mild regularity conditions.

\begin{lemma}[Translating global rates to local rates]\label{lemma:translate_RR}
Suppose the conditions of Lemma~\ref{lemma:bounded_RR_local} hold. Then
$
\mathcal{R}(\hat{\nu}_{\ell,\lambda})\lesssim \lambda^{-2} \mathcal{R}(\hat{\nu}_{\ell})$, $\mathcal{P}(\hat{\nu}_{\ell,\lambda})\lesssim \lambda^{-2} \mathcal{P}(\hat{\nu}_{\ell})$, $\mathcal{R}(\hat{\delta}_{\ell,\lambda})\lesssim \lambda^{-2} \mathcal{R}(\hat{\delta}_{\ell})$, $\mathcal{P}(\hat{\delta}_{\ell,\lambda})\lesssim \lambda^{-2} \mathcal{P}(\hat{\delta}_{\ell}).
$
\end{lemma}

\begin{proof}
We generalize \citet[Lemma 10]{chernozhukov2021simple}.
Write
\begin{align*}
    \mathcal{R}(\hat{\nu}_{\ell,\lambda})
    =\mathbb{E}[\{\ell_{\lambda}(V)\hat{\nu}_{\ell}(W)-\ell_{\lambda}(V)\nu_0(W)\}^2\mid I^c_{\ell}] 
    \leq \|\ell_{\lambda}\|^2_{\infty} \mathcal{R}(\hat{\nu}_{\ell}).
\end{align*}
From the proof of Lemma~\ref{lemma:bounded_RR_local}, $\|\ell_{\lambda}\|_{\infty}\lesssim \lambda^{-1}$. The remaining results are identical.
\end{proof}

\subsection{Approximation error}

Finally, we characterize the finite sample approximation error $\Delta_{\lambda}=
n^{1/2} \sigma^{-1}|\theta_{0,\lambda}-\theta_0|$ where 
$
\theta_{0}=\lim_{\lambda \rightarrow 0} \theta_{0,\lambda}.
$ 
Here, $\Delta_{\lambda}$ is bias from approximating a causal function using sequence of local functionals.
We define $m(v)= \mathbb{E} \left\{\nu_0(W) \mid V=v\right\}$ 
to lighten notation.

\begin{lemma}[Approximation error from localization \citep{chernozhukov2018global}]\label{lemma:approx}
Suppose there exist constants $(\lambda_0,s, \bar g_{s}, \bar f_{s}, \underline f, \bar{g})$ bounded away from zero and above such that the following conditions hold.
\begin{enumerate}
    \item Differentiability: on $N_{\lambda_0}(v)=\{v':|v'-v|\leq \lambda_0\}$, $m(v')$ and $f_V(v')$ are differentiable to the integer order $d$.
    \item Bounded derivatives: let $s= d\wedge o$ where $o$ is the order of the kernel $K$. Let $\partial^s_v$ denote the $s$ order derivative $\partial^{s}/(\partial v)^{s}$. Assume
    $
\sup_{ v' \in N_{\lambda_0}(v)}  \| \partial^{s}_v (m(v') f_V(v')) \|_{op} \leq \bar g_{s}$, $\sup_{ v' \in N_{\lambda_0}(v) } \|\partial^{s}_v f_V(v')  \|_{op} \leq \bar f_{s}$, $\inf_{ v' \in N_{\lambda_0}(v) } f_V(v') \geq \underline{f}.
$
\item Bounded conditional formula:
$
m(v)f_V(v)\leq \bar{g}.
$
\end{enumerate}
Then there exist constants $(C,\lambda_1)$ depending only on $(\lambda_0,K, s, \bar g_{s}$,  $\bar f_{s}$, $\underline f, \bar{g})$ such that for all $\lambda_1\in(\lambda,\lambda_0)$, $|\theta_{0,\lambda}-\theta_0|\leq C \lambda^s.$
In summary,
$
\Delta_{\lambda} \lesssim n^{1/2} \lambda^{s+1/2}.
$
\end{lemma}

\section{Proofs of propositions and corollaries}\label{sec:analytic}

\subsection{Analytical examples of relative well posedness}

\begin{proof}[Proof of Proposition~\ref{prop:linear}]
    Here, $(g_0,h_0)$ are real numbers solving $\mathbb{E}(YC')=g\mathbb{E}(AC')$ and $ h\mathbb{E}(BC)=g\mathbb{E}(AC)$. Clearly $S:g\mapsto \mathbb{E}(AC') g$ has operator norm $|\mathbb{E}(AC')|$ and $T:(0,g)\mapsto \mathbb{E}(BC) \cdot 0-\mathbb{E}(AC)g$ has operator norm $|\mathbb{E}(AC)|$.
\end{proof}

\begin{proof}[Proof of Proposition~\ref{prop:gaussian}]
    By \citet[eq. 2.7]{hoderlein2011demand},
    $\sigma_{j}(T_g)=|\rho_A|^{j}$ and  $\sigma_{j}(S)=|\rho'|^{j}$.  By Mehler's formula, the right singular functions coincide as Hermite polynomials of $A$. Therefore, 
 $\|(S^*S+\mu'I)^{-1}T_g^*T_g\|_{\op}=\frac{\sigma_j^2(T)}{\sigma_j^2(S)+\mu'}=\frac{|\rho_A|^{2j}}{|\rho'|^{2j}+\mu'}\leq \frac{|\rho_A|^{2j}}{|\rho'|^{2j}}$.
\end{proof}

\begin{proposition}[Nonlinear models with multivariate Gaussian data]
Suppose that $(A,B,C,C')$ are each multivariate normal vectors with mean zero and identity covariance, possibly of differing dimensions. Assumption~\ref{assumption:posedness} holds when $\max_{k \leq r \wedge r'} \frac{|\rho_k|}{|\rho_k'|} 
= O(1)$, where $\rho_{k}$ and $\rho_{k}'$
are the $k$th canonical correlations between $(A,C)$ and $(A, C')$, respectively. In the maximization, $r$ and $r'$ are the ranks of $\textsc{cov}(A,C)$ and $\textsc{cov}(A,C')$, respectively.
\end{proposition}

\begin{proof}
    The result is a straightforward generalization of Proposition~\ref{prop:gaussian}, using the tensorized Mehler's formula.
\end{proof}

The quantity $\max _{k \leq r \wedge r'} \frac{|\rho_k|}{|\rho_k'|}$ uniformly compares the relevance of the two instruments $C$ and $C'$ for $A$, in each canonical direction. We require that $C'$ is at least as strong an instrument as $C$ in every such direction. Here,  $k\leq r \wedge r'$ means  that we only compare the nondegenerate directions.

The result naturally generalizes to random vectors that are not mean zero and that do not have identity covariances, using heavier notation.

\subsection{Corollaries with compounding ill posedness}

\begin{proof}[Proof of Corollary~\ref{cor:limit_sequential}]
    The result is immediate from Lemma~\ref{lemma:innermax_weak}.
\end{proof}

\begin{proof}[Proof of Corollary~\ref{cor:L2}]
    To lighten notation, let $\bar{r}=\max\{\delta_n,\|\hat{g}-g_0\|_2\}$. We minimize the mean square bound in Theorem~\ref{theorem:L2}: $\|\hat{h}-h_0\|_2^2=O\{\mu^{\min(\beta,1)}\|w_h\|^2_2+\mu^{-1}\bar{r}^2\}$. 
    
    In the case $\beta \geq  1$, the first order condition is $\|w_h\|^2_2-\mu^{-2}\bar{r}^2=0$, suggesting $\mu\asymp \bar{r}$, $R_n\asymp \mu^2+\bar{r}^2\asymp \bar{r}^2$, and $\mu^{-1}R_n\asymp \bar{r}$. 
    
    In the case $\beta < 1$, the first order condition is $\mu^{\beta-1}\beta\|w_h\|^2_2-\mu^{-2}\bar{r}^2=0$, suggesting $\mu\asymp \bar{r}^{\frac{2}{\beta+1}}$, $R_n\asymp \mu^{\beta+1}+\bar{r}^2\asymp \bar{r}^2$, and $\mu^{-1}R_n\asymp \bar{r}^{2-\frac{2}{\beta+1}}=\bar{r}^{\frac{2\beta}{\beta+1}}$.
\end{proof}

\begin{proof}[Proof of Corollary~\ref{cor:L2_compound}]
    By Corollary~\ref{cor:L2}, we set $\mu_g=\delta_n^{\frac{2}{\min(\beta_g',1)+1}}$ to obtain $\|\hat{g}-g_0\|_2^2=O\left\{\delta_n^{2\textsc{well}(\beta_g')}\right\}$, which dominates $\delta_n^2$. Hence $\bar{r}^2=\delta_n^{2\textsc{well}(\beta_g')}$. Thus by Corollary~\ref{cor:L2}, we set $\mu_h=\bar{r}^\frac{2}{\min(\beta_h,1)+1}$ to obtain $\|T(\hat{h}-h_0)\|_2^2=O \left(\bar{r}^2\right) $ and $\|\hat{h}-h_0\|_2^2=O\left\{\bar{r}^{2\textsc{well}(\beta_h)}\right\}$.
\end{proof}

\begin{proposition}[Less robustness to ill posedness]\label{prop:robust}
    Suppose the conditions of Corollary~\ref{cor:L2_compound} hold for $(\hat{h}_1,\hat{h}_2,\hat{h}_3,\hat{h}_4)$. Write the largest critical radius as $\bar{\delta}_n=\tilde{O}(n^{-\alpha})$, and the source conditions as $(\vec{\beta}_1,\beta_2,\vec{\beta}_3,\beta_4)$.\footnote{Recall that for several leading examples, $(h_1,h_3)$ are nested NPIVs while $(h_2,h_4)$ are NPIVs.} Set the regularizations as in Corollary~\ref{cor:L2_compound}. Suppose $\sigma\asymp n^{\gamma}$. Then product rate condition of Theorem~\ref{theorem:ci} are satisfied when
    (i)
    $ \gamma + \alpha \left\{
  \textsc{well}(\beta_{1g}')\textsc{well}(\beta_{1h}) + 1
  \right\}>1/2;
    $
  %
    (ii) $ \gamma + \alpha \left\{
  \textsc{well}(\beta_{3g}')\textsc{well}(\beta_{3h}) + 1
  \right\}>1/2;
    $
    (iii) $
    \gamma + \alpha  \left\{
 \textsc{well}(\beta_2) \vee \textsc{well}(\beta_{4}) +1
  \right\} >1/2.$
\end{proposition}

\begin{proof}
To begin, recall the rates we have derived.   By Corollary~\ref{cor:L2},
$\|S(\hat{g}-g_0)\|_2=O (\bar{\delta}_n)$ and $\|\hat{g}-g_0\|_2=O \left\{\bar{\delta}_n^{\textsc{well}(\beta_g')}\right\}$. By Corollary~\ref{cor:L2_compound},
    $\|T(\hat{h}-h_0)\|_2=O \left\{\bar{\delta}_n^{\textsc{well}(\beta_g')}\right\}$ and $\|\hat{h}-h_0\|_2=O \left\{\bar{\delta}_n^{\textsc{well}(\beta_h)\textsc{well}(\beta_g')}\right\}$.
\begin{enumerate}
    \item The first product rate condition of Theorem~\ref{theorem:ci} is satisfied when $$
  n^{1/2}\sigma^{-1}\{\|T_1(\hat{h}_1-h_1)\|_2\|\hat{h}_4-h_4\|_2 \wedge \|\hat{h}_1-h_1\|_2\|T_4(\hat{h}_4-h_4)\|_2\}=o_p(1).
    $$
    Since $h_1$ is a nested NPIV and $h_4$ is an NPIV, the former term is
    $$
  n^{1/2}\sigma^{-1}\|T_1(\hat{h}_1-h_1)\|_2\|\hat{h}_4-h_4\|_2
  =O\left[n^{\frac{1}{2}-\gamma-\alpha\left\{
 \textsc{well}(\beta_{1g}')
  +
 \textsc{well}(\beta_4)
  \right\}}\right].
    $$
    Meanwhile the latter term is
    $$
  n^{1/2}\sigma^{-1}\|\hat{h}_1-h_1\|_2\|T_4(\hat{h}_4-h_4)\|_2
  =O\left[n^{\frac{1}{2}-\gamma-\alpha\left\{
  \textsc{well}(\beta_{1h})\textsc{well}(\beta_{1g}') + 1
  \right\}}\right].
    $$
    In summary, the first product rate condition requires
    $$
    \left[1/2-\gamma-\alpha\left\{
 \textsc{well}(\beta_{1g}')
  +
 \textsc{well}(\beta_4)
  \right\} \right]\wedge 
  \left[1/2-\gamma-\alpha\left\{
  \textsc{well}(\beta_{1h})\textsc{well}(\beta_{1g}') + 1
  \right\} \right]<0.
    $$
    Rearranging,
    $
    \frac{1}{2} < \gamma + \alpha \left[ \left\{
 \textsc{well}(\beta_{1g}')
  +
 \textsc{well}(\beta_4)
  \right\} \vee  \left\{
  \textsc{well}(\beta_{1h})\textsc{well}(\beta_{1g}') + 1
  \right\}\right].
    $
    The latter branch of the maximum weakly dominates the former since $\textsc{well}(\beta)\in[0,1/2]$ and $\textsc{well}(\beta_{1g}')$ appears in both branches.
    \item The second product rate condition is similar.
    Again, the latter branch weakly dominates the former.
    \item For the third product rate condition, we require 
    $$
    n^{1/2}\sigma^{-1}\{\|T_2(\hat{h}_2-h_2)\|_2\|\hat{h}_4-h_4\|_2 \wedge \|\hat{h}_2-h_2\|_2\|T_4(\hat{h}_4-h_4)\|_2\}.
    $$
    Since $(h_2,h_4)$ are NPIVs, the former term is  
    $$
  n^{1/2}\sigma^{-1}\|T_2(\hat{h}_2-h_2)\|_2\|\hat{h}_4-h_4\|_2
  =O\left[n^{\frac{1}{2}-\gamma-\alpha\left\{
 1
  +
 \textsc{well}(\beta_4)
  \right\}}\right].
    $$
    Meanwhile the latter term is
    $$
  n^{1/2}\sigma^{-1}\|\hat{h}_2-h_2\|_2\|T_4(\hat{h}_4-h_4)\|_2
  =O\left[n^{\frac{1}{2}-\gamma-\alpha\left\{
  \textsc{well}(\beta_{2}) + 1
  \right\}}\right].
    $$
    In summary, the third product rate condition requires
    $$
    \left[1/2-\gamma-\alpha\left\{
 1
  +
 \textsc{well}(\beta_4)
  \right\} \right]\wedge 
  \left[1/2-\gamma-\alpha\left\{
  \textsc{well}(\beta_{2}) + 1
  \right\} \right]<0.
    $$
    Rearranging,
    $
    \frac{1}{2} < \gamma + \alpha \left[ \left\{
 1
  +
 \textsc{well}(\beta_4)
  \right\} \vee  \left\{
  \textsc{well}(\beta_{2}) + 1
  \right\}\right].
    $ \qedhere
\end{enumerate}

\end{proof}

\begin{proof}[Proof of Corollary~\ref{cor:general}]
    The result is immediate from Theorem~\ref{theorem:general}.
\end{proof}

\subsection{Corollaries without compounding ill posedness}

\begin{proof}[Proof of Corollary~\ref{cor:limit_simultaneous}]
    The result is immediate from Lemmas~\ref{lemma:innermax_weak} and~\ref{lemma:innermax_weak2}.
\end{proof}

\begin{proof}[Proof of Corollary~\ref{cor:L2_compound2}]
    To lighten notation, let $\underline{\beta}=\min(\beta_h,\beta_g',\beta_g)$ and $\|\bar{w}\|_2^2=\max(\|w_h\|_2^2,\|w_g'\|^2_2,\|w_g\|^2_2)$. We minimize the mean square bound in Theorem~\ref{theorem:joint}: $\|\hat{h}-h_0\|_2^2=O\{\mu^{\min(\underline{\beta},1)}\|\bar{w}\|_2^2+\mu^{-1}\delta_n^2\}$. 
    
    In the case $\underline{\beta} \geq  1$, the first order condition is $\|\bar{w}\|^2_2-\mu^{-2}\delta_n^2=0$, suggesting $\mu\asymp \delta_n$, $R_n\asymp \mu^2+\delta_n^2\asymp \delta_n^2$, and $\mu^{-1}R_n\asymp \delta_n$. 
    
    In the case $\underline{\beta} < 1$, the first order condition is $\mu^{\underline{\beta}-1}\underline{\beta}\|\bar{w}\|^2_2-\mu^{-2}\delta_n^2=0$, suggesting $\mu\asymp \delta_n^{\frac{2}{\underline{\beta}+1}}$, $R_n\asymp \mu^{\underline{\beta}+1}+\delta_n^2\asymp \delta_n^2$, and $\mu^{-1}R_n\asymp \delta_n^{2-\frac{2}{\underline{\beta}+1}}=\delta_n^{\frac{2\underline{\beta}}{\underline{\beta}+1}}$.
\end{proof}

\begin{proof}[Proof of Proposition~\ref{prop:robust2}]
By Corollary~\ref{cor:L2_compound2}, the projected rates are $O(\bar{\delta}_n)$ and the mean square rates are $O\left\{\bar{\delta}_n^{\textsc{well}(\underline{\beta}_j)}\right\}$ where $\underline{\beta}_j=\min(\vec{\beta}_j)$.
The first product rate condition of Theorem~\ref{theorem:ci} is satisfied when $$
  n^{1/2}\sigma^{-1}\{\|T_1(\hat{h}_1-h_1)\|_2\|\hat{h}_4-h_4\|_2 \wedge \|\hat{h}_1-h_1\|_2\|T_4(\hat{h}_4-h_4)\|_2\}=o_p(1).
    $$
   The former term is
    $
  n^{1/2}\sigma^{-1}\|T_1(\hat{h}_1-h_1)\|_2\|\hat{h}_4-h_4\|_2
  =O\left[n^{\frac{1}{2}-\gamma-\alpha\left\{
 1
  +
 \textsc{well}(\underline{\beta}_4)
  \right\}}\right].
    $
    Meanwhile the latter term is
    $
  n^{1/2}\sigma^{-1}\|\hat{h}_1-h_1\|_2\|T_4(\hat{h}_4-h_4)\|_2
  =O\left[n^{\frac{1}{2}-\gamma-\alpha\left\{
 \textsc{well}(\underline{\beta}_{1}) + 1
  \right\}}\right].
    $
    In summary, the first product rate condition requires
    $$
    \left[1/2-\gamma-\alpha\left\{
  1
  +
 \textsc{well}(\underline{\beta}_4)
  \right\} \right]\wedge 
  \left[1/2-\gamma-\alpha\left\{
   \textsc{well}(\underline{\beta}_{1}) + 1
  \right\} \right]<0.
    $$
    Rearranging,
    $
    \frac{1}{2} < \gamma + \alpha \left[ \left\{
 \textsc{well}(\underline{\beta}_1)
  +
 1
  \right\} \vee  \left\{
  \textsc{well}(\underline{\beta}_4) + 1
  \right\}\right].
    $
The other product rate conditions are similar. 
\end{proof}

\begin{proof}[Proof of Proposition~\ref{prop:negative}]
    The argument is similar to Proposition~\ref{prop:robust2}, however product rate conditions are of the form $n^{1/2}\sigma^{-1}\|\hat{h}_1-h_1\|_2\|\hat{h}_4-h_4\|_2=o_p(1)$, leading to $\frac{1}{2} < \gamma + \alpha \left\{
 \textsc{well}(\underline{\beta}_1)
  +
 \textsc{well}(\underline{\beta}_4)
  \right\} $. For causal scalars, $\gamma=0$. At best, $\alpha=\textsc{well}(\underline{\beta}_1)=\textsc{well}(\underline{\beta}_4)=1/2$, so the inequality fails.
\end{proof}

\section{Simulation and application details}\label{sec:details}

\subsection{Nested NPIV design}

Let $g_0(A)=A_1^3$ and $h_0(B)=f(B_1)$, where $f:\mathbb{R}\rightarrow \mathbb{R}$ is one of four possible functions from \cite{dikkala2020minimax}. 
For simplicity, we write $p=dim(A)=dim(B)=dim(C)=dim(C')$. We define $F:(x_1,...,x_p)\mapsto (x_1^{1/3},...,x_p^{1/3})$ and $1_p=(1,...,1)^\top \in\mathbb{R}^p$.

Each observation is generated as follows. Independently draw the instruments $C\sim \mathcal{N}(0,... I_p)$ and $C'\sim \mathcal{N}(0,... I_p)$. Next draw the noise terms $U\sim \mathcal{N}(0,1)$, $U_A\sim\mathcal{N}\{0,\min(1,|C_1|)\}$, $U_B \sim \mathcal{N}(0,0.1)$,  and $U_Y\sim\mathcal{N}\{0,\min(1,|C_1'|)\}$. 
Finally, set $B=C+1_p \cdot U+1_p \cdot U_B$, $A=F\{h_0(B)\cdot 1_p+U \cdot 1_p+C'+U_A \cdot 1_p\}$, and $Y=g_0(A)+U+U_Y$.

\begin{proposition}[Nested NPIV simulation]
    The nested NPIV simulation design satisfies $\mathbb{E}\{h_0(B)|C\}=\mathbb{E}\{g_0(A)|C\}$ and $\mathbb{E}\{g_0(A)|C'\}=\mathbb{E}(Y|C')$.
\end{proposition}

\begin{proof}
    Clearly $\mathbb{E}\{Y-g_0(A)|C'\}=\mathbb{E}\{U+U_Y|C'\}=0$. Moreover, $\mathbb{E}\{g_0(A)-h_0(B)|C\}=\mathbb{E}\{A_1^{3}-h_0(B)|C'\}=\mathbb{E}\{U+C_1'+U_A|C\}=0$.
\end{proof}

\subsection{Coverage design}

To begin, we recap the linear design of \cite{dukes2023proximal} in our notation. Each variable is a scalar except for $X\in\mathbb{R}^2$. Each observation is generated as follows:
\begin{enumerate}
    \item $(X_1,X_2,U)^{\top}\sim \mathcal{N}\{(0.25,0.25,0)^{\top},\Sigma\}$ where $\Sigma=\begin{pmatrix}
        0.25 & 0.00 & 0.05 \\
        0.00 & 0.25 & 0.05 \\
        0.05 & 0.05 & 1.00    
    \end{pmatrix}$;
    \item $D|X,U\sim \textsc{Bernoulli}[1+\exp\{(0.5,0.5)^{\top} X+0.4 U\}]^{-1}$;
    \item $Z|X,D,U \sim \mathcal{N}\{0.2-0.52D+(0.2,0.2)^{\top} X -U,1\}$;
    \item $W|X,U \sim \mathcal{N} \{0.3+(0.2,0.2)^{\top} X-0.6U,1\})$;
    \item $M| X,D,U \sim \mathcal{N} \{-0.3D-(0.5,0.5)^{\top} X+0.4U,1\}$;
    \item $Y|X,D,M,W,U = 2+2D+M+2W-(1,1)^{\top} X-U+2\mathcal{N}(0,1)$.
\end{enumerate}

 Recall that $h_0$ is an outcome confounding bridge that solves $\mathbb{E}\{h(X,D,W)|X,Z,D=0\}=\mathbb{E}\{g_0(X,1,M,W)|X,Z,D=0\}$, where $g_0$ solves
$\mathbb{E}\{g(X,D,M,W)|X,Z,D=1,M\}=\mathbb{E}(Y|X,Z,D=1,M)$. 

Recall that $h_0'$ is a treatment confounding bridge that solves 
$\mathbb{E}\{h'(X,Z,D,M)|X,D=1,M,W\}=\mathbb{E}\left\{g_0'(X,Z,0)\frac{\mathbb{P}(D=0|X,M,W)}{\mathbb{P}(D=1|X,M,W)}|X,D=1,M,W\right\}$, 
where $g_0'$ solves
$\mathbb{E}\{g'(X,Z,D)|X,D=0,W\}=\mathbb{E}\left\{\frac{1}{\mathbb{P}(D=0|X,W)}|X,D=0,W\right\}$. 

We write the nuisances as
$h_1(X,W)=h_0(X,0,W)$, 
    $h_2(X,M,W)=g_0(X,1,M,W)$, 
    $h_3(X,Z,D,M)=1_{D=1}h_0'(X,Z,1,M)$, and   
    $h_4(X,Z,D)=1_{D=0}g_0'(X,Z,0)$.

Consider the notation $\mathbb{E}(A|B,C)=\beta_{A0}+\beta_{AB}B+\beta_{AC}C$ and  $\mathbb{V}(A|B,C)=v^2_{A|B,C}$. Also define $\mathbb{P}(D=0|X,U)=[1+\exp\{-(\pi_0+\pi_X^{\top}X+
\pi_UU)\}]^{-1}$ and $
\log \left\{\frac{\mathbb{P}(D=0|X,M,U)}{\mathbb{P}(D=1|X,M,U)}\right\}=\rho_0+\rho_X^{\top}X+\rho_MM+\rho_UU.
$

\begin{proposition}[Linear coverage simulation; c.f. Supplementary Material of \cite{dukes2023proximal}]\label{prop:dukes}
    The linear design for coverage simulations satisfies 
     {\small 
    \begin{align*}
        h_1(X,W)&=\nu_0+\nu_X^{\top} X+\nu_W W, \\
        h_2(X,M,W)&=\delta_0+\delta_X^{\top}X+\delta_MM+\delta_WW, \\
        h_3(X,Z,D,M)&=1_{D=1}\left[1 + \exp\left\{-(\eta_0+\eta_X^{\top}X+\eta_ZZ)\right\}\right] \exp\left(\alpha_0+\alpha_X^{\top}X+\alpha_ZZ+\alpha_MM\right), \\
        h_4(X,Z,D)&=1_{D=0}\left[1 + \exp\left\{-(\eta_0+\eta_X^{\top}X+\eta_ZZ)\right\}\right],
    \end{align*}
    }
    where $(\nu,\delta,\alpha,\eta)$ can be expressed in terms of $(\beta,v,\pi,\rho)$ as follows:
    {\small 
    \begin{align*}
    (\nu_0,\nu_X^{\top},\nu_W)
        &=\left(\delta_0+\delta_M\beta_{M0}+\delta_W\beta_{W0}-\nu_W\beta_{W0},
        \delta_M\beta_{MX}^{\top}+\delta_W\beta_{WX}^{\top}+\delta_X^{\top}-\nu_W\beta_{WX}^{\top},
        \frac{ \delta_M\beta_{MU}+\delta_W\beta_{WU}}{\beta_{WU}}\right); \\
    (\delta_0,\delta_X^{\top},\delta_M,\delta_W)
        &=\left(\beta_{Y0}+\beta_{YD}+\beta_{YW}\beta_{W0}-\delta_W\beta_{W0},
        \beta_{YW}\beta_{WX}^{\top}+\beta_{YX}^{\top}-\delta_W\beta_{WX}^{\top},
        \beta_{YM},
        \frac{\beta_{YW}\beta_{WU}+\beta_{YU}}{\beta_{WU}}\right); \\
    (\alpha_0,\alpha_X^{\top},\alpha_Z,\alpha_M)
        &=\left\{\rho_0-\alpha_Z(\beta_{Z0}+\beta_{ZD})-\frac{\alpha_Z^2v^2_{Z|D,X,U}}{2},
        \rho_X^{\top}-\alpha_Z\beta_{ZX}^{\top},
       \frac{\rho_U}{\beta_{ZU}},
     \rho_M\right \}; \\
      (\eta_0,\eta_X^{\top},\eta_Z)
        &=\left(\pi_0-\eta_Z\beta_{Z0}+\frac{\eta_Z^2v^2_{Z|D,X,U}}{2},
        \pi_X^{\top}-\eta_Z\beta_{ZX}^{\top},
        \frac{\pi_U}{\beta_{ZU}}\right).
    \end{align*}
    }
    Moreover, $\rho$ can be expressed in terms of $(\beta,v,\pi)$ as 
    {\small
     \begin{align*}
    (\rho_0,\rho_X^{\top},\rho_M,\rho_U)
    &=\left\{\frac{\beta_{MD}}{v^2_{M|D,X,U}}(\beta_{M0}+\beta_{MD}/2)+\pi_0,
    \frac{\beta_{MD}}{v^2_{M|D,X,U}}\beta_{MX}^{\top}+\pi_X^{\top},
    \frac{-\beta_{MD}}{v^2_{M|D,X,U}},
    \frac{\beta_{MD}}{v^2_{M|D,X,U}}\beta_{MU}+\pi_U)\right\}.
 \end{align*}
 } Finally, in this design, $\theta_0=4.05$.
\end{proposition}

Next, we extend the linear design to a nonlinear design. Instead of observing $(X_1,X_2,Z,D,M,Y,W)$, we observe $(\tilde{X}_1,\tilde{X}_2,\tilde{Z},\tilde{D},\tilde{M},\tilde{Y},\tilde{W})=\{f(X_1),f(X_2),f(Z),D,f(M),Y,f(W)\}$ where $f:\mathbb{R}\rightarrow \mathbb{R}$ is one of four possible functions.
Let $\tilde{\theta}_0$ be the causal parameter in the nonlinear design, and let $(\tilde{h}_1,\tilde{h}_2,\tilde{h}_3,\tilde{h}_4)$ be the nuisances.

\begin{proposition}[Nonlinear coverage simulation]
    The nonlinear design for coverage simulations satisfies $\tilde{h}_1(\tilde{X}_1,\tilde{X}_2,\tilde{W})=h_1\{f^{-1}(\tilde{X}_1),f^{-1}(\tilde{X}_2),f^{-1}(\tilde{W})\}$,
     {\small 
    \begin{align*}
        \tilde{h}_2(\tilde{X}_1,\tilde{X}_2,\tilde{M},\tilde{W})&=h_2\{f^{-1}(\tilde{X}_1),f^{-1}(\tilde{X}_2),f^{-1}(\tilde{M}),f^{-1}(\tilde{W})\}, \\
        \tilde{h}_3(\tilde{X}_1,\tilde{X}_2,\tilde{Z},\tilde{D},\tilde{M})&=h_3\{f^{-1}(\tilde{X}_1),f^{-1}(\tilde{X}_2),f^{-1}(\tilde{Z}),\tilde{D},f^{-1}(\tilde{M})\}, \\
        \tilde{h}_4(\tilde{X}_1,\tilde{X}_2,\tilde{Z},\tilde{D})&=h_4\{f^{-1}(\tilde{X}_1),f^{-1}(\tilde{X}_2),f^{-1}(\tilde{Z}),\tilde{D}\},
    \end{align*}
    }
    where $(h_1,h_2,h_3,h_4)$ are characterized in Proposition~\ref{prop:dukes}. Hence $\tilde{h}_1$ is linear in $f^{-1}(\tilde{X})$ yet nonlinear in $\tilde{X}$, and so on. In this design, $\tilde{\theta}_0=4.05$.
\end{proposition}

\begin{proof}
By Proposition~\ref{prop:dukes}, $0=\mathbb{E}\{Y-h_2(X,M,W)|X,Z,D=1,M\}$, which equals
   \begin{align*}
  &\mathbb{E}[Y-h_2\{f^{-1}(\tilde{X}_1),f^{-1}(\tilde{X}_2),f^{-1}(\tilde{M}),f^{-1}(\tilde{W})\}|f^{-1}(\tilde{X}_1),f^{-1}(\tilde{X}_2), f^{-1}(\tilde{Z}),\tilde{D}=1,f^{-1}(\tilde{M})] \\
    &= \mathbb{E}\{Y- \tilde{h}_2(\tilde{X}_1,\tilde{X}_2,\tilde{Z},\tilde{D})|\tilde{X}_1,\tilde{X}_2, \tilde{Z},\tilde{D}=1,\tilde{M}\}
    \end{align*}
    and similarly for the other bridge functions.
\end{proof}

\subsection{Implementation details}

All estimators are implemented in a \texttt{python} package called \texttt{nnpiv} that accompanies the paper as supplementary material. In general, we follow the standard tuning procedures of previous work on NPIV \citep{dikkala2020minimax,bennett2020deepgeneralizedmethodmoments}.

\textbf{Benchmark: 2SLS.} We employ an iterated two stage least squares (2SLS) method using the command \texttt{tsls.tsls()}. Sequential and simultaneous estimation coincide.

\textbf{Bechmark: Series.} We implement the iterated 2SLS method after applying a cubic transformation to the data. Again, sequential and simultaneous estimation coincide. 

\textbf{Benchmark: Regularized series.} In addition to the cubic transformation, we also introduce regularization via \texttt{ElasticNetCV} with cross-validation (\texttt{cv=3}). This method blends both the $\ell_1$ and $\ell_2$ penalties of the lasso and ridge. We implement it in the command \texttt{tsls.regtsls()}.

\textbf{Proposal: RKHS.} Our estimator has a closed form solution, derived in Appendix~\ref{sec:rkhs} and implemented via the command \texttt{rkhsiv.RKHSIVL2}. We use the Gaussian kernel, which has a hyperparameter called the lengthscale controlling the width of the kernel, and hence its sensitivity to local variations. We set the lengthscale by cross validation, which gives the value $0.0013$. We set the ridge regularization parameters $(\mu,\mu')=n^{-1.6}$, following the choice of \cite{dikkala2020minimax}.

\textbf{Proposal: Neural network.} We optimize the objective using the optimistic Adam algorithm of \cite{Daskalakis}, as described in Appendix~\ref{sec:nn} and implemented via the command \texttt{agmm2.AGMM2L2(learnerh, learnerg, adversary1, adversary2)}. 
Following \cite{dikkala2020minimax,bennett2020deepgeneralizedmethodmoments}, the architecture consists of fully connected layers with dropout layers throughout to prevent overfitting, with a dropout probability of $0.1$, hidden layer width set to $100$,  and \texttt{LeakyReLU} activation for both the learner and the adversary models. Following \cite{dikkala2020minimax,bennett2020deepgeneralizedmethodmoments}, the learning rates for both the learner and adversary are set to $1 \times 10^{-4}$, with a weight-decay parameter of $1 \times 10^{-3}$ for the learner and $1 \times 10^{-4}$ for the adversary. The ridge regularization parameters are $(\mu,\mu')= 1 \times 10^{-12}$, similar to those works.  The model is trained with a batch size of $100$ samples. 
Since we consider joint estimation of $(\hat{h},\hat{g})$, we double the number of epochs of \cite{dikkala2020minimax} to $600$. 
The final neural network model is computed by averaging models, burning in the initial $200$ epochs to allow the model to stabilize. 

\textbf{Standard practice: Debiased machine learning.} In Algorithm~\ref{algo:dml}, we use five folds. When using neural networks, we reduce the number of epochs to $500$ because the inference process is computationally expensive. We increase the burn-in period to $400$ epochs to enhance model stability, particularly for accurate estimation of the treatment bridge function.

\textbf{Computational resources.} The simulations were run on a cluster partition with $2$ nodes, each equipped with $110$ Intel "Sapphire Rapids" CPUs and $600$ GB of memory. For the neural network, we utilized a separate cluster partition with $64$ Intel "Ice Lake" CPUs, $4$ NVIDIA A100 GPUs per node, and $600$ GB of memory.

\subsection{Proxy mediation analysis of US Job Corps} 

We directly extend the real world application of \cite{dukes2023proximal}, from parametric estimation to semiparametric estimation. In the Job Corps study, participants were randomized to receive eligibility for  the Job Corps job training program. The treatment is whether an individual received job training in the first year after randomization. The mediator is the fraction of weeks employed in the second year after randomization. The outcome is the number of arrests in the fourth year after randomization.
 The sample consists of $10,775$ participants, with $257$ participants removed from the analysis due to missing mediator information.

The parametric estimate is derived from \cite{dukes2023proximal}. 

For the RKHS estimate, we use the \texttt{RKHS2IVL2} command with a Gaussian kernel and lengthscale $0.1$, chosen by cross validation. The ridge regularization is tuned as in the simulations.

For the neural network estimate, the learning and regularization parameters are set as in the simulations. Since the application is less computationally intensive than the coverage simulations, we increase the hidden layer width to $400$. The final model is computed using $500$ epochs, with a burn-in period of $350$ epochs.

\subsection{Heterogeneous long term effects of Project STAR} 

We directly extend the real world application of \cite{athey2020combining}, from average long term  effects to heterogeneous long term  effects. We use data from the Tennessee Student Teacher
Achievement Ratio experiment (Project STAR), where each student was randomly assigned to either a small or regular kindergarten class. The short term outcome is the third grade test score. Observational data come from the New York City (NYC) public school system, and include the same variables, along with fourth to eighth grade test scores.

The full sample consists of $11,600$ students: $3,497$ from STAR, and $8,103$ from NYC. Due to missing values in the prior ability variable, analysis of heterogeneous long term effects was conducted on $1,954$ students from STAR, and $8,103$ students from NYC. We refer to the subpopulation with non-missing prior ability as the selected sample. Figure~\ref{fig:star_ate_selection} illustrates the average long term effects for this subpopulation, which are lower than the average long term effects for the full population reported in Section~\ref{sec:experiments}.

We compute localization weights to assess heterogeneous effects with respect to prior ability $X_1$. For the localization, we use a Gaussian kernel with bandwidth chosen according to Silverman’s rule: $\lambda = 0.9 \min \left( \hat{\sigma}_{X_1}, \frac{\textsc{IQR}_{X_1}}{1.34} \right) n^{-\frac{1}{5}}$, where $\hat{\sigma}_{X_1}$ is the standard  deviation and $\textsc{IQR}_{X_1}$  is the interquantile range of $X_1$. 

Oracle average long term effects are estimated as in \cite{athey2020combining}. Oracle heterogeneous long term effects are localized averages thereof.

\begin{figure}
   \captionsetup[subfigure]{justification=Centering}
\begin{subfigure}[t]{0.48\textwidth}
         \centering
        \resizebox{\textwidth}{!}{%
       \includegraphics[width=\textwidth]{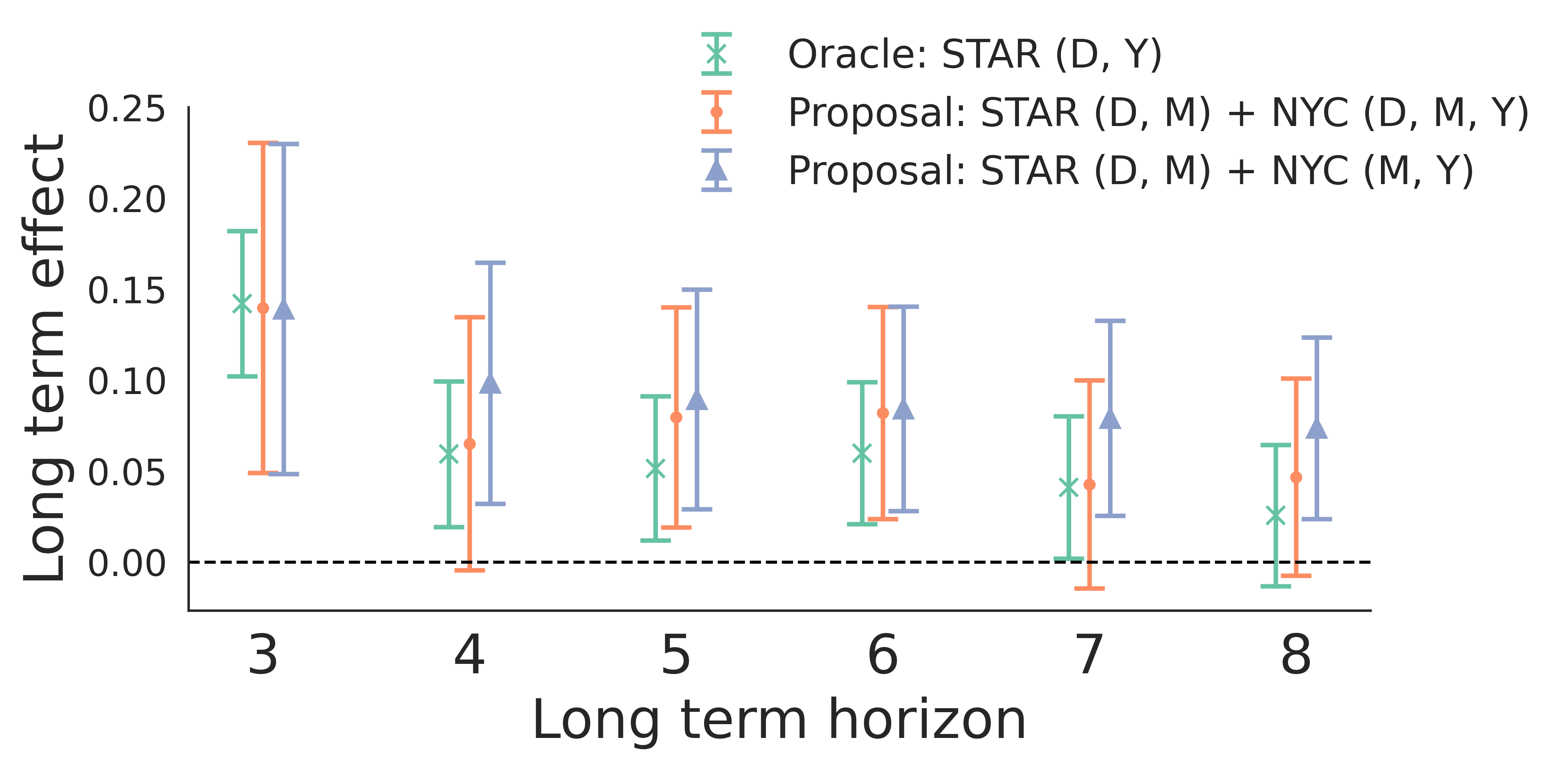}
        }
    \caption{RKHS}
\end{subfigure}\hspace{\fill} 
\begin{subfigure}[t]{0.48\textwidth}
          \centering
        \resizebox{\textwidth}{!}{%
      \includegraphics[width=\textwidth]{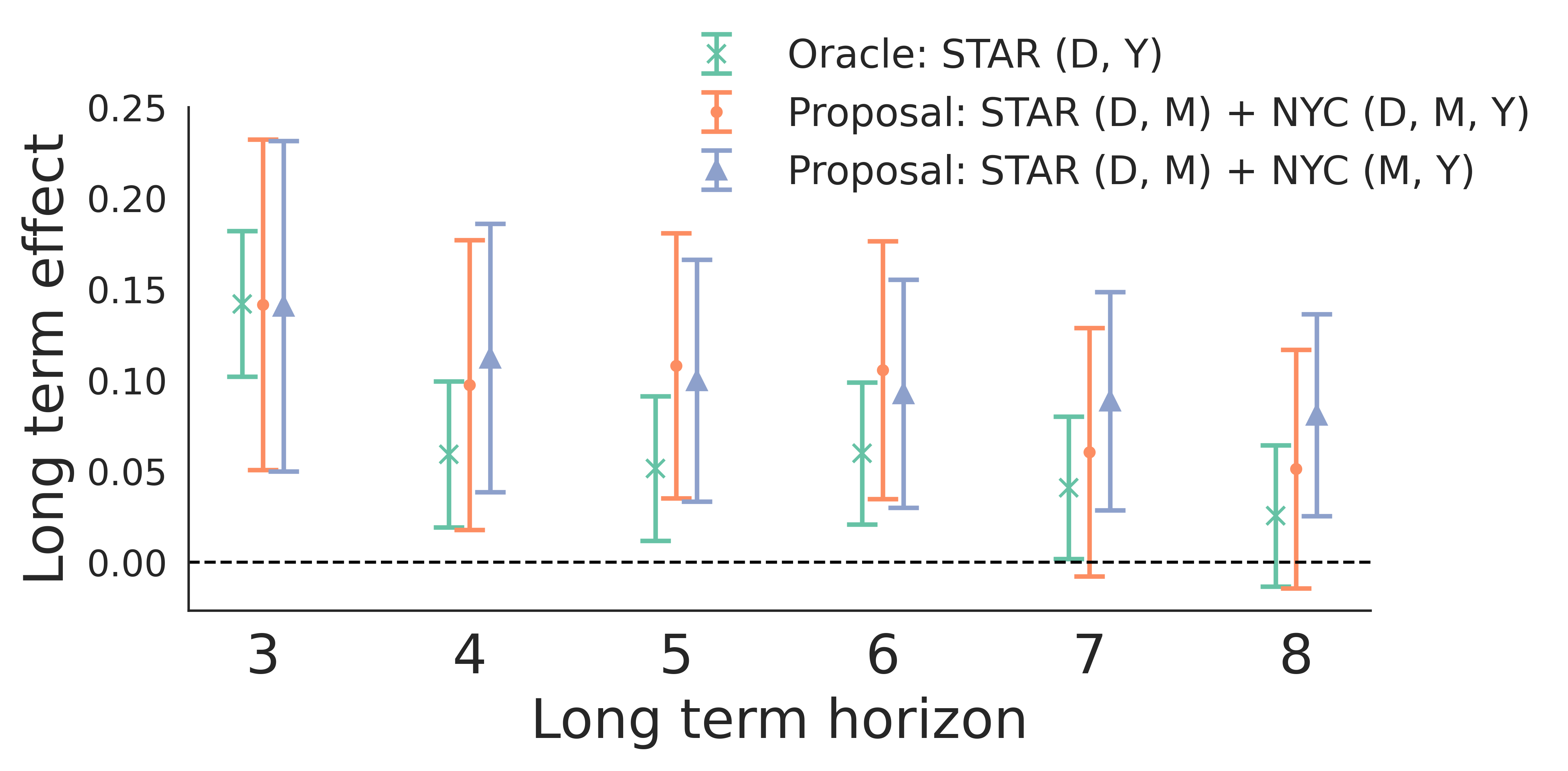}
        }
    \caption{Neural network}
\end{subfigure}
\caption{Average long term  treatment effect over different horizons for selected sample}\label{fig:star_ate_selection}
\end{figure}

For our RKHS proposal, we use the Gaussian kernel.  The lengthscale is chosen in a data driven way, following the median distance heuristic. For our neural network proposal, we choose the same hyperparameters as in the simulations.

Section~\ref{sec:experiments} reports heterogeneous long term effects of small kindergarten class size on seventh grade test scores. Figure~\ref{fig:longterm_cate_grades} reports heterogeneous long term effects of small kindergarten class size over different horizons: third through eight grade test scores. Across horizons and function spaces, our proposals recover the oracle well.

\begin{figure}
   \captionsetup[subfigure]{justification=Centering}
\begin{subfigure}[t]{0.48\textwidth}
         \centering
        \resizebox{\textwidth}{!}{%
       \includegraphics[width=\textwidth]{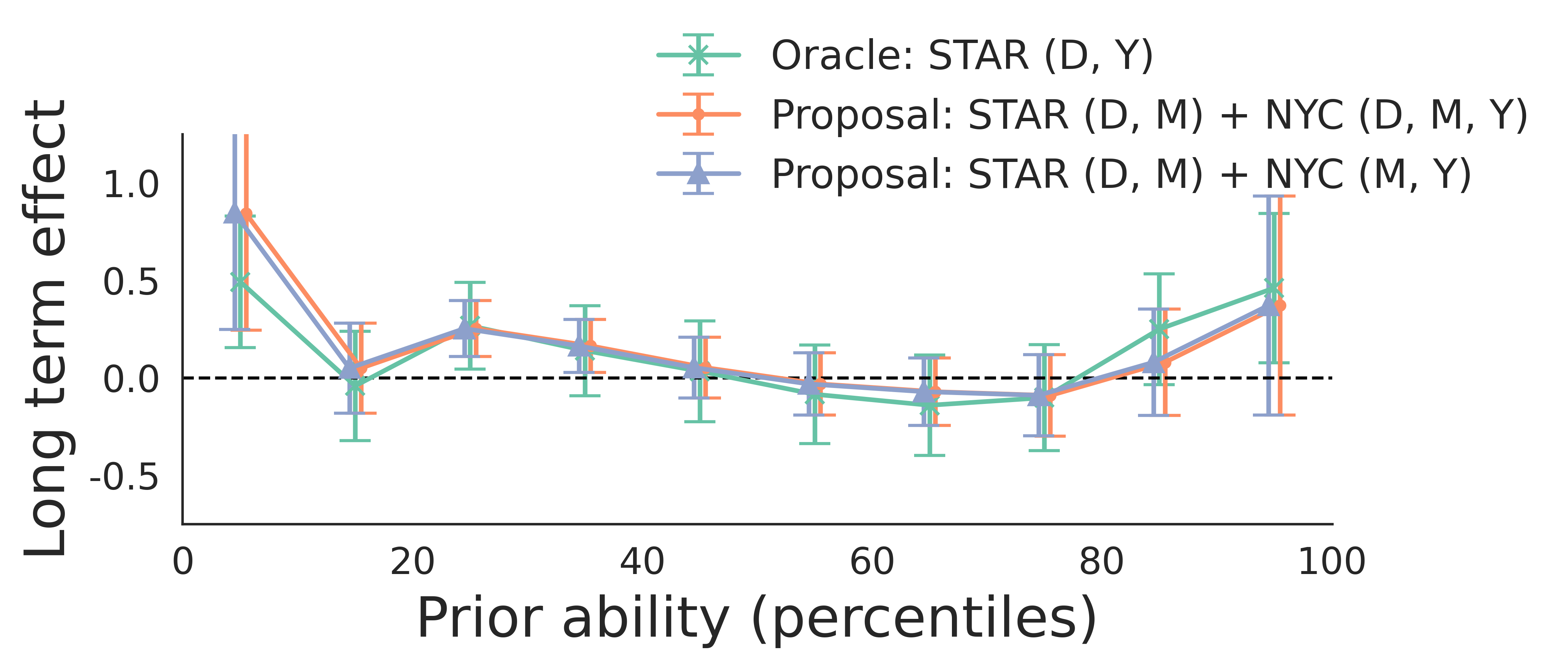}
        }
        \vspace{-40pt}
    \caption{RKHS, third grade}
\end{subfigure}\hspace{\fill} 
\begin{subfigure}[t]{0.48\textwidth}
          \centering
        \resizebox{\textwidth}{!}{%
      \includegraphics[width=\textwidth]{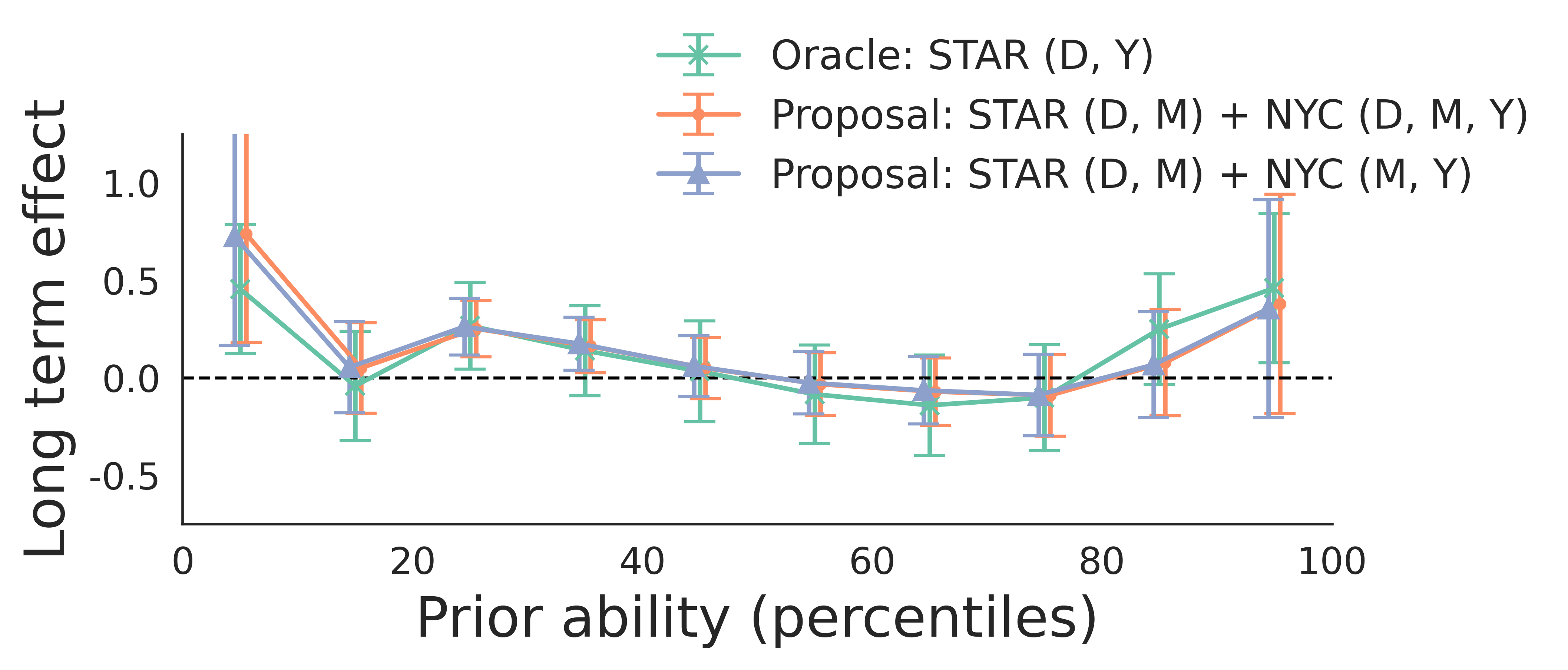}
        }
        \vspace{-40pt}
    \caption{Neural network, third grade}
\end{subfigure}


  \captionsetup[subfigure]{justification=Centering}
\begin{subfigure}[t]{0.48\textwidth}
         \centering
        \resizebox{\textwidth}{!}{%
       \includegraphics[width=\textwidth]{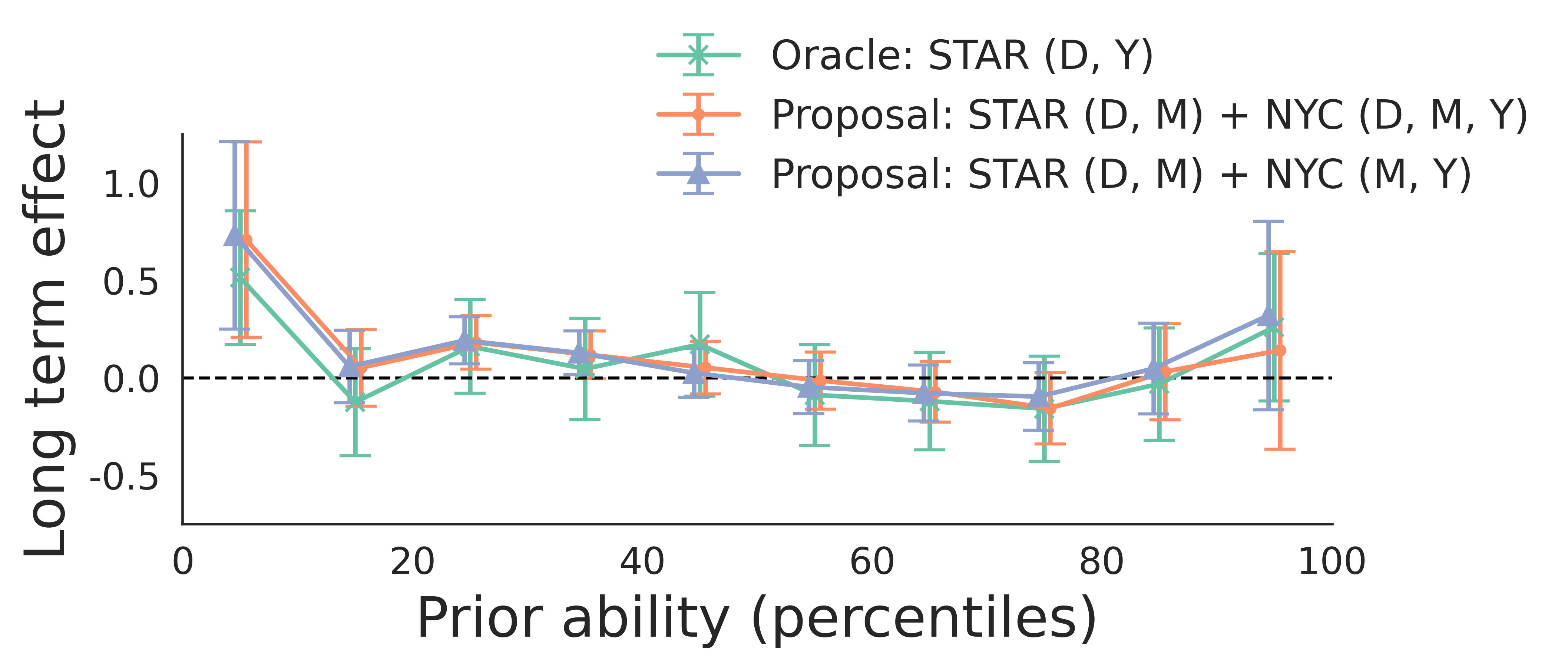}
        }
        \vspace{-40pt}
    \caption{RKHS, fourth grade}
\end{subfigure}\hspace{\fill} 
\begin{subfigure}[t]{0.48\textwidth}
          \centering
        \resizebox{\textwidth}{!}{%
      \includegraphics[width=\textwidth]{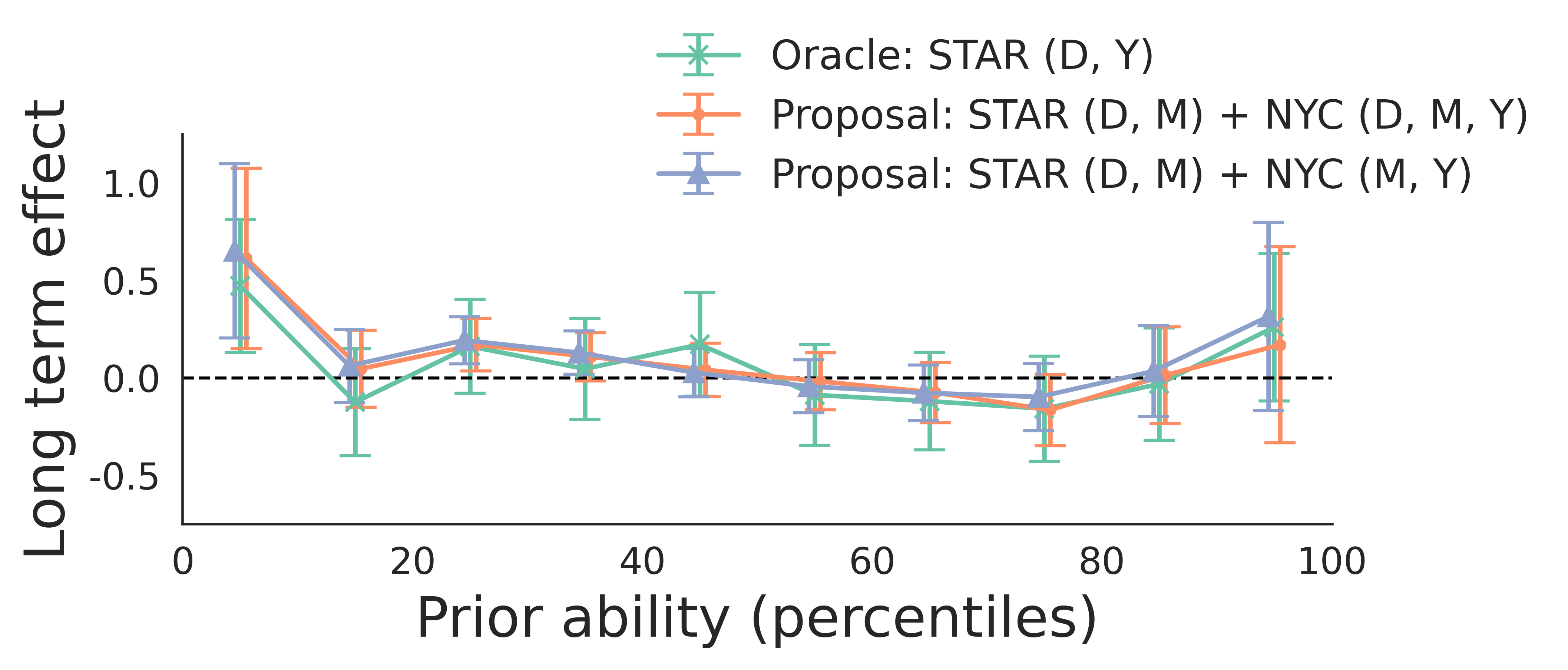}
        }
        \vspace{-40pt}
    \caption{Neural network, fourth grade}
\end{subfigure}


  \captionsetup[subfigure]{justification=Centering}
\begin{subfigure}[t]{0.48\textwidth}
         \centering
        \resizebox{\textwidth}{!}{%
       \includegraphics[width=\textwidth]{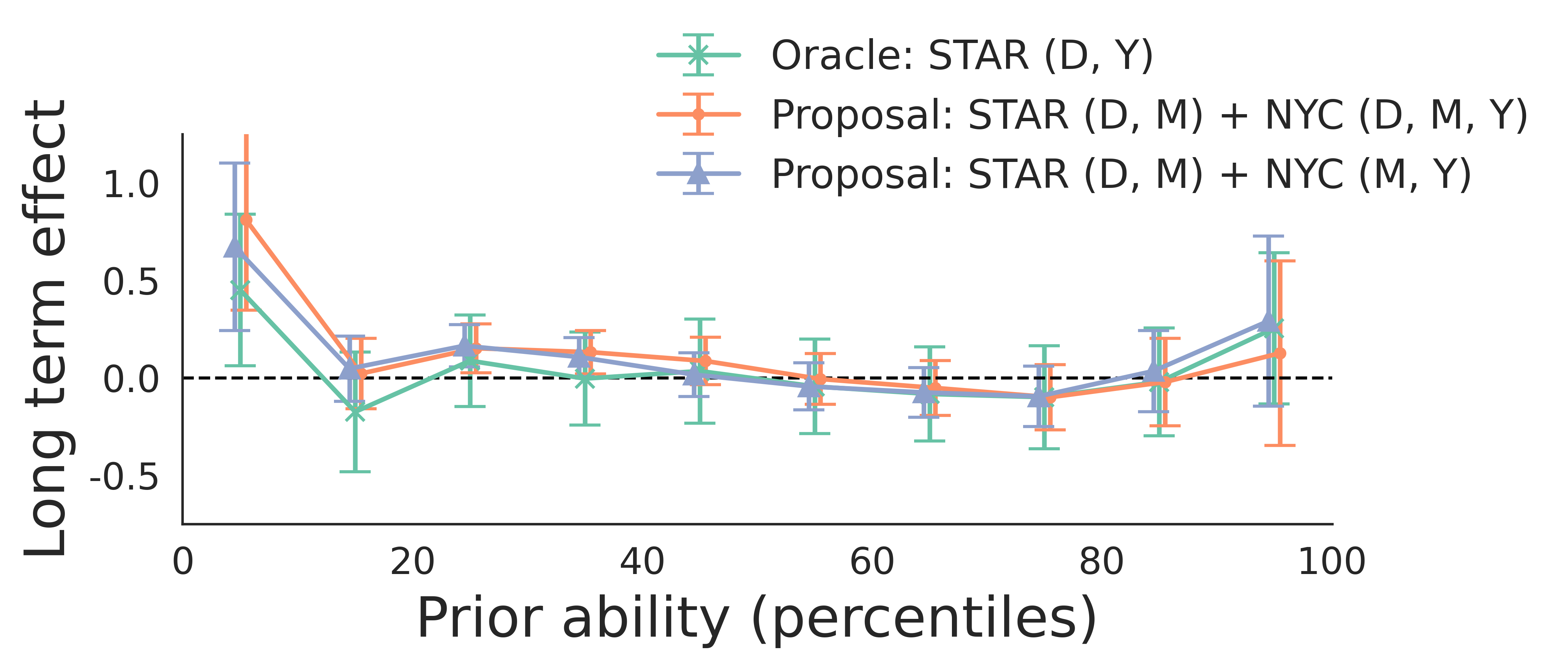}
        }
        \vspace{-40pt}
    \caption{RKHS, fifth grade}
\end{subfigure}\hspace{\fill} 
\begin{subfigure}[t]{0.48\textwidth}
          \centering
        \resizebox{\textwidth}{!}{%
      \includegraphics[width=\textwidth]{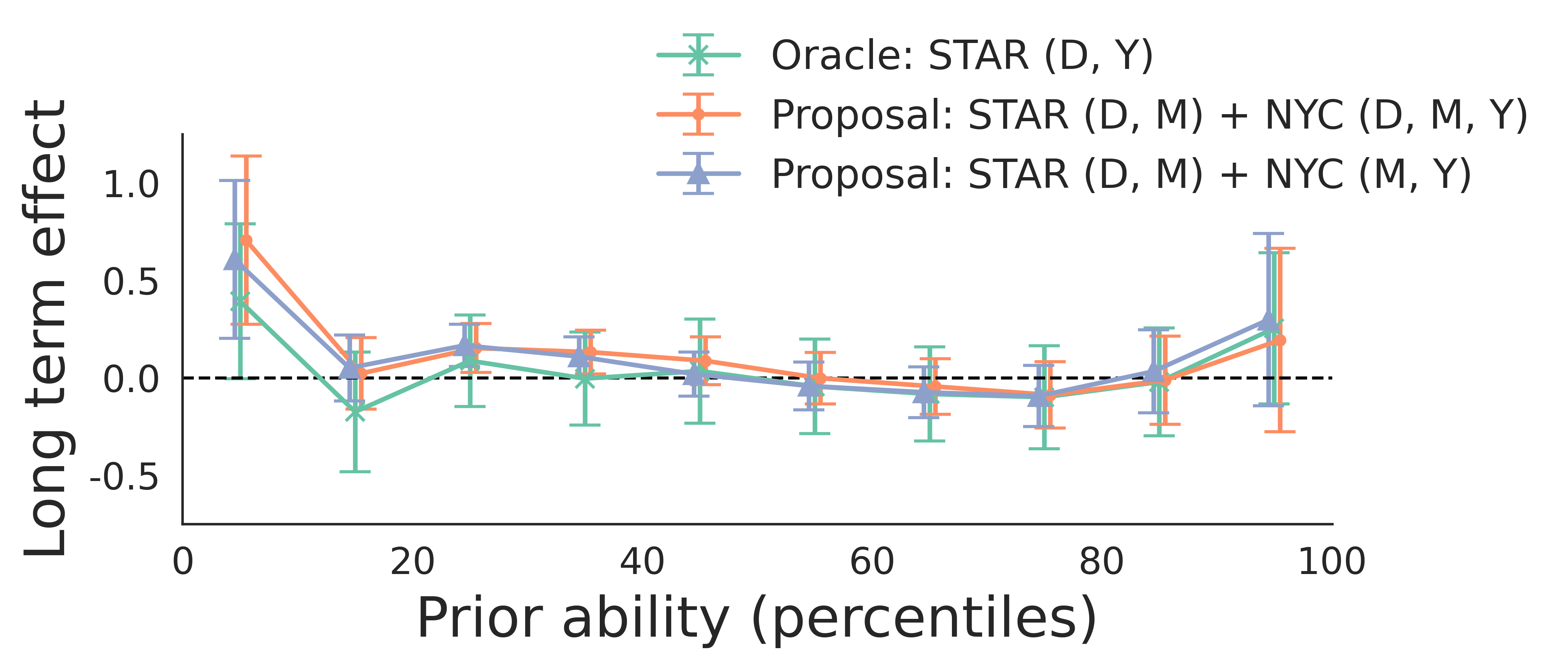}
        }
        \vspace{-40pt}
    \caption{Neural network, fifth grade}
\end{subfigure}


  \captionsetup[subfigure]{justification=Centering}
\begin{subfigure}[t]{0.48\textwidth}
         \centering
        \resizebox{\textwidth}{!}{%
       \includegraphics[width=\textwidth]{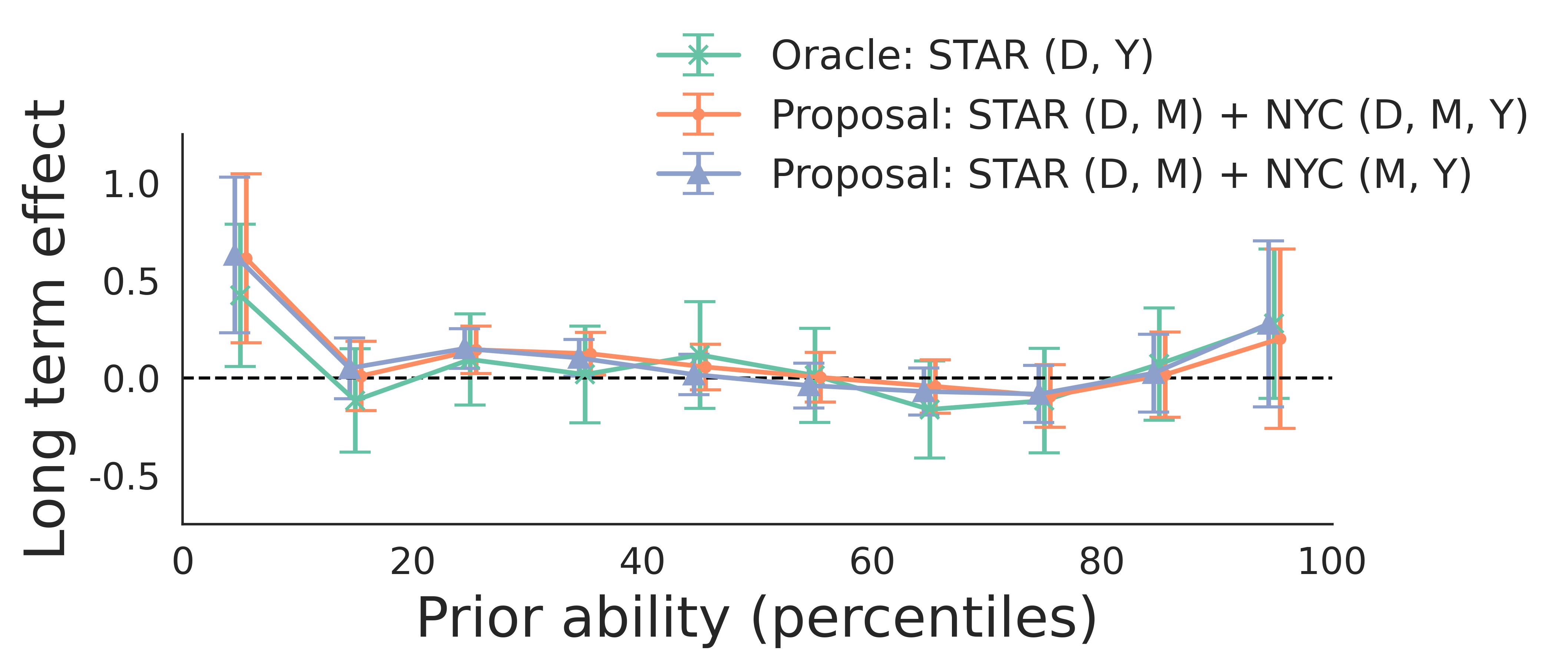}
        }
        \vspace{-40pt}
    \caption{RKHS, sixth grade}
\end{subfigure}\hspace{\fill} 
\begin{subfigure}[t]{0.48\textwidth}
          \centering
        \resizebox{\textwidth}{!}{%
      \includegraphics[width=\textwidth]{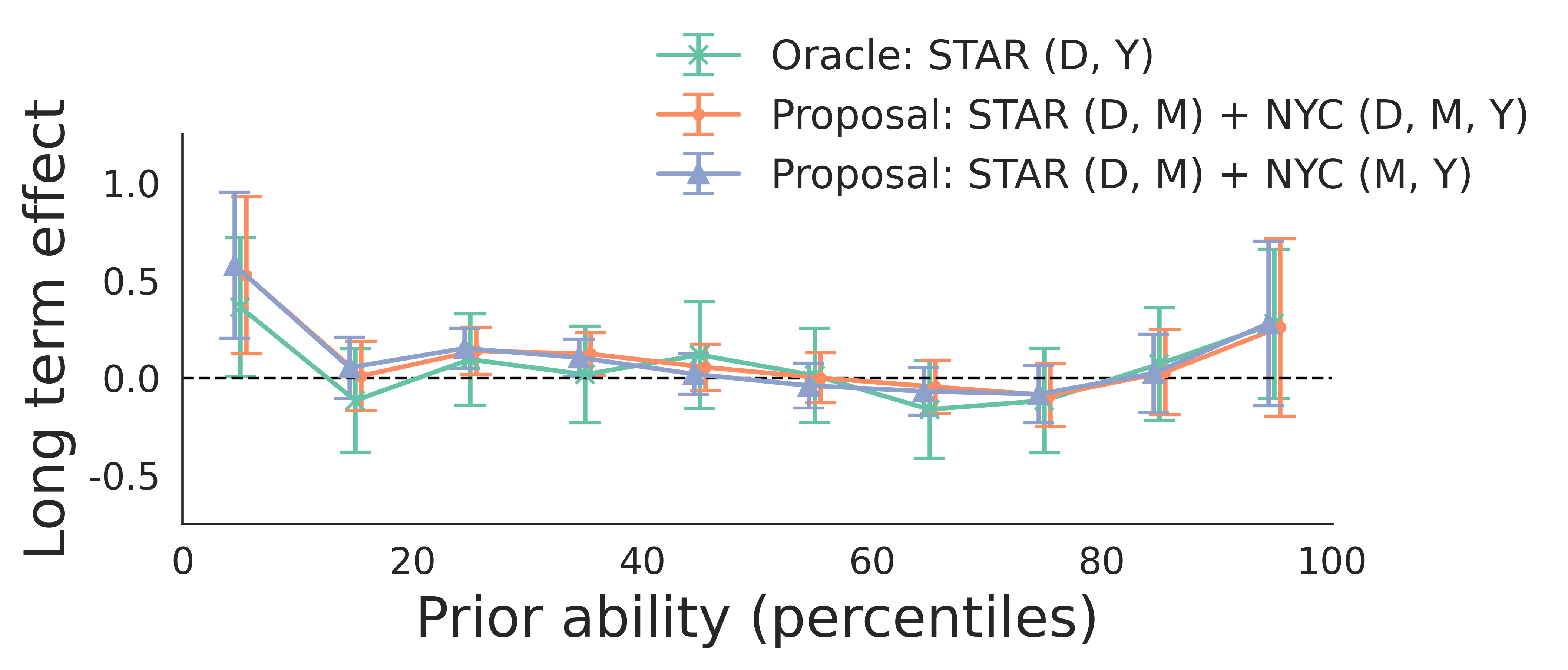}
        }
        \vspace{-40pt}
    \caption{Neural network, sixth grade}
\end{subfigure}


  \captionsetup[subfigure]{justification=Centering}
\begin{subfigure}[t]{0.48\textwidth}
         \centering
        \resizebox{\textwidth}{!}{%
       \includegraphics[width=\textwidth]{figures/plot_longterm_cate_Y7_rkhs.png}
        }
        \vspace{-40pt}
    \caption{RKHS, seventh grade}
\end{subfigure}\hspace{\fill} 
\begin{subfigure}[t]{0.48\textwidth}
          \centering
        \resizebox{\textwidth}{!}{%
      \includegraphics[width=\textwidth]{figures/plot_longterm_cate_Y7_agmm2.png}
        }
        \vspace{-40pt}
    \caption{Neural network, seventh grade}
\end{subfigure}


  \captionsetup[subfigure]{justification=Centering}
\begin{subfigure}[t]{0.48\textwidth}
         \centering
        \resizebox{\textwidth}{!}{%
       \includegraphics[width=\textwidth]{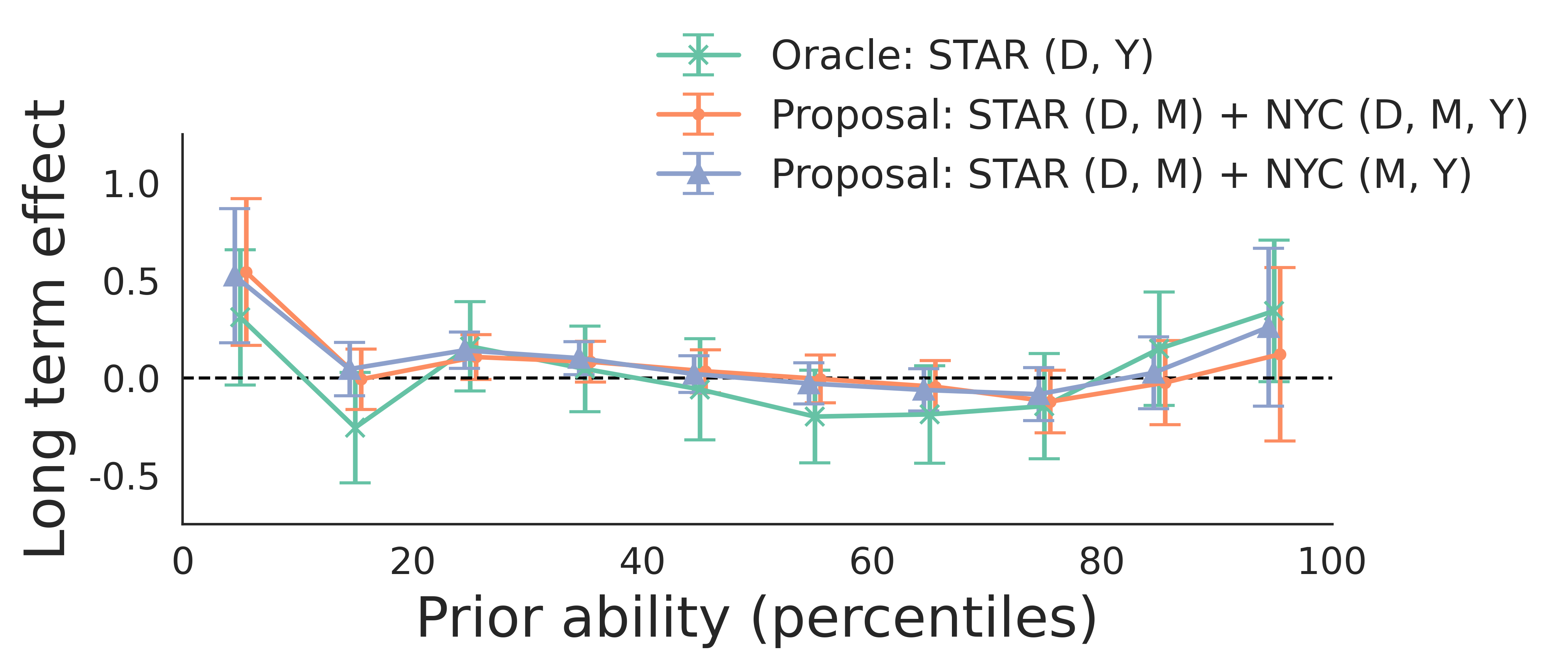}
        }
        \vspace{-40pt}
    \caption{RKHS, eight grade}
\end{subfigure}\hspace{\fill} 
\begin{subfigure}[t]{0.48\textwidth}
          \centering
        \resizebox{\textwidth}{!}{%
      \includegraphics[width=\textwidth]{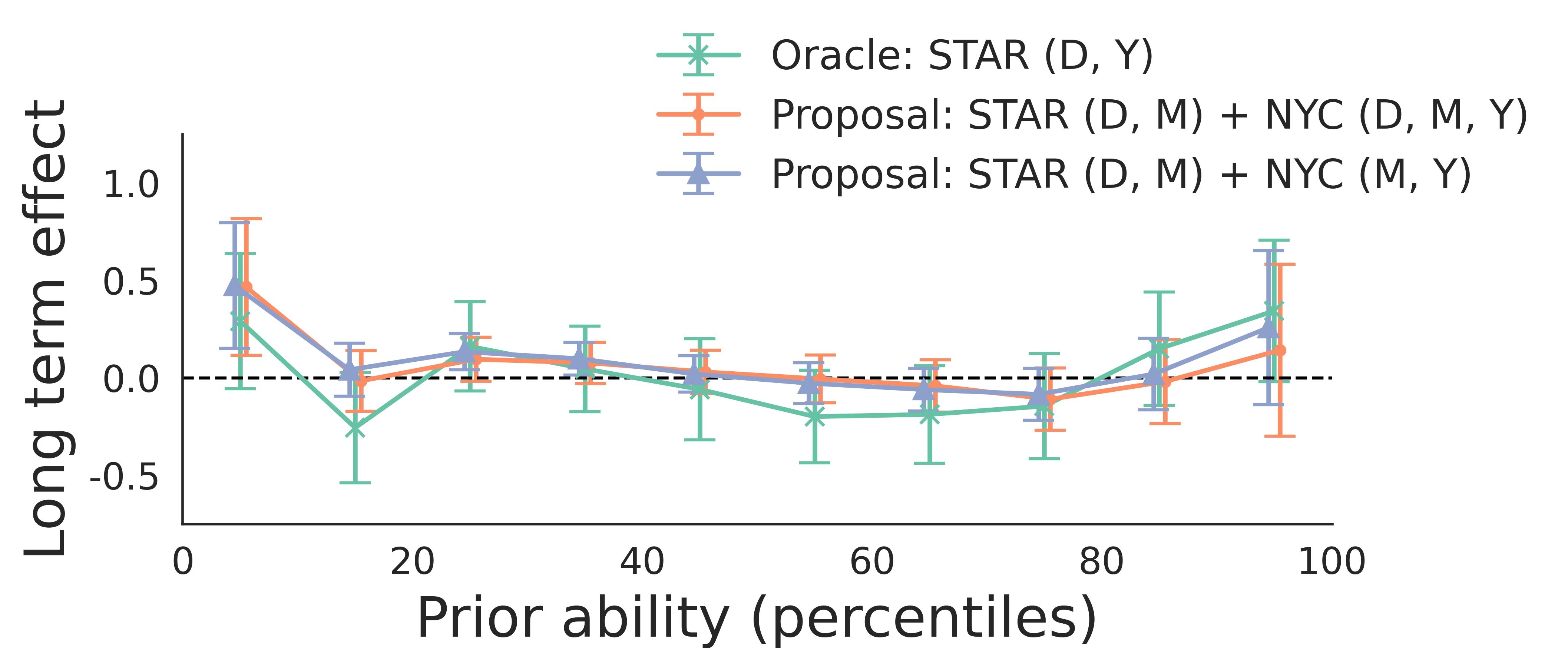}
        }
        \vspace{-40pt}
    \caption{Neural network, eight grade}
\end{subfigure}

\caption{Heterogeneous long term treatment effects with respect to prior ability, over different horizons.}\label{fig:longterm_cate_grades}
\end{figure}

\section{Computation details}\label{sec:tuning}

In this appendix, we give RKHS and neural network derivations for Algorithm~\ref{algo:simultaneous}:
\begin{align*}
    (\hat{g},\hat{h})&=\argmin _{g\in\mathcal{G}, h \in \mathcal{H}} 
    \max_{f' \in \mathcal{F}'} \mathbb{E}_n\left[2\left\{g(A)-Y\right\} f'(C')-f'(C')^2\right]
     +\mu'\E_n\{g(A)^2\} \\
    &\quad +
    \max_{f \in \mathcal{F}} \mathbb{E}_n\left[2\left\{h(B)-g(A)\right\} f(C)-f(C)^2\right]   
    +\mu\E_n\{h(B)^2\}. 
\end{align*}
Results for Algorithms~\ref{algo:sequential} and~\ref{algo:npiv_general} are similar and can be shared upon request. So can derivations for lasso and random forest function spaces.

In formal remarks, we also study the ``subsetted'' modification
\begin{align*}
    (\hat{g},\hat{h})&=\argmin _{g\in\mathcal{G}, h \in \mathcal{H}} 
    \max_{f' \in \mathcal{F}} \mathbb{E}_p\left[2\left\{g(A)-Y\right\} f'(C')-f'(C')^2\right]
     +\mu'\E_n\{g(A)^2\} \\
    &\quad +
    \max_{f \in \mathcal{F}} \mathbb{E}_q\left[2\left\{h(B)-g(A)\right\} f(C)-f(C)^2\right]   
    +\mu\E_n\{h(B)^2\}
\end{align*}
where $[p]$ and $[q]$ partition $[n]=(1,...,n)$, so $p+q=n$. This modification applies the estimator to conditional moments over different subpopulations  indexed by $[p]$ and $[q]$.

For the index set $[p]$, let $I_{[p]}\in\mathbb{R}^{p\times n}$ be the matrix of ones and zeros such that $V_{[p]}=I_{[p]}V$ gives the elements of $V$ whose indices are in $[p]$.

Finally,  we abbreviate $\mu I=\mu$ and $\mu' I=\mu'$ when it is clear from context.

\subsection{RKHS: Closed form}\label{sec:rkhs}

Let $V_{g,h}'=g(A)-Y$ and $V_{g,h}=h(B)-g(A)$. Let $\Phi_C:\mathcal{F}\rightarrow\mathbb{R}^n$ be an operator with $i$th row $\langle \phi(C_i),\cdot \rangle_{\mathcal{F}}$. Define $\Phi_{C'}$ analogously, replacing $C_i$ with $C_i'$. Let $K_C$ and $K_{C'}$ be the corresponding kernel matrices.

\textbf{Maximizers.}

\begin{lemma}[Existence of maximizers]\label{lemma:max_exist}
  There exist coefficients $\hat{\gamma}_{g,h},\hat{\gamma}'_{g,h}\in\mathbb{R}^n$ such that maximizers take the form $\hat{f}_{g,h}=\Phi_C^* \hat{\gamma}_{g,h}$ and $\hat{f}'_{g,h}=\Phi_{C'}^*\hat{\gamma}'_{g,h}$.
\end{lemma}

\begin{remark}[Subsetted estimator]\label{remark:max_exist}
    For the subsetted estimator, the same results hold but with $\hat{\gamma}_{g,h;[q]}\in\mathbb{R}^q $ and $\hat{\gamma}'_{g,h;[p]}\in\mathbb{R}^p$, acting on appropriately modified feature operators $\Phi^*_{C;    [q]}$ and $\Phi^*_{C';[p]}$.
\end{remark}

\begin{proof}
      Write the objectives for the maximizers as
      \begin{align*}
          \mathcal{E}'(f')&=\mathbb{E}_n\left\{2V'_{g,h} f'(C')-f'(C')^2\right\},\quad 
          \mathcal{E}(f)=\mathbb{E}_n\left\{2V_{g,h} f(C)-f(C)^2\right\}.
      \end{align*}

We prove the former result; the latter is similar. By the Riesz representation theorem,
$$
\mathcal{E}(f)=\mathbb{E}_n\left\{2V_{g,h} \langle f, \phi(C)\rangle_{\mathcal{F}}-\langle f, \phi(C)\rangle_{\mathcal{F}}^2\right\}.
$$
For an RKHS, evaluation is a continuous functional represented as the inner product with the feature map. Due to the ridge penalty, the stated objective has a maximizer $\hat{f}_{g,h}$ that obtains the maximum.

To lighten notation, we suppress the indexing of $\hat{f}_{g,h}$ by $(g,h)$ for the rest of this argument. Write $\hat{f}=\hat{f}_n+\hat{f}^{\perp}_n$ where $\hat{f}_n\in row(\Phi_C)$ and $\hat{f}_n^{\perp}\in null(\Phi_C)$. Substituting this decomposition of $\hat{f}$ into the objective, we see that
$
\mathcal{E}(\hat{f})=\mathcal{E}(\hat{f}_n).
$
Hence if $\hat{f}$ is a maximizer, then there exists $\hat{f}_n$ that is also a maximizer.
\end{proof}

\begin{lemma}[Formula of maximizers]\label{lemma:max}
    The explicit formula for the coefficients is $\hat{\gamma}_{g,h}=K_C^{\dagger}\vec{V}_{g,h}$ and $\hat{\gamma}'_{g,h}=K_{C'}^{\dagger}\vec{V}'_{g,h}$.
\end{lemma}

\begin{remark}[Subsetted estimator]\label{remark:max}
    For the subsetted estimator, the same results hold but with $\hat{\gamma}_{g,h;[q]} =K_{C;[q,q]}^{\dagger}\vec{V}_{g,h;[q]}$ and $\hat{\gamma}'_{g,h;[p]}=K_{C';[p,p]}^{\dagger}\vec{V}'_{g,h;[p]}$.
\end{remark}

\begin{proof}
   We prove the former result; the latter is similar. Write the objective as
   $
   \mathcal{E}(f)= 2\langle f, \hat{\mu}_{g,h}\rangle_{\mathcal{F}}-\langle f, \hat{T}_C f\rangle_{\mathcal{F}}, 
   $
   where $\hat{\mu}_{g,h}=\mathbb{E}_n\{V_{g,h}\phi(C)\}=\frac{1}{n}\Phi_C^* \vec{V}_{g,h}$ and $\hat{T}_C=\mathbb{E}_n\{\phi(C)\otimes \phi(C)^*\}=\frac{1}{n}\Phi_C^*\Phi_C$. Hence by Lemma~\ref{lemma:max_exist},
    $$
   \mathcal{E}(\gamma)= 2\langle \Phi_C^* \gamma_{g,h}, \hat{\mu}_{g,h}\rangle_{\mathcal{F}}-\langle \Phi_C^* \gamma_{g,h}, \hat{T}_C \Phi_C^* \gamma_{g,h}\rangle_{\mathcal{F}}=\frac{2}{n}\gamma_{g,h}^{\top}\Phi_C \Phi_C^* \vec{V}_{g,h}-\frac{1}{n}\gamma_{g,h}^{\top} \Phi_C \Phi_C^*\Phi_C \Phi_C^* \gamma_{g,h}.
   $$
   Since $K_C=\Phi_C\Phi_C^*$, the first order condition yields $K_C\vec{V}_{g,h}=K_C^2 \hat{\gamma}_{g,h}$, i.e. $\hat{\gamma}_{g,h}=K_C^{\dagger}\vec{V}_{g,h}$ where $K_C^{\dagger}$ is the pseudoinverse of $K_C$.
\end{proof}

\subsubsection{Minimizers}

Let $\Phi_A:\mathcal{H}\rightarrow\mathbb{R}^n$ be an operator with $i$th row $\langle \phi(A_i),\cdot \rangle_{\mathcal{H}}$. Define $\Phi_B$ analogously, replacing $A_i$ with $B_i$. Let $K_A$ and $K_B$ be the corresponding kernel matrices.

\begin{lemma}[Existence of minimizers]\label{lemma:min_exist}
    There exist coefficients $\alpha,\beta \in\mathbb{R}^n$ such that minimizers take the form $\hat{g}=\Phi_A^*\hat{\alpha}$ and $\hat{h}=\Phi_B^*\hat{\beta}$.
\end{lemma}

\begin{remark}[Subsetted estimator]\label{remark:min_exist}
    The result remains true for the subsetted estimator.
\end{remark}

\begin{proof}
    To begin, write the objective $\mathcal{E}(g,h)$ as 
  \begin{align*}
   \mathbb{E}_n\left\{2V'_{g,h} \hat{f}_{g,f}'(C')-\hat{f}_{g,h}'(C')^2\right\}
     +\mu'\E_n\{g(A)^2\} 
    +
     \mathbb{E}_n\left\{2V_{g,h} \hat{f}_{g,h}(C)-\hat{f}_{g,h}(C)^2\right\}   
    +\mu\E_n\{h(B)^2\}.
     \end{align*}
     By Lemmas~\ref{lemma:max_exist} and~\ref{lemma:max},
     \begin{align*}
         \hat{f}_{g,f}'(C')
     &=\langle \hat{f}_{g,f}',  \phi(C')\rangle_{\mathcal{F}}
     =\langle \Phi_{C'}^*K_{C'}^{\dagger}\vec{V}'_{g,h},  \phi(C')\rangle_{\mathcal{F}} \\
     \hat{f}_{g,h}(C)
     &=\langle \hat{f}_{g,f},  \phi(C)\rangle_{\mathcal{F}}
     =\langle \Phi_{C}^*K_{C}^{\dagger}\vec{V}_{g,h},  \phi(C)\rangle_{\mathcal{F}}.
     \end{align*}
     Hence $(g,h)$ only appear via $V'_{g,h}=g(A)-Y$, $V_{g,h}=h(B)-g(A)$, and directly as $g(A)$ and $h(B)$. In all of these expressions, they can be further expressed as $g(A)=\langle g,\phi(A)\rangle_{\mathcal{G}}$ and $h(B)=\langle h,\phi(B)\rangle_{\mathcal{H}}$, which is a linear functional. The overall objective is quadratic in such terms, so the stated objective has maximizers $(\hat{g},\hat{h})$ that obtain the maximum.

     By a similar argument to Lemma~\ref{lemma:max_exist}, for any $(\hat{g},\hat{h})$ attaining the maximum, $\mathcal{E}(\hat{g},\hat{h})=\mathcal{E}(\hat{g}_n,\hat{h}_n)$ where $\hat{g}_n\in row(\Phi_A)$ and $\hat{h}_n\in row(\Phi_B)$.
\end{proof}

\begin{lemma}[Properties of pseudo-inverse]\label{lemma:pseudo1}
    For any square symmetric matrix $K\in\mathbb{R}^{n\times n}$, its eigendecomposition is $K=U\Sigma U^{\top}$ where $\Sigma\in\mathbb{R}^{r\times r}$ has nonzero diagonal entries and $r\leq n$. Its pseudo-inverse is $K^{\dagger}=U\Sigma^{-1} U^{\top}$. Moreover $K^{\dagger}K=KK^{\dagger}=UU^{\top}$, which is a projection. 
\end{lemma}

To lighten notation, let $K_C^{\dagger}K_C=P_C$.

\begin{proposition}[Formula of minimizers]\label{prop:min}
    The explicit formula for the coefficients is
    \begin{align*}
    \hat{\beta} &= \left[K_A\left\{-P_C+\left(P_{C'}+P_C+\mu'\right)K_A\left(K_BP_CK_A\right)^{\dagger}K_B\left(P_C+\mu\right)\right\}K_B\right]^{\dagger}K_AP_{C'}Y\\
    \hat{\alpha}&=  \left(K_BP_CK_A\right)^{\dagger}K_B\left(P_C+\mu\right)K_B\hat{\beta}      
    \end{align*}
\end{proposition}

\begin{proof}
We proceed in steps.

\begin{enumerate}
    \item Write the objective $\mathcal{E}(g,h)$ as
    $$
   2\langle \hat{f}'_{g,h}, \hat{\mu}'_{g,h}\rangle_{\mathcal{F}}-\langle \hat{f}'_{g,h}, \hat{T}_{C'} \hat{f}'_{g,h}\rangle_{\mathcal{F}}  
     +\mu'\langle g,\hat{T}_A g\rangle_{\mathcal{G}} 
    +
    2\langle \hat{f}_{g,h}, \hat{\mu}_{g,h}\rangle_{\mathcal{F}}-\langle \hat{f}_{g,h}, \hat{T}_C \hat{f}_{g,h}\rangle_{\mathcal{F}}  
    +\mu\langle h,\hat{T}_B h\rangle_{\mathcal{H}}
    $$
    where 
    $\hat{\mu}'_{g,h}=\frac{1}{n}\Phi_{C'}^* \vec{V}'_{g,h}$, 
    $\hat{\mu}_{g,h}=\frac{1}{n}\Phi_C^* \vec{V}_{g,h}$, and the covariance operators are defined analogously to Lemma~\ref{lemma:max}. Hence by Lemma~\ref{lemma:max},
    \begin{align*}
        \mathcal{E}(g,h)
        &=2\langle \Phi_{C'}^*K_{C'}^{\dagger}\vec{V}'_{g,h}, \hat{\mu}'_{g,h}\rangle_{\mathcal{F}}-\langle \Phi_{C'}^*K_{C'}^{\dagger}\vec{V}'_{g,h}, \hat{T}_{C'} \Phi_{C'}^*K_{C'}^{\dagger}\vec{V}'_{g,h}\rangle_{\mathcal{F}}  
     +\mu'\langle g,\hat{T}_A g\rangle_{\mathcal{G}}  \\
    &+
    2\langle \Phi_{C}^*K_{C}^{\dagger}\vec{V}_{g,h}, \hat{\mu}_{g,h}\rangle_{\mathcal{F}}-\langle \Phi_{C}^*K_{C}^{\dagger}\vec{V}_{g,h}, \hat{T}_C \Phi_{C}^*K_{C}^{\dagger}\vec{V}_{g,h}\rangle_{\mathcal{F}}  
    +\mu\langle h,\hat{T}_B h\rangle_{\mathcal{H}} \\
    %
    %
    &=\frac{2}{n}
    (\vec{V}'_{g,h})^{\top}K_{C'}^{\dagger}\Phi_{C'}\Phi_{C'}^* \vec{V}'_{g,h}
    -
    \frac{1}{n}(\vec{V}'_{g,h})^{\top}K_{C'}^{\dagger}\Phi_{C'} \Phi_{C'}^*\Phi_{C'}  \Phi_{C'}^*K_{C'}^{\dagger}\vec{V}'_{g,h} 
     +\mu'\langle g,\hat{T}_A g\rangle_{\mathcal{G}}  \\
    &+
    \frac{2}{n}\vec{V}_{g,h}^{\top}K_{C}^{\dagger}\Phi_{C} \Phi_C^* \vec{V}_{g,h}
    -\frac{1}{n}\vec{V}_{g,h}^{\top}K_{C}^{\dagger}\Phi_{C} \Phi_{C}^*\Phi_{C} \Phi_{C}^*K_{C}^{\dagger}\vec{V}_{g,h}  
    +\mu\langle h,\hat{T}_B h\rangle_{\mathcal{H}} \\
    &=
    \frac{1}{n}(\vec{V}'_{g,h})^{\top} P_{C'}\vec{V}'_{g,h}
     +\mu'\langle g,\hat{T}_A g\rangle_{\mathcal{G}}  +
    \frac{1}{n}\vec{V}_{g,h}^{\top}P_C\vec{V}_{g,h}
    +\mu\langle h,\hat{T}_B h\rangle_{\mathcal{H}}.
    \end{align*}
    \item Let $Y,G,H\in\mathbb{R}^n$ be defined with $G_i=g(A_i)$ and $H_i=h(B_i)$. In this notation,
    \begin{align*}
        \frac{1}{n}(\vec{V}'_{g,h})^{\top} P_{C'}\vec{V}'_{g,h} 
        &
        =\frac{1}{n}(Y^{\top}P_{C'}Y-2G^{\top}P_{C'}Y+G^{\top}P_{C'}G)
        ,\quad \mu'\langle g,\hat{T}_A g\rangle_{\mathcal{G}} 
        = \frac{\mu'}{n} G^{\top}G \\
        \frac{1}{n}\vec{V}_{g,h}^{\top}P_C\vec{V}_{g,h}
        &
        =\frac{1}{n}(H^{\top}P_CH-2G^{\top}P_CH+G^{\top}P_CG),\quad 
        \mu\langle h,\hat{T}_B h\rangle_{\mathcal{H}} 
        =\frac{\mu}{n} H^{\top}H.
    \end{align*}
   Combining with $G=\Phi_Ag=K_A\alpha$ and $H=\Phi_B h=K_B\beta$ from Lemma~\ref{lemma:min_exist},
    \begin{align*}
        n\mathcal{E}(\alpha,\beta)&=Y^{\top}P_{C'}Y-2G^{\top}(P_{C'}Y+P_CH)+G^{\top}(P_{C'}+P_C+\mu')G+H^{\top}(P_C+\mu)H \\
        &=Y^{\top}P_{C'}Y-2\alpha^{\top}K_A(P_{C'}Y+P_CK_B\beta)+\alpha^{\top}K_A(P_{C'}+P_C+\mu') K_A\alpha\\
        &\quad +\beta^{\top}K_B (P_C+\mu) K_B\beta. 
    \end{align*}
   \item  The first order conditions yield
\begin{align*}
    0&=-2K_A(P_{C'}Y+P_CK_B\hat{\beta})+2 K_A(P_{C'}+P_C+\mu') K_A\hat{\alpha} \\
        0&=-2K_BP_C K_A\hat{\alpha}+2K_B (P_C+\mu) K_B \hat{\beta}.
       \end{align*}
    Rearranging and taking pseudo-inverses, we arrive at two equations:
\begin{align*}
    K_AP_{C'}Y+K_AP_CK_B\hat{\beta}&=K_A(P_{C'}+P_C+\mu') K_A\hat{\alpha} \\
    K_BP_C K_A\hat{\alpha}=K_B(P_C+\mu) K_B \hat{\beta} &\Longrightarrow \hat{\alpha} = \left(K_BP_CK_A\right)^{\dagger}K_B\left(P_C+\mu\right)K_B\hat{\beta}.
\end{align*}

\item Substituting the latter into the former,
    $$
      K_AP_{C'}Y+K_AP_CK_B\hat{\beta}=K_A(P_{C'}+P_C+\mu') K_A\left(K_BP_CK_A\right)^{\dagger}K_B\left(P_C+\mu\right)K_B\hat{\beta},
    $$
    and solving for $\hat{\beta}$,
  $$
    \hat{\beta} = \left[K_A\left\{
    -P_C+\left(P_{C'}+P_C+\mu'\right)K_A\left(K_BP_CK_A\right)^{\dagger}K_B\left(P_C+\mu\right)\right\}K_B\right]^{\dagger}K_AP_{C'}Y. \qedhere
    $$   
\end{enumerate}
\end{proof}

\begin{remark}[Subsetted estimator]\label{remark:min}
    The explicit formula for the coefficients is
    \begin{align*}
    \hat{\beta} &= \left[K_A\left\{-\tilde{P}_C+\left(\tilde{P}_{C'}+\tilde{P}_C+\mu'\right)K_A\left(K_B\tilde{P}_CK_A\right)^{\dagger}K_B\left(\tilde{P}_C+\mu\right)\right\}K_B\right]^{\dagger}K_A\tilde{P}_{C'}Y\\
    \hat{\alpha}&=  \left(K_B\tilde{P}_CK_A\right)^{\dagger}K_B\left(\tilde{P}_C+\mu\right)K_B\hat{\beta}      
    \end{align*}
where $\tilde{P}_{C'}=\frac{n}{p}I_{[p]}^{\top}P_{C';[p,p]}I_{[p]}$ and $\tilde{P}_{C}=\frac{n}{q}I_{[q]}^{\top}P_{C;[q,q]}I_{[q]}$. Note that $P_{C';[p,p]}=(K_{C';[p,p]})^{\dagger}K_{C';[p,p]}$ and  $K_{C';[p,p]}=I_{[p]}K_{C'}I_{[p]}^{\top}$.
\end{remark}

\begin{proof}
We proceed in steps.

\begin{enumerate}
    \item Write the objective $\mathcal{E}(g,h)$ as
    \begin{align*}
        &2\langle \hat{f}'_{g,h}, \hat{\mu}'_{g,h;[p]}\rangle_{\mathcal{F}}-\langle \hat{f}'_{g,h}, \hat{T}_{C';[p,p]} \hat{f}'_{g,h}\rangle_{\mathcal{F}}  
     +\mu'\langle g,\hat{T}_A g\rangle_{\mathcal{G}} \\
    &\quad +
    2\langle \hat{f}_{g,h}, \hat{\mu}_{g,h;[q]}\rangle_{\mathcal{F}}-\langle \hat{f}_{g,h}, \hat{T}_{C;[q,q]} \hat{f}_{g,h}\rangle_{\mathcal{F}}  
    +\mu\langle h,\hat{T}_B h\rangle_{\mathcal{H}}
     \end{align*}
    where 
    $\hat{\mu}'_{g,h;[p]}=\frac{1}{p}\Phi_{C';[p]}^* \vec{V}'_{g,h;[p]}$, 
    $\hat{\mu}_{g,h;[q]}=\frac{1}{q}\Phi_{C;[q]}^* \vec{V}_{g,h;[q]}$, and the covariance operators are defined analogously to Remark~\ref{remark:max}. Hence by Remark~\ref{remark:max} and the same argument as in Proposition~\ref{prop:min},
    \begin{align*}
        \mathcal{E}(g,h)
    &=
    \frac{1}{p}(\vec{V}'_{g,h;[p]})^{\top} P_{C';[p,p]}\vec{V}'_{g,h;[p]}
     +\mu'\langle g,\hat{T}_A g\rangle_{\mathcal{G}}  +
    \frac{1}{q}\vec{V}_{g,h;[q]}^{\top}P_{C;[q,q]}\vec{V}_{g,h;[q]}
    +\mu\langle h,\hat{T}_B h\rangle_{\mathcal{H}}.
    \end{align*}
    \item Let $Y,G,H\in\mathbb{R}^n$ be defined with $G_i=g(A_i)$ and $H_i=h(B_i)$ as before. Now, let $\tilde{P}_{C'}=\frac{n}{p}I_{[p]}^{\top}P_{C';[p,p]}I_{[p]} \in \mathbb{R}^{n\times n}$ and
    $\tilde{P}_C=\frac{n}{q}I_{[q]}^{\top}P_{C';[q,q]}I_{[q]} \in \mathbb{R}^{n\times n}$. Then
    \begin{align*}
        \frac{1}{p}(\vec{V}'_{g,h;[p]})^{\top} P_{C';[p,p]}\vec{V}'_{g,h;[p]} 
        &
        =\frac{1}{n}(Y^{\top}\tilde{P}_{C'} Y-2G^{\top}\tilde{P}_{C'}Y+G^{\top}\tilde{P}_{C'}G)
        \\
        \mu'\langle g,\hat{T}_A g\rangle_{\mathcal{G}} 
        &= \frac{\mu'}{n} G^{\top}G \\
        \frac{1}{q}\vec{V}_{g,h;[q]}^{\top}P_{C;[q,q]}\vec{V}_{g,h;[q]}
        &
        =\frac{1}{n}(H^{\top}\tilde{P}_CH-2G^{\top}\tilde{P}_CH+G^{\top}\tilde{P}_CG)\\ 
        \mu\langle h,\hat{T}_B h\rangle_{\mathcal{H}} 
        &=\frac{\mu}{n} H^{\top}H.
    \end{align*}
    Hereafter we use the same argument as in Proposition~\ref{prop:min}. \qedhere
\end{enumerate}
\end{proof}

\subsubsection{Nystr\"om approximation}

Computation of kernel methods may be demanding due to the inversions of matrices that scale with $n$ such as $K_B\in\mathbb{R}^{n\times n}$. One solution is Nystr\"om approximation. We now provide alternative expressions for the minimizers $(\hat{g},\hat{h})$ that lend themselves to Nystr\"om approximation, then describe the procedure.

\begin{lemma}[Minimizer sufficient statistics]\label{lemma:n2}
The minimizers may be expressed as 
\begin{align*}
    \hat{g}&=\left(\Phi_B^*P_C \Phi_A\right)^{\dagger}\Phi_B^* (P_C+\mu) \Phi_B \hat{h} \\
    \hat{h}&=\left[\Phi_A^*\left\{-P_C+\left(P_{C'}+P_C+\mu'\right)\Phi_A\left(\Phi_B^*P_C\Phi_A\right)^{\dagger}\Phi_B^*\left(P_C+\mu\right)\right\}\Phi_B\right]^{\dagger}\Phi_A^*P_{C'}Y.
\end{align*}
\end{lemma}

\begin{proof}
We proceed in steps.
\begin{enumerate}
    \item By the proof of Proposition~\ref{prop:min},
       with $G=\Phi_Ag$ and $H=\Phi_Bh$,
    \begin{align*}
        n\mathcal{E}(g,h)&=Y^{\top}P_{C'}Y-2G^{\top}(P_{C'}Y+P_CH)+G^{\top}(P_{C'}+P_C+\mu')G+H^{\top}(P_C+\mu)H \\
        &=Y^{\top}P_{C'}Y-2g^*\Phi_A^*(P_{C'}Y+P_C\Phi_Bh)+g^*\Phi_A^*(P_{C'}+P_C+\mu')\Phi_Ag+h^*\Phi_B^*(P_C+\mu)\Phi_Bh.
    \end{align*}
\item Informally, the first order conditions yield
\begin{align*}
    0&=-2\Phi_A^*(P_{C'}Y+P_C\Phi_B\hat{h})+2 \Phi_A^*(P_{C'}+P_C+\mu') \Phi_A\hat{g} \\
        0&=-2\Phi_B^*P_C \Phi_A\hat{g}+2\Phi_B^* (P_C+\mu) \Phi_B \hat{h}.
       \end{align*}
       See \citet[Proof of Proposition 2]{de2005risk} for the formal way of deriving the first order condition, which incurs additional notation. 
    Rearranging and taking pseudo-inverses, we arrive at two equations:
\begin{align*}
    \Phi_A^*(P_{C'}+P_C+\mu') \Phi_A\hat{g}&=\Phi_A^*(P_{C'}Y+P_C\Phi_B\hat{h}) \\
    \Phi_B^*P_C \Phi_A\hat{g}=\Phi_B^* (P_C+\mu) \Phi_B \hat{h} 
    &\Longrightarrow \hat{g} = \left(\Phi_B^*P_C \Phi_A\right)^{\dagger}\Phi_B^* (P_C+\mu) \Phi_B \hat{h}.
\end{align*}

\item Substituting the latter into the former,
    $$
      \Phi_A^*P_{C'}Y+\Phi_A^*P_C\Phi_B\hat{h}
      =\Phi_A^*(P_{C'}+P_C+\mu') \Phi_A \left(\Phi_B^*P_C \Phi_A\right)^{\dagger}\Phi_B^* (P_C+\mu) \Phi_B \hat{h},
    $$
    and solving for $\hat{h}$,
    $$
    \hat{h} = \left[\Phi_A^*\left\{-P_C+\left(P_{C'}+P_C+\mu'\right)\Phi_A\left(\Phi_B^*P_C\Phi_A\right)^{\dagger}\Phi_B^*\left(P_C+\mu\right)\right\}\Phi_B\right]^{\dagger}\Phi_A^*P_{C'}Y. \qedhere
    $$
    \end{enumerate}
\end{proof}

\begin{remark}[Subsetted estimator]\label{remark:n2}
    The subsetted minimizers may be expressed as
    \begin{align*}
    \hat{g}&=\left(\Phi_B^*\tilde{P}_C \Phi_A\right)^{\dagger}\Phi_B^* (\tilde{P}_C+\mu) \Phi_B \hat{h} \\
    \hat{h}&=\left[\Phi_A^*\left\{-\tilde{P}_C+\left(\tilde{P}_{C'}+\tilde{P}_C+\mu'\right)\Phi_A\left(\Phi_B^*\tilde{P}_C\Phi_A\right)^{\dagger}\Phi_B^*\left(\tilde{P}_C+\mu\right)\right\}\Phi_B\right]^{\dagger}\Phi_A^*\tilde{P}_{C'}Y.
\end{align*}
\end{remark}

\begin{proof}
    The argument is analogous to Remark~\ref{remark:min}.
\end{proof}

\begin{lemma}[Properties of pseudo-inverse]\label{lemma:pseudo2}
    Continuing the notation of Lemma~\ref{lemma:pseudo1}, if $\Phi=U\Sigma^{1/2}V^{\top}$ and $K=\Phi\Phi^*$, then $P=UU^{\top}=K^{\dagger}K=\Phi\Phi^{\dagger}$. Remark~\ref{remark:min} relates $\tilde{P}$ to $P$.
\end{lemma}

Combining Lemmas~\ref{lemma:n2} and~\ref{lemma:pseudo2}, we conclude that sufficient statistics for $(\hat{g},\hat{h})$ are feature operators. Within the feature operator $\Phi$, the $i$th row $\langle \phi(X_i),\cdot\rangle$ may be viewed as an infinite dimensional vector.

Nystr\"om approximation is a way to approximate infinite dimensional vectors with finite dimensional ones. It uses the substitution
$
\phi(x)\mapsto \check{\phi}(x)= (K_{\mathcal{S}\mathcal{S}})^{-\frac{1}{2}}K_{\mathcal{S}x}
$, where
$\mathcal{S}$ is a subset of $s=|\mathcal{S}|\ll n$ observations called landmarks. $K_{\mathcal{S}\mathcal{S}}\in\mathbb{R}^{s\times s}$ is defined such that $(K_{\mathcal{S}\mathcal{S}})_{ij}=k(X_i,X_j)$ for $i,j\in\mathcal{S}$. Similarly, $K_{\mathcal{S}x}\in\mathbb{R}^s$ is defined such that $(K_{\mathcal{S}x})_i=k(X_i,x)$ for $i\in\mathcal{S}$. 

In summary, the approximate sufficient statistics are of the form $\check{\Phi}\in\mathbb{R}^{n\times s}$, i.e. a matrix whose $i$th row $\langle \check{\phi}(X_i),\cdot\rangle$ may be viewed as a vector in $\mathbb{R}^s$.
\subsection{Neural network: Stochastic gradient descent}\label{sec:nn}

Suppose the function classes are neural networks. Algorithm~\ref{algo:simultaneous} takes the form
\begin{align*}
    (\hat{g},\hat{h})&=\arg \min _{\theta_1,\theta_2} 
    \max_{\omega_1, \omega_2} \bigg\{\mathbb{E}_n\left[2\left\{g_{\theta_1}(A)-Y\right\} f_{\omega_1}'(C')-f_{\omega_1}'(C')^2\right]
     +\mu'\E_n\{g_{\theta_1}(A)^2\} \\
    &\quad + \mathbb{E}_n\left[2\left\{h_{\theta_2}(B)-g_{\theta_1}(A)\right\} f_{\omega_2}(C)-f_{\omega_2}(C)^2\right]   
    +\mu\E_n\{h_{\theta_2}(B)^2\}\bigg\}
\end{align*}
where $\theta_1, \theta_2, \omega_1,\omega_2$ are weights of the neural networks.

We use the optimistic Adam algorithm of \cite{Daskalakis}, which is a type of stochastic gradient descent.  This approach is standard in the literature \citep{dikkala2020minimax}.

\begin{remark}[Subsetted estimator]
To adapt the algorithm, it suffices to make one simple modification: for observations outside of the subset, we set the adversary's loss to zero.
 \end{remark}

\newpage 

\bibliographystyle{apalike}

\end{document}